\documentclass{article}
\pdfoutput=1
\usepackage{natbib}
\setcitestyle{numbers,square}

\usepackage[final]{neurips_2020}

\usepackage[utf8]{inputenc} 
\usepackage[T1]{fontenc}    
\usepackage{hyperref}       
\usepackage{url}            
\usepackage{booktabs}       
\usepackage{amsfonts}       
\usepackage{nicefrac}       
\usepackage{microtype}      
\usepackage{enumitem}
\usepackage{amsmath,color}
\usepackage[ruled,vlined]{algorithm2e}
\newcommand{\kerP}{\mathbf{P}}
\newcommand\numberthis{\addtocounter{equation}{1}\tag{\theequation}}
\usepackage{graphicx}
\usepackage{float}
\usepackage{amsthm}
\usepackage{subcaption}
\usepackage{cleveref} 
\usepackage{enumitem}

\newcommand{\EE}{\mathbb{E}}

\newcommand{\calA}{\mathcal{A}}
\newcommand{\calF}{\mathcal{F}}
\newcommand{\calO}{\mathcal{O}}

\newcommand{\calS}{\mathcal{S}}

\newcommand{\PP}{\mathbb{P}}
\newcommand{\dd}{\mathrm{d}}
\newcommand{\RR}{\mathbb{R}}

\makeatletter
\newcommand*\dotp{\mathpalette\dotp@{.5}}
\newcommand*\dotp@[2]{\mathbin{\vcenter{\hbox{\scalebox{#2}{$\m@th#1\bullet$}}}}}
\makeatother

\newcommand{\Var}{\mathrm{Var}}

\newcommand{\nn}{\nonumber}
\usepackage{color}

\newcommand{\blue}[1]{{\color{black} #1}}
\newtheorem{definition}{Definition}[section]
\newtheorem*{definition*}{Definition}

\newtheorem*{fact*}{Fact} 

\newtheorem{theorem}[definition]{Theorem}

\newtheorem{lemma}[definition]{Lemma}

\newtheorem{corollary}[definition]{Corollary}
\newtheorem{assumption}[definition]{Assumption}

\newtheorem*{theorem*}{Theorem}
\newtheorem*{lemma*}{Lemma}
\newtheorem*{proposition*}{Proposition}
\newtheorem{corollary*}{Corollary}   
\newtheorem{assumption*}{Assumption}
\newtheorem{condition*}{Condition}
\newtheorem{exercise*}{Exercise}
\newtheorem*{example*}{Example}

\usepackage{comment} 

\usepackage[toc,page,header]{appendix}
\usepackage{minitoc}

\title{Variance-Reduced Off-Policy TDC Learning: Non-Asymptotic Convergence Analysis}
\author{%
  Shaocong Ma \\
  Department of ECE\\
  University of Utah\\
  Salt Lake City, UT 84112 \\
  \texttt{s.ma@utah.edu}\\
  \And
  Yi Zhou \\
  Department of ECE\\
  University of Utah\\
  Salt Lake City, UT 84112 \\
  \texttt{yi.zhou@utah.edu}\\
  \And
  Shaofeng Zou \\
  Department of EE\\
  University at Buffalo\\
  Buffalo, NY 14260 \\
  \texttt{szou3@buffalo.edu}\\}

\begin{document}

\doparttoc 
\faketableofcontents 
\part{} 

\maketitle

\begin{abstract}
   Variance reduction techniques have been successfully applied to temporal-difference (TD) learning and help to improve the sample complexity in policy evaluation. However, the existing work applied variance reduction to either the less popular one time-scale TD algorithm or the two time-scale GTD algorithm but with a finite number of i.i.d.\ samples, and both algorithms apply to only the on-policy setting. In this work, we develop a variance reduction scheme for the two time-scale TDC algorithm in the off-policy setting and analyze its non-asymptotic convergence rate over both i.i.d.\ and Markovian samples. In the i.i.d.\ setting, our algorithm \blue{matches the best-known lower bound $\Tilde{\calO}(\epsilon^{-1}$).} In the Markovian setting, our algorithm achieves the state-of-the-art sample complexity $\calO(\epsilon^{-1} \log {\epsilon}^{-1})$ that is near-optimal. Experiments demonstrate that the proposed variance-reduced TDC achieves a smaller asymptotic convergence error than both the conventional TDC and the variance-reduced TD. 
\end{abstract}

\section{Introduction}
\vspace{-5pt}

In reinforcement learning applications, we often need to evaluate the value function of a target policy by sampling the Markov decision process (MDP) generated by either the target policy (on-policy) or a certain behavior policy (off-policy) \cite{bertsekas1995neuro,dayan1992q,rummery1994line,Sutton2018,sun2020finite,wang2019multistep}. In the on-policy setting, temporal-difference (TD) learning \cite{sutton1988learning,Sutton2018} and Q-learning \cite{dayan1992q} algorithms have been developed with convergence guarantee. However, in the more popular off-policy setting, these conventional policy evaluation algorithms have been shown to possibly diverge under linear function approximation \cite{BAIRD1995}. To address this issue, \cite{maei2011gradient,sutton2009fast,sutton2008convergent} developed a family of gradient-based TD (GTD) algorithms that have convergence guarantee in the off-policy setting. In particular, the TD with gradient correction (TDC) algorithm has been shown to have superior performance and is widely used in practice.

Although TD-type algorithms achieve a great success in policy evaluation, their convergence suffer from a large variance caused by the stochastic samples obtained from a dynamic environment. A conventional approach that addresses this issue is to use a diminishing stepsize \cite{bhandari2018finite,srikant2019finite}, but it significantly slows down the convergence in practice.
Recently, several work proposed to apply the variance reduction technique developed in the stochastic optimization literature to reduce the variance of TD learning. Specifically, \cite{du2017stochastic} considered a convex-concave TD learning problem with a finite number of i.i.d.\ samples, and they applied the SVRG \cite{johnson2013accelerating} and SAGA \cite{defazio2014saga} variance reduction techniques to develop variance-reduced primal-dual batch gradient algorithms for solving the problem. In \cite{peng2019svrg}, two variants of SVRG-based GTD2 algorithms for i.i.d.\ samples were proposed to save the computation cost. While these work developed variance-reduced TD-type algorithms for i.i.d.\ samples, practical reinforcement learning applications often involve {\em non-i.i.d.} samples that are generated by an underlying MDP. In \cite{korda2015td}, the authors proposed a variance-reduced TD (VRTD) algorithm for Markovian samples in the on-policy setting. However, their analysis of the algorithm has a technical error and has been corrected in the recent work \cite{Xu2020Reanalysis}. To summarize, the existing developments of variance-reduced TD-type algorithms  consider only the on-policy setting, and only the one time-scale VRTD algorithm applies to Markovian samples. Therefore, it is very much desired to develop a two time-scale variance-reduced algorithm in the off-policy setting for Markovian samples, which constitutes to \textbf{the goal of this paper:} we develop two variance-reduced TDC algorithms in the off-policy setting for i.i.d.\ samples and Markovian samples, respectively, and analyze their non-asymptotic convergence rates. We summarize our contributions as follows.
\vspace{-3pt}
\subsection{Our Contributions}
\vspace{-3pt}
We develop two variance-reduced TDC algorithms (named VRTDC) respectively for i.i.d.\ samples and Markovian samples in the off-policy setting and analyze their non-asymptotic convergence rates as well as sample complexities. To the best of our knowledge, our work provides the first study on variance reduction for two time-scale TD learning over Markovian samples.

For i.i.d.\ samples with constant step sizes $\alpha, \beta$, we show that VRTDC converges at a linear rate to a neighborhood of the optimal solution with an asymptotic convergence error $\calO(\beta M^{-1}+\beta^4)$, where $M$ denotes the batch size of the outer-loops. Consequently, to achieve an $\epsilon$-accurate solution, the required sample complexity for VRTDC is \blue{$\calO(\epsilon^{-1} \log{\epsilon}^{-1})$, which matches VRTD for i.i.d.\ samples \cite{Xu2020Reanalysis}}. Also, the tracking error of VRTDC diminishes linearly with an asymptotic error \blue{$\calO(  {M}^{-1}+\beta^3)$ and has a corresponding sample complexity $\calO(\epsilon^{-1} \log{\epsilon}^{-1})$}. For Markovian samples with constant step sizes $\alpha, \beta$, we show that VRTDC converges at a linear rate to a neighborhood of the optimal solution with an asymptotic convergence error $\calO({M}^{-1}+\beta^2)$, and the required sample complexity to achieve an $\epsilon$-accurate solution is  $\calO(\epsilon^{-1} \log {\epsilon}^{-1})$, which matches the best existing result of VRTD \cite{Xu2020Reanalysis} and TDC \cite{kaledin2020finite} and nearly matches the theoretical lower bound $\calO(\epsilon^{-1})$ \cite{kaledin2020finite}. Also, the tracking error of VRTDC diminishes linearly with an asymptotic error $\calO( {M}^{-1}+\beta^3)$ and has a corresponding sample complexity $\calO(\epsilon^{-1} \log{\epsilon}^{-1})$.
Furthermore, our experiments on the Garnet problem and frozen lake game demonstrate that VRTDC achieves a smaller asymptotic convergence error than both the conventional TDC and the variance-reduced TD. 

Our analysis of VRTDC requires substantial developments of bounding techniques. Specifically, we develop much refined bounds for the tracking error and the convergence error via a recursive refinement strategy: we first develop a preliminary bound for the tracking error $\|\tilde{z}^{(m)}\|^2$ and then use it to develop a preliminary bound for the convergence error $\|\theta^{(m)}-\theta^\ast\|^2$. Then, by leveraging the relation between tracking error and convergence error induced by the two time-scale updates of VRTDC, we further utilize the preliminary bound for the convergence error to obtain a refined bound for the tracking error. Finally, we apply the refined bound for the tracking error to develop a refined bound for the convergence error by leveraging the two time-scale updates. These refined bounds are the key to establish the reported sample complexities of VRTDC.

\vspace{-3pt}
\subsection{Related Work}
\vspace{-3pt}
\paragraph{Off-policy two time-scale TDC and SA:} Two time-scale policy evaluation algorithms such as TDC and GTD2 were first introduced in \cite{sutton2008convergent}, where the asymptotic convergence of both algorithms with i.i.d.\ samples were established. Their non-asymptotic convergence rates were established in \cite{dalal2017finite} as special cases of a two time-scale linear stochastic approximation (SA) algorithm. Recently, the non-asymptotic convergence analysis of TDC and two-time scale linear SA over Makovian samples were established in \cite{xu2019two} and \cite{kaledin2020finite}, respectively. 

\paragraph{TD learning with variance reduction:} In the existing literature, two settings of variance reduction for TD learning have been considered. The first setting considers evaluating the value function based on a fixed number of samples. In this setting, it is preferred to use the batch TD algorithm \cite{lange2012batch}. \cite{du2017stochastic} rewrote the original MSPBE minimization problem into a convex-concave saddle-point optimization problem and applied SVRG and SAGA to the primal-dual batch gradient algorithm. The second setting considers online policy evaluation problem, where the trajectory follows from an MDP. In \cite{bhandari2018finite}, the variance-reduced TD algorithm was introduced for solving the MSPBE minimization problem. \cite{Xu2020Reanalysis} pointed out a technical error in the analysis of \cite{bhandari2018finite} and provided a correct non-asymptotic analysis for the variance-reduced TD algorithm over Markovian samples. 

\section{Preliminaries: Off-Policy TDC with Linear Function Approximation}
\vspace{-5pt}
In this section, we provide an overview of TDC learning with linear function approximation in the off-policy setting and define the notations that are used throughout the paper.
\vspace{-3pt}
\subsection{Off-Policy Value Function Evaluation} 
\vspace{-3pt}
In reinforcement learning, an agent interacts with the environment via a Markov decision process (MDP) that is denoted as $(\calS, \calA, \kerP, r, \gamma)$. Specifically, $\calS$ denotes a state space and $\calA$ denotes an action space, both of which are assumed to have finite number of elements. Then, a given policy $\pi$ maps a state $s\in \calS$ to a certain action in $\calA$ following a conditional distribution $\pi(\cdot|s)$. The associated Markov chain is denoted as $p(s^\prime|s):=\sum_{a\in\calA}\kerP(s^\prime|s,a)\pi(a|s)$ and is assumed to be ergodic, and the induced stationary distribution  is denoted as $\mu_{\pi}(s^\prime) :=\sum_{s}p(s^\prime|s)\mu_{\pi}(s)$. 

In the off-policy setting, the action of the agent is determined by a behavior policy $\pi_b$, which controls the behavior of the agent as follows: Suppose the agent is in a certain state $s_t$ at time-step $t$ and takes an action $a_t$ based on the policy $\pi_b(\cdot|s_t)$. Then, the agent transfers to a new state $s_{t+1}$ in the next time-step according to the transition kernel $\kerP(s_{t+1} |s_t,a_t)$ and receives a reward $r_t=r(s_t, a_t, s_{t+1})$. The goal of off-policy value function evaluation is to use samples of the MDP to estimate the following value function $V^\pi(s)$ of the target policy $\pi$.  
\[V^\pi(s)=\mathbb{E}\Big[\sum_{t=0}^{\infty}\gamma^t r(s_t,a_t, s_{t+1})|s_0=s,\pi\Big],\]  
where $\gamma\in(0,1)$ is a discount factor. Define the Bellman operator $T^\pi$ for any function $\xi(s)$ as $T^\pi \xi(s):= r^\pi(s)+\gamma\mathbb{E}_{s^\prime |s} \xi(s^\prime)$, where $r^\pi(s)=\mathbb{E}_{a, s^\prime|s}r(s,a,s^\prime)$ is the expected reward of the Markov chain induced by the policy $\pi$. It is known that the value function $V^\pi(s)$ is the unique fixed point of the Bellman operator $T^\pi$, i.e., $V^\pi(s) = T^\pi V^\pi(s)$. 


\vspace{-3pt}
\subsection{TDC Learning with Linear Function Approximation} \label{subsec: TDC}
\vspace{-3pt}
In practice, the state space $\calS$ usually contains a large number of states that induces much computation overhead in policy evaluation. To address this issue, a popular approach is to approximate the value function via linear function approximation. Specifically, given a set of feature functions $\phi_i: \calS \to \RR$ for $i=1,2,\dots,d$ and define $\phi=[\phi_1,...,\phi_d]^\top$,
the value function of a given state $s$ can be approximated via the linear model
$\widehat{V}_\theta(s) := \phi(s)^\top \theta,$
where $\theta=[\theta_1, \theta_2, ..., \theta_d]^\top$ denotes all the model parameters. Suppose the state space includes states $s_1,...,s_n$, we denote the total value function as $\widehat{V}_\theta:=[\widehat{V}_\theta(s_1), ..., \widehat{V}_\theta(s_n)]^\top=\Phi\theta$, where $\Phi:=[\phi(s_1), ..., \phi(s_n)]^\top$.
In TDC learning, the goal is to evaluate the value function under linear function approximation via minimizing the following mean square projected Bellman error (MSPBE).
\begin{align*}
\text{MSPBE}(\theta) := \EE_{\mu_b}\| \widehat{V}_\theta - \Pi_{R_{\theta}} T^{\pi} \widehat{V}_\theta\|^2,
\end{align*} 
where $\Pi_{R_{\theta}}$ is the projection operator to the Euclidean ball with radius $R_{\theta}$.


In the off-policy TDC learning, we sample a trajectory of the MDP induced by the behavior policy $\pi_b$ and obtain samples $\{s_0, a_0, r_0, s_1, \dots, s_t, a_t, r_t, s_{t+1}, ...\}$. For the $t$-th sample $x_t=(s_t,a_t,r_t,s_{t+1})$, we define the following parameters 
\begin{align}\label{eq:def_abc}
    A_t:=  \rho(s_{t}, a_{t}) \phi(s_{t})(\gamma\phi(s_{t+1}) - \phi(s_{t}))^\top, \quad b_t:=  r_t \rho(s_{t}, a_{t})\phi(s_{t}),\\
    B_t :=  -\gamma  \rho(s_{t}, a_{t}) \phi(s_{t+1})\phi(s_{t})^\top, \quad
    C_t:=-\phi(s_{t})\phi(s_{t})^\top \nn, 
\end{align}
 where $\rho(s,a):= \frac{\pi(a|s)}{\pi_b(a|s)}$ is the importance sampling ratio. 
Then, with learning rates $\alpha, \beta>0$ and initialization parameters $\theta_0, w_0$, the two time-scale off-policy TDC algorithm takes the following recursive updates for $t=0,1,2,...$
\begin{equation}
\text{(Off-Policy TDC):}
\left\{
\begin{aligned}
\theta_{t+1} &= \theta_t + \alpha (A_t \theta_t + b_t + B_t w_t),  \\
w_{t+1} &= w_t + \beta (A_t \theta_t + b_t + C_t w_t).
\end{aligned}
\right.
\end{equation}
Also, for an arbitrary sample $(s,a,r,s')$, we define the following expectation terms for convenience of the analysis: $A:=\EE_{\mu_{\pi_b}} [ \rho(s,a) \phi(s)(\gamma\phi(s') - \phi(s))^\top ]$, $B:=  -\gamma\EE_{\mu_{\pi_b}}[ \rho(s,a) \phi(s')\phi(s)^\top],$ $C:=-\EE_{\mu_{\pi_b}}[ \phi(s)\phi(s)^\top]$ and $b:= \EE_{\mu_{\pi_b}}\left[r \rho(s,a)\phi(s) \right].$ 

With these notations, we introduce the following standard assumptions for our analysis \cite{Xu2020Reanalysis}.
\begin{assumption}[Problem solvability]\label{ass: solvability}
	The matrix $A$ and $C$ are non-singular. 
\end{assumption}

\begin{assumption}[Boundedness]\label{ass: boundedness} For all states $s,s'\in \calS$ and all actions $a\in \calA$,
\begin{enumerate}[leftmargin=*,noitemsep,topsep=0pt]
    \item The feature function is uniformly bounded as $\|\phi(s)\|\leq 1$;
    \item The reward is uniformly bounded as $r(s,a,s') \le r_{\max}$;
    \item The importance sampling ratio is uniformly bounded as
	$\rho(s,a) \leq \rho_{\max}$.
\end{enumerate}
\end{assumption}

\begin{assumption}[Geometric ergodicity]\label{ass: ergodicity}
	There exists $\kappa>0$ and $\rho \in (0,1)$ such that 	for all $t \geq 0$,
	\begin{align}
	    \sup_{s\in \calS} d_{\text{TV}}\left(  \PP(s_t \in \cdot | s_0 = s), \mu_{\pi_b}  \right) \leq \kappa \rho^t, \nonumber
	\end{align}
	where $d_{\text{TV}}(P,Q)$ denotes the total-variation distance between the probability measures $P$ and $Q$.
\end{assumption}

 Under \Cref{ass: solvability}, the optimal parameter $\theta^\ast$ can be written as $\theta^\ast = - A^{-1}b$.


\section{Variance-Reduced TDC for I.I.D.\ Samples}
\vspace{-5pt}
In this section, we propose a variance reduction scheme for the off-policy TDC over i.i.d.\ samples and analyze its non-asymptotic convergence rate.

\begin{algorithm}
	\SetAlgoLined 
	\textbf{Input:} learning rates $\alpha,\beta$, batch size $M$, initial parameters $\tilde{\theta}^{(0)},\tilde{w}^{(0)}$.
	
	\For{$m=1,2,\dots$}{
		Initialize $\theta^{(m)}_0 = \tilde{\theta}^{(m-1)}, w^{(m)}_0 = \tilde{w}^{(m-1)}$.
		
		Query a set $B_m$ of $M$ independent samples  from $\mu_{\pi_b}$ and compute
		\begin{align}
		    \widetilde{G}^{(m)} \!=\! \frac{1}{M} \!\!\sum_{x\in B_m}\!\!\! \big( A_x \tilde{\theta}^{(m-1)} \!+\! b_x \!+\! B_x \tilde{w}^{(m-1)}\big),
		    ~\widetilde{H}^{(m)} \!=\! \frac{1}{M}\!\! \sum_{x\in B_m} \!\!\!\big( A_x \tilde{\theta}^{(m-1)} \!+\! b_x \!+\! C_x \tilde{w}^{(m-1)} \big). \nonumber
		\end{align}
	
		\For{$t=0,1,\dots, M-1$}{
			Query a new sample $x_t^{(m)}$ from $\mu_{\pi_b}$ and compute
			\begin{align}
			\theta_{t+1}^{(m)} &= \Pi_{R_\theta}\Big[ \theta_{t}^{(m)} + \alpha\big(G_{t}^{(m)}(\theta_{t}^{(m)}, w_{t}^{(m)})  - G_{t}^{(m)}(\tilde{\theta}^{(m-1)}, \tilde{w}^{(m-1)})  +  \widetilde{G}^{(m)}\big)  \Big], \nonumber\\
			w_{t+1}^{(m)} &= \Pi_{R_w} \Big[ w_{t}^{(m)} + \beta \big(   H_{t}^{(m)}(\theta_{t}^{(m)}, w_{t}^{(m)})  - H_{t}^{(m)}(\tilde{\theta}^{(m-1)}, \tilde{w}^{(m-1)})  +  \widetilde{H}^{(m)}\big) \Big], \nonumber
			\end{align}
		where for any $\theta, w$, 
		$
		    G_{t}^{(m)}(\theta, w) =  A^{(m)}_t \theta + b^{(m)}_t + B^{(m)}_t w, ~~H_{t}^{(m)}(\theta, w) =  A^{(m)}_t \theta + b^{(m)}_t + C^{(m)}_t w. 
		$
		}
		Set $\tilde{\theta}^{(m)} = \frac{1}{M}\sum_{t=0}^{M-1} \theta_{t}^{(m-1)}, ~\tilde{w}^{(m)} = \frac{1}{M}\sum_{t=0}^{M-1} w_{t}^{(m-1)}$.
	}
	\textbf{Output:} $\tilde{\theta}^{(m)}$.
	\caption{Variance-Reduced TDC for I.I.D. Samples}
	 \label{alg: iid}
\end{algorithm}

\vspace{-3pt}
\subsection{Algorithm Design}
\vspace{-3pt}
In the i.i.d.\ setting, we assume that one can query independent samples $x=(s,a,r,s')$ from the stationary distribution $\mu_{\pi_b}$ induced by the behavior policy $\pi_b$. In particular, we define $A_x, B_x, C_x, b_x$ based on the sample $x$ and define $A_t^{(m)}, B_t^{(m)}, C_t^{(m)}, b_t^{(m)}$ based on the sample $x_t^{(m)}$ in a similar way as how we define $A_t, B_t, C_t, b_t$ based on the sample $x_t$ in \cref{eq:def_abc}. 

We then propose the variance-reduced TDC algorithm for i.i.d.\ samples in \Cref{alg: iid}. 
To elaborate, the algorithm runs for $m$ outer-loops, each of which consists of $M$ inner-loops. Specifically, in the $m$-th outer-loop, we first initialize the parameters $\theta_0^{(m)}, w_0^{(m)}$ with $\tilde{\theta}^{(m-1)}, \tilde{w}^{(m-1)}$, respectively, which are the output of the previous outer-loop. Then, we query $M$ independent samples from $\mu_{\pi_b}$ and compute a pair of batch pseudo-gradients $\widetilde{G}^{(m)}, \widetilde{H}^{(m)}$ to be used in the inner-loops. In the $t$-th inner-loop, we query a new independent sample $x_t^{(m)}$ and compute the corresponding stochastic pseudo-gradients $G_t^{(m)}(\theta_t^{(m)}, w_t^{(m)}), G_t^{(m)}(\tilde{\theta}^{(m-1)}, \tilde{w}^{(m-1)})$. Then, we update the parameters $\theta_{t+1}^{(m)}$ and  $w_{t+1}^{(m)}$ using the batch pseudo-gradient and stochastic pseudo-gradients via the SVRG variance reduction scheme. At the end of each outer-loop, we set the parameters $\tilde{\theta}^{(m)}, \tilde{w}^{(m)}$ to be the average of the parameters $\{\theta_t^{(m-1)}, w_t^{(m-1)}\}_{t=0}^{M-1}$ obtained in the inner-loops, respectively. 
We note that the updates of $\theta_{t+1}^{(m)}$ and  $w_{t+1}^{(m)}$ involve two projection operators, which are widely adopted in the literature, e.g., \cite{bhandari2018finite,bubeck2015convex,dalal2018finite,dalal2017finite,kushner2010stochastic,lacoste2012simpler,xu2019two,zou2019finite}. Throughout this paper, we assume the radius $R_\theta,R_w$ of the projected Euclidean balls satisfy that $R_\theta \geq \max\{\|A\| \|b\|, \|\theta^\ast\| \}$, $R_w \geq 2 \|C^{-1} \|\|A\|R_\theta$.

Compare to the conventional variance-reduced TD for i.i.d.\ samples that applies variance reduction to only the one time-scale update of $\theta_{t}^{(m)}$ \cite{korda2015td,Xu2020Reanalysis}, our VRTDC for i.i.d.\ samples applies variance reduction to both $\theta_{t}^{(m)}$ and $w_{t}^{(m)}$ that are in two different time-scales. As we show in the following subsection, such a two time-scale variance reduction scheme leads to an improved sample complexity of VRTDC over that of VRTD \cite{Xu2020Reanalysis}.


\vspace{-3pt}
\subsection{Non-Asymptotic Convergence Analysis}
\vspace{-3pt}

The following theorem presents the convergence result of VRTDC for i.i.d.\ samples. Due to space limitation, we omit other constant factors in the bound. The exact bound can be found in \Cref{appendix: iid proof}.

\begin{theorem}\label{thm: iid}
	Let Assumptions \ref{ass: solvability}, \ref{ass: boundedness} and \ref{ass: ergodicity} hold. Connsider the VRTDC for i.i.d.\ samples in \Cref{alg: iid}. If the learning rates $\alpha$, $\beta$ and the batch size $M$ satisfy the conditions specified in \cref{eq: lr iid 1,eq: lr iid 2,eq: lr iid 3,eq: lr iid 4,eq: lr iid 5,eq: lr iid 6,eq: lr iid 7} (see the appendix) and $\beta = \calO(\alpha^{\frac{2}{3}})$, then, the output of the algorithm satisfies
 \blue{
	\begin{align}
	     \EE  \| \tilde{\theta}^{(m)} - \theta^\ast\|^2 &\leq \calO(m D^m+\beta^4+ {M}^{-1}), \nonumber
	\end{align}}
	where $D= \frac{12}{\lambda_{\widehat{A}}}\Big\{  \frac{1}{\alpha M} +  \alpha  \cdot 5(1+\gamma)^2\rho^2_{\max}\big(1 + \frac{\gamma \rho_{\max}}{ \min |\lambda(C)|}\big)^2   +     \frac{\alpha^2}{\beta^2} \cdot C_1  + \beta \cdot C_2   \Big\}\in (0,1)$ (see \Cref{supp: notation} for the definitions of $\lambda_{\widehat{A}}, C_1, C_2$). 
	In particular, choose $\alpha = \calO(\epsilon^{\frac{3}{5}}), \beta = \calO({\epsilon^{\frac{2}{5}}})$, $m = \calO(\log{\epsilon}^{-1})$ and \blue{$M = \calO(\epsilon^{-1})$}, then the {total sample complexity} to achieve $\EE \| \tilde{\theta}^{(m)} - \theta^\ast\|^2\le \epsilon$ is in the order of \blue{$\calO({\epsilon^{-1}}\log{\epsilon}^{-1})$}.
\end{theorem}

Theorem \ref{thm: iid} shows that in the i.i.d.\ setting, VRTDC with constant stepsizes converges to a neighborhood of the optimal solution $\theta^*$ at a linear convergence rate. In particular, the asymptotic convergence error is in the order of \blue{$\calO(\beta^4+M^{-1})$}, which can be driven arbitrarily close to zero by choosing a sufficiently small stepsize $\beta$. Also, the required sample complexity for VRTDC to achieve an $\epsilon$-accurate solution is \blue{$\calO(\epsilon^{-1}\log{\epsilon}^{-1})$, which matches the best-known sample complexity $\calO(\epsilon^{-1}\log{\epsilon}^{-1})$ required by the conventional VRTD for i.i.d.\ samples \cite{Xu2020Reanalysis}.} 

Our analysis of VRTDC in the i.i.d.\ setting requires substantial developments of new bounding techniques. To elaborate, note that in the analysis of the one time-scale VRTD \cite{Xu2020Reanalysis}, they only need to deal with the parameter $\tilde{\theta}^{(m)}$ and bound its variance error using constant-level bounds. As a comparison, in the analysis of VRTDC we need to develop much refined variance reduction bounds for both $\tilde{\theta}^{(m)}$ and $\tilde{w}^{(m)}$ that are correlated with each other. Specifically, we first develop a preliminary bound for the tracking error $\|\tilde{z}^{(m)}\|^2= \|\tilde{w}^{(m)} + C^{-1}(b + A(\tilde{\theta}^{(m)}))\|^2$ that is in the order of $\calO(D^m+\beta)$ (Lemma \ref{lemma: iid conv-z}), which is further used to develop a preliminary bound for the convergence error $\|\theta^{(m)}-\theta^\ast\|^2\le \calO(D^m+{M}^{-1})$ (Lemma \ref{lemma: iid theta}). Then, by leveraging the relation between tracking error and convergence error induced by the
two time-scale updates of VRTDC (Lemma \ref{lemma: iid pre-bound z}), we further apply the preliminary bound of $\|\theta^{(m)}-\theta^\ast\|^2$ to obtain a refined bound for the tracking error \blue{$\|z^{(m)}\|^2\leq \calO(D^m+\beta^3+{M}^{-1})$} (Lemma \ref{lemma: iid refined z}). Finally, the refined bound of $\|z^{(m)}\|^2$ is applied to derive a refined bound for the convergence error \blue{$\|\theta^{(m)}-\theta^\ast\|^2\le \calO(D^m+\beta^4+ {M}^{-1})$} by leveraging the two time-scale updates (Lemma \ref{lemma: iid pre-bound theta}). These refined bounds are the key to establish the improved sample complexity of VRTDC for i.i.d. samples over the state-of-the-art result.

We also obtain the following convergence rate of the tracking error of VRTDC in the i.i.d.\ setting. 
\begin{corollary}\label{cor: iid}
Under the same settings and parameter choices as those of \Cref{thm: iid}, the tracking error of VRTDC for i.i.d.\ samples satisfies
\blue{\begin{align}
    \EE \|\tilde{z}^{(m)}\|^2 &\leq \calO(D^m + \beta^3 +  {M}^{-1}). \nonumber
\end{align}}
Moreover, the total sample complexity to achieve $\EE \|\tilde{z}^{(m)}\|^2\le \epsilon$ is in the order of \blue{$\calO(\epsilon^{-1} \log{\epsilon}^{-1})$}.
\end{corollary}

\section{Variance-Reduced TDC for Markovian Samples}
\vspace{-5pt}
In this section, we propose a variance-reduced TDC algorithm for Markovian samples and characterize its non-asymptotic convergence rate.

\vspace{-5pt}
\begin{algorithm}
	\SetAlgoLined 
	\textbf{Input:} learning rates $\alpha,\beta$, batch size $M$, initial parameters $\tilde{\theta}^{(0)},\tilde{w}^{(0)}$,
	
	\qquad\quad Markovian samples of MDP $\{x_1, x_2, x_3,...\}$.
	
	\For{$m=1,2,\dots$}{
		Initialize $\theta^{(m)}_0 = \tilde{\theta}^{(m-1)}$,
		$w^{(m)}_0 = \tilde{w}^{(m-1)}$.
		
		Query the set of samples $B_m=\{x_{(m-1)M},..., x_{mM-1}\}$ and compute 
		\begin{align}
		    \widetilde{G}^{(m)} \!=\! \frac{1}{M} \!\!\!\sum_{x\in B_m} \!\!\!\big( A_{x} \tilde{\theta}^{(m-1)} \!+\! b_{x} \!+\! B_{x} \tilde{w}^{(m-1)}\big), ~
		    \widetilde{H}^{(m)} \!=\! \frac{1}{M} \!\!\!\sum_{x\in B_m}\!\!\!\big( A_{x} \tilde{\theta}^{(m-1)} \!+\! b_{x} \!+\! C_{x} \tilde{w}^{(m-1)} \big). \nonumber
		\end{align}
		
		\For{$t=0,1,\dots, M-1$}{
			Query a sample $x^{(m)}_t$ from $B_m$ uniformly at random and update
			\begin{align}
			\theta_{t+1}^{(m)} &= \Pi_{R_\theta}\Big[ \theta_{t}^{(m)} + \alpha\big(G_{t}^{(m)}(\theta_{t}^{(m)}, w_{t}^{(m)})  - G_{t}^{(m)}(\tilde{\theta}^{(m-1)}, \tilde{w}^{(m-1)})  +  \widetilde{G}^{(m)}\big)  \Big], \nonumber\\
			w_{t+1}^{(m)} &= \Pi_{R_w} \Big[ w_{t}^{(m)} + \beta \big(   H_{t}^{(m)}(\theta_{t}^{(m)}, w_{t}^{(m)})  - H_{t}^{(m)}(\tilde{\theta}^{(m-1)}, \tilde{w}^{(m-1)})  +  \widetilde{H}^{(m)}\big) \Big], \nonumber
			\end{align}
		where for any $\theta, w$, we define
		\begin{align}
		    G_{t}^{(m)}(\theta, w) =  A^{(m)}_t \theta + b^{(m)}_t + B^{(m)}_t w, ~~H_{t}^{(m)}(\theta, w) =  A^{(m)}_t \theta + b^{(m)}_t + C^{(m)}_t w. \nonumber
		\end{align}
		}
		Set $\tilde{\theta}^{(m)} = \frac{1}{M}\sum_{t=0}^{M-1} \theta_{t}^{(m-1)}, ~\tilde{w}^{(m)} = \frac{1}{M}\sum_{t=0}^{M-1} w_{t}^{(m-1)}$.
	}
	\textbf{Output:} $\tilde{\theta}^{(m)}$.
	\caption{TDC with Variance Reduction for Markovian Samples}\label{alg: markov}
\end{algorithm}

\vspace{-13pt}
\subsection{Algorithm Design}
\vspace{-3pt}
In the Markovian setting, we generate a single trajectory of the MDP and obtain a series of Markovian samples $\{x_1, x_2, x_3,...\}$, where the $t$-th sample is $x_t=(s_t,a_t,r_t,s_{t+1})$. The detailed steps of variance-reduced TDC for Markovian samples are presented in \Cref{alg: markov}.
To elaborate, in the Markovian case, we divide the samples of the Markovian trajectory into $m$ batches. In the $m$-th outer-loop, we query the $m$-th batch of samples to compute the batch pseudo-gradients $\widetilde{G}^{(m)},\widetilde{H}^{(m)}$. Then, in each of the corresponding inner-loops we query a sample from the same $m$-th batch of samples and perform variance-reduced updates on the parameters $\theta_t^{(m)},w_t^{(m)}$. As a comparison, in our design of VRTDC for i.i.d.\ samples, the samples used in both the outer-loops and the inner-loops are independently drawn from the stationary distribution.

\vspace{-3pt}
\subsection{Non-Asymptotic Convergence Analysis}
\vspace{-3pt}
In the following theorem, we present the convergence result of VRTDC for Markovian samples. Due to space limitation, we omit the constants in the bound. The exact bound can be found in \Cref{appendix: markov proof}.
\begin{theorem}\label{thm: markov}
	Let Assumptions \ref{ass: solvability}, \ref{ass: boundedness} and \ref{ass: ergodicity} hold and consider the VRTDC in  \Cref{alg: markov} for Markovian samples. Choose learning rates $\alpha$, $\beta$ and the batch size $M$ that satisfy the conditions specified in \cref{eq: lr markov 1,eq: lr markov 2,eq: lr markov 3,eq: lr markov 4,eq: lr markov 5,eq: lr markov 6} (see the appendix) and $\beta = \calO(\alpha^{\frac{2}{3}})$. Then, the output of the algorithm satisfies
	\begin{align*}
	\EE \| \tilde{\theta}^{(m)} - \theta^\ast\|^2 &\leq  \calO(D^m + {M}^{-1}+\beta^2), \nonumber 
	\end{align*}
	where $D= \frac{16}{\lambda_{\widehat{A}}}\Big[\frac{ 1}{\alpha M} +  \alpha  \cdot 5(1+\gamma)^2\rho^2_{\max}\Big(1 + \frac{\gamma \rho_{\max}}{ \min |\lambda(C)|}\Big)^2 + \frac{\alpha^2}{\beta^2}  \cdot C_1   +   \beta  \cdot C_2   \Big] \in (0,1)$ (see \Cref{supp: notation} for the definitions of $\lambda_{\widehat{A}}, C_1, C_2$).
	In particular, choose $\alpha = \calO({\epsilon^{\frac{3}{4}}}),\beta = \calO({\epsilon^{\frac{1}{2}}})$, $m=\calO(\log{\epsilon}^{-1})$ and $M=\calO({\epsilon}^{-1})$, the {total sample complexity} to achieve $\EE \| \tilde{\theta}^{(m)} - \theta^\ast\|^2\le \epsilon$ is in the order of $\calO(\epsilon^{-1}\log{\epsilon}^{-1})$.
\end{theorem}


 
 
The above theorem shows that in the Markovian setting, VRTDC with constant learning rates converges linearly to a neighborhood of $\theta^*$ with an asymptotic convergence error $\calO({M}^{-1}  + \beta^2)$. As expected, such an asymptotic convergence error is larger than that of VRTDC for i.i.d. samples established in \Cref{thm: iid}. 
We also note that the overall sample complexity of VRTDC for Markovian samples is in the order of $\calO(\epsilon^{-1}\log{\epsilon}^{-1})$, which matches that of TDC for Markovian samples \cite{kaledin2020finite} and VRTD for Markovian samples \cite{Xu2020Reanalysis}. Such a complexity result nearly matches the theoretical lower bound $\calO(\epsilon^{-1})$ given in \cite{kaledin2020finite}. Moreover, as we show later in the experiments, our VRTDC always achieves a smaller asymptotic convergence error than that of TDC and VRTD in practice.

 
We note that in the Markovian case, one can also develop refined error bounds following the recursive refinement strategy used in the proof of \Cref{thm: iid}. However, the proof of \Cref{thm: markov} only applies the preliminary bounds for the tracking error and the convergence error, which suffices to obtain the above desired near-optimal sample complexity result. This is because the error terms in \Cref{thm: markov} are dominated by $\frac{1}{M}$ and hence applying the refined bounds does not lead to a better sample complexity result.

We also obtain the following convergence rate of the tracking error of VRTDC in the Markovian setting, where the total sample complexity matches that of the TDC for Markovian samples \cite{kaledin2020finite,xu2019two}.
\begin{corollary}\label{cor: markov}
	Under the same settings and parameter choices as those of \Cref{thm: markov}, the tracking error of VRTDC for Markovian samples  satisfies
	\begin{align}
	\EE \|\tilde{z}^{(m)}\|^2 &\leq \calO(D^m +  \beta^3 + {M}^{-1}). \nonumber
	\end{align}
	Moreover, the total sample complexity to achieve $\EE \|\tilde{z}^{(m)}\|^2\le \epsilon$ is in the order of $\calO(\epsilon^{-1} \log{\epsilon}^{-1})$.
\end{corollary}

\vspace{-5pt}
\section{Experiments}
\vspace{-5pt}
\label{sec:exp}
In this section, we conduct two reinforcement learning experiments, Garnet problem and Frozen Lake game, to explore the performance of VRTDC in the off-policy setting with Markovian samples, and compare it with  TD, TDC, VRTD and VRTDC in the Markovian setting. 

\vspace{-3pt}
\subsection{Garnet Problem} 
\vspace{-3pt}
We first consider the Garnet problem \cite{archibald1995generation, xu2019two} that is specified as $\mathcal{G}(n_\calS, n_\calA, b, d)$, where $n_\calS$ and $n_\calA$ denote the cardinality of the state and action spaces, respectively, $b$ is referred to as the branching factor--the number of states that have strictly positive probability to be visited after an action is taken, and $d$ denotes the dimension of the features. We set $n_\calS = 500$, $n_\calA = 20$, $b=50, d=15$ and generate the features $\Phi \in \RR^{n_\calS \times d}$ via the uniform distribution on $[0,1]$. We then normalize its rows to have unit norm. Then, we randomly generate a state-action transition kernel $\kerP \in \RR^{n_\calS \times n_\calA \times n_\calS}$ via the uniform distribution on $[0,1]$ (with proper normalization). We set the behavior policy as the uniform policy, i.e., $\pi_b(a|s) = n_\calA^{-1}$ for any $s$ and $a$, and we generate the target policy $\pi(\cdot|s)$ via the uniform distribution on $[0,1]$ with proper normalization for every state $s$. The discount factor is set to be $\gamma = 0.95$. As the transition kernel and the features are known, we compute $\theta^\ast$ and use $\|\theta- \theta^\ast\|$ to evaluate the performance of all the algorithms.

We set the learning rate $\alpha = 0.1$ for all the four algorithms, and set the other learning rate $\beta = 0.02$ for both VRTDC and TDC. For VRTDC and VRTD, we set the batch size $M=3000$. In \Cref{fig:example} (Left), we plot the convergence error as a function of number of pseudo stochastic gradient computations for all these algorithms. Specifically, we use $100$ Garnet trajectories with length $50$k to obtain 100 convergence error curves for each algorithm. 
The upper and lower envelopes of the curves correspond to the $95\%$ and $5\%$ percentiles of the $100$ curves, respectively. 
It can be seen that both VRTD and VRTDC outperform TD and TDC and asymptotically achieve smaller mean convergence errors with reduced numerical variances of the curves. This demonstrates the advantage of applying variance reduction to policy evaluation. Furthermore, comparing VRTDC with VRTD, one can observe that VRTDC achieves a smaller mean convergence error than that achieved by VRTD, and the numerical variance of the VRTDC curves is smaller than that of the VRTD curves. This demonstrates the effectiveness of applying variance reduction to two time-scale updates.    

We further compare the asymptotic error of VRTDC and VRTD under different batch sizes $M$. We use the same learning rate setting as mentioned above and run $100$k iterations for each of the $250$ Garnet trajectories. For each trajectory, we use the mean of the convergence error of the last $10$k iterations as an estimate of the asymptotic convergence error (the training curves are already  flattened).
\Cref{fig:example} (Right) shows the box plot of the 250 samples of the asymptotic convergence error of VRTDC and VRTD under different batch sizes $M$. It can be seen that the two time-scale VRTDC achieves a smaller mean asymptotic convergence error with a smaller numerical variance than that achieved by the one time-scale VRTD.


\begin{figure}[tbh]
\captionsetup[subfigure]{justification=centering}
	\centering
	\vspace{-10pt}
	\begin{subfigure}{0.35\linewidth}
		\includegraphics[width=\linewidth]{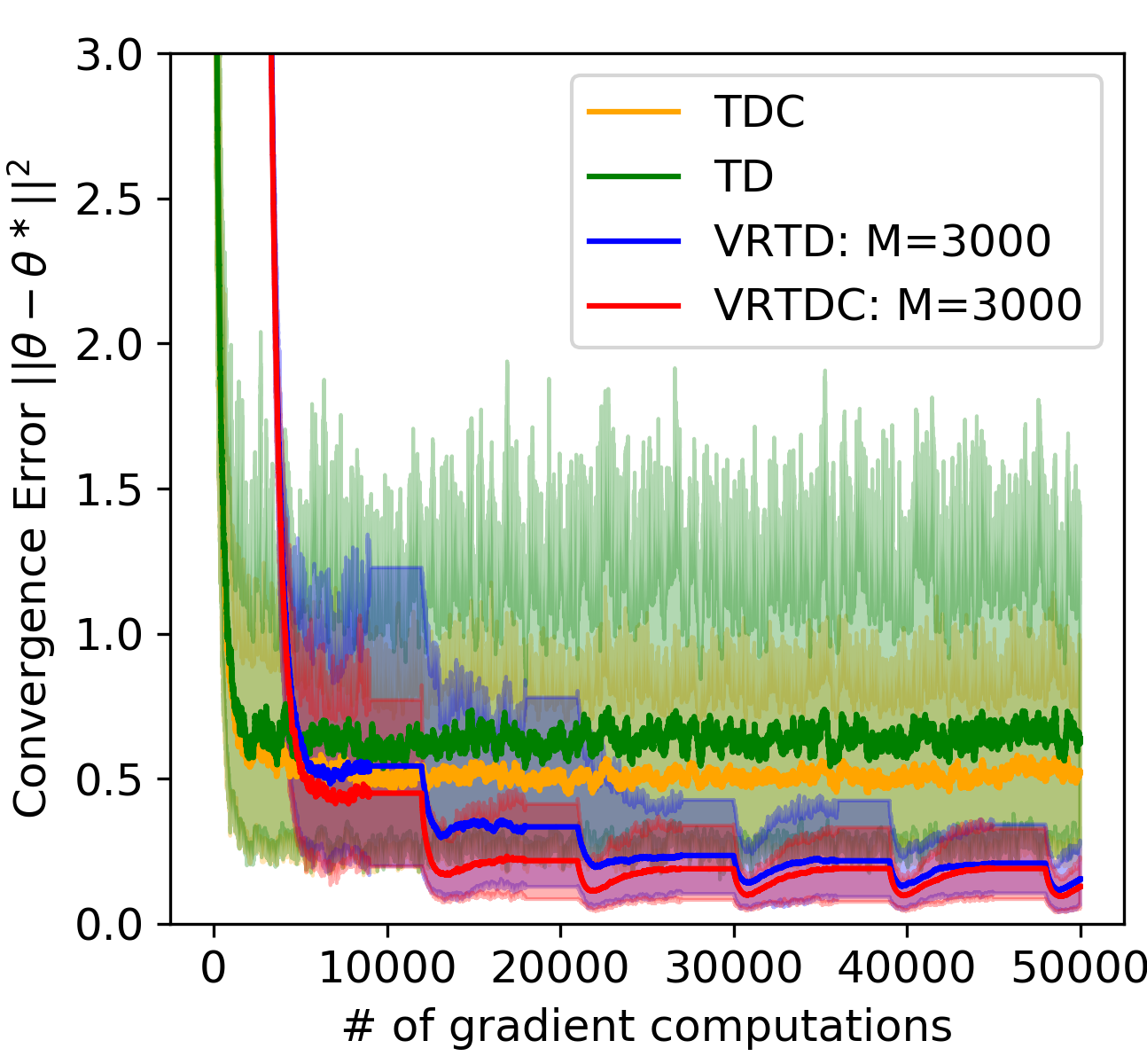}
	\end{subfigure}
	\hspace{20pt}
	\begin{subfigure}{0.35\linewidth}
		\includegraphics[width=\linewidth]{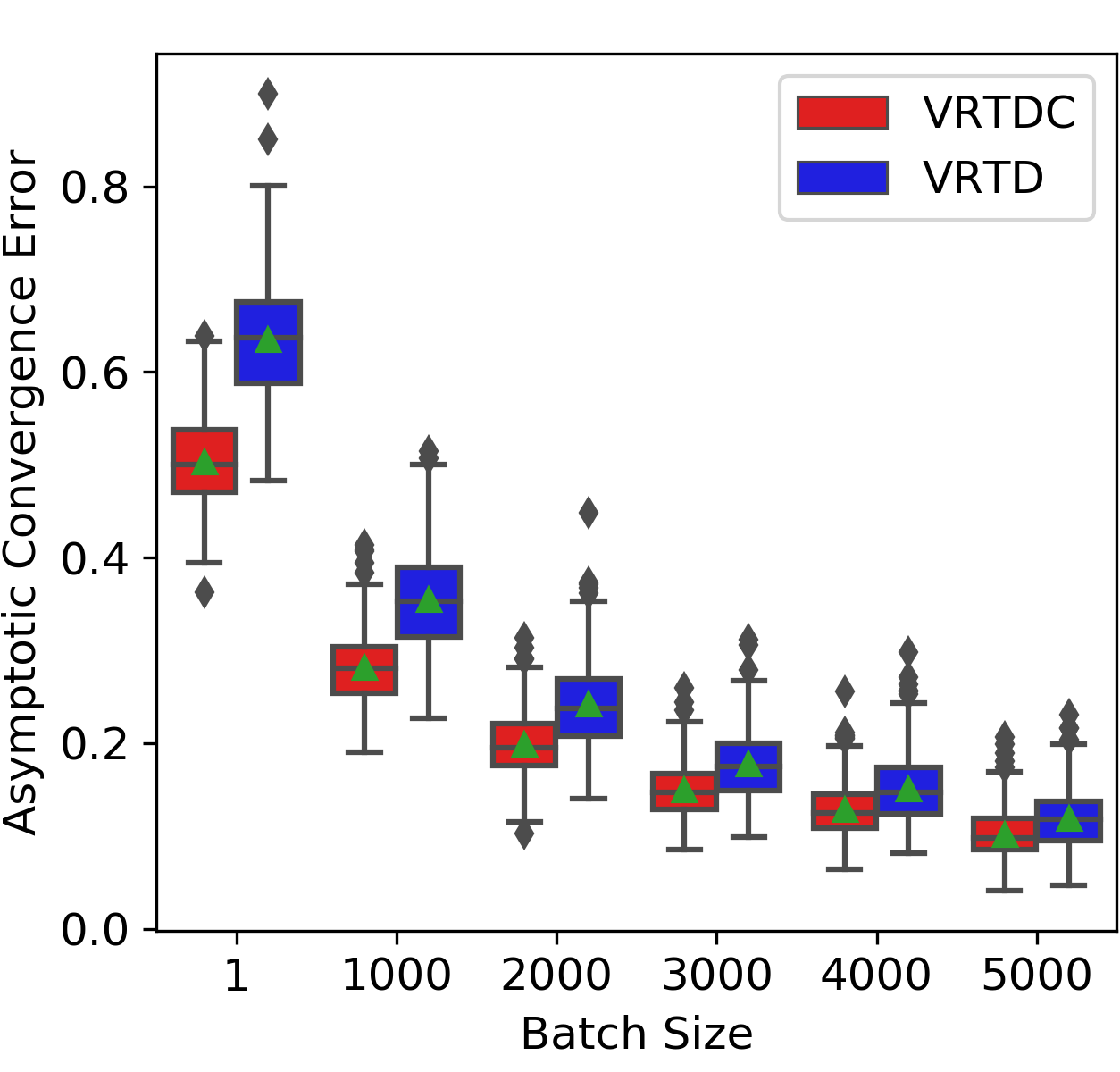}
	\end{subfigure}
	\vspace{-5pt}
	\caption{Comparison of TD, TDC, VRTD, VRTDC in solving the Garnet problem.}%
	\label{fig:example}%
	\vspace{-5pt}
\end{figure}  

We further study the variance reduction effect of VRTDC. We plot the estimated variance of the stochastic updates of $\theta$ (see \Cref{fig:example3} left) and $w$ (see \Cref{fig:example3} right) for different algorithms. It can be seen that VRTDC significantly reduces the variance of TDC in both time-scales. For each step, we estimate the variance of each pseudo-gradient update using Monte-Carlo method with $500$ additional samples under the same learning rates and batch size setting as mentioned previously.
\begin{figure}[tbh]
\captionsetup[subfigure]{justification=centering}
	\centering
	\vspace{-10pt}
	\begin{subfigure}{0.35\linewidth}
		\includegraphics[width=\linewidth]{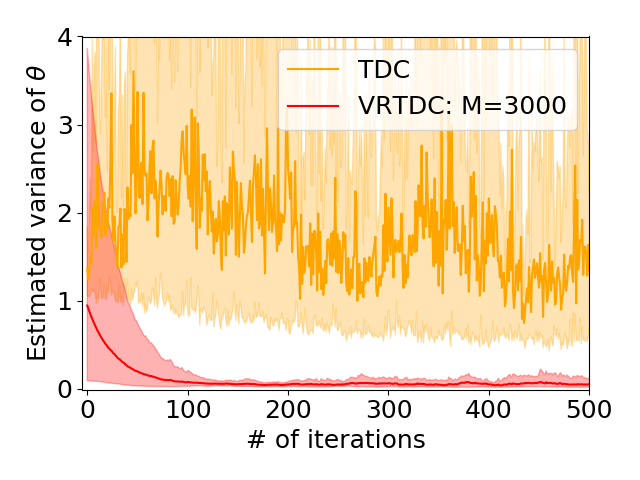}
	\end{subfigure}
	\hspace{20pt}
	\begin{subfigure}{0.35\linewidth}
		\includegraphics[width=\linewidth]{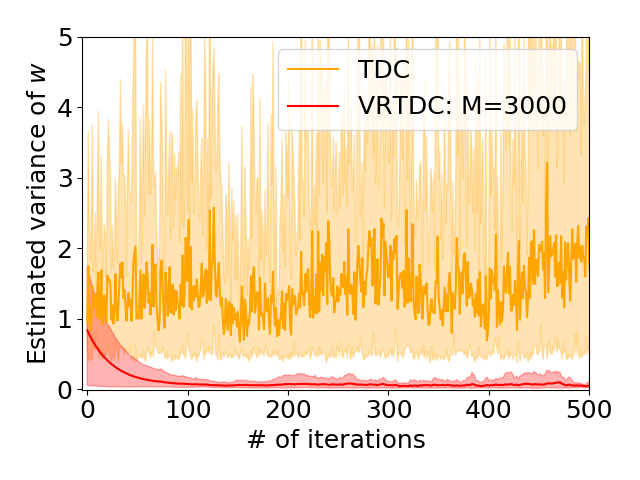}
	\end{subfigure}

	\vspace{-5pt}
	\caption{Comparison of TDC, VRTDC in solving the Garnet problem.}%
	\label{fig:example3}%
	\vspace{-5pt}
\end{figure}  

\vspace{-3pt}
\subsection{Frozen Lake Game}
\vspace{-3pt}
Our second experiment considers the frozen lake game in the OpenAI Gym \cite{1606.01540}. We generate a Gaussian feature matrix with dimension $4$ to linearly approximate the value function and we aim to evaluate a target policy based on a behavior policy,  generated via the uniform distribution. We set the learning rates $\alpha = 0.1$ and $\beta = 0.01$ for all the algorithms and set the batch size $M=3000$ for the variance-reduced algorithms. We run $50$k iterations for each of the $100$ trajectories. \Cref{fig:example2} (Left) plots the convergence error as a function of number of gradient computations for the four algorithms using $5\%$ and $95\%$ percentiles of the 100 curves. One can see that our VRTDC asymptotically achieves the smallest mean convergence error with the smallest numerical variance of the curves. In particular, one can see that TDC achieves a comparable asymptotic error to that of VRTD, and VRTDC outperforms both of the algorithms.
\Cref{fig:example2} (Right) further compares the asymptotic convergence error of VRTDC with that of VRTD under different batch sizes $M$. Similar to the Garnet experiment, for each of the 100 trajectories we use the mean of the convergence errors of the last $10$k iterations as an estimate of the asymptotic error, and the boxes include the samples between $25\%$ percentile and $75\%$ percentile. One can see from the figure that VRTDC achieves smaller asymptotic errors with smaller numerical variances than VRTD under all choices of batch size.

\begin{figure}[tbh]
\captionsetup[subfigure]{justification=centering}
	\centering
	\vspace{-10pt}
	\begin{subfigure}{0.35\linewidth}
		\includegraphics[width=\linewidth]{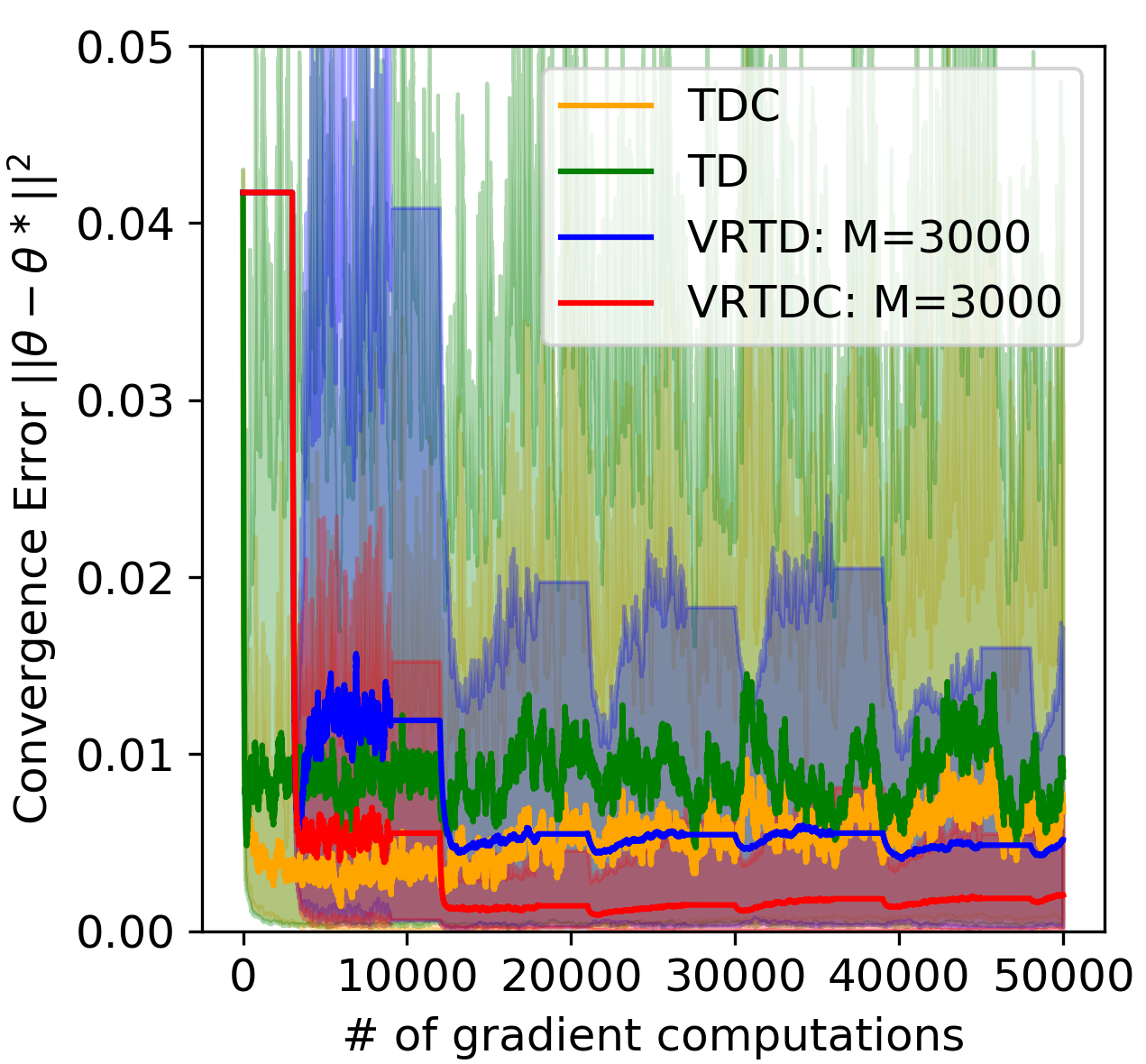}
	\end{subfigure}
	\hspace{20pt}
	\begin{subfigure}{0.35\linewidth}
		\includegraphics[width=\linewidth]{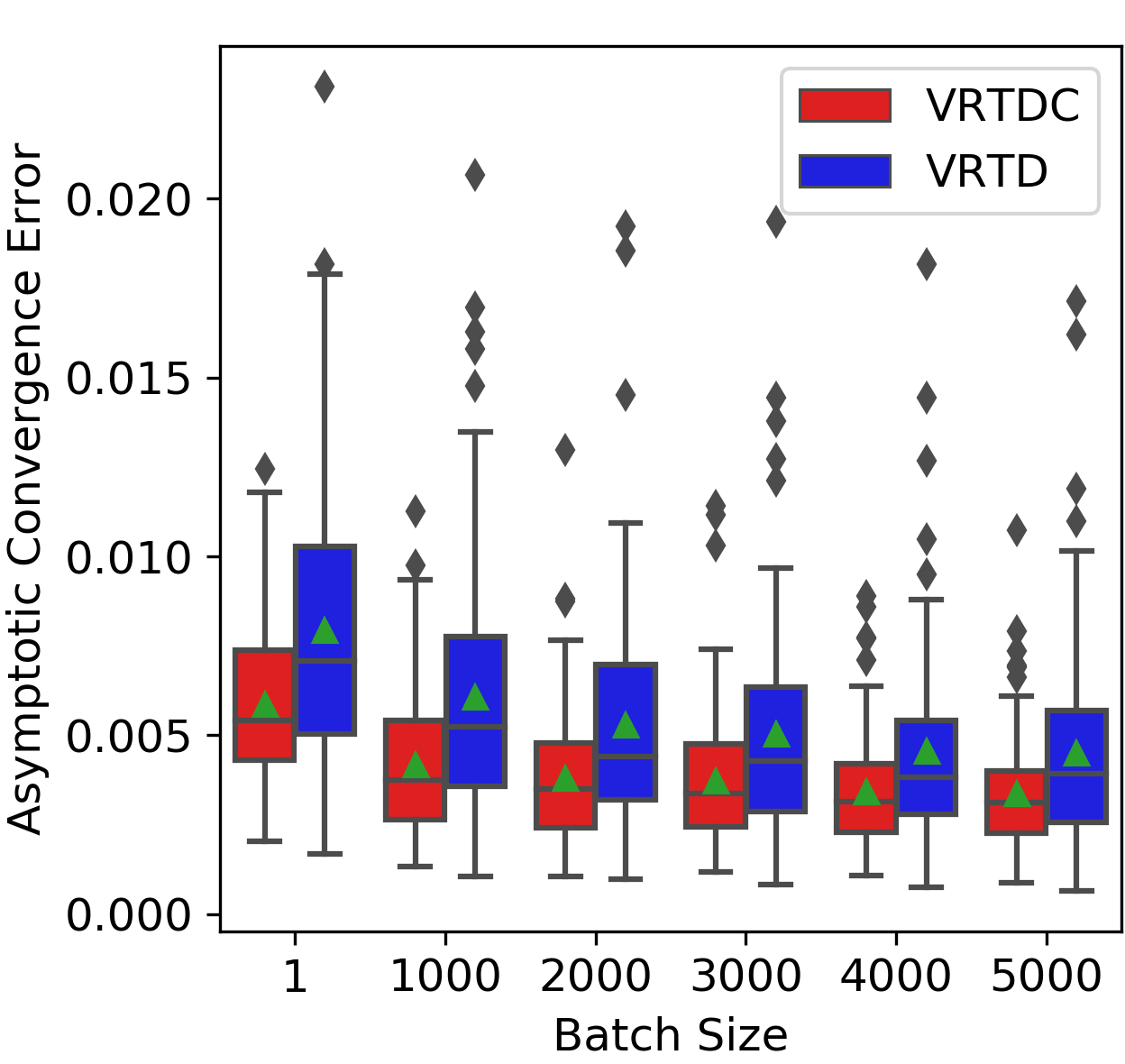}
	\end{subfigure}

	\vspace{-5pt}
	\caption{Comparison of TD, TDC, VRTD, VRTDC in solving the frozen lake problem.}%
	\label{fig:example2}%
	\vspace{-5pt}
\end{figure}

Similar to Figure \ref{fig:example3}, we also estimate the variance of each pseudo-gradient update using $500$ Monte Carlo samples for the Frozen Lake problem. We can find that VRTDC also has lower variance for the Frozen Lake problem.
\begin{figure}[tbh]
\captionsetup[subfigure]{justification=centering}
	\centering
	\vspace{-10pt}
	\begin{subfigure}{0.35\linewidth}
		\includegraphics[width=\linewidth]{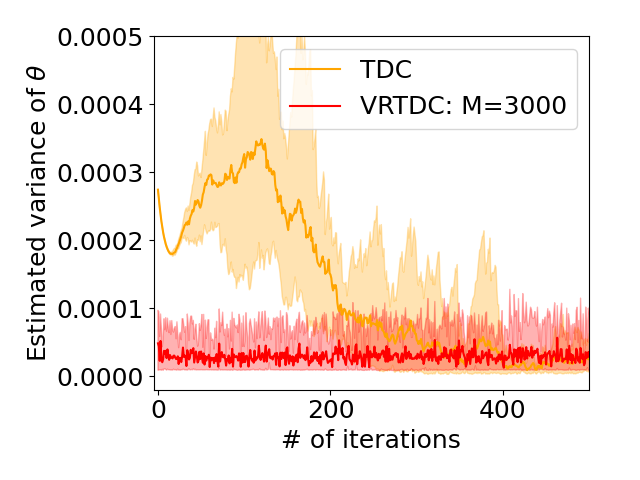}
	\end{subfigure}
	\hspace{20pt}
	\begin{subfigure}{0.35\linewidth}
		\includegraphics[width=\linewidth]{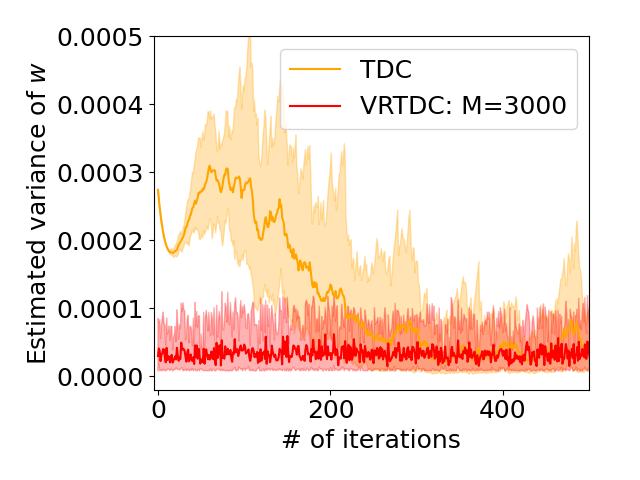}
	\end{subfigure}

	\vspace{-5pt}
	\caption{Comparison of TDC, VRTDC in solving the frozen lake problem.}%
	\label{fig:example4}%
	\vspace{-5pt}
\end{figure}  
\vspace{-10pt}

\section{Conclusion}
\vspace{-5pt}
In this paper, we proposed two variance-reduced off-policy TDC algorithms for policy evaluation with i.i.d.\ samples and Markovian samples, respectively. We developed new analysis techniques and showed that VRTDC for i.i.d.\ samples achieves  \blue{a state-of-the-art sample complexity}, and VRTDC for Markovian samples achieves the best existing sample complexity. We expect that the developed VRTDC algorithm can help reduce the stochastic variance in reinforcement learning applications and improve the solution accuracy.

\section*{Acknowledgement}
The work of S. Zou was supported by the National Science Foundation under Grant CCF-2007783.

\section*{Broader Impact}
This work exploits techniques in multidisciplinary areas including reinforcement learning, stochastic optimization and statistics, and contributes new technical developments to analyze TD learning algorithm under stochastic variance reduction.
The proposed two time-scale VRTDC  significantly improves the solution quality of TD learning and reduces the variance and uncertainty in training reinforcement learning policies. Therefore it has the potential to be applied to reinforcement learning applications such as autonomous driving, decision making and control to reduce the risk caused by uncertainty of the policy.

\bibliographystyle{abbrvnat}

\newpage
\appendix

\addcontentsline{toc}{section}{Appendix} 
\part{Appendix} 
\parttoc 
\allowdisplaybreaks

\section{Filtration, Additional Notations and List of Constants}\label{supp: notation}
\subsection*{Filtration for I.I.D. samples}
The definition of filtration is similar as that in VRTD (Appendix D, \cite{Xu2020Reanalysis}). Recall that in Algorithm \ref{alg: iid}, $B_m$ consists of $M$ independent samples that are sampled from $\mu_{\pi_b}$, and $x_t^{(m)}$ is another independent sample sampled in the $t$-th iteration of the $m$-th epoch. Let $\sigma(A \cup B)$ be the smallest $\sigma$-field that includes both $A$ and $B$. Then we define the filtration for I.I.D. samples as follow
\begin{align*}
    & F_{1,0} = \sigma(\tilde{\theta}^{(0)}, \tilde{w}^{(0)} ), F_{1,1}=\sigma(F_{1,0} \cup \sigma(B_1) \cup \sigma(x_{0}^{(1)})),\dots, F_{1,M}= \sigma(F_{1,M-1} \cup \sigma(x_{M-1}^{(1)})) \\
      & F_{2,0} = \sigma\big( F_{1,M} \cup \sigma(\tilde{\theta}^{(1)}, \tilde{w}^{(1)} )  \big), F_{2,1}=\sigma(F_{2,0} \cup \sigma(B_2) \cup \sigma(x_{0}^{(2)})),\dots, F_{2,M}= \sigma(F_{2,M-1} \cup \sigma(x_{M-1}^{(2)})) \\
    &\vdots\\
      & F_{m,0} = \sigma\big( F_{m-1,M} \cup \sigma(\tilde{\theta}^{(m-1)}, \tilde{w}^{(m-1)} )  \big), F_{m,1}=\sigma(F_{m,0} \cup \sigma(B_m) \cup \sigma(x_{0}^{(m)})),\dots, \\
      &\quad F_{m,M}= \sigma(F_{m,M-1} \cup \sigma(x_{M-1}^{(m)})) .
\end{align*}
Moreover, we define $\EE_{t,m}$ as the conditional expectation with respect to the $\sigma$-field $F_{t,m}$.

\subsection*{Filtration for Markovian samples}
The definition of filtration is similar as that in VRTD (Appendix E, \cite{Xu2020Reanalysis}). We first recall that $B_m$ denotes the set of Markovian samples used in the $m$-th epoch, and we also abuse the notation here by letting $x^{(m)}_t$ be the sample picked in the $t$-th iteration of the $m$-th epoch. Then, we define the filtration for Markovian samples as follows
\begin{align*}
    & F_{1,0} = \sigma( B_0 \cup \sigma(\tilde{\theta}^{(0)}, \tilde{w}^{(0)} ) ), F_{1,1}=\sigma(F_{1,0} \cup  \sigma(x_{0}^{(1)})),\dots, F_{1,M}= \sigma(F_{1,M-1} \cup \sigma(x_{M-1}^{(1)})) \\
      & F_{2,0} = \sigma\big( B_1  \cup F_{1,M}\cup \sigma(\tilde{\theta}^{(1)}, \tilde{w}^{(1)} )  \big), F_{2,1}=\sigma(F_{2,0}  \cup \sigma(x_{0}^{(2)})),\dots, F_{2,M}= \sigma(F_{2,M-1} \cup \sigma(x_{M-1}^{(2)})) \\
    &\vdots\\
      & F_{m,0} = \sigma\big(B_{m-1}  \cup F_{m-1,M} \cup \sigma(\tilde{\theta}^{(m-1)}, \tilde{w}^{(m-1)} )  \big), F_{m,1}=\sigma(F_{m,0}   \cup \sigma(x_{0}^{(m)})),\dots,\\ 
      &\quad F_{m,M}= \sigma(F_{m,M-1} \cup \sigma(x_{M-1}^{(m)})) .
\end{align*}
Moreover, we define $\EE_{t,m}$ as the conditional expectation with respect to the $\sigma$-field $F_{t,m}$.

\subsection*{Additional Notations}
Recall the one-step TDC update at $\theta_{t}^{(m)}$:
	\begin{align*}
		\ A_{t}^{(m)}\theta + b_{t}^{(m)}+ B_{t}^{(m)} w 
		 =& \  A_{t}^{(m)}\theta + b_{t}^{(m)}+ B_{t}^{(m)} z + B_{t}^{(m)}\left( -C^{-1} (b + A\theta) \right) \\
		=& \ \left( A_{t}^{(m)} - B_{t}^{(m)} C^{-1}A  \right) \theta + b_{t}^{(m)} - B_{t}^{(m)} C^{-1} b + B_{t}^{(m)} z.
	\end{align*}
	Define $\widehat{A}_{t}^{(m)}:=A_{t}^{(m)} - B_{t}^{(m)} C^{-1}A$ and $\hat{b}_{t}^{(m)}:=b_{t}^{(m)} - B_{t}^{(m)} C^{-1} b$. Then, we further define
	\begin{align*}
		G_{t}^{(m)}(\theta, z):= \widehat{A}_{t}^{(m)}\theta + \hat{b}_{t}^{(m)}+ B_{t}^{(m)} z.
	\end{align*}
	Moreover, we define
	\[\lambda_{\widehat{A}} := -\lambda_{\max} (\widehat{A} + \widehat{A}^\top) = -\lambda_{\max}(2 A^\top C^{-1} A) > 0 .\] 
Similarly, recall the one-step TDC update at $w_{t}^{(m)}$:
	\begin{align*}
		 A_{t}^{(m)}\theta + b_{t}^{(m)}+ C_{t}^{(m)} w 
		=& \ A_{t}^{(m)}\theta + b_{t}^{(m)}+ C_{t}^{(m)} z + C_{t}^{(m)}\left( -C^{-1} (b + A\theta) \right) \\
		= & \ \left( A_{t}^{(m)} - C_{t}^{(m)} C^{-1}A  \right) \theta + b_{t}^{(m)} - C_{t}^{(m)} C^{-1} b + C_{t}^{(m)} z.
	\end{align*}
	Define $\bar{A}_{t}^{(m)}:=A_{t}^{(m)} - C_{t}^{(m)} C^{-1}A$ and $\bar{b}_{t}^{(m)}:=b_{t}^{(m)} - C_{t}^{(m)} C^{-1} b$. Then, we further define
	\begin{align*}
	H_{t}^{(m)}(\theta, z):= \bar{A}_{t}^{(m)}\theta + \bar{b}_{t}^{(m)}+ C_{t}^{(m)} z.
	\end{align*} 
	Moreover, we define
	$$\lambda_{C} := -\lambda_{\max} ({C} + {C}^\top) = -\lambda_{\max}(2C) > 0 .$$
	
	\subsection*{List of Constants}
	We summerize all the constants that are used in the proof as follows. 
	
	\textbf{Constants for both i.i.d.\ and Markovian setting:}
	\begin{itemize}
		\item $G_{\text{VR}} := 3 \left[ (1+\gamma)R_\theta + r_{\max} \right]\rho_{\max}\left(1 + \frac{\gamma \rho_{\max}}{ \min |\lambda(C)|}\right).$
		
		\item $H_{\text{VR}} := 3 \left[(1+\gamma)R_\theta + r_{\max}\right]\rho_{\max}\left( 1 + \frac{1}{ \min |\lambda(C)|} \right).$
	\end{itemize}
	
	\textbf{Constants for i.i.d.\ setting:} 
	\begin{itemize} 
	\item $K_1 =\big[ (1+\gamma)R_\theta + r_{\max} \big]^2\rho^2_{\max}\big(1 + \frac{\gamma \rho_{\max}}{ \min |\lambda(C)|}\big)^2$,
	
	\item $K_2 = \big[  \big(  1 + \gamma\big)R_\theta  + r_{\max}\big]^2\big(  1 + \frac{1}{\min|\lambda_{C}|} \big)^2$,
	\item $C_1 = \frac{ 2 \rho^2_{\max} \gamma^2}{\lambda_{\widehat{A}}} \frac{3}{\lambda_{C}}  \cdot 10 (1+\gamma)^2 \rho^2 _{\max} \cdot \big(1 + \frac{\gamma \rho_{\max}}{ \min |\lambda(C)|}\big)^2 \big(1 + \frac{2 }{\lambda_C } \big)\cdot  \big(\rho_{\max} \frac{ 1+\gamma }{ \min |\lambda(C)|}\big)^2,$
	\item $ C_2 = \frac{ 2 \rho^2_{\max} \gamma^2}{\lambda_{\widehat{A}}} \frac{3}{\lambda_{C}}  \cdot 10 (1+\gamma)^2 \rho^2 _{\max} \cdot \big( 1 + \frac{1}{ \min |\lambda(C)|} \big)^2.$
	
    \item 	$C_3 =  10 (1+\gamma)^2 \rho^2 _{\max} \cdot  \big(1 + \frac{\gamma \rho_{\max}}{ \min |\lambda(C)|}\big)^2 \big(1 + \frac{2 }{\lambda_C } \big)\cdot  \big(\rho_{\max} \frac{ 1+\gamma }{ \min |\lambda(C)|}\big)^2 $,
	
	\item
	$C_4 =  10 (1+\gamma)^2 \rho^2 _{\max} \cdot\big( 1 + \frac{1}{ \min |\lambda(C)|} \big)^2$,
	
	\item $D= \frac{12}{\lambda_{\widehat{A}}}\Big\{  \frac{1}{\alpha M} +  \alpha  \cdot 5(1+\gamma)^2\rho^2_{\max}\big(1 + \frac{\gamma \rho_{\max}}{ \min |\lambda(C)|}\big)^2   +     \frac{\alpha^2}{\beta^2} \cdot C_1  + \beta \cdot C_2   \Big\},$

\item $E= \frac{1}{M\beta} \cdot \frac{2}{\lambda_C },$ 
\item $F= \frac{4}{\lambda_{C}}\Big[\frac{1}{\beta M}+  \beta \cdot 10 + \frac{\alpha^2}{\beta^2} \cdot 10\gamma^2\rho_{\max}^2 \big(1 + \frac{2 }{\lambda_C } \big)\cdot  \big(\rho_{\max} \frac{ 1+\gamma }{ \min |\lambda(C)|}\big)^2    + 
	\big(  \frac{\alpha^3}{\beta^2}\cdot C_3 + \alpha \beta \cdot C_4 \big) \frac{30}{\lambda_{\widehat{A}}} \gamma^2\rho_{\max}^2  
	\Big]$.
	\end{itemize}

\subsection*{Constants for Markovian setting:} 
\begin{itemize} 
		\item $K_1 := \left[ (1+\gamma)R_\theta + r_{\max} \right]^2\rho^2_{\max}\left(1 + \frac{\gamma \rho_{\max}}{ \min |\lambda(C)|}\right)^2  \cdot \left( 1+\kappa \frac{2 \rho}{1-\rho} \right)$,
		
		\item $K_2 :=  \frac{2}{\lambda_{\hat{A}}}\left[R_\theta^2 (1+\gamma)^2 + r^2_{\max}\right] \cdot 4 \rho_{\max}^2  \left(1 + \frac{\gamma \rho_{\max}}{ \min |\lambda(C)|}\right)^2 \left[1 + \kappa \frac{\rho}{1-\rho} \right]$,
		
		\item $K_3 := \left( \frac{32}{\lambda_C} \left[R_\theta^2(1+\gamma)^2+   r^2_{\max}\right]\cdot \rho_{\max}^2 +  \frac{16}{\lambda_C}  \frac{\rho_{\max}(1+\gamma) R_\theta + \rho_{\max} r_{\max}}{\min |\lambda(C)|}\right)  \left[ 1 + \kappa \frac{\rho}{1-\rho} \right]$,
		
		\item $K_4 := \frac{12}{\lambda_C} R_w^2 \left[ 1 + \kappa \frac{\rho}{1-\rho} \right]$,
		
		\item $K_5:=\left[(1+\gamma)R_\theta + r_{\max}\right]^2\rho_{\max}^2\left( 1 + \frac{1}{ \min |\lambda(C)|} \right)^2  \cdot \left( 1   +    \kappa \frac{2\rho}{1-\rho}   \right)$,
		
		\item $C_1 = \Big(1 + \frac{\gamma \rho_{\max}}{ \min |\lambda(C)|}\Big)^2 \Big(1 + \frac{2 }{\lambda_C }  \Big)\cdot  \Big(\rho_{\max} \frac{ 1+\gamma }{ \min |\lambda(C)|}\Big)^2 \cdot  \frac{96}{\lambda_{\widehat{A}} \lambda_{C}}\gamma^2 \rho_{\max}^2 \cdot 10 (1+\gamma)^2 \rho^2 _{\max} $,
		
		\item $C_2 = \Big( 1 + \frac{1}{ \min |\lambda(C)|} \Big)^2 \cdot  \frac{96}{\lambda_{\widehat{A}} \lambda_{C}}\gamma^2 \rho_{\max}^2 \cdot 10 (1+\gamma)^2 \rho^2 _{\max} $,
		
		\item $D= \frac{16}{\lambda_{\widehat{A}}}\Big[\frac{ 1}{\alpha M} +  \alpha  \cdot 5(1+\gamma)^2\rho^2_{\max}\Big(1 + \frac{\gamma \rho_{\max}}{ \min |\lambda(C)|}\Big)^2 + \frac{\alpha^2}{\beta^2}  \cdot C_1   +   \beta  \cdot C_2   \Big] $, 
		
		\item $E= \frac{1}{M\beta} \cdot \frac{12}{\lambda_C }$,
		
		\item $F= \frac{24}{\lambda_{C}} \cdot  \Big[  \frac{1}{\beta M}+ 10\beta  + 10\gamma^2\rho_{\max}^2 \big(1 + \frac{1}{\lambda_C }  \big)\cdot  \Big(\rho_{\max} \frac{ 1+\gamma }{ \min |\lambda(C)|}\Big)^2    \frac{\alpha^2}{\beta^2}    +  
\alpha \cdot 120 (1+\gamma)^2 \rho^2 _{\max} \frac{1}{\lambda_{\widehat{A}}} \cdot \Big[\Big(1 + \frac{\gamma \rho_{\max}}{ \min |\lambda(C)|}\Big)^2 \Big(1+ \frac{2  }{\lambda_C }   \Big)\cdot  \Big(\rho_{\max} \frac{ 1+\gamma }{ \min |\lambda(C)|}\Big)^2   \frac{\alpha^2}{\beta^2}     +
+  \Big( 1 + \frac{1}{ \min |\lambda(C)|} \Big)^2 \beta   \Big]  \cdot  5 \gamma^2\rho_{\max}^2 
\Big]$.
\end{itemize}
\section{Proof of Theorem \ref{thm: iid}}\label{appendix: iid proof}
Throughout the proof, we assume the learning rates $\alpha,\beta$ and the batch size $M$ satisfy the following conditions.
\begin{align}
    &\alpha \leq \min \Big\{  \frac{1}{5 \lambda_{\widehat{A}}} , \frac{\lambda_{\widehat{A}}}{60} / \Big[(1+\gamma)^2\rho^2_{\max}\big(1 + \frac{\gamma \rho_{\max}}{ \min |\lambda(C)|}\big)^2\Big]  \Big\}, \label{eq: lr iid 1}\\
& \frac{\alpha^2}{\beta^2}\cdot C_3 + \beta  \cdot C_4      \leq \min \Big\{  	\frac{1-D}{144}\frac{\lambda_{\widehat{A}}^2 \lambda_{C}}{   \rho^2_{\max} \gamma^2} ,5 (1-D), C_4\Big\}  , \label{eq: lr iid 2}\\
& M\beta > \frac{4}{\lambda_{C}},\label{eq: lr iid 3}\\
	&\frac{\lambda_{C}}{6}\beta - 10\beta^2 - 10\gamma^2\rho_{\max}^2 \big(\alpha^2 + \frac{2 \alpha^2}{\lambda_C } \frac{1}{\beta} \big)\cdot  \big(\rho_{\max} \frac{ 1+\gamma }{ \min |\lambda(C)|}\big)^2 \geq 0, \label{eq: lr iid 4}\\
	&\frac{1}{\alpha M}   +  \alpha  \cdot 5(1+\gamma)^2\rho^2_{\max}\big(1 + \frac{\gamma \rho_{\max}}{ \min |\lambda(C)|}\big)^2    \leq  \frac{\lambda_{\widehat{A}}}{6}, \label{eq: lr iid 5}\\
	&\frac{\alpha}{\beta^2 M} \cdot \frac{ 72 \rho^2_{\max} \gamma^2}{\lambda_{\widehat{A}}^2} \frac{1}{\lambda_{C}} +  \beta \cdot \frac{ 720 \rho^2_{\max} \gamma^2}{\lambda_{\widehat{A}}^2} \frac{1}{\lambda_{C}} \nonumber \\ 
	&\quad + \frac{\alpha^2}{\beta^2} \cdot \frac{ 720 \rho^2_{\max} \gamma^4}{\lambda_{\widehat{A}}^2} \frac{1}{\lambda_{C}}  \rho_{\max}^2 \big(1 + \frac{2 }{\lambda_C } \big)\cdot  \big(\rho_{\max} \frac{ 1+\gamma }{ \min |\lambda(C)|}\big)^2   +  \alpha\cdot\frac{60}{\lambda_{\widehat{A}}} \gamma^2\rho_{\max}^2 \leq 1, \label{eq: lr iid 6}\\
	&\max\{D,E,F\} < 1, \label{eq: lr iid 7}
\end{align} 
where $C_3$ and $C_4$ are specified in eq.(\ref{eq: iid def-C3}) and eq.(\ref{eq: iid def-C4}), respectively, and $D,E,F$ are specified in eq.(\ref{eq: iid def-D}), eq.(\ref{eq: iid def-E}), and eq.(\ref{eq: iid def-F}), respectively. We note that under the above conditions, all the supporting lemmas for proving the theorem are satisfied. Also, we note that for a sufficiently small target accuracy $\epsilon$, our choices of learning rates and batch size $\alpha = \calO(\epsilon^{\frac{3}{5}}), \beta = \calO({\epsilon^{\frac{2}{5}}})$, $M = \calO(\epsilon^{-\frac{3}{5}})$ that are stated in the theorem satisfy the above conditions \cref{eq: lr iid 1,eq: lr iid 2,eq: lr iid 3,eq: lr iid 4,eq: lr iid 5,eq: lr iid 6,eq: lr iid 7}. 

\paragraph{Proof Sketch.} The proof consists of the following key steps. 
\begin{enumerate}
	\item Develop \textit{preliminary bound for  $\sum_{t=0}^{M-1} \| \theta_{t}^{(m)} - \theta^\ast\|^2$} (Lemma \ref{lemma: iid pre-bound theta}).
	
	We bound $\sum_{t=0}^{M-1}\| \theta_{t}^{(m)}- \theta^\ast\|^2$ in terms of $\sum_{t=0}^{M-1} \|z_{t}^{(m)}\|^2$, $\|\tilde{z}^{(m-1)}\|^2$, and $\|\tilde{\theta}^{(m-1)}- \theta^\ast\|^2$. 
	
	\item Develop \textit{preliminary bound for  $\sum_{t=0}^{M-1} \| z_{t}^{(m)}\|^2$} (Lemma \ref{lemma: iid pre-bound z}).
	
	We bound $\sum_{t=0}^{M-1} \| z_{t}^{(m)}\|^2$ in terms of  $\sum_{t=0}^{M-1} \|\theta_{t}^{(m)}- \theta^\ast\|^2$, $\|\tilde{z}^{(m-1)}\|^2$, and $\|\tilde{\theta}^{(m-1)}- \theta^\ast\|^2$, and plug it into the preliminary bound of  $\sum_{t=0}^{M-1} \| \theta_{t}^{(m)}- \theta^\ast\|^2$. Then, we obtain an upper bound of $\sum_{t=0}^{M-1} \| \theta_{t}^{(m)}- \theta^\ast\|^2$ in terms of $\|\tilde{z}^{(m-1)}\|^2$, and $\|\tilde{\theta}^{(m-1)}- \theta^\ast\|^2$.  
	
	\item Develop \textit{preliminary non-asymptotic bound for $\|\tilde{z}^{(m)}\|^2$} (Lemma \ref{lemma: iid conv-z}).
	
	We develop a non-asymptotic bound for $\|\tilde{z}^{(m)}\|^2$. 
	
	\item Develop \textit{preliminary non-asymptotic bound for $\|\tilde{\theta}^{m} - \theta^\ast\|^2$} (Lemma \ref{lemma: iid theta}).
	
	We plug the bound in Lemma \ref{lemma: iid conv-z} into the previous upper bounds. Then, we obtain an inequality between $\EE \| \tilde{\theta}^{(m )}- \theta^\ast\|$ and $\EE \| \tilde{\theta}^{(m-1 )}- \theta^\ast\|$. Telescoping this inequality leads to the final result.

    \item Develop \textit{refined bound for $\|\tilde{z}^{(m)}\|^2$} (Lemma \ref{lemma: iid refined z}).
    
    We bound $\sum_{t=0}^{M-1} \| z_{t}^{(m)}\|^2$ in terms of $\sum_{t=0}^{M-1} \|\theta_{t}^{(m)}- \theta^\ast\|^2$, $\|\tilde{z}^{(m-1)}\|^2$, and $\|\tilde{\theta}^{(m-1)}- \theta^\ast\|^2$. Then, we apply Lemma \ref{lemma: iid pre-bound theta} and obtain an upper bound of $\sum_{t=0}^{M-1} \| z_{t}^{(m)}\|^2$ in terms of  $\|\tilde{z}^{(m-1)}\|^2$ and $\|\tilde{\theta}^{(m-1)}- \theta^\ast\|^2$. Moreover, we apply Lemma \ref{lemma: iid conv-z} and the preliminary non-asymptotic bound of $\|\tilde{\theta}^{m} - \theta^\ast\|^2$ to obtain an upper bound of $\sum_{t=0}^{M-1} \| z_{t}^{(m)}\|^2$ in terms of $\|\tilde{z}^{(m-1)}\|^2$. This gives the desired refined bound of $\|\tilde{z}^{m}\|^2$.
    
    \item Develop \textit{refined bound for $\|\tilde{\theta}^{m} - \theta^\ast\|^2$} (Theorem \ref{thm: iid}).
    
    We use the refined bound of $\|\tilde{z}^{m}\|^2$ instead of the preliminary bound obtained in the step 4.
\end{enumerate} 

First, based on Lemma \ref{lemma: iid pre-bound theta}, we have the following result
\begin{align*}
&  \frac{\lambda_{\widehat{A}}}{6}\alpha   \sum_{t=0}^{M-1}\EE_{m,0} \| \theta_{t}^{(m)} - \theta^\ast\|^2 \\
\leq & \Big[ 1 +  \alpha^2M \cdot 5(1+\gamma)^2\rho^2_{\max}\big(1 + \frac{\gamma \rho_{\max}}{ \min |\lambda(C)|}\big)^2 \Big] \EE_{m,0}\| \tilde{\theta}^{(m-1)} - \theta^\ast\|^2    + \alpha^2 \cdot5K_1 \\
&+ \alpha \cdot  \frac{ 2 \rho^2_{\max} \gamma^2}{\lambda_{\widehat{A}}}  \sum_{t=0}^{M-1}\EE_{m,0}\| z_{t}^{(m)}\|^2+ \alpha^2 M \cdot 5 \gamma^2\rho_{\max}^2 \EE_{m,0}\|\tilde{z}^{(m-1)}\|^2.
\end{align*}
Apply Lemma  \ref{lemma: iid pre-bound z} to  bound the term $\sum_{t=0}^{M-1} \EE_{m,0}\| z_{t}^{(m)}\|^2$ in the above inequality, re-arrange the obtained result and note that $\frac{\lambda_{C}}{6}\beta - 10\beta^2 - 10\gamma^2\rho_{\max}^2 \big(\alpha^2 + \frac{2 \alpha^2}{\lambda_C } \frac{1}{\beta} \big)\cdot  \big(\rho_{\max} \frac{ 1+\gamma }{ \min |\lambda(C)|}\big)^2 \geq 0$. Then, we obtain the following inequality,
\begin{align*}
&  \frac{\lambda_{\widehat{A}}}{12}\alpha  \sum_{t=0}^{M-1} \EE_{m,0} \| \theta_{t}^{(m)} - \theta^\ast \|^2\\
\leq &  \Big\{  1 +  \alpha^2M \cdot 5(1+\gamma)^2\rho^2_{\max}\big(1 + \frac{\gamma \rho_{\max}}{ \min |\lambda(C)|}\big)^2   +  \alpha M \cdot  \big(  \frac{\alpha^2}{\beta^2} \cdot C_1  + \beta \cdot C_2 \big)    \Big\} \EE_{m,0}\| \tilde{\theta}^{(m-1)} - \theta^\ast\|^2     \\
&+ \Big\{ \frac{\alpha}{\beta} \cdot  \frac{ 2 \rho^2_{\max} \gamma^2}{\lambda_{\widehat{A}}} \frac{3}{\lambda_{C}}  \Big[1+ \Big[10\beta^2 + 10\gamma^2\rho_{\max}^2 \big(1 + \frac{2 }{\lambda_C } \big)\cdot  \big(\rho_{\max} \frac{ 1+\gamma }{ \min |\lambda(C)|}\big)^2 \frac{\alpha^2}{\beta} \Big]M \Big] + \alpha^2 M \cdot 5 \gamma^2\rho_{\max}^2   \Big\}\\
&\quad \cdot\EE_{m,0}\|  \tilde{z}^{(m-1)}\|^2\\
&+   \alpha\beta  \cdot  \frac{ 60 \rho^2_{\max} \gamma^2}{\lambda_{\widehat{A}}\lambda_{C} }   \cdot K_2   + \frac{\alpha^3}{\beta^2} \cdot  \frac{ 60 \rho^2_{\max} \gamma^2}{\lambda_{\widehat{A}}\lambda_{C}}   \cdot \big(\rho_{\max} \frac{ 1+\gamma }{ \min |\lambda(C)|}\big)^2  K_1 \big(1+ \frac{2  }{\lambda_C }  \big)   + \alpha^2 \cdot5K_1.
\end{align*}
where $C_1$ and $C_2$ {are specified in eq.(\ref{eq: iid def-C1}) and eq.(\ref{eq: iid def-C2}), respectively}. 
Dividing $\frac{\lambda_{\widehat{A}}}{12}\alpha M$ and taking total expectation on both sides of the above inequality, and applying Lemma  \ref{lemma: iid refined z} to  bound the term $\EE \|  \tilde{z}^{(m-1)}\|^2$, we obtain that
 \begin{align*}
 &  \EE  \| \tilde{ \theta}^{(m)} - \theta^\ast \|^2\\
 \leq & D \cdot \EE_{m,0}\| \tilde{\theta}^{(m-1)} - \theta^\ast\|^2     \\
 &+ \frac{12}{\lambda_{\widehat{A}}} \frac{1}{\alpha M}\Big\{ \frac{\alpha}{\beta} \cdot  \frac{ 2 \rho^2_{\max} \gamma^2}{\lambda_{\widehat{A}}} \frac{3}{\lambda_{C}}  \Big[1+ \Big[10\beta^2 + 10\gamma^2\rho_{\max}^2 \big(1 + \frac{2 }{\lambda_C } \big)\cdot  \big(\rho_{\max} \frac{ 1+\gamma }{ \min |\lambda(C)|}\big)^2 \frac{\alpha^2}{\beta} \Big]M \Big] + \alpha^2 M \cdot 5 \gamma^2\rho_{\max}^2   \Big\}   \\
 & \times \Big\{ F^{m-1} \cdot \EE \|  \tilde{z}^{(0)}\|^2 +  \frac{D^{m-1} - F^{m-1}}{D - F} \cdot \EE\| \tilde{\theta}^{(0)} - \theta^\ast\|^2    + \frac{\frac{D^{m-1} - F^{m-1}}{D - F} - \frac{E^{m-1} - F^{m-1}}{E - F}}{D - E}\EE \|  \tilde{z}^{(0)}\|^2   \\
 &+   \frac{1}{1 - F}\frac{1}{1-D} \frac{1152}{\lambda_{\widehat{A}}^2}  \frac{   \rho^2_{\max} \gamma^2}{\lambda_C^2}  \big(  \frac{\alpha^2}{\beta^2}\cdot C_3 + \beta  \cdot C_4 \big)    \Big[  \frac{1}{\beta M}  + \beta \cdot 20 + \frac{\alpha^2}{\beta^2} \cdot 20\gamma^2\rho_{\max}^2 \big(1 + \frac{2 }{\lambda_C } \big)\cdot  \big(\rho_{\max} \frac{ 1+\gamma }{ \min |\lambda(C)|}\big)^2   \Big]   \\
 &\times  \big[ \beta \cdot \frac{24}{\lambda_C }   H_{\text{VR}}^2 + \frac{\alpha^2}{\beta^2} \cdot \big(1 + \frac{2 }{\lambda_C }  \big) \cdot \frac{2}{\lambda_C }\big(\rho_{\max} \frac{ 1+\gamma }{ \min |\lambda(C)|}\big)^2 G_{\text{VR}}^2 \big] \\  
 &+   \frac{1}{1 - F} \frac{1}{1-D}\Big[ \frac{\beta}{M} \cdot\big(  \frac{80}{\lambda_{C}} K_2 + C_4  \frac{600}{\lambda_{C}} \frac{K_1}{\lambda_{\widehat{A}}}  \big)  +  \frac{\alpha^2}{\beta^2}\frac{1}{M} \cdot \big(\frac{80}{\lambda_{C}}\big(\rho_{\max} \frac{ 1+\gamma }{ \min |\lambda(C)|}\big)^2  K_1 \big(1+ \frac{2  }{\lambda_C }  \big) + C_3  \frac{600}{\lambda_{C}} \frac{K_1}{\lambda_{\widehat{A}}}\big)      \Big] \Big\}  \\
 &+ \frac{12}{\lambda_{\widehat{A}}} \frac{1}{\alpha M} \Big\{  \alpha\beta  \cdot  \frac{ 60 \rho^2_{\max} \gamma^2}{\lambda_{\widehat{A}}\lambda_{C} }   \cdot K_2   + \frac{\alpha^3}{\beta^2} \cdot  \frac{ 60 \rho^2_{\max} \gamma^2}{\lambda_{\widehat{A}}\lambda_{C}}   \cdot \big(\rho_{\max} \frac{ 1+\gamma }{ \min |\lambda(C)|}\big)^2  K_1 \big(1+ \frac{2  }{\lambda_C }  \big)   + \alpha^2 \cdot5K_1 \Big\}.
 \end{align*} 
 Note that the second coefficient in the above inequality can be simplified as 
 \begin{align*}
 	&\frac{12}{\lambda_{\widehat{A}}} \frac{1}{\alpha M}\Big\{ \frac{\alpha}{\beta} \cdot  \frac{ 2 \rho^2_{\max} \gamma^2}{\lambda_{\widehat{A}}} \frac{3}{\lambda_{C}}  \Big[1+ \Big[10\beta^2 + 10\gamma^2\rho_{\max}^2 \big(1 + \frac{2 }{\lambda_C } \big)\cdot  \big(\rho_{\max} \frac{ 1+\gamma }{ \min |\lambda(C)|}\big)^2 \frac{\alpha^2}{\beta} \Big]M \Big] + \alpha^2 M \cdot 5 \gamma^2\rho_{\max}^2   \Big\}\\
 	=& \frac{12}{\lambda_{\widehat{A}}}  \Big\{ \frac{\alpha}{\beta} \cdot  \frac{ 2 \rho^2_{\max} \gamma^2}{\lambda_{\widehat{A}}} \frac{3}{\lambda_{C}}  \Big[\frac{1}{\beta M}+  10 \beta + 10\gamma^2\rho_{\max}^2 \big(1 + \frac{2 }{\lambda_C } \big)\cdot  \big(\rho_{\max} \frac{ 1+\gamma }{ \min |\lambda(C)|}\big)^2 \frac{\alpha^2}{\beta^2}  \Big] + \alpha \cdot 5 \gamma^2\rho_{\max}^2   \Big\} \\ 
 	=& \frac{\alpha}{\beta^2 M} \cdot \frac{ 72 \rho^2_{\max} \gamma^2}{\lambda_{\widehat{A}}^2} \frac{1}{\lambda_{C}} +  \beta \cdot \frac{ 720 \rho^2_{\max} \gamma^2}{\lambda_{\widehat{A}}^2} \frac{1}{\lambda_{C}} + \frac{\alpha^2}{\beta^2} \cdot \frac{ 720 \rho^2_{\max} \gamma^4}{\lambda_{\widehat{A}}^2} \frac{1}{\lambda_{C}}  \rho_{\max}^2 \big(1 + \frac{2 }{\lambda_C } \big)\cdot  \big(\rho_{\max} \frac{ 1+\gamma }{ \min |\lambda(C)|}\big)^2 \\
 	&\quad+  \alpha\cdot\frac{60}{\lambda_{\widehat{A}}} \gamma^2\rho_{\max}^2,
 \end{align*} 
and note that we have assumed that
 \begin{align}\label{eq: tmp}
 	&\frac{\alpha}{\beta^2 M} \cdot \frac{ 72 \rho^2_{\max} \gamma^2}{\lambda_{\widehat{A}}^2} \frac{1}{\lambda_{C}} +  \beta \cdot \frac{ 720 \rho^2_{\max} \gamma^2}{\lambda_{\widehat{A}}^2} \frac{1}{\lambda_{C}} \nonumber\\
 	&\quad+ \frac{\alpha^2}{\beta^2} \cdot \frac{ 720 \rho^2_{\max} \gamma^4}{\lambda_{\widehat{A}}^2} \frac{1}{\lambda_{C}}  \rho_{\max}^2 \big(1 + \frac{2 }{\lambda_C } \big)\cdot  \big(\rho_{\max} \frac{ 1+\gamma }{ \min |\lambda(C)|}\big)^2 +  \alpha\cdot\frac{60}{\lambda_{\widehat{A}}} \gamma^2\rho_{\max}^2 \leq 1.
 \end{align}
 Then, we obtain that 
  \begin{align*}
 &  \EE  \| \tilde{ \theta}^{(m)} - \theta^\ast \|^2\\
 \leq & D \cdot \EE_{m,0}\| \tilde{\theta}^{(m-1)} - \theta^\ast\|^2 +  F^{m-1} \cdot \EE \|  \tilde{z}^{(0)}\|^2 +  \frac{D^{m-1} - F^{m-1}}{D - F} \cdot \EE\| \tilde{\theta}^{(0)} - \theta^\ast\|^2  \\
 &\quad+ \frac{\frac{D^{m-1} - F^{m-1}}{D - F} - \frac{E^{m-1} - F^{m-1}}{E - F}}{D - E}\EE \|  \tilde{z}^{(0)}\|^2     \\
 &+ \Big\{ \frac{\alpha}{\beta^2 M} \cdot \frac{ 72 \rho^2_{\max} \gamma^2}{\lambda_{\widehat{A}}^2} \frac{1}{\lambda_{C}} +  \beta \cdot \frac{ 720 \rho^2_{\max} \gamma^2}{\lambda_{\widehat{A}}^2} \frac{1}{\lambda_{C}} \\
 &\quad+ \frac{\alpha^2}{\beta^2} \cdot \frac{ 720 \rho^2_{\max} \gamma^4}{\lambda_{\widehat{A}}^2} \frac{1}{\lambda_{C}}  \rho_{\max}^2 \big(1 + \frac{2 }{\lambda_C } \big)\cdot  \big(\rho_{\max} \frac{ 1+\gamma }{ \min |\lambda(C)|}\big)^2 +  \alpha\cdot\frac{60}{\lambda_{\widehat{A}}} \gamma^2\rho_{\max}^2  \Big\}   \\
 & \times \Big\{  \frac{1}{1 - F}\frac{1}{1-D} \frac{1152}{\lambda_{\widehat{A}}^2}  \frac{   \rho^2_{\max} \gamma^2}{\lambda_C^2}  \big(  \frac{\alpha^2}{\beta^2}\cdot C_3 + \beta  \cdot C_4 \big) \Big[  \frac{1}{\beta M}  + \beta \cdot 20 + \frac{\alpha^2}{\beta^2} \cdot 20\gamma^2\rho_{\max}^2 \big(1 + \frac{2 }{\lambda_C } \big)\cdot  \big(\rho_{\max} \frac{ 1+\gamma }{ \min |\lambda(C)|}\big)^2   \Big]  \\
 &\quad\cdot\big[ \beta \cdot \frac{24}{\lambda_C }   H_{\text{VR}}^2 + \frac{\alpha^2}{\beta^2} \cdot \big(1 + \frac{2 }{\lambda_C }  \big) \cdot \frac{2}{\lambda_C }\big(\rho_{\max} \frac{ 1+\gamma }{ \min |\lambda(C)|}\big)^2 G_{\text{VR}}^2 \big] \\  
 &\quad+   \frac{1}{1 - F} \frac{1}{1-D}\Big[ \frac{\beta}{M} \cdot\big(  \frac{80}{\lambda_{C}} K_2 + C_4  \frac{600}{\lambda_{C}} \frac{K_1}{\lambda_{\widehat{A}}}  \big)  +  \frac{\alpha^2}{\beta^2}\frac{1}{M} \cdot \big(\frac{80}{\lambda_{C}}\big(\rho_{\max} \frac{ 1+\gamma }{ \min |\lambda(C)|}\big)^2  K_1 \big(1+ \frac{2  }{\lambda_C }  \big) + C_3  \frac{600}{\lambda_{C}} \frac{K_1}{\lambda_{\widehat{A}}}\big)      \Big] \Big\}  \\
 &\quad+ \frac{12}{\lambda_{\widehat{A}}} \frac{1}{\alpha M} \Big\{  \alpha\beta  \cdot  \frac{ 60 \rho^2_{\max} \gamma^2}{\lambda_{\widehat{A}}\lambda_{C} }   \cdot K_2   + \frac{\alpha^3}{\beta^2} \cdot  \frac{ 60 \rho^2_{\max} \gamma^2}{\lambda_{\widehat{A}}\lambda_{C}}   \cdot \big(\rho_{\max} \frac{ 1+\gamma }{ \min |\lambda(C)|}\big)^2  K_1 \big(1+ \frac{2  }{\lambda_C }  \big)   + \alpha^2 \cdot5K_1 \Big\}\\
 \leq & D \cdot \EE_{m,0}\| \tilde{\theta}^{(m-1)} - \theta^\ast\|^2 +  F^{m-1} \cdot \EE \|  \tilde{z}^{(0)}\|^2 +  \frac{D^{m-1} - F^{m-1}}{D - F} \cdot \EE\| \tilde{\theta}^{(0)} - \theta^\ast\|^2  \\
 &\quad+ \frac{\frac{D^{m-1} - F^{m-1}}{D - F} - \frac{E^{m-1} - F^{m-1}}{E - F}}{D - E}\EE \|  \tilde{z}^{(0)}\|^2     \\
 &+ \Big\{ \frac{\alpha}{\beta^2 M} \cdot \frac{ 72 \rho^2_{\max} \gamma^2}{\lambda_{\widehat{A}}^2} \frac{1}{\lambda_{C}} +  \beta \cdot \frac{ 720 \rho^2_{\max} \gamma^2}{\lambda_{\widehat{A}}^2} \frac{1}{\lambda_{C}} \\
 &\quad+ \frac{\alpha^2}{\beta^2} \cdot \frac{ 720 \rho^2_{\max} \gamma^4}{\lambda_{\widehat{A}}^2} \frac{1}{\lambda_{C}}  \rho_{\max}^2 \big(1 + \frac{2 }{\lambda_C } \big)\cdot  \big(\rho_{\max} \frac{ 1+\gamma }{ \min |\lambda(C)|}\big)^2 +  \alpha\cdot\frac{60}{\lambda_{\widehat{A}}} \gamma^2\rho_{\max}^2  \Big\}   \\
 & \quad\cdot \Big\{  \frac{1}{1 - F}\frac{1}{1-D} \frac{1152}{\lambda_{\widehat{A}}^2}  \frac{   \rho^2_{\max} \gamma^2}{\lambda_C^2}  \big(  \frac{\alpha^2}{\beta^2}\cdot C_3 + \beta  \cdot C_4 \big)    \Big[  \frac{1}{\beta M}  + \beta \cdot 20 + \frac{\alpha^2}{\beta^2} \cdot 20\gamma^2\rho_{\max}^2 \big(1 + \frac{2 }{\lambda_C } \big)\cdot  \big(\rho_{\max} \frac{ 1+\gamma }{ \min |\lambda(C)|}\big)^2   \Big]    \\
 &\quad\cdot \big[ \beta \cdot \frac{24}{\lambda_C }   H_{\text{VR}}^2 + \frac{\alpha^2}{\beta^2} \cdot \big(1 + \frac{2 }{\lambda_C }  \big) \cdot \frac{2}{\lambda_C }\big(\rho_{\max} \frac{ 1+\gamma }{ \min |\lambda(C)|}\big)^2 G_{\text{VR}}^2 \big] \Big\}\\  
 &\quad+   \frac{1}{1 - F} \frac{1}{1-D}\Big[ \frac{\beta}{M} \cdot\big(  \frac{80}{\lambda_{C}} K_2 + C_4  \frac{600}{\lambda_{C}} \frac{K_1}{\lambda_{\widehat{A}}}  \big)  +  \frac{\alpha^2}{\beta^2}\frac{1}{M} \cdot \big(\frac{80}{\lambda_{C}}\big(\rho_{\max} \frac{ 1+\gamma }{ \min |\lambda(C)|}\big)^2  K_1 \big(1+ \frac{2  }{\lambda_C }  \big) + C_3  \frac{600}{\lambda_{C}} \frac{K_1}{\lambda_{\widehat{A}}}\big)      \Big]   \\
 &\quad+   \frac{1}{1 - F} \frac{1}{1-D} \Big\{  \frac{\beta}{  M}  \cdot  \frac{ 720 \rho^2_{\max} \gamma^2}{\lambda_{\widehat{A}}^2\lambda_{C} }   \cdot K_2   + \frac{\alpha^2}{\beta^2} \frac{1}{M}\cdot  \frac{ 720 \rho^2_{\max} \gamma^2}{\lambda_{\widehat{A}}^2\lambda_{C}}   \cdot \big(\rho_{\max} \frac{ 1+\gamma }{ \min |\lambda(C)|}\big)^2  K_1 \big(1+ \frac{2  }{\lambda_C }  \big)   + \frac{\alpha }{M} \cdot \frac{60}{\lambda_{\widehat{A}}} K_1 \Big\}\\
= & D \cdot \EE_{m,0}\| \tilde{\theta}^{(m-1)} - \theta^\ast\|^2 +  F^{m-1} \cdot \EE \|  \tilde{z}^{(0)}\|^2 +  \frac{D^{m-1} - F^{m-1}}{D - F} \cdot \EE\| \tilde{\theta}^{(0)} - \theta^\ast\|^2  \\
&\quad+ \frac{\frac{D^{m-1} - F^{m-1}}{D - F} - \frac{E^{m-1} - F^{m-1}}{E - F}}{D - E}\EE \|  \tilde{z}^{(0)}\|^2     \\
&+ \Big\{ \frac{\alpha}{\beta^2 M} \cdot \frac{ 72 \rho^2_{\max} \gamma^2}{\lambda_{\widehat{A}}^2} \frac{1}{\lambda_{C}} +  \beta \cdot \frac{ 720 \rho^2_{\max} \gamma^2}{\lambda_{\widehat{A}}^2} \frac{1}{\lambda_{C}} \\
&\quad+ \frac{\alpha^2}{\beta^2} \cdot \frac{ 720 \rho^2_{\max} \gamma^4}{\lambda_{\widehat{A}}^2} \frac{1}{\lambda_{C}}  \rho_{\max}^2 \big(1 + \frac{2 }{\lambda_C } \big)\cdot  \big(\rho_{\max} \frac{ 1+\gamma }{ \min |\lambda(C)|}\big)^2 +  \alpha\cdot\frac{60}{\lambda_{\widehat{A}}} \gamma^2\rho_{\max}^2  \Big\}   \\
& \quad\cdot \Big\{  \frac{1}{1 - F}\frac{1}{1-D} \frac{1152}{\lambda_{\widehat{A}}^2}  \frac{   \rho^2_{\max} \gamma^2}{\lambda_C^2}  \big(  \frac{\alpha^2}{\beta^2}\cdot C_3 + \beta  \cdot C_4 \big)    \Big[  \frac{1}{\beta M}  + \beta \cdot 20 + \frac{\alpha^2}{\beta^2} \cdot 20\gamma^2\rho_{\max}^2 \big(1 + \frac{2 }{\lambda_C } \big)\cdot  \big(\rho_{\max} \frac{ 1+\gamma }{ \min |\lambda(C)|}\big)^2   \Big]    \\
&\quad\cdot \big[ \beta \cdot \frac{24}{\lambda_C }   H_{\text{VR}}^2 + \frac{\alpha^2}{\beta^2} \cdot \big(1 + \frac{2 }{\lambda_C }  \big) \cdot \frac{2}{\lambda_C }\big(\rho_{\max} \frac{ 1+\gamma }{ \min |\lambda(C)|}\big)^2 G_{\text{VR}}^2 \big] \Big\}\\   
&\quad+   \frac{1}{1 - F} \frac{1}{1-D} \Big\{  \frac{\beta}{  M}  \cdot \Big[   \frac{ 720 \rho^2_{\max} \gamma^2}{\lambda_{\widehat{A}}^2\lambda_{C} }     K_2 + \big(  \frac{80}{\lambda_{C}} K_2 + C_4  \frac{600}{\lambda_{C}} \frac{K_1}{\lambda_{\widehat{A}}}  \big) \Big] \\
&\quad+ \frac{\alpha^2}{\beta^2} \frac{1}{M}\cdot  \Big[ \frac{ 720 \rho^2_{\max} \gamma^2}{\lambda_{\widehat{A}}^2\lambda_{C}}   \cdot \big(\rho_{\max} \frac{ 1+\gamma }{ \min |\lambda(C)|}\big)^2  K_1 \big(1+ \frac{2  }{\lambda_C }  \big) +  \big(\frac{80}{\lambda_{C}}\big(\rho_{\max} \frac{ 1+\gamma }{ \min |\lambda(C)|}\big)^2  K_1 \big(1+ \frac{2  }{\lambda_C }  \big) + C_3  \frac{600}{\lambda_{C}} \frac{K_1}{\lambda_{\widehat{A}}}\big)  \Big] \\
&\quad+ \frac{\alpha }{M} \cdot \frac{60}{\lambda_{\widehat{A}}} K_1 \Big\},
 \end{align*} 
 where we use eq.(\ref{eq: tmp}) in the first step, use $1< \frac{1}{1-D} \frac{1}{1-F}$ in the second step, and rearrange the terms in the last step. Next, we telescope the above inequality over $m$. To further simplify the result, we choose the optimal relation between $\alpha$ and $\beta$, i.e., $\beta = \calO(\alpha^{2/3})$. Then, for sufficiently small $\alpha$ and $\beta$, we have $D> E$ and $D > F$, and we obtain that
 \begin{align*}
 &  \EE  \| \tilde{ \theta}^{(m)} - \theta^\ast \|^2\\
 \leq & D^m \cdot  \| \tilde{\theta}^{(0)} - \theta^\ast\|^2 +  D^{m-1} \cdot \EE \|  \tilde{z}^{(0)}\|^2 +  \frac{ m D^{m-1} }{D - F} \cdot \EE\| \tilde{\theta}^{(0)} - \theta^\ast\|^2    + \frac{ m D^{m-1}  }{(D-E)(D-F)}\EE \|  \tilde{z}^{(0)}\|^2     \\
 &+ \frac{1}{1-D}\Big\{ \frac{\alpha}{\beta^2 M} \cdot \frac{ 72 \rho^2_{\max} \gamma^2}{\lambda_{\widehat{A}}^2} \frac{1}{\lambda_{C}} \\
 &\quad+  \beta \cdot \frac{ 720 \rho^2_{\max} \gamma^2}{\lambda_{\widehat{A}}^2} \frac{1}{\lambda_{C}} + \frac{\alpha^2}{\beta^2} \cdot \frac{ 720 \rho^2_{\max} \gamma^4}{\lambda_{\widehat{A}}^2} \frac{1}{\lambda_{C}}  \rho_{\max}^2 \big(1 + \frac{2 }{\lambda_C } \big)\cdot  \big(\rho_{\max} \frac{ 1+\gamma }{ \min |\lambda(C)|}\big)^2 +  \alpha\cdot\frac{60}{\lambda_{\widehat{A}}} \gamma^2\rho_{\max}^2  \Big\}   \\
 & \times \Big\{  \frac{1}{1 - F}\frac{1}{1-D} \frac{1152}{\lambda_{\widehat{A}}^2}  \frac{   \rho^2_{\max} \gamma^2}{\lambda_C^2}  \big(  \frac{\alpha^2}{\beta^2}\cdot C_3 + \beta  \cdot C_4 \big)    \Big[  \frac{1}{\beta M}  + \beta \cdot 20 + \frac{\alpha^2}{\beta^2} \cdot 20\gamma^2\rho_{\max}^2 \big(1 + \frac{2 }{\lambda_C } \big)\cdot  \big(\rho_{\max} \frac{ 1+\gamma }{ \min |\lambda(C)|}\big)^2   \Big]    \\
 &\times \big[ \beta \cdot \frac{24}{\lambda_C }   H_{\text{VR}}^2 + \frac{\alpha^2}{\beta^2} \cdot \big(1 + \frac{2 }{\lambda_C }  \big) \cdot \frac{2}{\lambda_C }\big(\rho_{\max} \frac{ 1+\gamma }{ \min |\lambda(C)|}\big)^2 G_{\text{VR}}^2 \big] \Big\}\\   
 &\quad+   \frac{1}{1 - F} \frac{1}{(1-D)^2} \Big\{  \frac{\beta}{  M}  \cdot \Big[   \frac{ 720 \rho^2_{\max} \gamma^2}{\lambda_{\widehat{A}}^2\lambda_{C} }     K_2 + \big(  \frac{80}{\lambda_{C}} K_2 + C_4  \frac{600}{\lambda_{C}} \frac{K_1}{\lambda_{\widehat{A}}}  \big) \Big] \\
 &\quad+ \frac{\alpha^2}{\beta^2} \frac{1}{M}\cdot  \Big[ \frac{ 720 \rho^2_{\max} \gamma^2}{\lambda_{\widehat{A}}^2\lambda_{C}}   \cdot \big(\rho_{\max} \frac{ 1+\gamma }{ \min |\lambda(C)|}\big)^2  K_1 \big(1+ \frac{2  }{\lambda_C }  \big) +  \big(\frac{80}{\lambda_{C}}\big(\rho_{\max} \frac{ 1+\gamma }{ \min |\lambda(C)|}\big)^2  K_1 \big(1+ \frac{2  }{\lambda_C }  \big) \\
 &\quad+ C_3  \frac{600}{\lambda_{C}} \frac{K_1}{\lambda_{\widehat{A}}}\big)  \Big]   + \frac{\alpha }{M} \cdot \frac{60}{\lambda_{\widehat{A}}} K_1 \Big\}.
 \end{align*}
 The first four terms in the right hand side of the above inequality are dominated by the order $\calO(mD^m)$.
 Also, under the choices of learning rates, the { fifth} term is in the order of $\calO(\beta^4)$ and the last term (in the last three lines) is in the order of $\calO(\frac{\beta}{M})$. {To elaborate this, we note that the { fifth} term is a product of two curly brackets: the first one is in the order of $\calO(\frac{\alpha}{\beta^2} \frac{1}{M}) + \calO(\beta)$, the second one is in the order of $\calO(\beta)\times \big(  \calO(\frac{1}{\beta M}) + \calO(\beta) \big) \times \calO(\beta)= \calO(\frac{\beta}{M}) + \calO(\beta^3)$. So, their product is in the order of $\big(\calO(\frac{\alpha}{\beta^2} \frac{1}{M}) + \calO(\beta)\big)\times \big(\calO(\frac{\beta}{M}) + \calO(\beta^3)\big) = \calO(\frac{\alpha}{\beta}\frac{1}{M^2}) + \calO(\frac{\beta^2}{M})+\calO(\frac{\alpha \beta}{M})+\calO(\beta^4) = \calO(\frac{\beta^2}{M}) + \calO(\beta^4)$}. The last term is in the order of 
$\frac{\beta}{M}$. Therefore, the above inequality implies that 
 
 \begin{align*}
  \EE  \| \tilde{ \theta}^{(m)} - \theta^\ast \|^2 \leq   \calO(m D^m) + \calO(\beta^4) + \calO(\frac{\beta}{M}).
 \end{align*}

 Next, we compute the sample complexity for achieving $\EE  \| \tilde{ \theta}^{(m)} - \theta^\ast \|^2 \leq\epsilon$. The above convergence rate implies that, for sufficiently small $\beta$ and sufficiently large $M$, there always exists constants $I_1, I_2, I_3$ such that 
		\[ \EE  \| \tilde{\theta}^{(m)} - \theta^\ast\|^2 \leq m D^m I_1 +  \beta^4 I_2 +  \frac{\beta}{M} I_3.\]
		We require
		\begin{enumerate}[label=(\roman*)]
		    \item $\beta^4 I_2 \leq \epsilon/3\Rightarrow \beta \leq I_2^{1/4}\epsilon^{1/4} = \calO(\epsilon^{1/4})$.
		    \item $\frac{\beta}{M} I_3 \leq \epsilon/3\Rightarrow M\geq \calO(\frac{ \beta}{\epsilon})$.
		    \item $m D^m I_1 \leq \epsilon/3$. We notice that this inequality implies $D^m I_1 \leq \epsilon/(3m)$. 
		\end{enumerate}
		We choose $m = \calO(\log {\epsilon}^{-1})$ so that $m \le \calO(\log \epsilon^{-2})$. Using the upper bound of $m$, the requirement in (iii) suffices to require that $D^m \le \calO(\epsilon/\log \epsilon^{-2})$, which further implies that $m\ge \calO(\log \epsilon^{-1} + \log\log \epsilon^{-2}) / \log D^{-1} = \calO(\log \epsilon^{-1})$ (note that $D<1$). Hence, it is valid to choose $m = \calO(\log {\epsilon}^{-1})$.
		Also, since $\alpha = \calO(\beta^{2/3})$, then $D \leq 1$ requires that $M \geq \calO(\beta^{-3/2})$, which combines with (ii) further requires that 
		\[M \geq \max \{ \calO(\frac{\beta}{\epsilon}), \calO(\beta^{-3/2}) \}.\]
		Let $\calO(\frac{\beta}{\epsilon})$ and $\calO(\beta^{-3/2})$ be of the same order, i.e.,
		$\beta = \calO(\epsilon^{2/5})$,
which satisfies (i). So overall we require that
	 $M\geq \calO(\epsilon^{-3/5})$,
		which leads to the sample complexity
		\[mM \geq \calO(\epsilon^{-3/5}\log {\epsilon}^{-1}).\] 
 
\section{Proof of \Cref{cor: iid}}

\begin{corollary} 
Under the same assumptions as those of \Cref{thm: iid}, choose the learning rates $\alpha,\beta$ and the batch size $M$  such that all requirements of \Cref{thm: iid} are satisfied.
	Then, the following refined bound holds.  
	\begin{align*}
	&  \EE \|\tilde{z}^{(m)}\|^2 \\
	\leq    & F^m \cdot \EE \|  \tilde{z}^{(0)}\|^2 +  \frac{D^{m} - F^m}{D - F} \cdot \EE\| \tilde{\theta}^{(0)} - \theta^\ast\|^2    + \frac{\frac{D^{m} - F^m}{D - F} - \frac{E^{m} - F^m}{E - F}}{D - E}\EE \|  \tilde{z}^{(0)}\|^2   \\
	&+   \frac{1}{1 - F}\frac{1}{1-D} \frac{1152}{\lambda_{\widehat{A}}^2}  \frac{   \rho^2_{\max} \gamma^2}{\lambda_C^2}  \big(  \frac{\alpha^2}{\beta^2}\cdot C_3 + \beta  \cdot C_4 \big)    \Big[  \frac{1}{\beta M}  + \beta \cdot 20 + \frac{\alpha^2}{\beta^2} \cdot 20\gamma^2\rho_{\max}^2 \big(1 + \frac{2 }{\lambda_C } \big)\cdot  \big(\rho_{\max} \frac{ 1+\gamma }{ \min |\lambda(C)|}\big)^2   \Big]   \\
	&\times  \big[ \beta \cdot \frac{24}{\lambda_C }   H_{\text{VR}}^2 + \frac{\alpha^2}{\beta^2} \cdot \big(1 + \frac{2 }{\lambda_C }  \big) \cdot \frac{2}{\lambda_C }\big(\rho_{\max} \frac{ 1+\gamma }{ \min |\lambda(C)|}\big)^2 G_{\text{VR}}^2 \big] \\  
	&+   \frac{1}{1 - F} \frac{1}{1-D}\Big[ \frac{\beta}{M} \cdot\big(  \frac{80}{\lambda_{C}} K_2 + C_4  \frac{600}{\lambda_{C}} \frac{K_1}{\lambda_{\widehat{A}}}  \big)  +  \frac{\alpha^2}{\beta^2}\frac{1}{M} \cdot \big(\frac{80}{\lambda_{C}}\big(\rho_{\max} \frac{ 1+\gamma }{ \min |\lambda(C)|}\big)^2  K_1 \big(1+ \frac{2  }{\lambda_C }  \big) + C_3  \frac{600}{\lambda_{C}} \frac{K_1}{\lambda_{\widehat{A}}}\big)      \Big].
	\end{align*}
	where $K_1$ is specified in eq.(\ref{eq: iid def-K1}) in Lemma \ref{lemma: iid GVR}, $K_2$ is specified in eq.(\ref{eq: iid def-K5}) in Lemma \ref{lemma: iid HVR}, $D$ and $E$ are specified in eq.(\ref{eq: iid def-D}) and eq.(\ref{eq: iid def-E}) in \Cref{lemma: iid theta}.
\end{corollary}
\begin{proof}
    See Lemma \ref{lemma: iid refined z}. Next, we derive its asymptotic upper bound under the setting $\beta = \calO(\alpha^{2/3})$. We note that all the conditions of \Cref{thm: iid} on the learning rates $\alpha, \beta$ and the batch size $M$ can be satisfied with a sufficiently small $\alpha$ in this setting. The first three terms are in the order of $D^m$ (because $D>E, D>F$).
    Here we mainly discuss the order of the fourth and fifth term in the above bound, and note that this term is a product of three brackets. Since we set $\beta = \calO(\alpha^{2/3})$, the first bracket of this product is in the order of $\calO(\beta)$, and the second bracket of this product is in the order of $\calO(\frac{1}{\beta M}) + \calO(\beta)$. The last bracket of this product is in the order of $\calO(\beta)$. Therefore, the fourth term in the above bound is in the order of $\calO(\frac{\beta}{M}) + \calO(\beta^3)$. Also, the last term of this upper bound is in the order of $\calO(\frac{\beta}{M})$. Overall, we can obtain that
    \begin{align*}
        \EE \|\tilde{z}^{(m)}\|^2 = \calO(D^m) + \calO(\frac{\beta}{M}) + \calO(\beta^3).
    \end{align*}
By following the same proof logic of Theorem \ref{thm: iid}, we obtain the desired complexity result.

\end{proof} 

\section{Key Lemmas for Proving \Cref{thm: iid}} 
\begin{lemma}[Preliminary Bound for $\sum_{t=0}^{M-1} \EE_{m,0}  \|\theta_{t}^{(m)} - \theta^\ast \|^2$] \label{lemma: iid pre-bound theta} 
Under the same assumptions as those of \Cref{thm: iid}, choose the learning rate $\alpha$ such that
\begin{align}
	\alpha \leq \min \Big\{  \frac{1}{5 \lambda_{\widehat{A}}} , \frac{\lambda_{\widehat{A}}}{60} / \Big[(1+\gamma)^2\rho^2_{\max}\big(1 + \frac{\gamma \rho_{\max}}{ \min |\lambda(C)|}\big)^2\Big]  \Big\}.
\end{align}
Then, the following preliminary bound holds, { where $K_1$ is specified in eq.(\ref{eq: iid def-K1}) in Lemma \ref{lemma: iid GVR}.}
\begin{align*}
&  \frac{\lambda_{\widehat{A}}}{6}\alpha   \sum_{t=0}^{M-1}\EE_{m,0} \| \theta_{t}^{(m)} - \theta^\ast\|^2 \\
\leq & \Big[ 1 +  \alpha^2M \cdot 5(1+\gamma)^2\rho^2_{\max}\big(1 + \frac{\gamma \rho_{\max}}{ \min |\lambda(C)|}\big)^2 \Big] \EE_{m,0}\| \tilde{\theta}^{(m-1)} - \theta^\ast\|^2    + \alpha^2 \cdot5K_1 \\
&+ \alpha \cdot  \frac{ 2 \rho^2_{\max} \gamma^2}{\lambda_{\widehat{A}}}  \sum_{t=0}^{M-1}\EE_{m,0}\| z_{t}^{(m)}\|^2+ \alpha^2 M \cdot 5 \gamma^2\rho_{\max}^2 \EE_{m,0}\|\tilde{z}^{(m-1)}\|^2.
\end{align*}
\end{lemma}

\begin{proof}
	Based on the update rule of VRTDC for i.i.d.\ samples, we obtain that 
	$$\theta_{t+1}^{(m)} = \Pi_{R_\theta}\big[ \theta_{t}^{(m)} + \alpha[ G_{t}^{(m)}(\theta_{t}^{(m)}, z_{t}^{(m)})  - G_{t}^{(m)}(\tilde{\theta}^{(m-1)}, \tilde{z}^{(m-1)})  + {G}^{(m)}(\tilde{\theta}^{(m-1)}, \tilde{z}^{(m-1)}) ] \big].$$
	The above update rule further implies that
	\begin{align*}
		&\|\theta_{t+1}^{(m)} - \theta^\ast \|^2 \\ &\overset{(i)}{\leq} \| \theta_{t}^{(m)} -\theta^\ast + \alpha[ G_{t}^{(m)}(\theta_{t}^{(m)}, z_{t}^{(m)})  - G_{t}^{(m)}(\tilde{\theta}^{(m-1)}, \tilde{z}^{(m-1)})  + {G}^{(m)}(\tilde{\theta}^{(m-1)}, \tilde{z}^{(m-1)}) ]\|^2 \\
		&=\| \theta_{t}^{(m)} -\theta^\ast\|^2 + \alpha^2 \| G_{t}^{(m)}(\theta_{t}^{(m)}, z_{t}^{(m)})  - G_{t}^{(m)}(\tilde{\theta}^{(m-1)}, \tilde{z}^{(m-1)})  + {G}^{(m)}(\tilde{\theta}^{(m-1)}, \tilde{z}^{(m-1)}) \|^2 \\
		&\quad + 2\alpha \langle \theta_{t}^{(m)} -\theta^\ast, G_{t}^{(m)}(\theta_{t}^{(m)}, z_{t}^{(m)})  - G_{t}^{(m)}(\tilde{\theta}^{(m-1)}, \tilde{z}^{(m-1)})  + {G}^{(m)}(\tilde{\theta}^{(m-1)}, \tilde{z}^{(m-1)})\rangle, \numberthis \label{eq: 2}
	\end{align*}
	where (i) uses the assumption that $R_\theta \ge \|\theta^*\|$ (i.e., $\theta^*$ is in the ball with radius $R_\theta$) and the fact that $\Pi_{R_\theta}$ is 1-Lipschitz. 
	Then, we take the expectation $\EE_{m,0}$ on both sides. In particular, an upper bound for the second variance term is given in \Cref{lemma: iid GVR}. Next, we bound the last term. Note that $\theta_{t}^{(m)} \in \calF_{m, t-1}$ by the definition of the given filtration. Also, the i.i.d.\ sampling implies that $\EE_{m,t-1}  [  - G_{t}^{(m)}(\tilde{\theta}^{(m-1)}, \tilde{z}^{(m-1)})  + {G}^{(m)}(\tilde{\theta}^{(m-1)}, \tilde{z}^{(m-1)}) ] = 0.$ Therefore, for the last term of the above equation, we obtain that
	\begin{align*}
	 &\EE_{m,0} \langle \theta_{t}^{(m)} -\theta^\ast, G_{t}^{(m)}(\theta_{t}^{(m)}, z_{t}^{(m)})  - G_{t}^{(m)}(\tilde{\theta}^{(m-1)}, \tilde{z}^{(m-1)})  + {G}^{(m)}(\tilde{\theta}^{(m-1)}, \tilde{z}^{(m-1)})\rangle \\ 
	 = &\EE_{m,0} \langle \theta_{t}^{(m)} -\theta^\ast, \EE_{m,t-1}\big[G_{t}^{(m)}(\theta_{t}^{(m)}, z_{t}^{(m)})  - G_{t}^{(m)}(\tilde{\theta}^{(m-1)}, \tilde{z}^{(m-1)})  + {G}^{(m)}(\tilde{\theta}^{(m-1)}, \tilde{z}^{(m-1)})\big]\rangle \\  
	 = &\EE_{m,0} \langle \theta_{t}^{(m)} -\theta^\ast, \EE_{m,t-1} G_{t}^{(m)}(\theta_{t}^{(m)}, z_{t}^{(m)})  \rangle \nonumber\\
	 &\quad+ \EE_{m,0} \langle \theta_{t}^{(m)} -\theta^\ast, \EE_{m,t-1}  \big[  - G_{t}^{(m)}(\tilde{\theta}^{(m-1)}, \tilde{z}^{(m-1)})  + {G}^{(m)}(\tilde{\theta}^{(m-1)}, \tilde{z}^{(m-1)}) \big]\rangle\\ 
   = &\EE_{m,0} \langle \theta_{t}^{(m)} -\theta^\ast, \EE_{m,t-1} G_{t}^{(m)}(\theta_{t}^{(m)}, z_{t}^{(m)})  \rangle \\ 
	 = &\EE_{m,0} \langle \theta_{t}^{(m)} -\theta^\ast, \EE_{m,t-1} \big[ \widehat{A}_{t}^{(m)}\theta_{t}^{(m)} +  \widehat{b}_{t}^{(m)} + B_{t}^{(m)} z_{t}^{(m)}\big]  \rangle \\ 
	 = &\EE_{m,0} \langle \theta_{t}^{(m)} -\theta^\ast, \EE_{m,t-1} \big[ \widehat{A}_{t}^{(m)}\theta_{t}^{(m)} +  \widehat{b}_{t}^{(m)} \big]  \rangle + \EE_{m,0} \langle \theta_{t}^{(m)} -\theta^\ast,  B_{t}^{(m)} z_{t}^{(m)} \rangle. \numberthis \label{eq: 1}
	\end{align*}
	Note that $\EE_{m,t-1} \widehat{A}_t^{(m)} = \widehat{A}$ and $\widehat{A} \theta^\ast + \widehat{b} = 0$, the first term of \cref{eq: 1} above can be simplified as
	\begin{align*}
		\EE_{m,0} \langle \theta_{t}^{(m)} -\theta^\ast, \EE_{m,t-1} \big[ \widehat{A}_{t}^{(m)}\theta_{t}^{(m)} +  \widehat{b}_{t}^{(m)} \big]\rangle  &=  \EE_{m,0} \langle \theta_{t}^{(m)} -\theta^\ast, \widehat{A}\theta_{t}^{(m)} + \widehat{b} \rangle \\
		&= \EE_{m,0} \langle \theta_{t}^{(m)} -\theta^\ast, \widehat{A} (\theta_{t}^{(m)} - \theta^\ast  ) + \widehat{A} \theta^\ast + \widehat{b} \rangle \\
		&= \EE_{m,0} \langle \theta_{t}^{(m)} -\theta^\ast, \widehat{A} (\theta_{t}^{(m)} - \theta^\ast  ) \rangle \\
		&\le -\frac{\lambda_{\widehat{A}}}{2}\EE_{m,0}\| \theta_{t}^{(m)} -\theta^\ast\|^2,
	\end{align*}
	where the last inequality uses the property of negative definite matrix
	$(\theta- \theta^\ast )^T \widehat{A} (\theta- \theta^\ast)\le \lambda_{\max}(\widehat{A}) \| \theta - \theta^\ast \|^2$ and the definition that $\lambda_{\widehat{A}}:= - \lambda_{\max} (\widehat{A} + \widehat{A}^\top)$.
	The last term of \cref{eq: 1} can be bounded using the inequality $2 \langle u,v\rangle \le  \| u\|^2 + \|v \|^2$ as 
	\begin{align*}
		\EE_{m,0} \langle \theta_{t}^{(m)} -\theta^\ast,  {B}^{(m)} z_t^{(m)}\rangle &\leq \frac{1}{2}\cdot \frac{\lambda_{\widehat{A}}}{2}\EE_{m,0}\| \theta_{t}^{(m)} -\theta^\ast\|^2 + \frac{1}{2} \cdot \frac{2}{\lambda_{\widehat{A}}}\rho_{\max}^2 \gamma^2 \cdot \EE_{m,0}\| z_{t}^{(m)} \|^2,
	\end{align*}
	{where we have used the fact that $\|{B}^{(m)}_t\|\le \rho_{\max} \gamma$}. Substituting  these inequalities into \cref{eq: 1}, we obtain that
	\begin{align*}
		&\EE_{m,0} \langle \theta_{t}^{(m)} -\theta^\ast, G_{t}^{(m)}(\theta_{t}^{(m)}, z_{t}^{(m)})  - G_{t}^{(m)}(\tilde{\theta}^{(m-1)}, \tilde{z}^{(m-1)})  + {G}^{(m)}(\tilde{\theta}^{(m-1)}, \tilde{z}^{(m-1)})\rangle \\
		\leq & -\frac{\lambda_{\widehat{A}}}{4}\EE_{m,0}  \|\theta_{t}^{(m)} -\theta^\ast\|^2  +   \frac{\rho^2_{\max} \gamma^2}{\lambda_{\widehat{A}}}\EE_{m,0}\| z_{t}^{(m)} \|^2 .
	\end{align*}
	
	Substituting the above inequality into \cref{eq: 2} yields that 
	\begin{align*}
		&\EE_{m,0} \|\theta_{t+1}^{(m)} - \theta^\ast \|^2 \\
		\leq&  \EE_{m,0} \| \theta_{t}^{(m)} -\theta^\ast\|^2 +  \alpha\big[-\frac{\lambda_{\widehat{A}}}{4}\EE_{m,0}  \|\theta_{t}^{(m)} -\theta^\ast\|^2 +   \frac{\rho_{\max}^2 \gamma^2}{\lambda_{\widehat{A}}}\EE_{m,0}\| z_{t}^{(m)} \|^2 \big]\\
		& +\alpha^2\big[  5(1+\gamma)^2\rho^2_{\max}\big(1 + \frac{\gamma \rho_{\max}}{ \min |\lambda(C)|}\big)^2\big( \EE_{m,0} \| \theta_{t}^{(m)} - \theta^\ast\|^2  + \EE_{m,0}\| \tilde{\theta}^{(m-1)} - \theta^\ast\|^2 \big)\big] \\
		& + \alpha^2 \big[ 5 \gamma^2\rho_{\max}^2 \big(\EE_{m,0}\| z_{t}^{(m)}\|^2  + \EE_{m,0}\|\tilde{z}^{(m-1)}\|^2\big) + \frac{5K_1}{M}  \big] \\
		=&\EE_{m,0} \| \theta_{t}^{(m)} -\theta^\ast\|^2 - \big( \frac{\lambda_{\widehat{A}}}{4}\alpha - \alpha^2 \cdot 5(1+\gamma)^2\rho^2_{\max}\big(1 + \frac{\gamma \rho_{\max}}{ \min |\lambda(C)|}\big)^2  \big) \EE_{m,0} \| \theta_{t}^{(m)} - \theta^\ast\|^2 \\
		&+  \alpha^2 \cdot 5(1+\gamma)^2\rho^2_{\max}\big(1 + \frac{\gamma \rho_{\max}}{ \min |\lambda(C)|}\big)^2  \EE_{m,0}\| \tilde{\theta}^{(m-1)} - \theta^\ast\|^2  + \frac{\alpha^2 }{M} \cdot 5K_1 \\
		&+ \big(\alpha \cdot  \frac{\rho_{\max}^2 \gamma^2}{\lambda_{\widehat{A}}} + \alpha^2 \cdot 5 \gamma^2\rho_{\max}^2\big) \EE_{m,0}\| z_{t}^{(m)}\|^2+ \alpha^2 \cdot 5 \gamma^2\rho_{\max}^2 \EE_{m,0}\|\tilde{z}^{(m-1)}\|^2.
	\end{align*}
	Summing the above inequality over $t=0, \dots, M-1$ yields that
	\begin{align*}
		& \big( \frac{\lambda_{\widehat{A}}}{4}\alpha - \alpha^2 \cdot 5(1+\gamma)^2\rho^2_{\max}\big(1 + \frac{\gamma \rho_{\max}}{ \min |\lambda(C)|}\big)^2  \big)  \sum_{t=0}^{M-1}\EE_{m,0} \| \theta_{t}^{(m)} - \theta^\ast\|^2 \\
		\leq & \big[ 1 +  \alpha^2M \cdot 5(1+\gamma)^2\rho^2_{\max}\big(1 + \frac{\gamma \rho_{\max}}{ \min |\lambda(C)|}\big)^2 \big] \EE_{m,0}\| \tilde{\theta}^{(m-1)} - \theta^\ast\|^2    + \alpha^2 \cdot5K_1 \\
		&+ \big(\alpha \cdot  \frac{\rho^2_{\max} \gamma^2}{\lambda_{\widehat{A}}} + \alpha^2 \cdot 5 \gamma^2\rho_{\max}^2\big) \sum_{t=0}^{M-1}\EE_{m,0}\| z_{t}^{(m)}\|^2+ \alpha^2 M \cdot 5 \gamma^2\rho_{\max}^2 \EE_{m,0}\|\tilde{z}^{(m-1)}\|^2.
	\end{align*}
	To further simplify the above inequality, we choose a sufficiently small $\alpha$ such that $\alpha^2 \cdot 5 \gamma^2\rho_{\max}^2 \leq  \alpha \cdot  \frac{\rho^2_{\max} \gamma^2}{\lambda_{\widehat{A}}}$ and $\frac{\lambda_{\widehat{A}}}{4}\alpha - \alpha^2 \cdot 5(1+\gamma)^2\rho^2_{\max}\big(1 + \frac{\gamma \rho_{\max}}{ \min |\lambda(C)|}\big)^2 \geq \frac{\lambda_{\widehat{A}}}{6} \alpha$. Then, the above inequality can be rewritten as
	\begin{align*}
	&  \frac{\lambda_{\widehat{A}}}{6}\alpha   \sum_{t=0}^{M-1}\EE_{m,0} \| \theta_{t}^{(m)} - \theta^\ast\|^2 \\
	\leq & \big[ 1 +  \alpha^2M \cdot 5(1+\gamma)^2\rho^2_{\max}\big(1 + \frac{\gamma \rho_{\max}}{ \min |\lambda(C)|}\big)^2 \big] \EE_{m,0}\| \tilde{\theta}^{(m-1)} - \theta^\ast\|^2    + \alpha^2 \cdot5K_1 \\
	&+ \alpha \cdot  \frac{ 2 \rho^2_{\max} \gamma^2}{\lambda_{\widehat{A}}}  \sum_{t=0}^{M-1}\EE_{m,0}\| z_{t}^{(m)}\|^2+ \alpha^2 M \cdot 5 \gamma^2\rho_{\max}^2 \EE_{m,0}\|\tilde{z}^{(m-1)}\|^2.
	\end{align*}
\end{proof}

\begin{lemma}[Preliminary bound for $\EE\| \tilde{z}^{(m)} \|^2$]
	 \label{lemma: iid conv-z}
	 Under the same assumptions as those of \Cref{thm: iid}, choose the learning rate $\beta$ and the batch size $M$ such that
	 $\beta < 1$ and $M\beta > \frac{4}{\lambda_{C}}$. 
	 Then, the following preliminary bound holds.
	\begin{align*}
	\EE\| \tilde{z}^{(m)} \|^2 &\leq   \big(\frac{1}{M\beta} \cdot \frac{2}{\lambda_C }\big)^m \EE \|  \tilde{z}^{(0)}\|^2 \\
	&\quad  + 2 \cdot \big[ \beta \cdot \frac{24}{\lambda_C }   H_{\text{VR}}^2 + \frac{\alpha^2}{\beta^2} \cdot \big(1 + \frac{2 }{\lambda_C }  \big) \cdot \frac{2}{\lambda_C }\big(\rho_{\max} \frac{ 1+\gamma }{ \min |\lambda(C)|}\big)^2 G_{\text{VR}}^2 \big].
	\end{align*}
	{ where $H_{\text{VR}}, G_{\text{VR}}$ is defined in Lemma \ref{lemma: HVR-const} and Lemma \ref{lemma: GVR-const}.}
\end{lemma}
\begin{proof}{ First, based on the update rule of $w^{(m)}_t$, we have 
\begin{align*}
	w^{(m)}_{t+1} = \Pi_{R_w}\big[ w^{(m)}_{t} +  \ A_{t}^{(m)} \theta^{(m)}_{t} + b_{t}^{(m)}+ C_{t}^{(m)}  w^{(m)}_{t} \big],
\end{align*}	
which further implies the following one-step update rule of the tracking error $z_t^{(m)}$.
\begin{align*}
z^{(m)}_{t+1} = \Pi_{R_w}\big[ w^{(m)}_{t} +  \ A_{t}^{(m)} \theta^{(m)}_{t} + b_{t}^{(m)}+ C_{t}^{(m)}  w^{(m)}_{t} \big]  + C^{-1}(b + A(\tilde{\theta}^{(m)})).
\end{align*}
		Then, its square norm can be bounded as}
		\begin{align*}
		\| z_{t+1}^{(m)}\|^2 &\overset{(i)}{\leq} \|  z_t^{(m)} + \beta\big[   H_{t}^{(m)}(\theta_{t}^{(m)}, z_{t}^{(m)})  - H_{t}^{(m)}(\tilde{\theta}^{(m-1)}, \tilde{z}^{(m-1)})  +  {H}^{(m)}(\tilde{\theta}^{(m-1)}, \tilde{z}^{(m-1)}) \big] \\
		&\quad+ C^{-1}A(\theta_{t+1}^{(m)}  - \theta_{t}^{(m)}) \|^2 \\
		&= \|  z_t^{(m)}\|^2 + 2\beta^2 \|    H_{t}^{(m)}(\theta_{t}^{(m)}, z_{t}^{(m)})  - H_{t}^{(m)}(\tilde{\theta}^{(m-1)}, \tilde{z}^{(m-1)})  +  {H}^{(m)}(\tilde{\theta}^{(m-1)}, \tilde{z}^{(m-1)}) \|^2   \\
		&\quad + 2\| C^{-1}A(\theta_{t+1}^{(m)}  - \theta_{t}^{(m)})\|^2 \\
		&\quad+ 2\beta \langle z_t^{(m)}, H_{t}^{(m)}(\theta_{t}^{(m)}, z_{t}^{(m)})  - H_{t}^{(m)}(\tilde{\theta}^{(m-1)}, \tilde{z}^{(m-1)})  +  {H}^{(m)}(\tilde{\theta}^{(m-1)}, \tilde{z}^{(m-1)}) \rangle  \\
		&\quad + 2\langle z_t^{(m)}, C^{-1}A(\theta_{t+1}^{(m)}  - \theta_{t}^{(m)})  \rangle, \numberthis \label{eq: 3}
	\end{align*}
	where (i) uses the assumption that $R_w\ge 2\|C^{-1}\|\|A\|R_{\theta}$ (i.e., $C^{-1}(b + A \theta_t^{(m)})$ is in the ball with radius $R_w$) and the fact that $\Pi_{R_w}$ is 1-Lipschitz.
	For the last term of \cref{eq: 3}, it can be bounded as
	$$2 \langle z_t^{(m)}, C^{-1}A(\theta_{t+1}^{(m)}  - \theta_{t}^{(m)})  \rangle \leq \frac{\lambda_C}{2}\beta \| z_t^{(m)}\|^2 + \frac{2}{\lambda_C } \frac{1}{\beta} \| C^{-1}A (\theta_{t+1}^{(m)} - \theta_{t}^{(m)})\|^2. $$
	Substituting the above inequality, \Cref{lemma: GVR-const} and \Cref{lemma: HVR-const} into \cref{eq: 3}, we obtain that 
	\begin{align*}
		\| z_{t+1}^{(m)}\|^2 &\leq \|  z_t^{(m)}\|^2 + 2\beta^2 H_{\text{VR}}^2 + \big(\alpha^2 + \frac{2 \alpha^2}{\lambda_C } \frac{1}{\beta} \big)\cdot 2\big(\rho_{\max} \frac{ 1+\gamma }{ \min |\lambda(C)|}\big)^2 G_{\text{VR}}^2 + \frac{\lambda_C}{2}\beta \| z_t^{(m)}\|^2  \\
		&\quad +  2\beta \langle z_t^{(m)}, H_{t}^{(m)}(\theta_{t}^{(m)}, z_{t}^{(m)})  - H_{t}^{(m)}(\tilde{\theta}^{(m-1)}, \tilde{z}^{(m-1)})  +  {H}^{(m)}(\tilde{\theta}^{(m-1)}, \tilde{z}^{(m-1)}) \rangle. \numberthis \label{eq: 4}
	\end{align*}
	Next, we bound the inner product term in the above inequality. Notice that $z_t^{(m)} \in \calF_{m, t-1}$ and by i.i.d.\ sampling we have $\EE_{m,t-1} \bar{A}_t^{(m)} = \bar{A}$. Therefore,
	\begin{align*}
		& \EE_{m,0} \langle z_t^{(m)}, H_{t}^{(m)}(\theta_{t}^{(m)}, z_{t}^{(m)})  - H_{t}^{(m)}(\tilde{\theta}^{(m-1)}, \tilde{z}^{(m-1)})  +  {H}^{(m)}(\tilde{\theta}^{(m-1)}, \tilde{z}^{(m-1)}) \rangle \\
		= & \EE_{m,0} \langle z_t^{(m)}, H_t^{(m)}(\theta_{t}^{(m)}, z_{t}^{(m)})   \rangle \\
		= & \EE_{m,0} \langle z_t^{(m)}, \bar{A}^{(m)}_t\theta_{t}^{(m)} +  \bar{b}_t^{(m)} +  {C}^{(m)}_t z_{t}^{(m)} \rangle \\
		= & \EE_{m,0} \langle z_t^{(m)}, \EE_{m,t-1}(\bar{A}^{(m)}_t)\theta_{t}^{(m)} +  \EE_{m,t-1}(\bar{b}^{(m)})\rangle +  \EE_{m,0} \langle z_t^{(m)}, \EE_{m,t-1}({C}_t^{(m)} - C) z_{t}^{(m)} \rangle + \EE_{m,0} \langle z_t^{(m)}, C z_{t}^{(m)} \rangle\\
		= & \EE_{m,0} \langle z_t^{(m)}, C z_{t}^{(m)} \rangle \\
		\leq &  - \frac{\lambda_C}{2} \EE_{m,0}\|z_t^{(m)}\|^2
	\end{align*}
	where the last inequality utilizes the negative definiteness of $C$ (recall that $\lambda_{C}:=-\lambda_{\max}(C+C^\top)$). Substituting the above inequality into \cref{eq: 4} (after taking expectation) yields that
	\begin{align*}
	\EE_{m,0} \| z_{t+1}^{(m)}\|^2 &\leq \EE_{m,0} \|  z_t^{(m)}\|^2 + 2\beta^2 H_{\text{VR}}^2 + \big(\alpha^2 + \frac{2 \alpha^2}{\lambda_C } \frac{1}{\beta} \big)\cdot 2\big(\rho_{\max} \frac{ 1+\gamma }{ \min |\lambda(C)|}\big)^2 G_{\text{VR}}^2 \\
	&\quad   - \frac{\lambda_C}{2}\beta \EE_{m,0} \| z_t^{(m)}\|^2  .\\ 
	\end{align*}
	Summing the above inequality over one batch yields that
	\begin{align*}
	\EE_{m,0} \| z_{M}^{(m)}\|^2 &\leq \EE_{m,0} \|  z_0^{(m)}\|^2 + 2\beta^2 M H_{\text{VR}}^2 + \big(\alpha^2 + \frac{2 \alpha^2}{\lambda_C } \frac{1}{\beta} \big)M\cdot 2\big(\rho_{\max} \frac{ 1+\gamma }{ \min |\lambda(C)|}\big)^2 G_{\text{VR}}^2 \\
	&\quad   - \frac{\lambda_C}{2}\beta \sum_{t=0}^{M-1}\EE_{m,0} \| z_t^{(m)}\|^2.\\ 
	\end{align*}
	Re-arranging the above inequality and omitting $\EE_{m,0} \| z_{M}^{(m)}\|^2$ further yields that
	\begin{align*}
		\frac{\lambda_C}{2}\beta M \EE_{m,0} \| \tilde{z}^{(m)} \|^2 \leq  \|  \tilde{z}^{(m-1)}\|^2 + 2\beta^2 M H_{\text{VR}}^2 + \big(\alpha^2 + \frac{2 \alpha^2}{\lambda_C } \frac{1}{\beta} \big)M\cdot 2\big(\rho_{\max} \frac{ 1+\gamma }{ \min |\lambda(C)|}\big)^2 G_{\text{VR}}^2.
	\end{align*}
	Dividing $\frac{\lambda_C}{2}\beta M $ on both sides of the above inequality, we obtain the following one-batch bound.
	\begin{align*}
	\EE\| \tilde{z}^{(m)} \|^2 &\leq   \frac{1}{M\beta} \cdot \frac{2}{\lambda_C }\EE \|  \tilde{z}^{(m-1)}\|^2 +  \beta \cdot \frac{4}{\lambda_C }   H_{\text{VR}}^2 + \big(\frac{\alpha^2}{\beta} + \frac{2 }{\lambda_C } \frac{\alpha^2}{\beta^2} \big) \cdot \frac{4}{\lambda_C }\big(\rho_{\max} \frac{ 1+\gamma }{ \min |\lambda(C)|}\big)^2 G_{\text{VR}}^2 .
	\end{align*}
	Finally, we recursively unroll the above inequality and obtain
	\begin{align*}
	\EE\| \tilde{z}^{(m)} \|^2 &\leq   \big(\frac{1}{M\beta} \cdot \frac{2}{\lambda_C }\big)^m \EE \|  \tilde{z}^{(0)}\|^2 \\
	&\quad  + \frac{1}{1 - \frac{1}{M\beta} \cdot \frac{2}{\lambda_C } }\big[ \beta \cdot \frac{24}{\lambda_C }   H_{\text{VR}}^2 + \big(\frac{\alpha^2}{\beta} + \frac{2 }{\lambda_C } \frac{\alpha^2}{\beta^2} \big) \cdot \frac{2}{\lambda_C }\big(\rho_{\max} \frac{ 1+\gamma }{ \min |\lambda(C)|}\big)^2 G_{\text{VR}}^2 \big].
	\end{align*}
	To further simplify the above inequality, we assume $\beta < 1$ and $M\beta > \frac{4}{\lambda_{C}}$. Then, we have
	\begin{align*}
	\EE\| \tilde{z}^{(m)} \|^2 &\leq   \big(\frac{1}{M\beta} \cdot \frac{2}{\lambda_C }\big)^m \EE \|  \tilde{z}^{(0)}\|^2 \\
	&\quad  + 2 \cdot \big[ \beta \cdot \frac{24}{\lambda_C }   H_{\text{VR}}^2 + \frac{\alpha^2}{\beta^2} \cdot \big(1 + \frac{2 }{\lambda_C }  \big) \cdot \frac{2}{\lambda_C }\big(\rho_{\max} \frac{ 1+\gamma }{ \min |\lambda(C)|}\big)^2 G_{\text{VR}}^2 \big].
	\end{align*}
\end{proof}

\begin{lemma}[Preliminary Bound for $\sum_{t=0}^{M-1} \|z_{t}^{(m)}\|^2$]
	 \label{lemma: iid pre-bound z} 
	 Under the same assumptions as those of \Cref{thm: iid}, choose the learning rate $\beta$ and the batch size $M$ such that
	 $\beta < 1$ and  
	 \begin{align}
	  \frac{\lambda_{C}}{2}\beta - 10\beta^2 - 10\gamma^2\rho_{\max}^2 \big(\alpha^2 + \frac{2 \alpha^2}{\lambda_C } \frac{1}{\beta} \big)\cdot  \big(\rho_{\max} \frac{ 1+\gamma }{ \min |\lambda(C)|}\big)^2 \geq \frac{\lambda_{C}}{3} \beta.
	 \end{align} 
	 Then the following preliminary bound holds.
	 \begin{align*}
	 &  \frac{\lambda_{C}}{3}\beta  \sum_{t=0}^{M-1}\EE_{m,0}\|z_t^{(m)}\|^2 \\
	 \leq &\Big[1+ \Big[10\beta^2 + 10\gamma^2\rho_{\max}^2 \big(1 + \frac{2 }{\lambda_C } \big)\cdot  \big(\rho_{\max} \frac{ 1+\gamma }{ \min |\lambda(C)|}\big)^2 \frac{\alpha^2}{\beta} \Big]M \Big]\EE_{m,0}\|  \tilde{z}^{(m-1)}\|^2\\
	 &+ 10 (1+\gamma)^2 \rho^2 _{\max} \cdot \Big[\big(1 + \frac{\gamma \rho_{\max}}{ \min |\lambda(C)|}\big)^2 \big(1 + \frac{2 }{\lambda_C } \big)\cdot  \big(\rho_{\max} \frac{ 1+\gamma }{ \min |\lambda(C)|}\big)^2 \frac{\alpha^2}{\beta}  \\
	 &+  \big( 1 + \frac{1}{ \min |\lambda(C)|} \big)^2 \beta^2  \Big] \big( \sum_{t=0}^{M-1}\EE_{m,0} \| \theta_{t}^{(m)} - \theta^\ast\|^2  + M \EE_{m,0}\| \tilde{\theta}^{(m-1)} - \theta^\ast\|^2 \big) \\ 
	 &+    10K_2 \beta^2 +  10\big(\rho_{\max} \frac{ 1+\gamma }{ \min |\lambda(C)|}\big)^2  K_1 \big(1+ \frac{2  }{\lambda_C }  \big)  \frac{\alpha^2}{\beta},
	 \end{align*}
	  { where $K_1$ is specified in eq.(\ref{eq: iid def-K1}) in Lemma \ref{lemma: iid GVR}.}
\end{lemma}
\begin{proof}
	Following the proof of \Cref{lemma: iid conv-z}, the one-step update of $z_{t}^{(m)}$ implies that
	\begin{align*}
	\| z_{t+1}^{(m)}\|^2  
	&\leq \|  z_t^{(m)}\|^2 + 2\beta^2 \|    H_{t}^{(m)}(\theta_{t}^{(m)}, z_{t}^{(m)})  - H_{t}^{(m)}(\tilde{\theta}^{(m-1)}, \tilde{z}^{(m-1)})  +  {H}^{(m)}(\tilde{\theta}^{(m-1)}, \tilde{z}^{(m-1)}) \|^2   \\
	&\quad + 2\| C^{-1}A(\theta_{t+1}^{(m)}  - \theta_{t}^{(m)})\|^2 +2\langle z_t^{(m)}, C^{-1}A(\theta_{t+1}^{(m)}  - \theta_{t}^{(m)})  \rangle \\
	&\quad +  2\beta \langle z_t^{(m)}, H_{t}^{(m)}(\theta_{t}^{(m)}, z_{t}^{(m)})  - H_{t}^{(m)}(\tilde{\theta}^{(m-1)}, \tilde{z}^{(m-1)})  +  {H}^{(m)}(\tilde{\theta}^{(m-1)}, \tilde{z}^{(m-1)}) \rangle \\
	&\leq  \|  z_t^{(m)}\|^2 + 2\beta^2 \|    H_{t}^{(m)}(\theta_{t}^{(m)}, z_{t}^{(m)})  - H_{t}^{(m)}(\tilde{\theta}^{(m-1)}, \tilde{z}^{(m-1)})  +  {H}^{(m)}(\tilde{\theta}^{(m-1)}, \tilde{z}^{(m-1)}) \|^2   \\
	&\quad + 2\| C^{-1}A(\theta_{t+1}^{(m)}  - \theta_{t}^{(m)})\|^2 +\frac{\lambda_C}{2}\beta \| z_t^{(m)}\|^2 + \frac{2}{\lambda_C } \frac{1}{\beta} \| C^{-1}A (\theta_{t+1} - \theta_{t})\|^2. \\
	&\quad +  2\beta \langle z_t^{(m)}, H_{t}^{(m)}(\theta_{t}^{(m)}, z_{t}^{(m)})  - H_{t}^{(m)}(\tilde{\theta}^{(m-1)}, \tilde{z}^{(m-1)})  +  {H}^{(m)}(\tilde{\theta}^{(m-1)}, \tilde{z}^{(m-1)}) \rangle \\
	&\leq \|  z_t^{(m)}\|^2 +\frac{\lambda_C}{2}\beta \| z_t^{(m)}\|^2 + 2\beta^2 \|    H_{t}^{(m)}(\theta_{t}^{(m)}, z_{t}^{(m)})  - H_{t}^{(m)}(\tilde{\theta}^{(m-1)}, \tilde{z}^{(m-1)})  +  {H}^{(m)}(\tilde{\theta}^{(m-1)}, \tilde{z}^{(m-1)}) \|^2   \\
	&\quad + \big(\alpha^2 + \frac{2 \alpha^2}{\lambda_C } \frac{1}{\beta} \big)\cdot 2\big(\rho_{\max} \frac{ 1+\gamma }{ \min |\lambda(C)|}\big)^2 \|    G_{t}^{(m)}(\theta_{t}^{(m)}, z_{t}^{(m)})  \\
	&\quad\qquad- G_{t}^{(m)}(\tilde{\theta}^{(m-1)}, \tilde{z}^{(m-1)})  +  {G}^{(m)}(\tilde{\theta}^{(m-1)}, \tilde{z}^{(m-1)}) \|^2 \\
	&\quad +  2\beta \langle z_t^{(m)}, H_{t}^{(m)}(\theta_{t}^{(m)}, z_{t}^{(m)})  - H_{t}^{(m)}(\tilde{\theta}^{(m-1)}, \tilde{z}^{(m-1)})  +  {H}^{(m)}(\tilde{\theta}^{(m-1)}, \tilde{z}^{(m-1)}) \rangle. \numberthis \label{eq: 5}
	\end{align*}
	For the last inner product term, we still have
	\begin{align*}
		\EE_{m,0} \langle z_t^{(m)}, H_{t}^{(m)}(\theta_{t}^{(m)}, z_{t}^{(m)})  - H_{t}^{(m)}(\tilde{\theta}^{(m-1)}, \tilde{z}^{(m-1)})  +  {H}^{(m)}(\tilde{\theta}^{(m-1)}, \tilde{z}^{(m-1)}) \rangle 
		\leq  -\frac{\lambda_C}{2}\EE_{m,0}\|z_t^{(m)}\|^2.
	\end{align*}
	 Instead of bounding the variance term in \cref{eq: 5} using Lemma \ref{lemma: GVR-const} and Lemma \ref{lemma: HVR-const}, we apply Lemma \ref{lemma: iid GVR} and Lemma \ref{lemma: iid HVR} to get a refined bound. Combining these together, we obtain from \cref{eq: 5} that
	\begin{align*}
		& \EE_{m,0}\| z_{t+1}^{(m)}\|^2   \\
		\leq & \EE_{m,0}  \|  z_t^{(m)}\|^2 -\frac{\lambda_C}{2}\beta \EE_{m,0}\|z_t^{(m)}\|^2 + 2\beta^2 \Big[5\big(   \EE_{m,0}\| z_{t}^{(m)}\|^2 + \EE_{m,0}\| \tilde{z}^{(m-1)}\|^2\big)  + \frac{5K_2}{M}\\
		&+    5 (1+\gamma)^2 \rho^2 _{\max}\big( 1 + \frac{1}{ \min |\lambda(C)|} \big)^2\big( \EE_{m,0}\| \theta_{t}^{(m)} - \theta^\ast\|^2 + \EE_{m,0}\| \tilde{\theta}^{(m-1)} - \theta^\ast\|^2  \big) \Big]\\
		&+ \big(\alpha^2 + \frac{2 \alpha^2}{\lambda_C } \frac{1}{\beta} \big)\cdot 2\big(\rho_{\max} \frac{ 1+\gamma }{ \min |\lambda(C)|}\big)^2 \Big[5 \gamma^2\rho_{\max}^2 \big(\EE_{m,0}\| z_{t}^{(m)}\|^2  + \EE_{m,0}\|\tilde{z}^{(m-1)}\|^2\big) + \frac{5K_1}{M} \\
		&+    5(1+\gamma)^2\rho^2_{\max}\big(1 + \frac{\gamma \rho_{\max}}{ \min |\lambda(C)|}\big)^2\big( \EE_{m,0} \| \theta_{t}^{(m)} - \theta^\ast\|^2  + \EE_{m,0}\| \tilde{\theta}^{(m-1)} - \theta^\ast\|^2 \big)  \Big].
	\end{align*}
	Re-arranging the above inequality yields that
	\begin{align*}
			& \EE_{m,0}\| z_{t+1}^{(m)}\|^2   \\
		\leq & \EE_{m,0}  \|  z_t^{(m)}\|^2 - \Big[ \frac{\lambda_{C}}{2}\beta - 10\beta^2 - 10\gamma^2\rho_{\max}^2 \big(\alpha^2 + \frac{2 \alpha^2}{\lambda_C } \frac{1}{\beta} \big)\cdot  \big(\rho_{\max} \frac{ 1+\gamma }{ \min |\lambda(C)|}\big)^2 \Big]\EE_{m,0}\|z_t^{(m)}\|^2  \\
		&+ \Big[10\beta^2 + 10\gamma^2\rho_{\max}^2 \big(\alpha^2 + \frac{2 \alpha^2}{\lambda_C } \frac{1}{\beta} \big)\cdot  \big(\rho_{\max} \frac{ 1+\gamma }{ \min |\lambda(C)|}\big)^2  \Big] \EE_{m,0}\| \tilde{z}^{(m-1)}\|^2 \\
		&+ 10 (1+\gamma)^2 \rho^2 _{\max} \cdot \Big[\big(1 + \frac{\gamma \rho_{\max}}{ \min |\lambda(C)|}\big)^2 \big(\alpha^2 + \frac{2 \alpha^2}{\lambda_C } \frac{1}{\beta} \big)\cdot  \big(\rho_{\max} \frac{ 1+\gamma }{ \min |\lambda(C)|}\big)^2   \\
		&+  \big( 1 + \frac{1}{ \min |\lambda(C)|} \big)^2 \beta^2  \Big] \big( \EE_{m,0} \| \theta_{t}^{(m)} - \theta^\ast\|^2  + \EE_{m,0}\| \tilde{\theta}^{(m-1)} - \theta^\ast\|^2 \big) \\ 
		&+ \frac{1}{M} \cdot \big[   10K_2 \beta^2 +  10\big(\rho_{\max} \frac{ 1+\gamma }{ \min |\lambda(C)|}\big)^2  K_1 \big(\alpha^2 + \frac{2 \alpha^2}{\lambda_C } \frac{1}{\beta} \big)\big].
	\end{align*}
	Telescoping the above inequality over one batch yields that
	\begin{align*}
		 & \Big[ \frac{\lambda_{C}}{2}\beta - 10\beta^2 - 10\gamma^2\rho_{\max}^2 \big(\alpha^2 + \frac{2 \alpha^2}{\lambda_C } \frac{1}{\beta} \big)\cdot  \big(\rho_{\max} \frac{ 1+\gamma }{ \min |\lambda(C)|}\big)^2 \Big]\sum_{t=0}^{M-1}\EE_{m,0}\|z_t^{(m)}\|^2 \\
		 \leq &\Big[1+ \Big[10\beta^2 + 10\gamma^2\rho_{\max}^2 \big(\alpha^2 + \frac{2 \alpha^2}{\lambda_C } \frac{1}{\beta} \big)\cdot  \big(\rho_{\max} \frac{ 1+\gamma }{ \min |\lambda(C)|}\big)^2  \Big]M \Big]\EE_{m,0}\|  \tilde{z}^{(m-1)}\|^2\\
		 &+ 10 (1+\gamma)^2 \rho^2 _{\max} \cdot \Big[\big(1 + \frac{\gamma \rho_{\max}}{ \min |\lambda(C)|}\big)^2 \big(\alpha^2 + \frac{2 \alpha^2}{\lambda_C } \frac{1}{\beta} \big)\cdot  \big(\rho_{\max} \frac{ 1+\gamma }{ \min |\lambda(C)|}\big)^2   \\
		 &+  \big( 1 + \frac{1}{ \min |\lambda(C)|} \big)^2 \beta^2  \Big] \big( \sum_{t=0}^{M-1}\EE_{m,0} \| \theta_{t}^{(m)} - \theta^\ast\|^2  + M \EE_{m,0}\| \tilde{\theta}^{(m-1)} - \theta^\ast\|^2 \big) \\ 
		 &+    10K_2 \beta^2 +  10\big(\rho_{\max} \frac{ 1+\gamma }{ \min |\lambda(C)|}\big)^2  K_1 \big(\alpha^2 + \frac{2 \alpha^2}{\lambda_C } \frac{1}{\beta} \big).
	\end{align*}
	To further simplify the above inequality, we let $\beta<1$ and 
	$$ \frac{\lambda_{C}}{2}\beta - 10\beta^2 - 10\gamma^2\rho_{\max}^2 \big(\alpha^2 + \frac{2 \alpha^2}{\lambda_C } \frac{1}{\beta} \big)\cdot  \big(\rho_{\max} \frac{ 1+\gamma }{ \min |\lambda(C)|}\big)^2 \geq \frac{\lambda_{C}}{3} \beta.$$
	Then, we finally obtain that
	\begin{align*}
	&  \frac{\lambda_{C}}{3}\beta  \sum_{t=0}^{M-1}\EE_{m,0}\|z_t^{(m)}\|^2 \\
	\leq &\Big[1+ \Big[10\beta^2 + 10\gamma^2\rho_{\max}^2 \big(1 + \frac{2 }{\lambda_C } \big)\cdot  \big(\rho_{\max} \frac{ 1+\gamma }{ \min |\lambda(C)|}\big)^2 \frac{\alpha^2}{\beta} \Big]M \Big]\EE_{m,0}\|  \tilde{z}^{(m-1)}\|^2\\
	&+ 10 (1+\gamma)^2 \rho^2 _{\max} \cdot \Big[\big(1 + \frac{\gamma \rho_{\max}}{ \min |\lambda(C)|}\big)^2 \big(1 + \frac{2 }{\lambda_C } \big)\cdot  \big(\rho_{\max} \frac{ 1+\gamma }{ \min |\lambda(C)|}\big)^2 \frac{\alpha^2}{\beta}  \\
	&+  \big( 1 + \frac{1}{ \min |\lambda(C)|} \big)^2 \beta^2  \Big] \big( \sum_{t=0}^{M-1}\EE_{m,0} \| \theta_{t}^{(m)} - \theta^\ast\|^2  + M \EE_{m,0}\| \tilde{\theta}^{(m-1)} - \theta^\ast\|^2 \big) \\ 
	&+    10K_2 \beta^2 +  10\big(\rho_{\max} \frac{ 1+\gamma }{ \min |\lambda(C)|}\big)^2  K_1 \big(1+ \frac{2  }{\lambda_C }  \big)  \frac{\alpha^2}{\beta}.
	\end{align*}
\end{proof}

\begin{lemma}[Preliminary bound for $\sum_{t=0}^{M-1}\EE_{m,0} \| \theta_{t}^{(m)} - \theta^\ast\|^2$]\label{lemma: iid theta}
Under the same assumptions as those of \Cref{thm: iid},  \Cref{lemma: iid pre-bound z}, \Cref{lemma: iid pre-bound theta}, and \Cref{lemma: iid conv-z}, choose the learning rates $\alpha, \beta$ and the batch size $M$ such that
\begin{align}\label{eq: iid def-D}
	D:= \frac{12}{\lambda_{\widehat{A}}}\Big\{  \frac{1}{\alpha M} +  \alpha  \cdot 5(1+\gamma)^2\rho^2_{\max}\big(1 + \frac{\gamma \rho_{\max}}{ \min |\lambda(C)|}\big)^2   +     \frac{\alpha^2}{\beta^2} \cdot C_1  + \beta \cdot C_2   \Big\} < 1,
\end{align}
 and 
 \begin{align}\label{eq: iid def-E}
 	E:= \frac{1}{M\beta} \cdot \frac{2}{\lambda_C } < 1,
 \end{align}
 where 
 \begin{align}\label{eq: iid def-C1}
 	C_1 = \frac{ 2 \rho^2_{\max} \gamma^2}{\lambda_{\widehat{A}}} \frac{3}{\lambda_{C}}  \cdot 10 (1+\gamma)^2 \rho^2 _{\max} \cdot \big(1 + \frac{\gamma \rho_{\max}}{ \min |\lambda(C)|}\big)^2 \big(1 + \frac{2 }{\lambda_C } \big)\cdot  \big(\rho_{\max} \frac{ 1+\gamma }{ \min |\lambda(C)|}\big)^2,
 \end{align}
 and 
 \begin{align}\label{eq: iid def-C2}
 C_2 = \frac{ 2 \rho^2_{\max} \gamma^2}{\lambda_{\widehat{A}}} \frac{3}{\lambda_{C}}  \cdot 10 (1+\gamma)^2 \rho^2 _{\max} \cdot \big( 1 + \frac{1}{ \min |\lambda(C)|} \big)^2.
 \end{align} 
Then, the following preliminary bound holds.
\begin{align*}
& \EE \| \tilde{ \theta}^{(m)} - \theta^\ast \|^2\\
\leq & D^m \cdot \EE\| \tilde{\theta}^{(0)} - \theta^\ast\|^2     \\
&+ \frac{12}{\lambda_{\widehat{A}}} \Big\{\frac{ 2 \rho^2_{\max} \gamma^2}{\lambda_{\widehat{A}}} \frac{3}{\lambda_{C}}  \Big[  \frac{1}{\beta M}  + \beta \cdot 10 + \frac{\alpha^2}{\beta^2} \cdot 10\gamma^2\rho_{\max}^2 \big(1 + \frac{2 }{\lambda_C } \big)\cdot  \big(\rho_{\max} \frac{ 1+\gamma }{ \min |\lambda(C)|}\big)^2   \Big] + \alpha  \cdot 5 \gamma^2\rho_{\max}^2   \Big\}   \\
&\times \Big\{ \frac{D^m - E^m}{D - E}\EE \|  \tilde{z}^{(0)}\|^2 + \frac{2}{1-D}\cdot \big[ \beta \cdot \frac{24}{\lambda_C }   H_{\text{VR}}^2 + \frac{\alpha^2}{\beta^2} \cdot \big(1 + \frac{2 }{\lambda_C }  \big) \cdot \frac{2}{\lambda_C }\big(\rho_{\max} \frac{ 1+\gamma }{ \min |\lambda(C)|}\big)^2 G_{\text{VR}}^2 \big]\Big\}\\  
&+  \frac{1}{1-D}\Big\{  \frac{\beta}{M}  \cdot  \frac{12}{\lambda_{\widehat{A}}} \frac{ 60 \rho^2_{\max} \gamma^2}{\lambda_{\widehat{A}}\lambda_{C} }   \cdot K_2   + \frac{\alpha^2}{\beta^2}\frac{1}{M} \cdot  \frac{12}{\lambda_{\widehat{A}}} \frac{ 60 \rho^2_{\max} \gamma^2}{\lambda_{\widehat{A}}\lambda_{C}}   \cdot \big(\rho_{\max} \frac{ 1+\gamma }{ \min |\lambda(C)|}\big)^2  K_1 \big(1+ \frac{2  }{\lambda_C }  \big)   + \frac{\alpha}{M} \cdot \frac{60 K_1}{\lambda_{\widehat{A}}}  \Big\}
\end{align*}
where $K_1$ is specified in eq.(\ref{eq: iid def-K1}) in Lemma \ref{lemma: iid GVR}, and $K_2$ is specified in eq.(\ref{eq: iid def-K5}) in Lemma \ref{lemma: iid HVR}.
\end{lemma}
\begin{proof}
    %
	First, recall that Lemma \ref{lemma: iid pre-bound theta} gives the following preliminary bound for $\sum_{t=0}^{M-1} \| \theta_{t}^{(m)} - \theta^\ast\|^2$.
	\begin{align*}
	&  \frac{\lambda_{\widehat{A}}}{6}\alpha   \sum_{t=0}^{M-1}\EE_{m,0} \| \theta_{t}^{(m)} - \theta^\ast\|^2 \\
	\leq & \Big[ 1 +  \alpha^2M \cdot 5(1+\gamma)^2\rho^2_{\max}\big(1 + \frac{\gamma \rho_{\max}}{ \min |\lambda(C)|}\big)^2 \Big] \EE_{m,0}\| \tilde{\theta}^{(m-1)} - \theta^\ast\|^2    + \alpha^2 \cdot5K_1 \\
	&+ \alpha \cdot  \frac{ 2 \rho^2_{\max} \gamma^2}{\lambda_{\widehat{A}}}  \sum_{t=0}^{M-1}\EE_{m,0}\| z_{t}^{(m)}\|^2+ \alpha^2 M \cdot 5 \gamma^2\rho_{\max}^2 \EE_{m,0}\|\tilde{z}^{(m-1)}\|^2.
	\end{align*}
	Then, we combine the above preliminary bound with Lemma \ref{lemma: iid pre-bound z} and obtain that
	\begin{align*}
	&  \frac{\lambda_{\widehat{A}}}{6}\alpha   \sum_{t=0}^{M-1}\EE_{m,0} \| \theta_{t}^{(m)} - \theta^\ast\|^2 \\
	\leq & \Big[ 1 +  \alpha^2M \cdot 5(1+\gamma)^2\rho^2_{\max}\big(1 + \frac{\gamma \rho_{\max}}{ \min |\lambda(C)|}\big)^2 \Big] \EE_{m,0}\| \tilde{\theta}^{(m-1)} - \theta^\ast\|^2    + \alpha^2 \cdot5K_1 \\
	&+ \frac{\alpha}{\beta} \cdot  \frac{ 2 \rho^2_{\max} \gamma^2}{\lambda_{\widehat{A}}} \frac{3}{\lambda_{C}} \Big\{\Big[1+ \Big[10\beta^2 + 10\gamma^2\rho_{\max}^2 \big(1 + \frac{2 }{\lambda_C } \big)\cdot  \big(\rho_{\max} \frac{ 1+\gamma }{ \min |\lambda(C)|}\big)^2 \frac{\alpha^2}{\beta} \Big]M \Big]\EE_{m,0}\|  \tilde{z}^{(m-1)}\|^2\\
	&+ 10 (1+\gamma)^2 \rho^2 _{\max} \cdot \Big[\big(1 + \frac{\gamma \rho_{\max}}{ \min |\lambda(C)|}\big)^2 \big(1 + \frac{2 }{\lambda_C } \big)\cdot  \big(\rho_{\max} \frac{ 1+\gamma }{ \min |\lambda(C)|}\big)^2 \frac{\alpha^2}{\beta}  \\
	&+  \big( 1 + \frac{1}{ \min |\lambda(C)|} \big)^2 \beta^2  \Big] \big( \sum_{t=0}^{M-1}\EE_{m,0} \| \theta_{t}^{(m)} - \theta^\ast\|^2  + M \EE_{m,0}\| \tilde{\theta}^{(m-1)} - \theta^\ast\|^2 \big) \\ 
	&+    10K_2 \beta^2 +  10\big(\rho_{\max} \frac{ 1+\gamma }{ \min |\lambda(C)|}\big)^2  K_1 \big(1+ \frac{2  }{\lambda_C }  \big)  \frac{\alpha^2}{\beta}\Big\} \\
	&+ \alpha^2 M \cdot 5 \gamma^2\rho_{\max}^2 \EE_{m,0}\|\tilde{z}^{(m-1)}\|^2\\
	= & \Big[ 1 +  \alpha^2M \cdot 5(1+\gamma)^2\rho^2_{\max}\big(1 + \frac{\gamma \rho_{\max}}{ \min |\lambda(C)|}\big)^2 \Big] \EE_{m,0}\| \tilde{\theta}^{(m-1)} - \theta^\ast\|^2    + \alpha^2 \cdot5K_1 \\
	&+ \frac{\alpha}{\beta} \cdot  \frac{ 2 \rho^2_{\max} \gamma^2}{\lambda_{\widehat{A}}} \frac{3}{\lambda_{C}}  \Big[1+ \Big[10\beta^2 + 10\gamma^2\rho_{\max}^2 \big(1 + \frac{2 }{\lambda_C } \big)\cdot  \big(\rho_{\max} \frac{ 1+\gamma }{ \min |\lambda(C)|}\big)^2 \frac{\alpha^2}{\beta} \Big]M \Big]\EE_{m,0}\|  \tilde{z}^{(m-1)}\|^2\\
	&+ \frac{\alpha}{\beta} \cdot  \frac{ 2 \rho^2_{\max} \gamma^2}{\lambda_{\widehat{A}}} \frac{3}{\lambda_{C}}  \cdot 10 (1+\gamma)^2 \rho^2 _{\max} \cdot \Big[\big(1 + \frac{\gamma \rho_{\max}}{ \min |\lambda(C)|}\big)^2 \big(1 + \frac{2 }{\lambda_C } \big)\cdot  \big(\rho_{\max} \frac{ 1+\gamma }{ \min |\lambda(C)|}\big)^2 \frac{\alpha^2}{\beta}  \\
	&+  \big( 1 + \frac{1}{ \min |\lambda(C)|} \big)^2 \beta^2  \Big] \big( \sum_{t=0}^{M-1}\EE_{m,0} \| \theta_{t}^{(m)} - \theta^\ast\|^2  + M \EE_{m,0}\| \tilde{\theta}^{(m-1)} - \theta^\ast\|^2 \big) \\ 
	&+   \frac{\alpha}{\beta} \cdot  \frac{ 2 \rho^2_{\max} \gamma^2}{\lambda_{\widehat{A}}} \frac{3}{\lambda_{C}}  \cdot 10K_2 \beta^2 + \frac{\alpha}{\beta} \cdot  \frac{ 2 \rho^2_{\max} \gamma^2}{\lambda_{\widehat{A}}} \frac{3}{\lambda_{C}}  \cdot 10\big(\rho_{\max} \frac{ 1+\gamma }{ \min |\lambda(C)|}\big)^2  K_1 \big(1+ \frac{2  }{\lambda_C }  \big)  \frac{\alpha^2}{\beta}  \\
	&+ \alpha^2 M \cdot 5 \gamma^2\rho_{\max}^2 \EE_{m,0}\|\tilde{z}^{(m-1)}\|^2\\
	= & \Big[ 1 +  \alpha^2M \cdot 5(1+\gamma)^2\rho^2_{\max}\big(1 + \frac{\gamma \rho_{\max}}{ \min |\lambda(C)|}\big)^2 \Big] \EE_{m,0}\| \tilde{\theta}^{(m-1)} - \theta^\ast\|^2     \\
	&+ \Big\{ \frac{\alpha}{\beta} \cdot  \frac{ 2 \rho^2_{\max} \gamma^2}{\lambda_{\widehat{A}}} \frac{3}{\lambda_{C}}  \Big[1+ \Big[10\beta^2 + 10\gamma^2\rho_{\max}^2 \big(1 + \frac{2 }{\lambda_C } \big)\cdot  \big(\rho_{\max} \frac{ 1+\gamma }{ \min |\lambda(C)|}\big)^2 \frac{\alpha^2}{\beta} \Big]M \Big] \\&+ \alpha^2 M \cdot 5 \gamma^2\rho_{\max}^2   \Big\}\EE_{m,0}\|  \tilde{z}^{(m-1)}\|^2\\
	&+ \frac{\alpha}{\beta} \cdot  \frac{ 2 \rho^2_{\max} \gamma^2}{\lambda_{\widehat{A}}} \frac{3}{\lambda_{C}}  \cdot 10 (1+\gamma)^2 \rho^2 _{\max} \cdot \Big[\big(1 + \frac{\gamma \rho_{\max}}{ \min |\lambda(C)|}\big)^2 \big(1 + \frac{2 }{\lambda_C } \big)\cdot  \big(\rho_{\max} \frac{ 1+\gamma }{ \min |\lambda(C)|}\big)^2 \frac{\alpha^2}{\beta}\\&  +  \big( 1 + \frac{1}{ \min |\lambda(C)|} \big)^2 \beta^2  \Big] \big( \sum_{t=0}^{M-1}\EE_{m,0} \| \theta_{t}^{(m)} - \theta^\ast\|^2  + M \EE_{m,0}\| \tilde{\theta}^{(m-1)} - \theta^\ast\|^2 \big) \\ 
	&+   \alpha\beta  \cdot  \frac{ 60 \rho^2_{\max} \gamma^2}{\lambda_{\widehat{A}}\lambda_{C} }   \cdot K_2   + \frac{\alpha^3}{\beta^2} \cdot  \frac{ 60 \rho^2_{\max} \gamma^2}{\lambda_{\widehat{A}}\lambda_{C}}   \cdot \big(\rho_{\max} \frac{ 1+\gamma }{ \min |\lambda(C)|}\big)^2  K_1 \big(1+ \frac{2  }{\lambda_C }  \big)   + \alpha^2 \cdot5K_1,
	\end{align*} 
	where in the first equality we expand the curly bracket and in the last equality we combine and re-arrange the terms. Then, we move the term $\sum_{t=0}^{M-1} \EE_{m,0} \| \theta_{t}^{(m)} - \theta^\ast \|^2$ in the last equality to the left-hand side and obtain that
	\begin{align*}
		&\Big\{ \frac{\lambda_{\widehat{A}}}{6}\alpha - \frac{\alpha}{\beta} \cdot  \frac{ 2 \rho^2_{\max} \gamma^2}{\lambda_{\widehat{A}}} \frac{3}{\lambda_{C}}  \cdot 10 (1+\gamma)^2 \rho^2 _{\max} \cdot \Big[\big(1 + \frac{\gamma \rho_{\max}}{ \min |\lambda(C)|}\big)^2 \big(1 + \frac{2 }{\lambda_C } \big)\cdot  \big(\rho_{\max} \frac{ 1+\gamma }{ \min |\lambda(C)|}\big)^2 \frac{\alpha^2}{\beta} \\& +  \big( 1 + \frac{1}{ \min |\lambda(C)|} \big)^2 \beta^2  \Big] \Big\} \sum_{t=0}^{M-1} \EE_{m,0} \| \theta_{t}^{(m)} - \theta^\ast \|^2\\
		\leq &  \Big\{\Big[ 1 +  \alpha^2M \cdot 5(1+\gamma)^2\rho^2_{\max}\big(1 + \frac{\gamma \rho_{\max}}{ \min |\lambda(C)|}\big)^2 \Big] \\& + \frac{\alpha}{\beta} \cdot  \frac{ 2 \rho^2_{\max} \gamma^2}{\lambda_{\widehat{A}}} \frac{3}{\lambda_{C}}  \cdot 10 (1+\gamma)^2 \rho^2 _{\max} \cdot \Big[\big(1 + \frac{\gamma \rho_{\max}}{ \min |\lambda(C)|}\big)^2 \big(1 + \frac{2 }{\lambda_C } \big)\cdot  \big(\rho_{\max} \frac{ 1+\gamma }{ \min |\lambda(C)|}\big)^2 \frac{\alpha^2}{\beta} \\& +  \big( 1 + \frac{1}{ \min |\lambda(C)|} \big)^2 \beta^2  \Big]  \cdot M \Big\} \EE_{m,0}\| \tilde{\theta}^{(m-1)} - \theta^\ast\|^2     \\
		&+ \Big\{ \frac{\alpha}{\beta} \cdot  \frac{ 2 \rho^2_{\max} \gamma^2}{\lambda_{\widehat{A}}} \frac{3}{\lambda_{C}}  \Big[1+ \Big[10\beta^2 + 10\gamma^2\rho_{\max}^2 \big(1 + \frac{2 }{\lambda_C } \big)\cdot  \big(\rho_{\max} \frac{ 1+\gamma }{ \min |\lambda(C)|}\big)^2 \frac{\alpha^2}{\beta} \Big]M \Big]\\& + \alpha^2 M \cdot 5 \gamma^2\rho_{\max}^2   \Big\}\EE_{m,0}\|  \tilde{z}^{(m-1)}\|^2\\  
		&+   \alpha\beta  \cdot  \frac{ 60 \rho^2_{\max} \gamma^2}{\lambda_{\widehat{A}}\lambda_{C} }   \cdot K_2   + \frac{\alpha^3}{\beta^2} \cdot  \frac{ 60 \rho^2_{\max} \gamma^2}{\lambda_{\widehat{A}}\lambda_{C}}   \cdot \big(\rho_{\max} \frac{ 1+\gamma }{ \min |\lambda(C)|}\big)^2  K_1 \big(1+ \frac{2  }{\lambda_C }  \big)   + \alpha^2 \cdot5K_1 \numberthis \label{eq: tmp1}
	\end{align*}
	Now we define the following constants to further simplify the result above.
	\begin{itemize}
		\item $C_1 = \frac{ 2 \rho^2_{\max} \gamma^2}{\lambda_{\widehat{A}}} \frac{3}{\lambda_{C}}  \cdot 10 (1+\gamma)^2 \rho^2 _{\max} \cdot \big(1 + \frac{\gamma \rho_{\max}}{ \min |\lambda(C)|}\big)^2 \big(1 + \frac{2 }{\lambda_C } \big)\cdot  \big(\rho_{\max} \frac{ 1+\gamma }{ \min |\lambda(C)|}\big)^2$,
		\item $C_2 = \frac{ 2 \rho^2_{\max} \gamma^2}{\lambda_{\widehat{A}}} \frac{3}{\lambda_{C}}  \cdot 10 (1+\gamma)^2 \rho^2 _{\max} \cdot \big( 1 + \frac{1}{ \min |\lambda(C)|} \big)^2$.
	\end{itemize}
	Then, eq.(\ref{eq: tmp1}) can be rewritten as
	\begin{align*}
	&\Big\{ \frac{\lambda_{\widehat{A}}}{6}\alpha -\alpha \cdot \big( \frac{\alpha^2}{\beta^2} \cdot C_1  + \beta \cdot C_2\big)  \Big\} \sum_{t=0}^{M-1} \EE_{m,0} \| \theta_{t}^{(m)} - \theta^\ast \|^2\\
	\leq &  \Big\{  1 +  \alpha^2M \cdot 5(1+\gamma)^2\rho^2_{\max}\big(1 + \frac{\gamma \rho_{\max}}{ \min |\lambda(C)|}\big)^2   +  \alpha M \cdot  \big(  \frac{\alpha^2}{\beta^2} \cdot C_1  + \beta \cdot C_2 \big)    \Big\} \EE_{m,0}\| \tilde{\theta}^{(m-1)} - \theta^\ast\|^2     \\
	&+ \Big\{ \frac{\alpha}{\beta} \cdot  \frac{ 2 \rho^2_{\max} \gamma^2}{\lambda_{\widehat{A}}} \frac{3}{\lambda_{C}}  \Big[1+ \Big[10\beta^2 + 10\gamma^2\rho_{\max}^2 \big(1 + \frac{2 }{\lambda_C } \big)\cdot  \big(\rho_{\max} \frac{ 1+\gamma }{ \min |\lambda(C)|}\big)^2 \frac{\alpha^2}{\beta} \Big]M \Big] + \alpha^2 M \cdot 5 \gamma^2\rho_{\max}^2   \Big\}\EE_{m,0}\|  \tilde{z}^{(m-1)}\|^2\\  
	&+   \alpha\beta  \cdot  \frac{ 60 \rho^2_{\max} \gamma^2}{\lambda_{\widehat{A}}\lambda_{C} }   \cdot K_2   + \frac{\alpha^3}{\beta^2} \cdot  \frac{ 60 \rho^2_{\max} \gamma^2}{\lambda_{\widehat{A}}\lambda_{C}}   \cdot \big(\rho_{\max} \frac{ 1+\gamma }{ \min |\lambda(C)|}\big)^2  K_1 \big(1+ \frac{2  }{\lambda_C }  \big)   + \alpha^2 \cdot5K_1.
	\end{align*}
	Apply Lemma \ref{lemma: iid conv-z} to the inequality above and taking total expectation on both sides, we obtain that
	\begin{align*}
	&\Big\{ \frac{\lambda_{\widehat{A}}}{6}\alpha -\alpha \cdot \big( \frac{\alpha^2}{\beta^2} \cdot C_1  + \beta \cdot C_2\big)  \Big\} \sum_{t=0}^{M-1} \EE \| \theta_{t}^{(m)} - \theta^\ast \|^2\\
	\leq &  \Big\{  1 +  \alpha^2M \cdot 5(1+\gamma)^2\rho^2_{\max}\big(1 + \frac{\gamma \rho_{\max}}{ \min |\lambda(C)|}\big)^2   +  \alpha M \cdot  \big(  \frac{\alpha^2}{\beta^2} \cdot C_1  + \beta \cdot C_2 \big)    \Big\} \EE\| \tilde{\theta}^{(m-1)} - \theta^\ast\|^2     \\
	&+ \Big\{ \frac{\alpha}{\beta} \cdot  \frac{ 2 \rho^2_{\max} \gamma^2}{\lambda_{\widehat{A}}} \frac{3}{\lambda_{C}}  \Big[1+ \Big[10\beta^2 + 10\gamma^2\rho_{\max}^2 \big(1 + \frac{2 }{\lambda_C } \big)\cdot  \big(\rho_{\max} \frac{ 1+\gamma }{ \min |\lambda(C)|}\big)^2 \frac{\alpha^2}{\beta} \Big]M \Big] + \alpha^2 M \cdot 5 \gamma^2\rho_{\max}^2   \Big\}   \\
	&\times \Big\{ \big(\frac{1}{M\beta} \cdot \frac{2}{\lambda_C }\big)^m \EE \|  \tilde{z}^{(0)}\|^2 + 2 \cdot \big[ \beta \cdot \frac{24}{\lambda_C }   H_{\text{VR}}^2 + \frac{\alpha^2}{\beta^2} \cdot \big(1 + \frac{2 }{\lambda_C }  \big) \cdot \frac{2}{\lambda_C }\big(\rho_{\max} \frac{ 1+\gamma }{ \min |\lambda(C)|}\big)^2 G_{\text{VR}}^2 \big]\Big\}\\  
	&+   \alpha\beta  \cdot  \frac{ 60 \rho^2_{\max} \gamma^2}{\lambda_{\widehat{A}}\lambda_{C} }   \cdot K_2   + \frac{\alpha^3}{\beta^2} \cdot  \frac{ 60 \rho^2_{\max} \gamma^2}{\lambda_{\widehat{A}}\lambda_{C}}   \cdot \big(\rho_{\max} \frac{ 1+\gamma }{ \min |\lambda(C)|}\big)^2  K_1 \big(1+ \frac{2  }{\lambda_C }  \big)   + \alpha^2 \cdot5K_1.
	\end{align*}
	Let $ \frac{\lambda_{\widehat{A}}}{6}\alpha -\alpha \cdot \big( \frac{\alpha^2}{\beta^2} \cdot C_1  + \beta \cdot C_2\big)  \geq \frac{\lambda_{\widehat{A}}}{12}\alpha$ and divide $\frac{\lambda_{\widehat{A}}}{12}\alpha M$ on both sides of the above inequality. Then, apply Jensen's inequality to the left-hand side of the inequality above, we obtain that
	\begin{align*}
	& \EE \| \tilde{ \theta}^{(m)} - \theta^\ast \|^2\\
	\leq & \frac{12}{\lambda_{\widehat{A}}}\Big\{  \frac{1}{\alpha M} +  \alpha  \cdot 5(1+\gamma)^2\rho^2_{\max}\big(1 + \frac{\gamma \rho_{\max}}{ \min |\lambda(C)|}\big)^2   +     \frac{\alpha^2}{\beta^2} \cdot C_1  + \beta \cdot C_2   \Big\} \EE\| \tilde{\theta}^{(m-1)} - \theta^\ast\|^2     \\
	&+ \frac{12}{\lambda_{\widehat{A}}} \Big\{ \frac{1}{\beta M} \cdot  \frac{ 2 \rho^2_{\max} \gamma^2}{\lambda_{\widehat{A}}} \frac{3}{\lambda_{C}}  \Big[1+ \Big[10\beta^2 + 10\gamma^2\rho_{\max}^2 \big(1 + \frac{2 }{\lambda_C } \big)\cdot  \big(\rho_{\max} \frac{ 1+\gamma }{ \min |\lambda(C)|}\big)^2 \frac{\alpha^2}{\beta} \Big]M \Big] + \alpha  \cdot 5 \gamma^2\rho_{\max}^2   \Big\}   \\
	&\times \Big\{ \big(\frac{1}{M\beta} \cdot \frac{2}{\lambda_C }\big)^{m-1} \EE \|  \tilde{z}^{(0)}\|^2 + 2 \cdot \big[ \beta \cdot \frac{24}{\lambda_C }   H_{\text{VR}}^2 + \frac{\alpha^2}{\beta^2} \cdot \big(1 + \frac{2 }{\lambda_C }  \big) \cdot \frac{2}{\lambda_C }\big(\rho_{\max} \frac{ 1+\gamma }{ \min |\lambda(C)|}\big)^2 G_{\text{VR}}^2 \big]\Big\}\\  
	&+   \frac{\beta}{M}  \cdot  \frac{12}{\lambda_{\widehat{A}}} \frac{ 60 \rho^2_{\max} \gamma^2}{\lambda_{\widehat{A}}\lambda_{C} }   \cdot K_2   + \frac{\alpha^2}{\beta^2}\frac{1}{M} \cdot  \frac{12}{\lambda_{\widehat{A}}} \frac{ 60 \rho^2_{\max} \gamma^2}{\lambda_{\widehat{A}}\lambda_{C}}   \cdot \big(\rho_{\max} \frac{ 1+\gamma }{ \min |\lambda(C)|}\big)^2  K_1 \big(1+ \frac{2  }{\lambda_C }  \big)   + \frac{\alpha}{M} \cdot \frac{60 K_1}{\lambda_{\widehat{A}}} .
	\end{align*}
	Next, we define $D:= \frac{12}{\lambda_{\widehat{A}}}\Big\{  \frac{1}{\alpha M} +  \alpha  \cdot 5(1+\gamma)^2\rho^2_{\max}\big(1 + \frac{\gamma \rho_{\max}}{ \min |\lambda(C)|}\big)^2   +     \frac{\alpha^2}{\beta^2} \cdot C_1  + \beta \cdot C_2   \Big\}$ and $E:= \frac{1}{M\beta} \cdot \frac{2}{\lambda_C }$. Telescoping the above inequality yields that
	\begin{align*}
	& \EE \| \tilde{ \theta}^{(m)} - \theta^\ast \|^2\\
	\leq & D^m \cdot \EE\| \tilde{\theta}^{(0)} - \theta^\ast\|^2     \\
	&+ \frac{12}{\lambda_{\widehat{A}}} \Big\{ \frac{1}{\beta M} \cdot  \frac{ 2 \rho^2_{\max} \gamma^2}{\lambda_{\widehat{A}}} \frac{3}{\lambda_{C}}  \Big[1+ \Big[10\beta^2 + 10\gamma^2\rho_{\max}^2 \big(1 + \frac{2 }{\lambda_C } \big)\cdot  \big(\rho_{\max} \frac{ 1+\gamma }{ \min |\lambda(C)|}\big)^2 \frac{\alpha^2}{\beta} \Big]M \Big] + \alpha  \cdot 5 \gamma^2\rho_{\max}^2   \Big\}   \\
	&\times \Big\{ \frac{D^m - E^m}{D - E}\EE \|  \tilde{z}^{(0)}\|^2 + \frac{2}{1-D}\cdot \big[ \beta \cdot \frac{24}{\lambda_C }   H_{\text{VR}}^2 + \frac{\alpha^2}{\beta^2} \cdot \big(1 + \frac{2 }{\lambda_C }  \big) \cdot \frac{2}{\lambda_C }\big(\rho_{\max} \frac{ 1+\gamma }{ \min |\lambda(C)|}\big)^2 G_{\text{VR}}^2 \big]\Big\}\\  
	&+  \frac{1}{1-D}\Big\{  \frac{\beta}{M}  \cdot  \frac{12}{\lambda_{\widehat{A}}} \frac{ 60 \rho^2_{\max} \gamma^2}{\lambda_{\widehat{A}}\lambda_{C} }   \cdot K_2   + \frac{\alpha^2}{\beta^2}\frac{1}{M} \cdot  \frac{12}{\lambda_{\widehat{A}}} \frac{ 60 \rho^2_{\max} \gamma^2}{\lambda_{\widehat{A}}\lambda_{C}}   \cdot \big(\rho_{\max} \frac{ 1+\gamma }{ \min |\lambda(C)|}\big)^2  K_1 \big(1+ \frac{2  }{\lambda_C }  \big)   + \frac{\alpha}{M} \cdot \frac{60 K_1}{\lambda_{\widehat{A}}}  \Big\}\\
	= & D^m \cdot \EE\| \tilde{\theta}^{(0)} - \theta^\ast\|^2     \\
	&+ \frac{12}{\lambda_{\widehat{A}}} \Big\{\frac{ 2 \rho^2_{\max} \gamma^2}{\lambda_{\widehat{A}}} \frac{3}{\lambda_{C}}  \Big[  \frac{1}{\beta M}  + \beta \cdot 10 + \frac{\alpha^2}{\beta^2} \cdot 10\gamma^2\rho_{\max}^2 \big(1 + \frac{2 }{\lambda_C } \big)\cdot  \big(\rho_{\max} \frac{ 1+\gamma }{ \min |\lambda(C)|}\big)^2   \Big] + \alpha  \cdot 5 \gamma^2\rho_{\max}^2   \Big\}   \\
	&\times \Big\{ \frac{D^m - E^m}{D - E}\EE \|  \tilde{z}^{(0)}\|^2 + \frac{2}{1-D}\cdot \big[ \beta \cdot \frac{24}{\lambda_C }   H_{\text{VR}}^2 + \frac{\alpha^2}{\beta^2} \cdot \big(1 + \frac{2 }{\lambda_C }  \big) \cdot \frac{2}{\lambda_C }\big(\rho_{\max} \frac{ 1+\gamma }{ \min |\lambda(C)|}\big)^2 G_{\text{VR}}^2 \big]\Big\}\\  
	&+  \frac{1}{1-D}\Big\{  \frac{\beta}{M}  \cdot  \frac{12}{\lambda_{\widehat{A}}} \frac{ 60 \rho^2_{\max} \gamma^2}{\lambda_{\widehat{A}}\lambda_{C} }   \cdot K_2   + \frac{\alpha^2}{\beta^2}\frac{1}{M} \cdot  \frac{12}{\lambda_{\widehat{A}}} \frac{ 60 \rho^2_{\max} \gamma^2}{\lambda_{\widehat{A}}\lambda_{C}}   \cdot \big(\rho_{\max} \frac{ 1+\gamma }{ \min |\lambda(C)|}\big)^2  K_1 \big(1+ \frac{2  }{\lambda_C }  \big)   + \frac{\alpha}{M} \cdot \frac{60 K_1}{\lambda_{\widehat{A}}}  \Big\},
	\end{align*}
	where in the last equality we re-arrange and simplify the upper bound to get the desired bound.
	
\end{proof}

\begin{lemma}[Refined Bound for $\EE \|\tilde{z}^{(m)}\|^2$]\label{lemma: iid refined z}
	Under the same assumptions as those of \Cref{thm: iid}, \Cref{lemma: iid pre-bound z}, \Cref{lemma: iid pre-bound theta}, \Cref{lemma: iid theta} and \Cref{lemma: iid conv-z}, choose the learning rates $\alpha, \beta$ and the batch size $M$ such that
	\begin{align}\label{eq: iid def-F}
	F:&= \frac{4}{\lambda_{C}}\Big[\frac{1}{\beta M}+  \beta \cdot 10 + \frac{\alpha^2}{\beta^2} \cdot 10\gamma^2\rho_{\max}^2 \big(1 + \frac{2 }{\lambda_C } \big)\cdot  \big(\rho_{\max} \frac{ 1+\gamma }{ \min |\lambda(C)|}\big)^2    \nonumber\\
	&\quad\qquad+ 
	\big(  \frac{\alpha^3}{\beta^2}\cdot C_3 + \alpha \beta \cdot C_4 \big) \frac{30}{\lambda_{\widehat{A}}} \gamma^2\rho_{\max}^2  
	\Big] < 1 
	\end{align}
	and
	\begin{itemize}
		\item $  \frac{1}{\alpha M}   +  \alpha  \cdot 5(1+\gamma)^2\rho^2_{\max}\big(1 + \frac{\gamma \rho_{\max}}{ \min |\lambda(C)|}\big)^2    \leq  \frac{\lambda_{\widehat{A}}}{6}$,
		\item $    \frac{\alpha^2}{\beta^2}\cdot C_3 + \beta  \cdot C_4     \leq 
	 \frac{1-D}{144}\frac{\lambda_{\widehat{A}}^2 \lambda_{C}}{   \rho^2_{\max} \gamma^2}$,
		\item  $    \frac{\alpha^2}{\beta^2}\cdot C_3 + \beta  \cdot C_4      \leq5(1-D) $,
	\end{itemize}
	where 
	\begin{align}\label{eq: iid def-C3}
	C_3 :=  10 (1+\gamma)^2 \rho^2 _{\max} \cdot  \big(1 + \frac{\gamma \rho_{\max}}{ \min |\lambda(C)|}\big)^2 \big(1 + \frac{2 }{\lambda_C } \big)\cdot  \big(\rho_{\max} \frac{ 1+\gamma }{ \min |\lambda(C)|}\big)^2
	\end{align}
	and 
	\begin{align}\label{eq: iid def-C4}
	C_4 :=  10 (1+\gamma)^2 \rho^2 _{\max} \cdot\big( 1 + \frac{1}{ \min |\lambda(C)|} \big)^2
	\end{align} 
	and $D$ is specified in eq.(\ref{eq: iid def-D}) in \Cref{lemma: iid theta}.
	Then, the following refined bound holds. 
	\begin{align*}
	&  \EE \|\tilde{z}^{(m)}\|^2 \\
	\leq    & F^m \cdot \EE \|  \tilde{z}^{(0)}\|^2 +  \frac{D^{m} - F^m}{D - F} \cdot \EE\| \tilde{\theta}^{(0)} - \theta^\ast\|^2    + \frac{\frac{D^{m} - F^m}{D - F} - \frac{E^{m} - F^m}{E - F}}{D - E}\EE \|  \tilde{z}^{(0)}\|^2   \\
	&+   \frac{1}{1 - F}\frac{1}{1-D} \frac{1152}{\lambda_{\widehat{A}}^2}  \frac{   \rho^2_{\max} \gamma^2}{\lambda_C^2}  \big(  \frac{\alpha^2}{\beta^2}\cdot C_3 + \beta  \cdot C_4 \big)    \Big[  \frac{1}{\beta M}  + \beta \cdot 20 + \frac{\alpha^2}{\beta^2} \cdot 20\gamma^2\rho_{\max}^2 \big(1 + \frac{2 }{\lambda_C } \big)\cdot  \big(\rho_{\max} \frac{ 1+\gamma }{ \min |\lambda(C)|}\big)^2   \Big]   \\
	&\times  \big[ \beta \cdot \frac{24}{\lambda_C }   H_{\text{VR}}^2 + \frac{\alpha^2}{\beta^2} \cdot \big(1 + \frac{2 }{\lambda_C }  \big) \cdot \frac{2}{\lambda_C }\big(\rho_{\max} \frac{ 1+\gamma }{ \min |\lambda(C)|}\big)^2 G_{\text{VR}}^2 \big] \\  
	&+   \frac{1}{1 - F} \frac{1}{1-D}\Big[ \frac{\beta}{M} \cdot\big(  \frac{80}{\lambda_{C}} K_2 + C_4  \frac{600}{\lambda_{C}} \frac{K_1}{\lambda_{\widehat{A}}}  \big)  +  \frac{\alpha^2}{\beta^2}\frac{1}{M} \cdot \big(\frac{80}{\lambda_{C}}\big(\rho_{\max} \frac{ 1+\gamma }{ \min |\lambda(C)|}\big)^2  K_1 \big(1+ \frac{2  }{\lambda_C }  \big) + C_3  \frac{600}{\lambda_{C}} \frac{K_1}{\lambda_{\widehat{A}}}\big)      \Big].
	\end{align*}
	where $K_1$ is specified in eq.(\ref{eq: iid def-K1}) in Lemma \ref{lemma: iid GVR}, $K_2$ is specified in eq.(\ref{eq: iid def-K5}) in Lemma \ref{lemma: iid HVR}, $D$ and $E$ are specified in eq.(\ref{eq: iid def-D}) and eq.(\ref{eq: iid def-E}) in \Cref{lemma: iid theta}.
\end{lemma}
\begin{proof} 
	From Lemma \ref{lemma: iid pre-bound z}, we have the following inequality
	\begin{align*}
	&  \frac{\lambda_{C}}{3}\beta  \sum_{t=0}^{M-1}\EE_{m,0}\|z_t^{(m)}\|^2 \\
	\leq &\Big[1+ \Big[10\beta^2 + 10\gamma^2\rho_{\max}^2 \big(1 + \frac{2 }{\lambda_C } \big)\cdot  \big(\rho_{\max} \frac{ 1+\gamma }{ \min |\lambda(C)|}\big)^2 \frac{\alpha^2}{\beta} \Big]M \Big]\EE_{m,0}\|  \tilde{z}^{(m-1)}\|^2\\
	&+ 10 (1+\gamma)^2 \rho^2 _{\max} \cdot \Big[\big(1 + \frac{\gamma \rho_{\max}}{ \min |\lambda(C)|}\big)^2 \big(1 + \frac{2 }{\lambda_C } \big)\cdot  \big(\rho_{\max} \frac{ 1+\gamma }{ \min |\lambda(C)|}\big)^2 \frac{\alpha^2}{\beta}  \\
	&+  \big( 1 + \frac{1}{ \min |\lambda(C)|} \big)^2 \beta^2  \Big] \big( \sum_{t=0}^{M-1}\EE_{m,0} \| \theta_{t}^{(m)} - \theta^\ast\|^2  + M \EE_{m,0}\| \tilde{\theta}^{(m-1)} - \theta^\ast\|^2 \big) \\ 
	&+    10K_2 \beta^2 +  10\big(\rho_{\max} \frac{ 1+\gamma }{ \min |\lambda(C)|}\big)^2  K_1 \big(1+ \frac{2  }{\lambda_C }  \big)  \frac{\alpha^2}{\beta}.
	\end{align*}
	Note that we have already bounded $\sum_{t=0}^{M-1}\EE_{m,0} \| \theta_{t}^{(m)} - \theta^\ast\|^2$ in Lemma \ref{lemma: iid pre-bound theta}. Then, we plug the result of Lemma \ref{lemma: iid pre-bound theta} into the above inequality and obtain that 
	\begin{align*}
	&  \frac{\lambda_{C}}{3}\beta  \sum_{t=0}^{M-1}\EE_{m,0}\|z_t^{(m)}\|^2 \\
	\leq &\Big[1+ \Big[10\beta^2 + 10\gamma^2\rho_{\max}^2 \big(1 + \frac{2 }{\lambda_C } \big)\cdot  \big(\rho_{\max} \frac{ 1+\gamma }{ \min |\lambda(C)|}\big)^2 \frac{\alpha^2}{\beta} \Big]M \Big]\EE_{m,0}\|  \tilde{z}^{(m-1)}\|^2\\
	&+ 10 (1+\gamma)^2 \rho^2 _{\max} \cdot \Big[\big(1 + \frac{\gamma \rho_{\max}}{ \min |\lambda(C)|}\big)^2 \big(1 + \frac{2 }{\lambda_C } \big)\cdot  \big(\rho_{\max} \frac{ 1+\gamma }{ \min |\lambda(C)|}\big)^2 \frac{\alpha^2}{\beta}  \\
	&+  \big( 1 + \frac{1}{ \min |\lambda(C)|} \big)^2 \beta^2  \Big] \big(  \frac{6}{\lambda_{\widehat{A}}} \frac{1}{\alpha}\Big\{ \Big[ 1 +  \alpha^2M \cdot 5(1+\gamma)^2\rho^2_{\max}\big(1 + \frac{\gamma \rho_{\max}}{ \min |\lambda(C)|}\big)^2 \Big] \EE_{m,0}\| \tilde{\theta}^{(m-1)} - \theta^\ast\|^2    + \alpha^2 \cdot5K_1 \\
	&+ \alpha \cdot  \frac{ 2 \rho^2_{\max} \gamma^2}{\lambda_{\widehat{A}}}  \sum_{t=0}^{M-1}\EE_{m,0}\| z_{t}^{(m)}\|^2+ \alpha^2 M \cdot 5 \gamma^2\rho_{\max}^2 \EE_{m,0}\|\tilde{z}^{(m-1)}\|^2 \Big\} + M \EE_{m,0}\| \tilde{\theta}^{(m-1)} - \theta^\ast\|^2 \big) \\ 
	&+    10K_2 \beta^2 +  10\big(\rho_{\max} \frac{ 1+\gamma }{ \min |\lambda(C)|}\big)^2  K_1 \big(1+ \frac{2  }{\lambda_C }  \big)  \frac{\alpha^2}{\beta}.
	\end{align*} 
	Define
	$C_3 =  10 (1+\gamma)^2 \rho^2 _{\max} \cdot  \big(1 + \frac{\gamma \rho_{\max}}{ \min |\lambda(C)|}\big)^2 \big(1 + \frac{2 }{\lambda_C } \big)\cdot  \big(\rho_{\max} \frac{ 1+\gamma }{ \min |\lambda(C)|}\big)^2 $
	and 
	$C_4 =  10 (1+\gamma)^2 \rho^2 _{\max} \cdot\big( 1 + \frac{1}{ \min |\lambda(C)|} \big)^2$, then the above inequality can be re-written as
	\begin{align*}
	&  \frac{\lambda_{C}}{3}\beta  \sum_{t=0}^{M-1}\EE_{m,0}\|z_t^{(m)}\|^2 \\
	\leq &\Big[1+ \Big[10\beta^2 + 10\gamma^2\rho_{\max}^2 \big(1 + \frac{2 }{\lambda_C } \big)\cdot  \big(\rho_{\max} \frac{ 1+\gamma }{ \min |\lambda(C)|}\big)^2 \frac{\alpha^2}{\beta} \Big]M \Big]\EE_{m,0}\|  \tilde{z}^{(m-1)}\|^2\\
	&+  \big(  \frac{\alpha^2}{\beta}\cdot C_3 + \beta^2 \cdot C_4 \big)\big(  \frac{6}{\lambda_{\widehat{A}}} \frac{1}{\alpha}\Big\{ \Big[ 1 +  \alpha^2M \cdot 5(1+\gamma)^2\rho^2_{\max}\big(1 + \frac{\gamma \rho_{\max}}{ \min |\lambda(C)|}\big)^2 \Big] \EE_{m,0}\| \tilde{\theta}^{(m-1)} - \theta^\ast\|^2    + \alpha^2 \cdot5K_1 \\
	&+ \alpha \cdot  \frac{ 2 \rho^2_{\max} \gamma^2}{\lambda_{\widehat{A}}}  \sum_{t=0}^{M-1}\EE_{m,0}\| z_{t}^{(m)}\|^2+ \alpha^2 M \cdot 5 \gamma^2\rho_{\max}^2 \EE_{m,0}\|\tilde{z}^{(m-1)}\|^2 \Big\} + M \EE_{m,0}\| \tilde{\theta}^{(m-1)} - \theta^\ast\|^2 \big) \\ 
	&+    10K_2 \beta^2 +  10\big(\rho_{\max} \frac{ 1+\gamma }{ \min |\lambda(C)|}\big)^2  K_1 \big(1+ \frac{2  }{\lambda_C }  \big)  \frac{\alpha^2}{\beta}.
	\end{align*} 
	Simplifying the above inequality yields that
	\begin{align*}
	& \Big[ \frac{\lambda_{C}}{3}\beta - \big(  \frac{\alpha^2}{\beta}\cdot C_3 + \beta^2 \cdot C_4 \big) \frac{ 12 \rho^2_{\max} \gamma^2}{\lambda^2_{\widehat{A}}} \Big]\sum_{t=0}^{M-1}\EE_{m,0}\|z_t^{(m)}\|^2 \\
	\leq &\Big[1+ \Big[10\beta^2 + 10\gamma^2\rho_{\max}^2 \big(1 + \frac{2 }{\lambda_C } \big)\cdot  \big(\rho_{\max} \frac{ 1+\gamma }{ \min |\lambda(C)|}\big)^2 \frac{\alpha^2}{\beta} \Big]M + 
	\big(  \frac{\alpha^2}{\beta}\cdot C_3 + \beta^2 \cdot C_4 \big) \frac{30}{\lambda_{\widehat{A}}} \gamma^2\rho_{\max}^2 \alpha M
	\Big]\EE_{m,0}\|  \tilde{z}^{(m-1)}\|^2\\
	&+  \big(  \frac{\alpha^2}{\beta}\cdot C_3 + \beta^2 \cdot C_4 \big)\big(  \frac{6}{\lambda_{\widehat{A}}} \frac{1}{\alpha}  \Big[ 1 +  \alpha^2M \cdot 5(1+\gamma)^2\rho^2_{\max}\big(1 + \frac{\gamma \rho_{\max}}{ \min |\lambda(C)|}\big)^2 \Big]   + M  \big) \EE_{m,0}\| \tilde{\theta}^{(m-1)} - \theta^\ast\|^2 \\ 
	&+    10K_2 \beta^2 +  10\big(\rho_{\max} \frac{ 1+\gamma }{ \min |\lambda(C)|}\big)^2  K_1 \big(1+ \frac{2  }{\lambda_C }  \big)  \frac{\alpha^2}{\beta}    + \big(  \frac{\alpha^2}{\beta}\cdot C_3 + \beta^2 \cdot C_4 \big)\alpha \cdot \frac{30}{\lambda_{\widehat{A}}}K_1.
	\end{align*} 
	Let $ \frac{\lambda_{C}}{3}\beta - \big(  \frac{\alpha^2}{\beta}\cdot C_3 + \beta^2 \cdot C_4 \big) \frac{ 12 \rho^2_{\max} \gamma^2}{\lambda^2_{\widehat{A}}} \geq \frac{\lambda_{C}}{4}\beta$. Dividing $\frac{\lambda_{C}}{4}\beta M$ and taking total expectation on both sides, and applying Jensen's inequality to the left-hand side, we obtain that
	\begin{align*}
	&  \EE \|\tilde{z}^{(m)}\|^2 \\
	\leq & \frac{4}{\lambda_{C}}\Big[\frac{1}{\beta M}+  \beta \cdot 10 + \frac{\alpha^2}{\beta^2} \cdot 10\gamma^2\rho_{\max}^2 \big(1 + \frac{2 }{\lambda_C } \big)\cdot  \big(\rho_{\max} \frac{ 1+\gamma }{ \min |\lambda(C)|}\big)^2    + 
	\big(  \frac{\alpha^3}{\beta^2}\cdot C_3 + \alpha \beta \cdot C_4 \big) \frac{30}{\lambda_{\widehat{A}}} \gamma^2\rho_{\max}^2  
	\Big]\EE \|  \tilde{z}^{(m-1)}\|^2\\
	&+  \frac{4}{\lambda_{C}} \big(  \frac{\alpha^2}{\beta^2}\cdot C_3 + \beta  \cdot C_4 \big)\big(  \frac{6}{\lambda_{\widehat{A}}} \frac{1}{\alpha M}  \Big[ 1 +  \alpha^2M \cdot 5(1+\gamma)^2\rho^2_{\max}\big(1 + \frac{\gamma \rho_{\max}}{ \min |\lambda(C)|}\big)^2 \Big]   + 1  \big) \EE \| \tilde{\theta}^{(m-1)} - \theta^\ast\|^2 \\ 
	&+    \frac{40}{\lambda_{C}} K_2 \frac{\beta}{M}  +  \frac{40}{\lambda_{C}}\big(\rho_{\max} \frac{ 1+\gamma }{ \min |\lambda(C)|}\big)^2  K_1 \big(1+ \frac{2  }{\lambda_C }  \big)  \frac{\alpha^2}{\beta^2}\frac{1}{M}    + \frac{\alpha}{M } \big(  \frac{\alpha^2}{\beta^2} \cdot C_3 + \beta  \cdot C_4 \big) \cdot \frac{4}{\lambda_{C}} \frac{30}{\lambda_{\widehat{A}}}K_1.
	\end{align*} 
	Define $F:= \frac{4}{\lambda_{C}}\Big[\frac{1}{\beta M}+  \beta \cdot 10 + \frac{\alpha^2}{\beta^2} \cdot 10\gamma^2\rho_{\max}^2 \big(1 + \frac{2 }{\lambda_C } \big)\cdot  \big(\rho_{\max} \frac{ 1+\gamma }{ \min |\lambda(C)|}\big)^2    + 
	\big(  \frac{\alpha^3}{\beta^2}\cdot C_3 + \alpha \beta \cdot C_4 \big) \frac{30}{\lambda_{\widehat{A}}} \gamma^2\rho_{\max}^2  
	\Big]$. The above inequality can be simplified as
	\begin{align*}
	&  \EE \|\tilde{z}^{(m)}\|^2 \\
	\leq & F \cdot \EE \|  \tilde{z}^{(m-1)}\|^2\\
	&+  \frac{4}{\lambda_{C}} \big(  \frac{\alpha^2}{\beta^2}\cdot C_3 + \beta  \cdot C_4 \big)\big(  \frac{6}{\lambda_{\widehat{A}}} \frac{1}{\alpha M}  \Big[ 1 +  \alpha^2M \cdot 5(1+\gamma)^2\rho^2_{\max}\big(1 + \frac{\gamma \rho_{\max}}{ \min |\lambda(C)|}\big)^2 \Big]   + 1  \big) \EE \| \tilde{\theta}^{(m-1)} - \theta^\ast\|^2 \\ 
	&+    \frac{40}{\lambda_{C}} K_2 \frac{\beta}{M}  +  \frac{40}{\lambda_{C}}\big(\rho_{\max} \frac{ 1+\gamma }{ \min |\lambda(C)|}\big)^2  K_1 \big(1+ \frac{2  }{\lambda_C }  \big)  \frac{\alpha^2}{\beta^2}\frac{1}{M}    + \frac{\alpha}{M } \big(  \frac{\alpha^2}{\beta^2} \cdot C_3 + \beta  \cdot C_4 \big) \cdot \frac{4}{\lambda_{C}} \frac{30}{\lambda_{\widehat{A}}}K_1.
	\end{align*} 
	Lastly, recall that we already have the preliminary convergence bound of $\EE \| \tilde{\theta}^{(m-1)} - \theta^\ast\|^2$ in Lemma \ref{lemma: iid theta}. Apply this result to the above inequality yields that
	\begin{align*}
	&  \EE \|\tilde{z}^{(m)}\|^2 \\
	\leq & F \cdot \EE \|  \tilde{z}^{(m-1)}\|^2\\
	&+  \frac{4}{\lambda_{C}} \big(  \frac{\alpha^2}{\beta^2}\cdot C_3 + \beta  \cdot C_4 \big)\big(  \frac{6}{\lambda_{\widehat{A}}} \frac{1}{\alpha M}  \Big[ 1 +  \alpha^2M \cdot 5(1+\gamma)^2\rho^2_{\max}\big(1 + \frac{\gamma \rho_{\max}}{ \min |\lambda(C)|}\big)^2 \Big]   + 1  \big) \Big[ D^{m-1} \cdot \EE\| \tilde{\theta}^{(0)} - \theta^\ast\|^2     \\
	&+ \frac{12}{\lambda_{\widehat{A}}} \Big\{\frac{ 2 \rho^2_{\max} \gamma^2}{\lambda_{\widehat{A}}} \frac{3}{\lambda_{C}}  \Big[  \frac{1}{\beta M}  + \beta \cdot 10 + \frac{\alpha^2}{\beta^2} \cdot 10\gamma^2\rho_{\max}^2 \big(1 + \frac{2 }{\lambda_C } \big)\cdot  \big(\rho_{\max} \frac{ 1+\gamma }{ \min |\lambda(C)|}\big)^2   \Big] + \alpha  \cdot 5 \gamma^2\rho_{\max}^2   \Big\}   \\
	&\times \Big\{ \frac{D^{m-1} - E^{m-1}}{D - E}\EE \|  \tilde{z}^{(0)}\|^2 + \frac{2}{1-D}\cdot \big[ \beta \cdot \frac{24}{\lambda_C }   H_{\text{VR}}^2 + \frac{\alpha^2}{\beta^2} \cdot \big(1 + \frac{2 }{\lambda_C }  \big) \cdot \frac{2}{\lambda_C }\big(\rho_{\max} \frac{ 1+\gamma }{ \min |\lambda(C)|}\big)^2 G_{\text{VR}}^2 \big]\Big\}\\  
	&+  \frac{1}{1-D}\Big\{  \frac{\beta}{M}  \cdot  \frac{12}{\lambda_{\widehat{A}}} \frac{ 60 \rho^2_{\max} \gamma^2}{\lambda_{\widehat{A}}\lambda_{C} }   \cdot K_2   + \frac{\alpha^2}{\beta^2}\frac{1}{M} \cdot  \frac{12}{\lambda_{\widehat{A}}} \frac{ 60 \rho^2_{\max} \gamma^2}{\lambda_{\widehat{A}}\lambda_{C}}   \cdot \big(\rho_{\max} \frac{ 1+\gamma }{ \min |\lambda(C)|}\big)^2  K_1 \big(1+ \frac{2  }{\lambda_C }  \big)   + \frac{\alpha}{M} \cdot \frac{60 K_1}{\lambda_{\widehat{A}}}  \Big\} \Big] \\ 
	&+    \frac{40}{\lambda_{C}} K_2 \frac{\beta}{M}  +  \frac{40}{\lambda_{C}}\big(\rho_{\max} \frac{ 1+\gamma }{ \min |\lambda(C)|}\big)^2  K_1 \big(1+ \frac{2  }{\lambda_C }  \big)  \frac{\alpha^2}{\beta^2}\frac{1}{M}    + \frac{\alpha}{M } \big(  \frac{\alpha^2}{\beta^2} \cdot C_3 + \beta  \cdot C_4 \big) \cdot \frac{4}{\lambda_{C}} \frac{30}{\lambda_{\widehat{A}}}K_1.
	\end{align*} 
	Telescoping the above inequality, we obtain the final non-asymptotic bound of $\EE \|\tilde{z}^{(m)}\|^2$ as 
	\begin{align*}
	&  \EE \|\tilde{z}^{(m)}\|^2 \\
	\leq & F^m \cdot \EE \|  \tilde{z}^{(0)}\|^2\\
	&+  \frac{4}{\lambda_{C}} \big(  \frac{\alpha^2}{\beta^2}\cdot C_3 + \beta  \cdot C_4 \big)\big(  \frac{6}{\lambda_{\widehat{A}}} \frac{1}{\alpha M}  \Big[ 1 +  \alpha^2M \cdot 5(1+\gamma)^2\rho^2_{\max}\big(1 + \frac{\gamma \rho_{\max}}{ \min |\lambda(C)|}\big)^2 \Big]   + 1  \big) \Big[ \frac{D^{m} - F^m}{D - F} \cdot \EE\| \tilde{\theta}^{(0)} - \theta^\ast\|^2     \\
	&+ \frac{12}{\lambda_{\widehat{A}}} \Big\{\frac{ 2 \rho^2_{\max} \gamma^2}{\lambda_{\widehat{A}}} \frac{3}{\lambda_{C}}  \Big[  \frac{1}{\beta M}  + \beta \cdot 10 + \frac{\alpha^2}{\beta^2} \cdot 10\gamma^2\rho_{\max}^2 \big(1 + \frac{2 }{\lambda_C } \big)\cdot  \big(\rho_{\max} \frac{ 1+\gamma }{ \min |\lambda(C)|}\big)^2   \Big] + \alpha  \cdot 5 \gamma^2\rho_{\max}^2   \Big\}   \\
	&\times \Big\{ \frac{\frac{D^{m} - F^m}{D - F} - \frac{E^{m} - F^m}{E - F}}{D - E}\EE \|  \tilde{z}^{(0)}\|^2 + \frac{1}{1 - F}\frac{2}{1-D}\cdot \big[ \beta \cdot \frac{24}{\lambda_C }   H_{\text{VR}}^2 + \frac{\alpha^2}{\beta^2} \cdot \big(1 + \frac{2 }{\lambda_C }  \big) \cdot \frac{2}{\lambda_C }\big(\rho_{\max} \frac{ 1+\gamma }{ \min |\lambda(C)|}\big)^2 G_{\text{VR}}^2 \big]\Big\}\\  
	&+  \frac{1}{1- F} \frac{1}{1-D}\Big\{  \frac{\beta}{M}  \cdot  \frac{12}{\lambda_{\widehat{A}}} \frac{ 60 \rho^2_{\max} \gamma^2}{\lambda_{\widehat{A}}\lambda_{C} }   \cdot K_2   + \frac{\alpha^2}{\beta^2}\frac{1}{M} \cdot  \frac{12}{\lambda_{\widehat{A}}} \frac{ 60 \rho^2_{\max} \gamma^2}{\lambda_{\widehat{A}}\lambda_{C}}   \cdot \big(\rho_{\max} \frac{ 1+\gamma }{ \min |\lambda(C)|}\big)^2  K_1 \big(1+ \frac{2  }{\lambda_C }  \big)   + \frac{\alpha}{M} \cdot \frac{60 K_1}{\lambda_{\widehat{A}}}  \Big\} \Big] \\ 
	&+  \frac{1}{1 - F} \Big[ \frac{\beta}{M} \cdot \frac{40}{\lambda_{C}} K_2   +  \frac{\alpha^2}{\beta^2}\frac{1}{M} \cdot \frac{40}{\lambda_{C}}\big(\rho_{\max} \frac{ 1+\gamma }{ \min |\lambda(C)|}\big)^2  K_1 \big(1+ \frac{2  }{\lambda_C }  \big)     + \frac{\alpha}{M } \big(  \frac{\alpha^2}{\beta^2} \cdot C_3 + \beta  \cdot C_4 \big) \cdot \frac{4}{\lambda_{C}} \frac{30}{\lambda_{\widehat{A}}}K_1\Big].
	\end{align*} 
	Lastly, we make the following assumption to further simplify the inequality above. We note that these requirements for the learning rate $\alpha,\beta$ and the batch size $M$ are not necessary; they are only used for simplification. 
	\begin{itemize}
		\item $  \frac{6}{\lambda_{\widehat{A}}} \frac{1}{\alpha M}  \Big[ 1 +  \alpha^2M \cdot 5(1+\gamma)^2\rho^2_{\max}\big(1 + \frac{\gamma \rho_{\max}}{ \min |\lambda(C)|}\big)^2 \Big]   + 1   \leq 2$,
		\item $ \frac{8}{\lambda_{C}} \big(  \frac{\alpha^2}{\beta^2}\cdot C_3 + \beta  \cdot C_4 \big) \cdot  \frac{1}{1- F} \frac{1}{1-D}  \frac{\beta}{M}  \cdot  \frac{12}{\lambda_{\widehat{A}}} \frac{ 60 \rho^2_{\max} \gamma^2}{\lambda_{\widehat{A}}\lambda_{C} }   \cdot K_2 \leq \frac{1}{1 - F} \frac{\beta}{M} \cdot \frac{40}{\lambda_{C}} K_2$,
		\item  $ \frac{8}{\lambda_{C}} \big(  \frac{\alpha^2}{\beta^2}\cdot C_3 + \beta  \cdot C_4 \big) \cdot \frac{1}{1- F} \frac{1}{1-D} \frac{\alpha^2}{\beta^2}\frac{1}{M} \cdot  \frac{12}{\lambda_{\widehat{A}}} \frac{ 60 \rho^2_{\max} \gamma^2}{\lambda_{\widehat{A}}\lambda_{C}}   \cdot \big(\rho_{\max} \frac{ 1+\gamma }{ \min |\lambda(C)|}\big)^2  K_1 \big(1+ \frac{2  }{\lambda_C }  \big)  \leq \frac{1}{1 - F}   \frac{\alpha^2}{\beta^2}\frac{1}{M} \cdot \frac{40}{\lambda_{C}}\big(\rho_{\max} \frac{ 1+\gamma }{ \min |\lambda(C)|}\big)^2  K_1 \big(1+ \frac{2  }{\lambda_C }  \big)$.
	\end{itemize}
	Apply these conditions to the above inequality yields that
	\begin{align*}
	&  \EE \|\tilde{z}^{(m)}\|^2 \\
	\leq & F^m \cdot \EE \|  \tilde{z}^{(0)}\|^2 +  \frac{8}{\lambda_{C}} \big(  \frac{\alpha^2}{\beta^2}\cdot C_3 + \beta  \cdot C_4 \big)  \Big[ \frac{D^{m} - F^m}{D - F} \cdot \EE\| \tilde{\theta}^{(0)} - \theta^\ast\|^2     \Big] \\
	&+ \frac{8}{\lambda_{C}} \big(  \frac{\alpha^2}{\beta^2}\cdot C_3 + \beta  \cdot C_4 \big)   \frac{12}{\lambda_{\widehat{A}}} \Big\{\frac{ 2 \rho^2_{\max} \gamma^2}{\lambda_{\widehat{A}}} \frac{3}{\lambda_{C}}  \Big[  \frac{1}{\beta M}  + \beta \cdot 10 + \frac{\alpha^2}{\beta^2} \cdot 10\gamma^2\rho_{\max}^2 \big(1 + \frac{2 }{\lambda_C } \big)\cdot  \big(\rho_{\max} \frac{ 1+\gamma }{ \min |\lambda(C)|}\big)^2   \Big] \\
	&\quad+ \alpha  \cdot 5 \gamma^2\rho_{\max}^2   \Big\}   \\
	&\times  \Big\{   \frac{\frac{D^{m} - F^m}{D - F} - \frac{E^{m} - F^m}{E - F}}{D - E}\EE \|  \tilde{z}^{(0)}\|^2 + \frac{1}{1 - F}\frac{2}{1-D}\cdot \big[ \beta \cdot \frac{24}{\lambda_C }   H_{\text{VR}}^2 + \frac{\alpha^2}{\beta^2} \cdot \big(1 + \frac{2 }{\lambda_C }  \big) \cdot \frac{2}{\lambda_C }\big(\rho_{\max} \frac{ 1+\gamma }{ \min |\lambda(C)|}\big)^2 G_{\text{VR}}^2 \big]\Big\}\\  
	&+  \frac{\alpha}{M} \big(  \frac{\alpha^2}{\beta^2}\cdot C_3 + \beta  \cdot C_4 \big)  \cdot  \frac{1}{1- F} \big(  \frac{1}{1-D}  \frac{8}{\lambda_{C}} \frac{60 K_1}{\lambda_{\widehat{A}}}  + \frac{4}{\lambda_{C}} \frac{30}{\lambda_{\widehat{A}}}K_1\big) \\
	&\quad+  \frac{1}{1 - F} \Big[ \frac{\beta}{M} \cdot \frac{80}{\lambda_{C}} K_2   +  \frac{\alpha^2}{\beta^2}\frac{1}{M} \cdot \frac{80}{\lambda_{C}}\big(\rho_{\max} \frac{ 1+\gamma }{ \min |\lambda(C)|}\big)^2  K_1 \big(1+ \frac{2  }{\lambda_C }  \big)    \Big].
	\end{align*} 
	Further note that $1 < \frac{1}{1-D}$ and $\alpha \leq 1$, the above inequality implies that
	\begin{align*}
	&  \EE \|\tilde{z}^{(m)}\|^2 \\
	\leq & F^m \cdot \EE \|  \tilde{z}^{(0)}\|^2 +  \frac{8}{\lambda_{C}} \big(  \frac{\alpha^2}{\beta^2}\cdot C_3 + \beta  \cdot C_4 \big)  \Big[ \frac{D^{m} - F^m}{D - F} \cdot \EE\| \tilde{\theta}^{(0)} - \theta^\ast\|^2     \Big] \\
	&+ \frac{8}{\lambda_{C}} \big(  \frac{\alpha^2}{\beta^2}\cdot C_3 + \beta  \cdot C_4 \big)   \frac{12}{\lambda_{\widehat{A}}} \Big\{\frac{ 2 \rho^2_{\max} \gamma^2}{\lambda_{\widehat{A}}} \frac{3}{\lambda_{C}}  \Big[  \frac{1}{\beta M}  + \beta \cdot 10 + \frac{\alpha^2}{\beta^2} \cdot 10\gamma^2\rho_{\max}^2 \big(1 + \frac{2 }{\lambda_C } \big)\cdot  \big(\rho_{\max} \frac{ 1+\gamma }{ \min |\lambda(C)|}\big)^2   \Big] \\
	&\quad+ \alpha  \cdot 5 \gamma^2\rho_{\max}^2   \Big\}   \\
	&\times  \Big\{   \frac{\frac{D^{m} - F^m}{D - F} - \frac{E^{m} - F^m}{E - F}}{D - E}\EE \|  \tilde{z}^{(0)}\|^2 + \frac{1}{1 - F}\frac{2}{1-D}\cdot \big[ \beta \cdot \frac{24}{\lambda_C }   H_{\text{VR}}^2 + \frac{\alpha^2}{\beta^2} \cdot \big(1 + \frac{2 }{\lambda_C }  \big) \cdot \frac{2}{\lambda_C }\big(\rho_{\max} \frac{ 1+\gamma }{ \min |\lambda(C)|}\big)^2 G_{\text{VR}}^2 \big]\Big\}\\  
	&+   \frac{1}{1 - F} \frac{1}{1-D}\Big[ \frac{\beta}{M} \cdot\big(  \frac{80}{\lambda_{C}} K_2 + C_4  \frac{600}{\lambda_{C}} \frac{K_1}{\lambda_{\widehat{A}}}  \big)  +  \frac{\alpha^2}{\beta^2}\frac{1}{M} \cdot \big(\frac{80}{\lambda_{C}}\big(\rho_{\max} \frac{ 1+\gamma }{ \min |\lambda(C)|}\big)^2  K_1 \big(1+ \frac{2  }{\lambda_C }  \big) + C_3  \frac{600}{\lambda_{C}} \frac{K_1}{\lambda_{\widehat{A}}}\big)      \Big].
	\end{align*} 
	Assume that $\alpha  \cdot 5 \gamma^2\rho_{\max}^2 \leq  \frac{ 2 \rho^2_{\max} \gamma^2}{\lambda_{\widehat{A}}} \frac{3}{\lambda_{C}} \Big[   \beta \cdot 10 + \frac{\alpha^2}{\beta^2} \cdot 10\gamma^2\rho_{\max}^2 \big(1 + \frac{2 }{\lambda_C } \big)\cdot  \big(\rho_{\max} \frac{ 1+\gamma }{ \min |\lambda(C)|}\big)^2  \Big]$. Then, we further obtain from the above inequality that 
	\begin{align*}
	&  \EE \|\tilde{z}^{(m)}\|^2 \\
	\leq & F^m \cdot \EE \|  \tilde{z}^{(0)}\|^2 +  \frac{8}{\lambda_{C}} \big(  \frac{\alpha^2}{\beta^2}\cdot C_3 + \beta  \cdot C_4 \big)  \Big[ \frac{D^{m} - F^m}{D - F} \cdot \EE\| \tilde{\theta}^{(0)} - \theta^\ast\|^2     \Big] \\
	&+ \frac{8}{\lambda_{C}} \big(  \frac{\alpha^2}{\beta^2}\cdot C_3 + \beta  \cdot C_4 \big)   \frac{12}{\lambda_{\widehat{A}}}  \frac{ 2 \rho^2_{\max} \gamma^2}{\lambda_{\widehat{A}}} \frac{3}{\lambda_{C}}  \Big[  \frac{1}{\beta M}  + \beta \cdot 20 + \frac{\alpha^2}{\beta^2} \cdot 20\gamma^2\rho_{\max}^2 \big(1 + \frac{2 }{\lambda_C } \big)\cdot  \big(\rho_{\max} \frac{ 1+\gamma }{ \min |\lambda(C)|}\big)^2   \Big]   \\
	&\times  \Big\{   \frac{\frac{D^{m} - F^m}{D - F} - \frac{E^{m} - F^m}{E - F}}{D - E}\EE \|  \tilde{z}^{(0)}\|^2 + \frac{1}{1 - F}\frac{2}{1-D}\cdot \big[ \beta \cdot \frac{24}{\lambda_C }   H_{\text{VR}}^2 + \frac{\alpha^2}{\beta^2} \cdot \big(1 + \frac{2 }{\lambda_C }  \big) \cdot \frac{2}{\lambda_C }\big(\rho_{\max} \frac{ 1+\gamma }{ \min |\lambda(C)|}\big)^2 G_{\text{VR}}^2 \big]\Big\}\\  
	&+   \frac{1}{1 - F} \frac{1}{1-D}\Big[ \frac{\beta}{M} \cdot\big(  \frac{80}{\lambda_{C}} K_2 + C_4  \frac{600}{\lambda_{C}} \frac{K_1}{\lambda_{\widehat{A}}}  \big)  +  \frac{\alpha^2}{\beta^2}\frac{1}{M} \cdot \big(\frac{80}{\lambda_{C}}\big(\rho_{\max} \frac{ 1+\gamma }{ \min |\lambda(C)|}\big)^2  K_1 \big(1+ \frac{2  }{\lambda_C }  \big) + C_3  \frac{600}{\lambda_{C}} \frac{K_1}{\lambda_{\widehat{A}}}\big)      \Big]\\
	=   & F^m \cdot \EE \|  \tilde{z}^{(0)}\|^2 +  \frac{8}{\lambda_{C}} \big(  \frac{\alpha^2}{\beta^2}\cdot C_3 + \beta  \cdot C_4 \big)  \Big[ \frac{D^{m} - F^m}{D - F} \cdot \EE\| \tilde{\theta}^{(0)} - \theta^\ast\|^2     \Big] \\
	&+    \frac{576}{\lambda_{\widehat{A}}^2}  \frac{   \rho^2_{\max} \gamma^2}{\lambda_C^2}  \big(  \frac{\alpha^2}{\beta^2}\cdot C_3 + \beta  \cdot C_4 \big)    \Big[  \frac{1}{\beta M}  + \beta \cdot 20 + \frac{\alpha^2}{\beta^2} \cdot 20\gamma^2\rho_{\max}^2 \big(1 + \frac{2 }{\lambda_C } \big)\cdot  \big(\rho_{\max} \frac{ 1+\gamma }{ \min |\lambda(C)|}\big)^2   \Big]   \\
	&\times  \Big\{   \frac{\frac{D^{m} - F^m}{D - F} - \frac{E^{m} - F^m}{E - F}}{D - E}\EE \|  \tilde{z}^{(0)}\|^2 + \frac{1}{1 - F}\frac{2}{1-D}\cdot \big[ \beta \cdot \frac{24}{\lambda_C }   H_{\text{VR}}^2 + \frac{\alpha^2}{\beta^2} \cdot \big(1 + \frac{2 }{\lambda_C }  \big) \cdot \frac{2}{\lambda_C }\big(\rho_{\max} \frac{ 1+\gamma }{ \min |\lambda(C)|}\big)^2 G_{\text{VR}}^2 \big]\Big\}\\  
	&+   \frac{1}{1 - F} \frac{1}{1-D}\Big[ \frac{\beta}{M} \cdot\big(  \frac{80}{\lambda_{C}} K_2 + C_4  \frac{600}{\lambda_{C}} \frac{K_1}{\lambda_{\widehat{A}}}  \big)  +  \frac{\alpha^2}{\beta^2}\frac{1}{M} \cdot \big(\frac{80}{\lambda_{C}}\big(\rho_{\max} \frac{ 1+\gamma }{ \min |\lambda(C)|}\big)^2  K_1 \big(1+ \frac{2  }{\lambda_C }  \big) + C_3  \frac{600}{\lambda_{C}} \frac{K_1}{\lambda_{\widehat{A}}}\big)      \Big].
	\end{align*} 
	Lastly, assume that $ \frac{576}{\lambda_{\widehat{A}}^2}  \frac{   \rho^2_{\max} \gamma^2}{\lambda_C^2}  \big(  \frac{\alpha^2}{\beta^2}\cdot C_3 + \beta  \cdot C_4 \big)    \Big[  \frac{1}{\beta M}  + \beta \cdot 20 + \frac{\alpha^2}{\beta^2} \cdot 20\gamma^2\rho_{\max}^2 \big(1 + \frac{2 }{\lambda_C } \big)\cdot  \big(\rho_{\max} \frac{ 1+\gamma }{ \min |\lambda(C)|}\big)^2   \Big] < 1$ and $\frac{8}{\lambda_{C}} \big(  \frac{\alpha^2}{\beta^2}\cdot C_3 + \beta  \cdot C_4 \big)  < 1$, the above inequality further implies that
	
	\begin{align*}
	&  \EE \|\tilde{z}^{(m)}\|^2 \\
	\leq    & F^m \cdot \EE \|  \tilde{z}^{(0)}\|^2 +  \frac{D^{m} - F^m}{D - F} \cdot \EE\| \tilde{\theta}^{(0)} - \theta^\ast\|^2    + \frac{\frac{D^{m} - F^m}{D - F} - \frac{E^{m} - F^m}{E - F}}{D - E}\EE \|  \tilde{z}^{(0)}\|^2   \\
	&+   \frac{1}{1 - F}\frac{1}{1-D} \frac{1152}{\lambda_{\widehat{A}}^2}  \frac{   \rho^2_{\max} \gamma^2}{\lambda_C^2}  \big(  \frac{\alpha^2}{\beta^2}\cdot C_3 + \beta  \cdot C_4 \big)    \Big[  \frac{1}{\beta M}  + \beta \cdot 20 + \frac{\alpha^2}{\beta^2} \cdot 20\gamma^2\rho_{\max}^2 \big(1 + \frac{2 }{\lambda_C } \big)\cdot  \big(\rho_{\max} \frac{ 1+\gamma }{ \min |\lambda(C)|}\big)^2   \Big]   \\
	&\times  \big[ \beta \cdot \frac{24}{\lambda_C }   H_{\text{VR}}^2 + \frac{\alpha^2}{\beta^2} \cdot \big(1 + \frac{2 }{\lambda_C }  \big) \cdot \frac{2}{\lambda_C }\big(\rho_{\max} \frac{ 1+\gamma }{ \min |\lambda(C)|}\big)^2 G_{\text{VR}}^2 \big] \\  
	&+   \frac{1}{1 - F} \frac{1}{1-D}\Big[ \frac{\beta}{M} \cdot\big(  \frac{80}{\lambda_{C}} K_2 + C_4  \frac{600}{\lambda_{C}} \frac{K_1}{\lambda_{\widehat{A}}}  \big)  +  \frac{\alpha^2}{\beta^2}\frac{1}{M} \cdot \big(\frac{80}{\lambda_{C}}\big(\rho_{\max} \frac{ 1+\gamma }{ \min |\lambda(C)|}\big)^2  K_1 \big(1+ \frac{2  }{\lambda_C }  \big) + C_3  \frac{600}{\lambda_{C}} \frac{K_1}{\lambda_{\widehat{A}}}\big)      \Big].
	\end{align*}  
\end{proof}

\section{Other Supporting Lemmas for Proving \Cref{thm: iid}}
\begin{lemma}[One-Step Update of $\theta_{t}^{(m)}$] \label{lemma: iid GVR}
	Under the same assumptions as those of Theorem \ref{thm: iid}, the square norm of one-step update of $\theta_{t}^{(m)}$ in Algorithm \ref{alg: iid} can be bounded as 
	\begin{align*}
	&  	\EE_{m,0}\| G_{t}^{(m)}(\theta_{t}^{(m)}, z_{t}^{(m)})  - G_{t}^{(m)}(\tilde{\theta}^{(m-1)}, \tilde{z}^{(m-1)})  +  {G}^{(m)}(\tilde{\theta}^{(m-1)}, \tilde{z}^{(m-1)}) \|^2 \\
	\leq & 5(1+\gamma)^2\rho^2_{\max}\big(1 + \frac{\gamma \rho_{\max}}{ \min |\lambda(C)|}\big)^2\big( \EE_{m,0} \| \theta_{t}^{(m)} - \theta^\ast\|^2  + \EE_{m,0}\| \tilde{\theta}^{(m-1)} - \theta^\ast\|^2 \big) \\
	& + 5 \gamma^2\rho_{\max}^2 \big(\EE_{m,0}\| z_{t}^{(m)}\|^2  + \EE_{m,0}\|\tilde{z}^{(m-1)}\|^2\big)   +    \frac{1}{M} \cdot   5 K_1,
	\end{align*}
	where
	\begin{align}\label{eq: iid def-K1}
		K_1 :=\big[ (1+\gamma)R_\theta + r_{\max} \big]^2\rho^2_{\max}\big(1 + \frac{\gamma \rho_{\max}}{ \min |\lambda(C)|}\big)^2 .
	\end{align}
\end{lemma}
\begin{proof} 
	Substituting the definitions of $G_{t}^{(m)}(\cdot)$ and ${G}^{(m)}(\cdot)$ into the update of $\theta_{t}^{(m)}$ yields that
 \begin{align*}
 	 & \| G_{t}^{(m)}(\theta_{t}^{(m)}, z_{t}^{(m)})  - G_{t}^{(m)}(\tilde{\theta}^{(m-1)}, \tilde{z}^{(m-1)})  +  {G}^{(m)}(\tilde{\theta}^{(m-1)}, \tilde{z}^{(m-1)}) \|^2 \\
 	= & \| \widehat{A}_{t}^{(m)}\theta_{t}^{(m)} + \widehat{b}_{t}^{(m)}+ B_{t}^{(m)} z_{t}^{(m)} - \widehat{A}_{t}^{(m)}\tilde{\theta}^{(m-1)} - \widehat{b}_{t}^{(m)}- B_{t}^{(m)} \tilde{z}^{(m-1)}  + \widehat{A}^{(m)}\tilde{\theta}^{(m-1)} + \widehat{b}^{(m)}+ B^{(m)} \tilde{z}^{(m-1)} \|^2\\
 	= & \| \big(\widehat{A}_{t}^{(m)}\theta_{t}^{(m)} - \widehat{A}_{t}^{(m)} \theta^\ast\big) + \big(\widehat{A}_{t}^{(m)} \theta^\ast - \widehat{A}_{t}^{(m)}\tilde{\theta}^{(m-1)}+ \widehat{A}^{(m)}\tilde{\theta}^{(m-1)} -  \widehat{A}^{(m)}\theta^\ast\big) +  \big(\widehat{A}^{(m)}\theta^\ast  + \widehat{b}^{(m)}\big) \\
 	\quad &  + B_{t}^{(m)} z_{t}^{(m)}  - B_{t}^{(m)} \tilde{z}^{(m-1)}+ B^{(m)} \tilde{z}^{(m-1)} \|^2\\
 	\leq &  5 \|\widehat{A}_{t}^{(m)} \|^2 \| \theta_{t}^{(m)} - \theta^\ast\|^2 + 5\| \big( \widehat{A}_{t}^{(m)}\tilde{\theta}^{(m-1)} - \widehat{A}_{t}^{(m)} \theta^\ast \big)- \big(\widehat{A}^{(m)}\tilde{\theta}^{(m-1)} -  \widehat{A}^{(m)}\theta^\ast\big) \|^2 + 5 \|\widehat{A}^{(m)}\theta^\ast  + \widehat{b}^{(m)} \|^2 \\
 	\quad & + 5\|B_{t}^{(m)}\|^2 \| z_{t}^{(m)}\|^2 + 5 \| B_{t}^{(m)} \tilde{z}^{(m-1)} -  B^{(m)} \tilde{z}^{(m-1)}\|^2,  \numberthis\label{eq: iid theta-one-step-update}
 \end{align*}
 where the last inequality uses Jensen's inequality $\| \sum_{i=1}^n a_i \|^2 = \|  \frac{1}{n}\sum_{i=1}^n  (n a_i) \|^2   \leq n \sum_{i=1}^n \| a_i \|^2$. Next, we bound the third term of the right hand side of the above inequality as follows:
 \begin{align*}
 &\|\widehat{A}^{(m)}\theta^\ast  + \widehat{b}^{(m)} \|^2 \\
 &\leq  \|\sum_{t=0}^{M-1}\widehat{A}^{(m)}_t\theta^\ast  + \sum_{t=0}^{M-1}\widehat{b}_{t}^{(m)} \|^2 \\
 &= \frac{1}{M^2} \big[  \sum_{i=j} \| \widehat{A}_i^{(m)}\theta^\ast  + \widehat{b}_i^{(m)}\|^2  + \sum_{i\neq j} \langle \widehat{A}_i^{(m)}\theta^\ast  + \widehat{b}_i^{(m)}, \widehat{A}_j^{(m)}\theta^\ast  + \widehat{b}_j^{(m)}\rangle \big] \\
 &\leq \frac{1}{M} \cdot \big[ (1+\gamma)R_\theta + r_{\max} \big]^2\rho^2_{\max}\big(1 + \frac{\gamma \rho_{\max}}{ \min |\lambda(C)|}\big)^2 + \frac{1}{M^2 }\sum_{i\neq j} \langle \widehat{A}_i^{(m)}\theta^\ast  + \widehat{b}_i^{(m)}, \widehat{A}_j^{(m)}\theta^\ast  + \widehat{b}_j^{(m)}\rangle, \numberthis\label{eq: iid theta-one-step-update 0}
 \end{align*}
 where in the last inequality we use Lemma \ref{lemma: const-2} to bound $\| \widehat{A}_i^{(m)}\theta^\ast  + \widehat{b}_i^{(m)}\|$. Next, consider the conditional expectation of the second term of the above inequality, and, without loss of generality, assume $i<j$, we obtain that  
 \begin{align*}
 &\EE_{m,0} \langle \widehat{A}_i^{(m)}\theta^\ast  + \widehat{b}_i^{(m)}, \widehat{A}_j^{(m)}\theta^\ast  + \widehat{b}_j^{(m)}\rangle \\
 =&\EE_{m,0}\langle \widehat{A}_i^{(m)}\theta^\ast  + \widehat{b}_i^{(m)}, \EE_{m,i}\big[\widehat{A}_j^{(m)}\theta^\ast  + \widehat{b}_j^{(m)}\big]\rangle \\
 =&\EE_{m,0}\langle \widehat{A}_i^{(m)}\theta^\ast  + \widehat{b}_i^{(m)},   \widehat{A} \theta^\ast  + \widehat{b} \rangle \\
 =& 0,
 \end{align*}
 which follows from the i.i.d.\ sampling scheme. Substituting the above result into (\ref{eq: iid theta-one-step-update 0}), we obtain that
 \begin{align}
 \EE_{m,0}	\|\widehat{A}^{(m)}\theta^\ast  + \widehat{b}^{(m)} \|^2 &\leq \frac{1}{M} \cdot \big[ (1+\gamma)R_\theta + r_{\max} \big]^2\rho^2_{\max}\big(1 + \frac{\gamma \rho_{\max}}{ \min |\lambda(C)|}\big)^2. \numberthis\label{eq: iid theta-one-step-update 1}
 \end{align}
  On the other hand, note that $\Var X  \leq \EE X^2$, we have 
 \begin{align*}
 	&\EE_{m,t-1} \| \big( \widehat{A}_{t}^{(m)}\tilde{\theta}^{(m-1)} - \widehat{A}_{t}^{(m)} \theta^\ast \big)- \big(\widehat{A}^{(m)}\tilde{\theta}^{(m-1)} -  \widehat{A}^{(m)}\theta^\ast\big) \|^2 \\
 	=& \Var_{m, t-1}\big( \widehat{A}_{t}^{(m)}\tilde{\theta}^{(m-1)} - \widehat{A}_{t}^{(m)} \theta^\ast \big) \\
 	\leq& \EE_{m,t-1} \big( \widehat{A}_{t}^{(m)}\tilde{\theta}^{(m-1)} - \widehat{A}_{t}^{(m)} \theta^\ast \big)^2 \\
 	\leq& \EE_{m,t-1} \| \widehat{A}_{t}^{(m)}\|^2 \| \tilde{\theta}^{(m-1)} - \theta^\ast\|^2 \numberthis\label{eq: iid theta-one-step-update 2}
  \end{align*} 
 and similarly,
 \begin{align} \label{eq: iid theta-one-step-update 3}
 	\EE_{m,t-1} \| B_{t}^{(m)} \tilde{z}^{(m-1)} -  B^{(m)} \tilde{z}^{(m-1)}\|^2  \leq  \EE_{m,t-1} \| B_{t}^{(m)} \|^2 \|\tilde{z}^{(m-1)}\|^2.
 \end{align}
 Substituting eqs. (\ref{eq: iid theta-one-step-update 1}), (\ref{eq: iid theta-one-step-update 2}), (\ref{eq: iid theta-one-step-update 3}) into (\ref{eq: iid theta-one-step-update}) yields that
 \begin{align*}
  &  	\EE_{m,0}\| G_{t}^{(m)}(\theta_{t}^{(m)}, z_{t}^{(m)})  - G_{t}^{(m)}(\tilde{\theta}^{(m-1)}, \tilde{z}^{(m-1)})  +  {G}^{(m)}(\tilde{\theta}^{(m-1)}, \tilde{z}^{(m-1)}) \|^2 \\
\leq & 5(1+\gamma)^2\rho^2_{\max}\big(1 + \frac{\gamma \rho_{\max}}{ \min |\lambda(C)|}\big)^2\big( \EE_{m,0} \| \theta_{t}^{(m)} - \theta^\ast\|^2  + \EE_{m,0}\| \tilde{\theta}^{(m-1)} - \theta^\ast\|^2 \big) \\
& + 5 \gamma^2\rho_{\max}^2 \big(\EE_{m,0}\| z_{t}^{(m)}\|^2  + \EE_{m,0}\|\tilde{z}^{(m-1)}\|^2\big)   +    \frac{1}{M} \cdot   5 \big[ (1+\gamma)R_\theta + r_{\max} \big]^2\rho^2_{\max}\big(1 + \frac{\gamma \rho_{\max}}{ \min |\lambda(C)|}\big)^2. 
 \end{align*}
\end{proof}

\begin{lemma}[One-Step Update of $z_{t}^{(m)}$] \label{lemma: iid HVR}
	Under the same assumptions as those of Theorem \ref{thm: iid}, the square norm of one-step update of $z_{t}^{(m)}$ in Algorithm \ref{alg: iid} is bounded as 
\begin{align*}
& \EE_{m,0}\| H_{t}^{(m)}(\theta_{t}^{(m)}, z_{t}^{(m)})  - H_{t}^{(m)}(\tilde{\theta}^{(m-1)}, \tilde{z}^{(m-1)})  +  {H}^{(m)}(\tilde{\theta}^{(m-1)}, \tilde{z}^{(m-1)}) \|^2 \\
\leq &  5 (1+\gamma)^2 \rho^2 _{\max}\big( 1 + \frac{1}{ \min |\lambda(C)|} \big)^2\big( \EE_{m,0}\| \theta_{t}^{(m)} - \theta^\ast\|^2 + \EE_{m,0}\| \tilde{\theta}^{(m-1)} - \theta^\ast\|^2  \big) \\
& + 5\big(   \EE_{m,0}\| z_{t}^{(m)}\|^2 + \EE_{m,0}\| \tilde{z}^{(m-1)}\|^2\big)  + \frac{5K_2}{M},
\end{align*}
where 
\begin{align}\label{eq: iid def-K5}
	K_2 := \big[  \big(  1 + \gamma\big)R_\theta  + r_{\max}\big]^2\big(  1 + \frac{1}{\min|\lambda_{C}|} \big)^2.
\end{align}
\end{lemma}
\begin{proof}
	Follow a similar proof logic as that of Lemma \ref{lemma: iid GVR}, we obtain the following bound for the  square norm of the one-step update of $z_{t}^{(m)}$,
	\begin{align*}
		& \| H_{t}^{(m)}(\theta_{t}^{(m)}, z_{t}^{(m)})  - H_{t}^{(m)}(\tilde{\theta}^{(m-1)}, \tilde{z}^{(m-1)})  +  {H}^{(m)}(\tilde{\theta}^{(m-1)}, \tilde{z}^{(m-1)}) \|^2 \\
		\leq &  5 \|\bar{A}_{t}^{(m)} \|^2 \| \theta_{t}^{(m)} - \theta^\ast\|^2 + 5\| \big( \bar{A}_{t}^{(m)}\tilde{\theta}^{(m-1)} - \bar{A}_{t}^{(m)} \theta^\ast \big)- \big(\bar{A}^{(m)}\tilde{\theta}^{(m-1)} -  \bar{A}^{(m)}\theta^\ast\big) \|^2 + 5 \|\bar{A}^{(m)}\theta^\ast  + \bar{b}^{(m)} \|^2 \\
		\quad & + 5\|C_{t}^{(m)}\|^2 \| z_{t}^{(m)}\|^2 + 5 \| C_{t}^{(m)} \tilde{z}^{(m-1)} -  C^{(m)} \tilde{z}^{(m-1)}\|^2 . \numberthis \label{eq: iid z-one-step-update}
	\end{align*}
	Then, we take $\EE_{m,0}$ on both sides, follow the same steps in the proof of Lemma \ref{lemma: iid GVR} and notice that $\EE_{m,0} \|\bar{A}^{(m)}\theta^\ast  + \bar{b}^{(m)} \|^2$ in \cref{eq: iid z-one-step-update} is bounded by
	\begin{align*}
	\EE_{m,0} \|\bar{A}^{(m)}\theta^\ast  + \bar{b}^{(m)} \|^2 &\leq \frac{1}{M^2} \sum_{i=j} \|\bar{A}_i^{(m)}\theta^\ast  + \bar{b}_i^{(m)} \|^2\\
	&\leq \frac{1}{M} \cdot \big[  \big(  1 + \gamma\big)R_\theta  + r_{\max}\big]^2\big(  1 + \frac{1}{\min|\lambda_{C}|} \big)^2
	\end{align*}
	using Lemma \ref{lemma: const-3}. Finally, we obtain that
	\begin{align*}
	& \EE_{m,0}\| H_{t}^{(m)}(\theta_{t}^{(m)}, z_{t}^{(m)})  - H_{t}^{(m)}(\tilde{\theta}^{(m-1)}, \tilde{z}^{(m-1)})  +  {H}^{(m)}(\tilde{\theta}^{(m-1)}, \tilde{z}^{(m-1)}) \|^2 \\
	\leq &  5 (1+\gamma)^2 \rho^2 _{\max}\big( 1 + \frac{1}{ \min |\lambda(C)|} \big)^2\big( \EE_{m,0}\| \theta_{t}^{(m)} - \theta^\ast\|^2 + \EE_{m,0}\| \tilde{\theta}^{(m-1)} - \theta^\ast\|^2  \big) \\
	& + 5\big(   \EE_{m,0}\| z_{t}^{(m)}\|^2 + \EE_{m,0}\| \tilde{z}^{(m-1)}\|^2\big)  + \frac{5}{M}\cdot \big[  \big(  1 + \gamma\big)R_\theta  + r_{\max}\big]^2\big(  1 + \frac{1}{\min|\lambda_{C}|} \big)^2.
	\end{align*} 
\end{proof}

\section{Proof of Theorem \ref{thm: markov}} \label{appendix: markov proof}
We assume the learning rates $\alpha,\beta$ and the batch size $M$ satisfy the following conditions.
\begin{align}
	&	\alpha \leq \min \Big\{  \frac{\lambda_{\widehat{A}}}{30} / \Big[ (1+\gamma)^2\rho^2_{\max}\Big(1 + \frac{\gamma \rho_{\max}}{ \min |\lambda(C)|}\Big)^2 \Big], \frac{3}{5}\frac{1}{\lambda_{\widehat{A}}  } \Big\},\\
	&\beta \leq 1,  \label{eq: lr markov 1}\\
	&M\beta > \frac{12}{\lambda_{C}},  \label{eq: lr markov 2}\\
	&\frac{\lambda_{C}}{48}\beta - 10\beta^2 - 10\gamma^2\rho_{\max}^2   \Big(\rho_{\max} \frac{ 1+\gamma }{ \min |\lambda(C)|}\Big)^2  \Big(\alpha^2 + \frac{2 \alpha^2}{\lambda_C } \frac{1}{\beta} \Big) \geq 0, \label{eq: lr markov 3}\\
	& \frac{16}{\lambda_{\widehat{A}}} \Big\{    \frac{96}{\lambda_{\widehat{A}} \lambda_{C}}\gamma^2 \rho_{\max}^2   \Big[  \frac{1}{\beta M}+ 10\beta + 10\gamma^2\rho_{\max}^2 \Big(1+ \frac{2  }{\lambda_C }  \Big) \Big(\rho_{\max} \frac{ 1+\gamma }{ \min |\lambda(C)|}\Big)^2 \frac{\alpha^2}{\beta^2}    \Big] +  5 \gamma^2\rho_{\max}^2   \alpha \Big\} \leq 1, \label{eq: lr markov 4}\\
	&\Big(1 + \frac{\gamma \rho_{\max}}{ \min |\lambda(C)|}\Big)^2 \Big(1 + \frac{2 1}{\lambda_C }  \Big)  \Big(\rho_{\max} \frac{ 1+\gamma }{ \min |\lambda(C)|}\Big)^2 \frac{\alpha^2}{\beta^2}  +  \Big( 1 + \frac{1}{ \min |\lambda(C)|} \Big)^2 \beta^2  \nonumber \\
	&\quad \leq \min\Big\{ \frac{\lambda_{\widehat{A}}}{48}/\big[\frac{96}{\lambda_{\widehat{A}} \lambda_{C}}\gamma^2 \rho_{\max}^2  \cdot 10 (1+\gamma)^2 \rho^2 _{\max}\big],  \frac{\lambda_{C}}{48}/\big[120 (1+\gamma)^2 \rho^2 _{\max} \frac{1}{\lambda_{\widehat{A}}}\big] \Big\}, \label{eq: lr markov 5}\\
	&\max\{D, E, F\} <   1  ,  \label{eq: lr markov 6}
\end{align} 
where $D,E,F$ are specified in eq.(\ref{eq: def-D}), eq.(\ref{eq: def-E}), and eq.(\ref{eq: def-F}), respectively. We note that under the above conditions, all the supporting lemmas for proving the theorem are satisfied. We also note that for a sufficiently small $\epsilon$, our choices of learning rates and batch size $\alpha = \calO({\epsilon^{\frac{3}{4}}}),\beta = \calO({\epsilon^{\frac{1}{2}}})$, $M=\calO({\epsilon}^{-1})$ that are stated in the theorem satisfy \cref{eq: lr markov 1,eq: lr markov 2,eq: lr markov 3,eq: lr markov 4,eq: lr markov 5,eq: lr markov 6}.

\paragraph{Proof Sketch} The proof consists of the following key steps. 
\begin{enumerate}
	\item Develop \textit{preliminary bound for  $\sum_{t=0}^{M-1} \| \theta_{t}^{(m)} - \theta^\ast \|^2$}. (Lemma \ref{lemma: pre-bound theta})
	
	We first bound $\sum_{t=0}^{M-1}\| \theta_{t}^{(m)}-\theta^\ast\|^2$ in terms of $\sum_{t=0}^{M-1} \|z_{t}^{(m)}\|^2$, $\|\tilde{z}^{(m-1)}\|^2$, and $\|\tilde{\theta}^{(m-1)}-\theta^\ast\|^2$. 
	
	\item Develop \textit{preliminary bound for  $\sum_{t=0}^{M-1} \| z_{t}^{(m)}\|^2$}.  (Lemma \ref{lemma: pre-bound z})
	
	Then we bound $\sum_{t=0}^{M-1} \| z_{t}^{(m)}\|^2$ in terms of $\sum_{t=0}^{M-1} \|\theta_{t}^{(m)}-\theta^\ast\|^2$, $\|\tilde{z}^{(m-1)}\|^2$, and $\|\tilde{\theta}^{(m-1)}-\theta^\ast\|^2$, and plug it into the preliminary bound of  $\sum_{t=0}^{M-1} \| \theta_{t}^{(m)}-\theta^\ast\|^2$. Then, we obtain an upper bound of $\sum_{t=0}^{M-1} \| \theta_{t}^{(m)}-\theta^\ast\|^2$ in terms of $\|\tilde{z}^{(m-1)}\|^2$, and $\|\tilde{\theta}^{(m-1)}-\theta^\ast\|^2$.  
	
	\item Develop \textit{non-asymptotic bound for $\|\tilde{z}^{m}\|^2$}. (Lemma \ref{lemma: conv-z})
	
	Lastly, we develop a non-asymptotic bound for $\|\tilde{z}^{m}\|^2$ and plug it into the previous upper bounds. Then, we obtain a relation between $\EE \| \tilde{\theta}^{(m )}-\theta^\ast\|$ and $\EE \| \tilde{\theta}^{(m-1 )}-\theta^\ast\|$. Recursively telescoping this inequality leads to our final result.
\end{enumerate}
 
	By Lemma \ref{lemma: pre-bound theta}, we have the following result:
	\begin{align*}
	&  \frac{\lambda_{\widehat{A}}}{12}\alpha   \sum_{t=0}^{M-1}\EE_{m,0} \| \theta_{t}^{(m)} - \theta^\ast\|^2 \\
	\leq & \Big[ 1 +  \alpha^2M \cdot 5(1+\gamma)^2\rho^2_{\max}\Big(1 + \frac{\gamma \rho_{\max}}{ \min |\lambda(C)|}\Big)^2 \Big] \EE_{m,0}\| \tilde{\theta}^{(m-1)} - \theta^\ast\|^2  + \alpha \cdot 2K_2 + \alpha^2 \cdot5K_1 \\
	&+  \alpha \cdot  \frac{6}{\lambda_{\widehat{A}}}\gamma^2 \rho_{\max}^2   \sum_{t=0}^{M-1}\EE_{m,0}\| z_{t}^{(m)}\|^2+ \alpha^2 M \cdot 5 \gamma^2\rho_{\max}^2 \EE_{m,0}\|\tilde{z}^{(m-1)}\|^2, \numberthis \label{eq: mc-tmpp1}
	\end{align*}
	where $K_1$ is specified in \cref{eq: def-K1-mc} of Lemma \ref{lemma: K1}, and $K_2$ is specified in   \cref{eq: def-K2-mc} of Lemma \ref{lemma: K2}. By Lemma \ref{lemma: pre-bound z}, we have that
	\begin{align*}
	&   \frac{\lambda_{C}}{16}\beta  \sum_{t=0}^{M-1}\EE_{m,0}\|z_t^{(m)}\|^2 \\
	\leq & \Big[  1+\Big[10\beta^2 + 10\gamma^2\rho_{\max}^2 \Big(\alpha^2 + \frac{2 \alpha^2}{\lambda_C } \frac{1}{\beta} \Big)\cdot  \Big(\rho_{\max} \frac{ 1+\gamma }{ \min |\lambda(C)|}\Big)^2  \Big]M  \Big]\EE_{m,0}\|  \tilde{z}^{(m-1)}\|^2\\
	&+ 10 (1+\gamma)^2 \rho^2 _{\max} \cdot \Big[\Big(1 + \frac{\gamma \rho_{\max}}{ \min |\lambda(C)|}\Big)^2 \Big(\alpha^2 + \frac{2 \alpha^2}{\lambda_C } \frac{1}{\beta} \Big)\cdot  \Big(\rho_{\max} \frac{ 1+\gamma }{ \min |\lambda(C)|}\Big)^2   \\
	&+  \Big( 1 + \frac{1}{ \min |\lambda(C)|} \Big)^2 \beta^2  \Big] \Big( \sum_{t=0}^{M-1}\EE_{m,0} \| \theta_{t}^{(m)} - \theta^\ast\|^2  + M \EE_{m,0}\| \tilde{\theta}^{(m-1)} - \theta^\ast\|^2 \Big) \\ 
	&+    (2K_3+2K_4)\beta + 10K_5 \beta^2 +  10\Big(\rho_{\max} \frac{ 1+\gamma }{ \min |\lambda(C)|}\Big)^2  K_1 \Big(\alpha^2 + \frac{2 \alpha^2}{\lambda_C } \frac{1}{\beta} \Big). \numberthis \label{eq: mc-tmpp2}
	\end{align*}
	Combining eq.(\ref{eq: mc-tmpp1}) and eq.(\ref{eq: mc-tmpp2}), we obtain the following upper bound of $\sum_{t=0}^{M-1} \| \theta_{t}^{(m)}-\theta^\ast\|^2$ in terms of $\|\tilde{z}^{(m-1)}\|^2$ and $\|\tilde{\theta}^{(m-1)}-\theta^\ast\|^2$. 
	\begin{align*}
	&  \frac{\lambda_{\widehat{A}}}{12}\alpha   \sum_{t=0}^{M-1}\EE_{m,0} \| \theta_{t}^{(m)} - \theta^\ast\|^2 \\
	\leq & \Big[ 1 +  \alpha^2M \cdot 5(1+\gamma)^2\rho^2_{\max}\Big(1 + \frac{\gamma \rho_{\max}}{ \min |\lambda(C)|}\Big)^2 \Big] \EE_{m,0}\| \tilde{\theta}^{(m-1)} - \theta^\ast\|^2  + \alpha \cdot 2K_2 + \alpha^2 \cdot5K_1 \\
	&+  \frac{\alpha}{\beta} \cdot  \frac{96}{\lambda_{\widehat{A}} \lambda_{C}}\gamma^2 \rho_{\max}^2   \Big\{ \Big[  1+\Big[10\beta^2 + 10\gamma^2\rho_{\max}^2 \Big(\alpha^2 + \frac{2 \alpha^2}{\lambda_C } \frac{1}{\beta} \Big)\cdot  \Big(\rho_{\max} \frac{ 1+\gamma }{ \min |\lambda(C)|}\Big)^2  \Big]M  \Big]\EE_{m,0}\|  \tilde{z}^{(m-1)}\|^2\\
	&+ 10 (1+\gamma)^2 \rho^2 _{\max} \cdot \Big[\Big(1 + \frac{\gamma \rho_{\max}}{ \min |\lambda(C)|}\Big)^2 \Big(\alpha^2 + \frac{2 \alpha^2}{\lambda_C } \frac{1}{\beta} \Big)\cdot  \Big(\rho_{\max} \frac{ 1+\gamma }{ \min |\lambda(C)|}\Big)^2   \\
	&+  \Big( 1 + \frac{1}{ \min |\lambda(C)|} \Big)^2 \beta^2  \Big] \Big( \sum_{t=0}^{M-1}\EE_{m,0} \| \theta_{t}^{(m)} - \theta^\ast\|^2  + M \EE_{m,0}\| \tilde{\theta}^{(m-1)} - \theta^\ast\|^2 \Big) \\ 
	&+    (2K_3+2K_4)\beta + 10K_5 \beta^2 +  10\Big(\rho_{\max} \frac{ 1+\gamma }{ \min |\lambda(C)|}\Big)^2  K_1 \Big(\alpha^2 + \frac{2 \alpha^2}{\lambda_C } \frac{1}{\beta} \Big) \Big\}   + \alpha^2 M \cdot 5 \gamma^2\rho_{\max}^2 \EE_{m,0}\|\tilde{z}^{(m-1)}\|^2.
	\end{align*}
	Re-arranging the above inequality yields that 
	\begin{align*}
	& \Big\{  \frac{\lambda_{\widehat{A}}}{12}\alpha -  \frac{\alpha}{\beta} \cdot  \frac{96}{\lambda_{\widehat{A}} \lambda_{C}}\gamma^2 \rho_{\max}^2  \cdot 10 (1+\gamma)^2 \rho^2 _{\max} \cdot \Big[\Big(1 + \frac{\gamma \rho_{\max}}{ \min |\lambda(C)|}\Big)^2 \Big(\alpha^2 + \frac{2 \alpha^2}{\lambda_C } \frac{1}{\beta} \Big)\cdot  \Big(\rho_{\max} \frac{ 1+\gamma }{ \min |\lambda(C)|}\Big)^2   \\
	&+  \Big( 1 + \frac{1}{ \min |\lambda(C)|} \Big)^2 \beta^2  \Big]  \Big\}
	 \sum_{t=0}^{M-1}\EE_{m,0} \| \theta_{t}^{(m)} - \theta^\ast\|^2 \\
	\leq & \Big[ 1 +  \alpha^2M \cdot 5(1+\gamma)^2\rho^2_{\max}\Big(1 + \frac{\gamma \rho_{\max}}{ \min |\lambda(C)|}\Big)^2\\
	& + \frac{\alpha M}{\beta} \cdot  \frac{96}{\lambda_{\widehat{A}} \lambda_{C}}\gamma^2 \rho_{\max}^2 \cdot 10 (1+\gamma)^2 \rho^2 _{\max} \cdot \big[\Big(1 + \frac{\gamma \rho_{\max}}{ \min |\lambda(C)|}\Big)^2 \Big(\alpha^2 + \frac{2 \alpha^2}{\lambda_C } \frac{1}{\beta} \Big)\cdot  \Big(\rho_{\max} \frac{ 1+\gamma }{ \min |\lambda(C)|}\Big)^2   \\
	&+  \Big( 1 + \frac{1}{ \min |\lambda(C)|} \Big)^2 \beta^2  \big] \Big] \EE_{m,0}\| \tilde{\theta}^{(m-1)} - \theta^\ast\|^2  + \alpha \cdot 2K_2 + \alpha^2 \cdot5K_1 \\
	&+  \Big\{ \frac{\alpha}{\beta} \cdot  \frac{96}{\lambda_{\widehat{A}} \lambda_{C}}\gamma^2 \rho_{\max}^2   \Big[  1+\Big[10\beta^2 + 10\gamma^2\rho_{\max}^2 \Big(\alpha^2 + \frac{2 \alpha^2}{\lambda_C } \frac{1}{\beta} \Big)\cdot  \Big(\rho_{\max} \frac{ 1+\gamma }{ \min |\lambda(C)|}\Big)^2  \Big]M  \Big]\\& + \alpha^2 M \cdot 5 \gamma^2\rho_{\max}^2  \Big\} \EE_{m,0}\|  \tilde{z}^{(m-1)}\|^2\\
	&+   \alpha \cdot  \frac{96}{\lambda_{\widehat{A}} \lambda_{C}}\gamma^2 \rho_{\max}^2   (2K_3+2K_4)  +  \alpha \beta \cdot  \frac{96}{\lambda_{\widehat{A}} \lambda_{C}}\gamma^2 \rho_{\max}^2   10K_5  \\&+   \frac{\alpha}{\beta} \big(\alpha^2 + \frac{2 \alpha^2}{\lambda_C } \frac{1}{\beta} \big)  \cdot  \frac{96}{\lambda_{\widehat{A}} \lambda_{C}}\gamma^2 \rho_{\max}^2   10\Big(\rho_{\max} \frac{ 1+\gamma }{ \min |\lambda(C)|}\Big)^2  K_1   . \numberthis \label{eq: mc-tmpp3}
	\end{align*}
	To simplify the above inequality, note that we assume that $  \frac{\lambda_{\widehat{A}}}{12}\alpha -  \frac{\alpha}{\beta} \cdot  \frac{96}{\lambda_{\widehat{A}} \lambda_{C}}\gamma^2 \rho_{\max}^2  \cdot 10 (1+\gamma)^2 \rho^2 _{\max} \cdot \Big[\Big(1 + \frac{\gamma \rho_{\max}}{ \min |\lambda(C)|}\Big)^2 \Big(\alpha^2 + \frac{2 \alpha^2}{\lambda_C } \frac{1}{\beta} \Big)\cdot  \Big(\rho_{\max} \frac{ 1+\gamma }{ \min |\lambda(C)|}\Big)^2   +  \Big( 1 + \frac{1}{ \min |\lambda(C)|} \Big)^2 \beta^2  \Big]  \geq \frac{\lambda_{\widehat{A}}}{16}\alpha$ and $\beta\leq 1$. Applying Jensen's inequality to the left-hand side of the above inequality, we obtain the following simplified inequality.
	\begin{align*}
	&  \frac{\lambda_{\widehat{A}}}{16}\alpha M \EE_{m,0} \| \tilde{ \theta}^{(m)} - \theta^\ast\|^2 \\
	\leq & \Big[ 1 +  \alpha^2M \cdot 5(1+\gamma)^2\rho^2_{\max}\Big(1 + \frac{\gamma \rho_{\max}}{ \min |\lambda(C)|}\Big)^2\\
	& + \alpha M \big[\Big(1 + \frac{\gamma \rho_{\max}}{ \min |\lambda(C)|}\Big)^2 \Big(1 + \frac{2 }{\lambda_C }  \Big)\cdot  \Big(\rho_{\max} \frac{ 1+\gamma }{ \min |\lambda(C)|}\Big)^2  \frac{\alpha^2}{\beta^2}  \\
	& +  \Big( 1 + \frac{1}{ \min |\lambda(C)|} \Big)^2 \beta  \big] \cdot  \frac{96}{\lambda_{\widehat{A}} \lambda_{C}}\gamma^2 \rho_{\max}^2 \cdot 10 (1+\gamma)^2 \rho^2 _{\max}  \Big] \EE_{m,0}\| \tilde{\theta}^{(m-1)} - \theta^\ast\|^2  \\
	&+  \Big\{ \frac{\alpha}{\beta} \cdot  \frac{96}{\lambda_{\widehat{A}} \lambda_{C}}\gamma^2 \rho_{\max}^2   \Big[  1+\Big[10\beta^2 + 10\gamma^2\rho_{\max}^2 \Big(1+ \frac{2  }{\lambda_C }  \Big)\cdot  \Big(\rho_{\max} \frac{ 1+\gamma }{ \min |\lambda(C)|}\Big)^2 \frac{\alpha^2}{\beta} \Big]M  \Big] \\&+ \alpha^2 M \cdot 5 \gamma^2\rho_{\max}^2  \Big\} \EE_{m,0}\|  \tilde{z}^{(m-1)}\|^2\\
	&+   \alpha \cdot  \Big[ \frac{96}{\lambda_{\widehat{A}} \lambda_{C}}\gamma^2 \rho_{\max}^2   (2K_3+2K_4) + 2K_2\Big]  +  \alpha \beta \cdot  \frac{96}{\lambda_{\widehat{A}} \lambda_{C}}\gamma^2 \rho_{\max}^2   10K_5  \\
	&+   \frac{\alpha^3}{\beta} \big(1 + \frac{2  }{\lambda_C }   \big)  \cdot  \frac{96}{\lambda_{\widehat{A}} \lambda_{C}}\gamma^2 \rho_{\max}^2   10\Big(\rho_{\max} \frac{ 1+\gamma }{ \min |\lambda(C)|}\Big)^2  K_1 + \alpha^2 \cdot5K_1 . 
	\end{align*}
	Define $C_1 = \Big(1 + \frac{\gamma \rho_{\max}}{ \min |\lambda(C)|}\Big)^2 \Big(1 + \frac{2 }{\lambda_C }  \Big)\cdot  \Big(\rho_{\max} \frac{ 1+\gamma }{ \min |\lambda(C)|}\Big)^2 \cdot  \frac{96}{\lambda_{\widehat{A}} \lambda_{C}}\gamma^2 \rho_{\max}^2 \cdot 10 (1+\gamma)^2 \rho^2 _{\max} $ and $C_2 = \Big( 1 + \frac{1}{ \min |\lambda(C)|} \Big)^2 \cdot  \frac{96}{\lambda_{\widehat{A}} \lambda_{C}}\gamma^2 \rho_{\max}^2 \cdot 10 (1+\gamma)^2 \rho^2 _{\max} $, and assume that
	\begin{align}\label{eq: def-D}
	D:= &\frac{16}{\lambda_{\widehat{A}}}\Big[\frac{ 1}{\alpha M} +  \alpha  \cdot 5(1+\gamma)^2\rho^2_{\max}\Big(1 + \frac{\gamma \rho_{\max}}{ \min |\lambda(C)|}\Big)^2 + \frac{\alpha^2}{\beta^2}  \cdot C_1   +   \beta  \cdot C_2   \Big]  < 1. 
	\end{align}
	Taking total expectation and dividing $ \frac{\lambda_{\widehat{A}}}{16}\alpha M $ on both sides of the previous simplified inequality,  we obtain that
	\begin{align*}
	&   \EE  \| \tilde{ \theta}^{(m)} - \theta^\ast\|^2 \\
	\leq & D \cdot \EE_{m,0}\| \tilde{\theta}^{(m-1)} - \theta^\ast\|^2  \\
	&+  \frac{16}{\lambda_{\widehat{A}}} \Big\{    \frac{96}{\lambda_{\widehat{A}} \lambda_{C}}\gamma^2 \rho_{\max}^2   \Big[  \frac{1}{\beta M}+ 10\beta + 10\gamma^2\rho_{\max}^2 \Big(1+ \frac{2  }{\lambda_C }  \Big)\cdot  \Big(\rho_{\max} \frac{ 1+\gamma }{ \min |\lambda(C)|}\Big)^2 \frac{\alpha^2}{\beta^2}    \Big] + \alpha  \cdot 5 \gamma^2\rho_{\max}^2  \Big\} \EE \|  \tilde{z}^{(m-1)}\|^2\\
	&+   \frac{1}{M} \cdot  \frac{16}{\lambda_{\widehat{A}}}  \Big[ \frac{96}{\lambda_{\widehat{A}} \lambda_{C}}\gamma^2 \rho_{\max}^2   (2K_3+2K_4) + 2K_2\Big]  +    \frac{\beta}{M} \cdot   \frac{16}{\lambda_{\widehat{A}}}  \frac{96}{\lambda_{\widehat{A}} \lambda_{C}}\gamma^2 \rho_{\max}^2   10K_5  \\
	&+   \frac{\alpha^2}{\beta M} \cdot \frac{16}{\lambda_{\widehat{A}}}  \big(1 + \frac{2  }{\lambda_C }   \big)  \cdot  \frac{96}{\lambda_{\widehat{A}} \lambda_{C}}\gamma^2 \rho_{\max}^2   10\Big(\rho_{\max} \frac{ 1+\gamma }{ \min |\lambda(C)|}\Big)^2  K_1 + \frac{\alpha}{M}  \cdot \frac{80 K_1}{\lambda_{\widehat{A}}}  . \numberthis \label{eq: mc-tmpp4}
	\end{align*}
	By Lemma \ref{lemma: conv-z}, we have that
	\begin{align*}
	\EE\| \tilde{z}^{(m)} \|^2 &\leq   \Big(\frac{1}{M\beta} \cdot \frac{12}{\lambda_C }\Big)^m \EE \|  \tilde{z}^{(0)}\|^2 \\
	&\quad  + 2 \cdot \Big[ \beta \cdot \frac{24}{\lambda_C }   H_{\text{VR}}^2 + \frac{\alpha^2}{\beta^2} \cdot \Big(1  + \frac{2 }{\lambda_C }  \Big) \cdot \frac{24}{\lambda_C }\Big(\rho_{\max} \frac{ 1+\gamma }{ \min |\lambda(C)|}\Big)^2 G_{\text{VR}}^2 + \frac{1}{  M} \frac{24}{\lambda_C } (K_3 + K_4)\Big],
	\end{align*} 
	where $K_3$ is specified in \cref{eq: def-K3-mc} of Lemma \ref{lemma: K3}, and $K_4$ is specified in   \cref{eq: def-K4-mc} of Lemma \ref{lemma: K4}. Here we assume that
	\begin{align}\label{eq: def-E}
		E:= \frac{1}{M\beta} \cdot \frac{12}{\lambda_C } < 1.
	\end{align}
	Substituting the above result of Lemma \ref{lemma: conv-z} into eq.(\ref{eq: mc-tmpp4}), we obtain that
	\begin{align*}
	&   \EE  \| \tilde{ \theta}^{(m)} - \theta^\ast\|^2 \\
	\leq & D \cdot \EE \| \tilde{\theta}^{(m-1)} - \theta^\ast\|^2  \\
	&+  \frac{16}{\lambda_{\widehat{A}}} \Big\{    \frac{96}{\lambda_{\widehat{A}} \lambda_{C}}\gamma^2 \rho_{\max}^2   \Big[  \frac{1}{\beta M}+ 10\beta + 10\gamma^2\rho_{\max}^2 \Big(1+ \frac{2  }{\lambda_C }  \Big)\cdot  \Big(\rho_{\max} \frac{ 1+\gamma }{ \min |\lambda(C)|}\Big)^2 \frac{\alpha^2}{\beta^2}    \Big] + \alpha  \cdot 5 \gamma^2\rho_{\max}^2  \Big\} \Big\{ E^{m-1} \EE \|  \tilde{z}^{(0)}\|^2 \\
	&\quad  + 2 \cdot \Big[ \beta \cdot \frac{24}{\lambda_C }   H_{\text{VR}}^2 + \frac{\alpha^2}{\beta^2} \cdot \Big(1  + \frac{2 }{\lambda_C }  \Big) \cdot \frac{24}{\lambda_C }\Big(\rho_{\max} \frac{ 1+\gamma }{ \min |\lambda(C)|}\Big)^2 G_{\text{VR}}^2 + \frac{1}{  M} \frac{24}{\lambda_C } (K_3 + K_4)\Big] \Big\}\\
	&+   \frac{1}{M} \cdot  \frac{16}{\lambda_{\widehat{A}}}  \Big[ \frac{96}{\lambda_{\widehat{A}} \lambda_{C}}\gamma^2 \rho_{\max}^2   (2K_3+2K_4) + 2K_2\Big]  +    \frac{\beta}{M} \cdot   \frac{16}{\lambda_{\widehat{A}}}  \frac{96}{\lambda_{\widehat{A}} \lambda_{C}}\gamma^2 \rho_{\max}^2   10K_5  \\
	&+   \frac{\alpha^2}{\beta M} \cdot \frac{16}{\lambda_{\widehat{A}}}  \big(1 + \frac{2  }{\lambda_C }   \big)  \cdot  \frac{96}{\lambda_{\widehat{A}} \lambda_{C}}\gamma^2 \rho_{\max}^2   10\Big(\rho_{\max} \frac{ 1+\gamma }{ \min |\lambda(C)|}\Big)^2  K_1 + \frac{\alpha}{M}  \cdot \frac{80 K_1}{\lambda_{\widehat{A}}}.  
	\end{align*}
	Furthermore, we assume that $ \frac{16}{\lambda_{\widehat{A}}} \Big\{    \frac{96}{\lambda_{\widehat{A}} \lambda_{C}}\gamma^2 \rho_{\max}^2   \Big[  \frac{1}{\beta M}+ 10\beta + 10\gamma^2\rho_{\max}^2 \Big(1+ \frac{2  }{\lambda_C }  \Big)\cdot  \Big(\rho_{\max} \frac{ 1+\gamma }{ \min |\lambda(C)|}\Big)^2 \frac{\alpha^2}{\beta^2}    \Big] + \alpha  \cdot 5 \gamma^2\rho_{\max}^2  \Big\} \leq 1$. Then, telescope the above inequality yields that
	\begin{align*}
	&   \EE  \| \tilde{ \theta}^{(m)} - \theta^\ast\|^2 \\
	\leq & D^m \cdot \EE \| \tilde{\theta}^{(0)} - \theta^\ast\|^2  + \frac{D^m - E^m}{D - E} \EE \|  \tilde{z}^{(0)}\|^2\\
	&+  \frac{32}{\lambda_{\widehat{A}}} \Big\{    \frac{96}{\lambda_{\widehat{A}} \lambda_{C}}\gamma^2 \rho_{\max}^2   \Big[  \frac{1}{\beta M}+ 10\beta + 10\gamma^2\rho_{\max}^2 \Big(1+ \frac{2  }{\lambda_C }  \Big)\cdot  \Big(\rho_{\max} \frac{ 1+\gamma }{ \min |\lambda(C)|}\Big)^2 \frac{\alpha^2}{\beta^2}    \Big] \\&+ \alpha  \cdot 5 \gamma^2\rho_{\max}^2  \Big\} \Big\{    \beta \cdot \frac{24}{\lambda_C }   H_{\text{VR}}^2 + \frac{\alpha^2}{\beta^2} \cdot \Big(1  + \frac{2 }{\lambda_C }  \Big) \cdot \frac{24}{\lambda_C }\Big(\rho_{\max} \frac{ 1+\gamma }{ \min |\lambda(C)|}\Big)^2 G_{\text{VR}}^2 + \frac{1}{  M} \frac{24}{\lambda_C } (K_3 + K_4) \Big\}\\
	&+   \frac{1}{M} \cdot  \frac{16}{\lambda_{\widehat{A}}}  \Big[ \frac{96}{\lambda_{\widehat{A}} \lambda_{C}}\gamma^2 \rho_{\max}^2   (2K_3+2K_4) + 2K_2\Big]  +    \frac{\beta}{M} \cdot   \frac{16}{\lambda_{\widehat{A}}}  \frac{96}{\lambda_{\widehat{A}} \lambda_{C}}\gamma^2 \rho_{\max}^2   10K_5  \\
	&+   \frac{\alpha^2}{\beta M} \cdot \frac{16}{\lambda_{\widehat{A}}}  \big(1 + \frac{2  }{\lambda_C }   \big)  \cdot  \frac{96}{\lambda_{\widehat{A}} \lambda_{C}}\gamma^2 \rho_{\max}^2   10\Big(\rho_{\max} \frac{ 1+\gamma }{ \min |\lambda(C)|}\Big)^2  K_1 + \frac{\alpha}{M}  \cdot \frac{80 K_1}{\lambda_{\widehat{A}}}.  
	\end{align*}
To further simplify, note that the first two terms are in the order of $D^m$. Also, the third term is a product of two  curly brackets, and it is easy to check that this term is dominated by $\calO(\beta^2)$ under the relation $\beta = \calO(\alpha^{2/3})$. { To further elaborate this, note that the first bracket of this product is in the order of $\calO(\frac{1}{\beta M})+\calO(\beta)$ and the second term of this product is in the order of $\calO(\beta)$ (by $E<1$, we have $\frac{1}{M} \leq \calO(\beta)$). Therefore, asymptotically the product is in the order of $\big(\calO(\frac{1}{\beta M})+\calO(\beta)\big) \times  \calO(\beta)  = \calO(\frac{1}{M}) + \calO(\beta^2)$.} Moreover, under the setting $\beta = \calO(\alpha^{2/3})$, $D > E$ for sufficiently small $\alpha$ and $\beta$. Therefore, we have the following asymptotic bound
	\begin{align*}
		 \EE  \| \tilde{ \theta}^{(m)} - \theta^\ast\|^2 \ = \calO(D^m) + \calO(\beta^2) + \calO(\frac{1}{M}).
	\end{align*}
	 
	Now we compute its complexity. For sufficiently small $\beta$ and sufficiently large $M$, there always exists constant $I_1, I_2, I_3$ such that 
	\[ \EE  \| \tilde{\theta}^{(m)} - \theta^\ast\|^2 \leq D^m I_1 +  \beta^2 I_2 +  \frac{1}{M} I_3.\]
	Now we require
	\begin{enumerate}[label=(\roman*)]
		\item $\beta^2 I_2 \leq \epsilon/3 \Rightarrow \beta \leq I_2^{1/2}\epsilon^{1/2} = \calO(\epsilon^{1/2})$.
		\item $\frac{1}{M} I_3 \leq \epsilon/3 \Rightarrow M\geq \calO({\epsilon}^{-1}).$
		\item $D^m I_1 \leq \epsilon/3 \Rightarrow m \log D \leq \log \frac{\epsilon}{3 I_1}\Rightarrow m \ge \calO(\log\epsilon^{-1})$.
	\end{enumerate}
	Note that in (iii), we have used the condition that $D<1$ and hence $\log D$ is a negative constant.  
	Since $\alpha = \calO(\beta^{2/3})$, the condition that $D \leq 1$ requires that $M \geq \calO(\beta^{-3/2})$, which combines with (ii) requires that
	\[M \geq \max \{ \calO({\epsilon}^{-1}), \calO(\beta^{-3/2}) \}.\]
	Taking into account the constraint on $\beta$ in (i), we just require that $M\geq \calO(\epsilon^{-1})$,
	which leads to the overall complexity result
	\[mM \geq \calO(\epsilon^{-1}\log {\epsilon}^{-1}).\]  
   
\section{Proof of \Cref{cor: markov}}

\begin{lemma}[Refined Bounds of $\EE \| \tilde{z} \|^2$]\label{lemma: refined z}
Under the same assumptions as those of \Cref{thm: markov}, choose the learning rate $\alpha,\beta$ and the batch size $M$  such that all requirements of \Cref{thm: markov} are satisfied.
	Then, the following refined bound holds. 
	\begin{align*}
	&    \EE \| \tilde{z}^{(m)}\|^2 \\ 
	\leq &  F^m \cdot \EE \|  \tilde{z}^{(0)}\|^2 + \frac{D^{m-1} - F^{m-1}}{D - F} \cdot \EE \| \tilde{\theta}^{(0)} - \theta^\ast\|^2  + \frac{\frac{D^{m-1} - F^{m-1}}{D - F} - \frac{E^{m-1} - F^{m-1}}{E - F}}{D - E} \EE \|  \tilde{z}^{(0)}\|^2\\
	& + \frac{1}{1- F}\frac{24}{\lambda_{C}}\Big\{   10 (1+\gamma)^2 \rho^2 _{\max} \cdot \Big[\Big(1 + \frac{\gamma \rho_{\max}}{ \min |\lambda(C)|}\Big)^2 \Big( 1 + \frac{2  }{\lambda_C } \Big)\cdot  \Big(\rho_{\max} \frac{ 1+\gamma }{ \min |\lambda(C)|}\Big)^2  \frac{\alpha^2}{\beta^2}   \\
	&+  \Big( 1 + \frac{1}{ \min |\lambda(C)|} \Big)^2 \beta  \Big]   +  120 (1+\gamma)^2 \rho^2 _{\max} \frac{1}{\lambda_{\widehat{A}}} \cdot \Big[\Big(1 + \frac{\gamma \rho_{\max}}{ \min |\lambda(C)|}\Big)^2 \Big(1 + \frac{2  }{\lambda_C }   \Big)\cdot  \Big(\rho_{\max} \frac{ 1+\gamma }{ \min |\lambda(C)|}\Big)^2  \frac{\alpha^2}{\beta^2}   \\
	&+  \Big( 1 + \frac{1}{ \min |\lambda(C)|} \Big)^2 \beta   \Big]  \cdot \Big[ \frac{1}{\alpha M}  +  \alpha  \cdot 5(1+\gamma)^2\rho^2_{\max}\Big(1 + \frac{\gamma \rho_{\max}}{ \min |\lambda(C)|}\Big)^2 \Big]  \Big\}\\
	&\times \Big\{  \frac{32}{\lambda_{\widehat{A}}} \Big\{    \frac{96}{\lambda_{\widehat{A}} \lambda_{C}}\gamma^2 \rho_{\max}^2   \Big[  \frac{1}{\beta M}+ 10\beta + 10\gamma^2\rho_{\max}^2 \Big(1+ \frac{2  }{\lambda_C }  \Big)\cdot  \Big(\rho_{\max} \frac{ 1+\gamma }{ \min |\lambda(C)|}\Big)^2 \frac{\alpha^2}{\beta^2}    \Big] + \alpha  \cdot 5 \gamma^2\rho_{\max}^2  \Big\} \\
	&\times \Big\{    \beta \cdot \frac{24}{\lambda_C }   H_{\text{VR}}^2 + \frac{\alpha^2}{\beta^2} \cdot \Big(1  + \frac{2 }{\lambda_C }  \Big) \cdot \frac{24}{\lambda_C }\Big(\rho_{\max} \frac{ 1+\gamma }{ \min |\lambda(C)|}\Big)^2 G_{\text{VR}}^2 + \frac{1}{  M} \frac{24}{\lambda_C } (K_3 + K_4) \Big\}\\
	&+   \frac{1}{M} \cdot  \frac{16}{\lambda_{\widehat{A}}}  \Big[ \frac{96}{\lambda_{\widehat{A}} \lambda_{C}}\gamma^2 \rho_{\max}^2   (2K_3+2K_4) + 2K_2\Big]  +    \frac{\beta}{M} \cdot   \frac{16}{\lambda_{\widehat{A}}}  \frac{96}{\lambda_{\widehat{A}} \lambda_{C}}\gamma^2 \rho_{\max}^2   10K_5  \\
	&+   \frac{\alpha^2}{\beta M} \cdot \frac{16}{\lambda_{\widehat{A}}}  \big(1 + \frac{2  }{\lambda_C }   \big)  \cdot  \frac{96}{\lambda_{\widehat{A}} \lambda_{C}}\gamma^2 \rho_{\max}^2   10\Big(\rho_{\max} \frac{ 1+\gamma }{ \min |\lambda(C)|}\Big)^2  K_1 + \frac{\alpha}{M}  \cdot \frac{80 K_1}{\lambda_{\widehat{A}}} \Big\}\\ 
	&+  \frac{1}{1-F} \Big\{ \frac{1}{M}   \frac{48K_3+48K_4}{\lambda_{C}}  + \frac{\beta}{M} \frac{240K_5}{\lambda_{C}}   +  \frac{\alpha^2}{\beta^2} \frac{1}{M}\frac{240}{\lambda_{C}}\Big(\rho_{\max} \frac{ 1+\gamma }{ \min |\lambda(C)|}\Big)^2  K_1 \Big(1 + \frac{2  }{\lambda_C }  \Big) \\
	&+   \frac{1}{M}    (1+\gamma)^2 \rho^2 _{\max} \frac{2880}{\lambda_{\widehat{A}} \lambda_{C}} \cdot \Big[\Big(1 + \frac{\gamma \rho_{\max}}{ \min |\lambda(C)|}\Big)^2 \Big(1 + \frac{2  }{\lambda_C }\Big)\cdot  \Big(\rho_{\max} \frac{ 1+\gamma }{ \min |\lambda(C)|}\Big)^2   \frac{\alpha^2}{\beta^2} \\&+  \Big( 1 + \frac{1}{ \min |\lambda(C)|} \Big)^2 \beta   \Big]  \cdot \Big\{    2K_2 + \alpha  \cdot 5K_1   \Big\}  \Big\}.
	\end{align*}
\end{lemma}
\begin{proof}
	Based on the preliminary bound of $ \sum_{t=0}^{M-1}\EE_{m,0}\|z_t^{(m)}\|^2$ (Lemma \ref{lemma: pre-bound z}), we have 
	\begin{align*}
&   \frac{\lambda_{C}}{16}\beta  \sum_{t=0}^{M-1}\EE_{m,0}\|z_t^{(m)}\|^2 \\
\leq & \Big[  1+\Big[10\beta^2 + 10\gamma^2\rho_{\max}^2 \Big(\alpha^2 + \frac{2 \alpha^2}{\lambda_C } \frac{1}{\beta} \Big)\cdot  \Big(\rho_{\max} \frac{ 1+\gamma }{ \min |\lambda(C)|}\Big)^2  \Big]M  \Big]\EE_{m,0}\|  \tilde{z}^{(m-1)}\|^2\\
&+ 10 (1+\gamma)^2 \rho^2 _{\max} \cdot \Big[\Big(1 + \frac{\gamma \rho_{\max}}{ \min |\lambda(C)|}\Big)^2 \Big(\alpha^2 + \frac{2 \alpha^2}{\lambda_C } \frac{1}{\beta} \Big)\cdot  \Big(\rho_{\max} \frac{ 1+\gamma }{ \min |\lambda(C)|}\Big)^2   \\
&+  \Big( 1 + \frac{1}{ \min |\lambda(C)|} \Big)^2 \beta^2  \Big] \Big( \sum_{t=0}^{M-1}\EE_{m,0} \| \theta_{t}^{(m)} - \theta^\ast\|^2  + M \EE_{m,0}\| \tilde{\theta}^{(m-1)} - \theta^\ast\|^2 \Big) \\ 
&+    (2K_3+2K_4)\beta + 10K_5 \beta^2 +  10\Big(\rho_{\max} \frac{ 1+\gamma }{ \min |\lambda(C)|}\Big)^2  K_1 \Big(\alpha^2 + \frac{2 \alpha^2}{\lambda_C } \frac{1}{\beta} \Big), \numberthis \label{eq: mc-tmppp-00}
\end{align*}
where $K_1$ is specified in \cref{eq: def-K1-mc} of Lemma \ref{lemma: K1}, $K_3$ is specified in  \cref{eq: def-K3-mc} of Lemma \ref{lemma: K3}, $K_4$ is specified in  \cref{eq: def-K4-mc} of Lemma \ref{lemma: K4}, and $K_5$ is specified in  \cref{eq: def-K5-mc} of Lemma \ref{lemma: K5}. Note that by Lemma \ref{lemma: pre-bound theta}, we have the preliminary bound of $ \sum_{t=0}^{M-1}\EE_{m,0} \| \theta_{t}^{(m)} - \theta^\ast\|^2$:
\begin{align*}
&  \frac{\lambda_{\widehat{A}}}{12}\alpha   \sum_{t=0}^{M-1}\EE_{m,0} \| \theta_{t}^{(m)} - \theta^\ast\|^2 \\
\leq & \Big[ 1 +  \alpha^2M \cdot 5(1+\gamma)^2\rho^2_{\max}\Big(1 + \frac{\gamma \rho_{\max}}{ \min |\lambda(C)|}\Big)^2 \Big] \EE_{m,0}\| \tilde{\theta}^{(m-1)} - \theta^\ast\|^2  + \alpha \cdot 2K_2 + \alpha^2 \cdot5K_1 \\
&+  \alpha \cdot  \frac{6}{\lambda_{\widehat{A}}}\gamma^2 \rho_{\max}^2   \sum_{t=0}^{M-1}\EE_{m,0}\| z_{t}^{(m)}\|^2+ \alpha^2 M \cdot 5 \gamma^2\rho_{\max}^2 \EE_{m,0}\|\tilde{z}^{(m-1)}\|^2, \numberthis \label{eq: mc-tmppp-01}
\end{align*}
where $K_1$ is specified in \cref{eq: def-K1-mc} of Lemma \ref{lemma: K1}, and $K_2$ is specified in  \cref{eq: def-K2-mc} of Lemma \ref{lemma: K2}. Now we combine eq.(\ref{eq: mc-tmppp-00}) and eq.(\ref{eq: mc-tmppp-01}) to obtain that
\begin{align*}
&   \frac{\lambda_{C}}{16}\beta  \sum_{t=0}^{M-1}\EE_{m,0}\|z_t^{(m)}\|^2 \\
\leq & \Big[  1+\Big[10\beta^2 + 10\gamma^2\rho_{\max}^2 \Big(\alpha^2 + \frac{2 \alpha^2}{\lambda_C } \frac{1}{\beta} \Big)\cdot  \Big(\rho_{\max} \frac{ 1+\gamma }{ \min |\lambda(C)|}\Big)^2  \Big]M  \Big]\EE_{m,0}\|  \tilde{z}^{(m-1)}\|^2\\
&+ \frac{1}{\alpha} \cdot 120 (1+\gamma)^2 \rho^2 _{\max} \frac{1}{\lambda_{\widehat{A}}} \cdot \Big[\Big(1 + \frac{\gamma \rho_{\max}}{ \min |\lambda(C)|}\Big)^2 \Big(\alpha^2 + \frac{2 \alpha^2}{\lambda_C } \frac{1}{\beta} \Big)\cdot  \Big(\rho_{\max} \frac{ 1+\gamma }{ \min |\lambda(C)|}\Big)^2   \\
&+  \Big( 1 + \frac{1}{ \min |\lambda(C)|} \Big)^2 \beta^2  \Big]  \cdot \Big\{ \Big[ 1 +  \alpha^2M \cdot 5(1+\gamma)^2\rho^2_{\max}\Big(1 + \frac{\gamma \rho_{\max}}{ \min |\lambda(C)|}\Big)^2 \Big] \EE_{m,0}\| \tilde{\theta}^{(m-1)} - \theta^\ast\|^2 \\& + \alpha \cdot 2K_2 + \alpha^2 \cdot5K_1  +  \alpha \cdot  \frac{6}{\lambda_{\widehat{A}}}\gamma^2 \rho_{\max}^2   \sum_{t=0}^{M-1}\EE_{m,0}\| z_{t}^{(m)}\|^2+ \alpha^2 M \cdot 5 \gamma^2\rho_{\max}^2 \EE_{m,0}\|\tilde{z}^{(m-1)}\|^2 \Big\}   \\
& +   10 (1+\gamma)^2 \rho^2 _{\max} \cdot \Big[\Big(1 + \frac{\gamma \rho_{\max}}{ \min |\lambda(C)|}\Big)^2 \Big(\alpha^2 + \frac{2 \alpha^2}{\lambda_C } \frac{1}{\beta} \Big)\cdot  \Big(\rho_{\max} \frac{ 1+\gamma }{ \min |\lambda(C)|}\Big)^2   \\
&+  \Big( 1 + \frac{1}{ \min |\lambda(C)|} \Big)^2 \beta^2  \Big] \cdot M \EE_{m,0}\| \tilde{\theta}^{(m-1)} - \theta^\ast\|^2   \\ 
&+    (2K_3+2K_4)\beta + 10K_5 \beta^2 +  10\Big(\rho_{\max} \frac{ 1+\gamma }{ \min |\lambda(C)|}\Big)^2  K_1 \Big(\alpha^2 + \frac{2 \alpha^2}{\lambda_C } \frac{1}{\beta} \Big).
\end{align*}
Next, we simplify the above inequality. We first expand the curly brackets and simplify, we obtain that
\begin{align*}
&  \Big\{ \frac{\lambda_{C}}{16}\beta  -  120 (1+\gamma)^2 \rho^2 _{\max} \frac{1}{\lambda_{\widehat{A}}} \cdot \Big[\Big(1 + \frac{\gamma \rho_{\max}}{ \min |\lambda(C)|}\Big)^2 \Big(\alpha^2 + \frac{2 \alpha^2}{\lambda_C } \frac{1}{\beta} \Big)\cdot  \Big(\rho_{\max} \frac{ 1+\gamma }{ \min |\lambda(C)|}\Big)^2   \\
&+  \Big( 1 + \frac{1}{ \min |\lambda(C)|} \Big)^2 \beta^2  \Big]  \cdot  \frac{6}{\lambda_{\widehat{A}}}\gamma^2 \rho_{\max}^2  \Big\}  \sum_{t=0}^{M-1}\EE_{m,0}\|z_t^{(m)}\|^2 \\
\leq & \Big[  1+\Big[10\beta^2 + 10\gamma^2\rho_{\max}^2 \Big(\alpha^2 + \frac{2 \alpha^2}{\lambda_C } \frac{1}{\beta} \Big)\cdot  \Big(\rho_{\max} \frac{ 1+\gamma }{ \min |\lambda(C)|}\Big)^2  \Big]M  \\&+ 
   120 (1+\gamma)^2 \rho^2 _{\max} \frac{1}{\lambda_{\widehat{A}}} \cdot \Big[\Big(1 + \frac{\gamma \rho_{\max}}{ \min |\lambda(C)|}\Big)^2 \Big(\alpha^2 + \frac{2 \alpha^2}{\lambda_C } \frac{1}{\beta} \Big)\cdot  \Big(\rho_{\max} \frac{ 1+\gamma }{ \min |\lambda(C)|}\Big)^2   \\
&+  \Big( 1 + \frac{1}{ \min |\lambda(C)|} \Big)^2 \beta^2  \Big]  \cdot \alpha  M \cdot 5 \gamma^2\rho_{\max}^2 
 \Big]\EE_{m,0}\|  \tilde{z}^{(m-1)}\|^2\\
&+   120 (1+\gamma)^2 \rho^2 _{\max} \frac{1}{\lambda_{\widehat{A}}} \cdot \Big[\Big(1 + \frac{\gamma \rho_{\max}}{ \min |\lambda(C)|}\Big)^2 \Big(\alpha^2 + \frac{2 \alpha^2}{\lambda_C } \frac{1}{\beta} \Big)\cdot  \Big(\rho_{\max} \frac{ 1+\gamma }{ \min |\lambda(C)|}\Big)^2  +  \Big( 1 + \frac{1}{ \min |\lambda(C)|} \Big)^2 \beta^2  \Big]  \\&\times \Big[    2K_2 + \alpha  \cdot5K_1   \Big]   \\
& + \Big\{   10 (1+\gamma)^2 \rho^2 _{\max} \cdot \Big[\Big(1 + \frac{\gamma \rho_{\max}}{ \min |\lambda(C)|}\Big)^2 \Big(\alpha^2 + \frac{2 \alpha^2}{\lambda_C } \frac{1}{\beta} \Big)\cdot  \Big(\rho_{\max} \frac{ 1+\gamma }{ \min |\lambda(C)|}\Big)^2   \\
&+  \Big( 1 + \frac{1}{ \min |\lambda(C)|} \Big)^2 \beta^2  \Big] \cdot M  + \frac{1}{\alpha} \cdot 120 (1+\gamma)^2 \rho^2 _{\max} \frac{1}{\lambda_{\widehat{A}}} \cdot \Big[\Big(1 + \frac{\gamma \rho_{\max}}{ \min |\lambda(C)|}\Big)^2 \Big(\alpha^2 + \frac{2 \alpha^2}{\lambda_C } \frac{1}{\beta} \Big)\cdot  \Big(\rho_{\max} \frac{ 1+\gamma }{ \min |\lambda(C)|}\Big)^2   \\
&+  \Big( 1 + \frac{1}{ \min |\lambda(C)|} \Big)^2 \beta^2  \Big]  \cdot \Big[ 1 +  \alpha^2M \cdot 5(1+\gamma)^2\rho^2_{\max}\Big(1 + \frac{\gamma \rho_{\max}}{ \min |\lambda(C)|}\Big)^2 \Big]  \Big\}\EE_{m,0}\| \tilde{\theta}^{(m-1)} - \theta^\ast\|^2   \\ 
&+    (2K_3+2K_4)\beta + 10K_5 \beta^2 +  10\Big(\rho_{\max} \frac{ 1+\gamma }{ \min |\lambda(C)|}\Big)^2  K_1 \Big(\alpha^2 + \frac{2 \alpha^2}{\lambda_C } \frac{1}{\beta} \Big) 
\end{align*}
Assume $ \frac{\lambda_{C}}{16}\beta  -  120 (1+\gamma)^2 \rho^2 _{\max} \frac{1}{\lambda_{\widehat{A}}} \cdot \Big[\Big(1 + \frac{\gamma \rho_{\max}}{ \min |\lambda(C)|}\Big)^2 \Big(\alpha^2 + \frac{2 \alpha^2}{\lambda_C } \frac{1}{\beta} \Big)\cdot  \Big(\rho_{\max} \frac{ 1+\gamma }{ \min |\lambda(C)|}\Big)^2   +  \Big( 1 + \frac{1}{ \min |\lambda(C)|} \Big)^2 \beta^2  \Big]  \cdot  \frac{6}{\lambda_{\widehat{A}}}\gamma^2 \rho_{\max}^2 \geq  \frac{\lambda_{C}}{24}\beta$ and define 
\begin{align}\label{eq: def-F}
    F:= &\frac{24}{\lambda_{C}} \cdot  \Big[  \frac{1}{\beta M}+ 10\beta  + 10\gamma^2\rho_{\max}^2 \big(1 + \frac{1}{\lambda_C }  \big)\cdot  \Big(\rho_{\max} \frac{ 1+\gamma }{ \min |\lambda(C)|}\Big)^2    \frac{\alpha^2}{\beta^2}    \nonumber \\&+  
\alpha \cdot 120 (1+\gamma)^2 \rho^2 _{\max} \frac{1}{\lambda_{\widehat{A}}} \cdot \Big[\Big(1 + \frac{\gamma \rho_{\max}}{ \min |\lambda(C)|}\Big)^2 \Big(1+ \frac{2  }{\lambda_C }   \Big)\cdot  \Big(\rho_{\max} \frac{ 1+\gamma }{ \min |\lambda(C)|}\Big)^2   \frac{\alpha^2}{\beta^2}     \nonumber \\&+
+  \Big( 1 + \frac{1}{ \min |\lambda(C)|} \Big)^2 \beta   \Big]  \cdot  5 \gamma^2\rho_{\max}^2 
\Big].
\end{align}
Also, assume that $F<1$. Applying Jensen's inequality on the left-hand side of the above inequality and dividing $\frac{\lambda_{C}}{24}\beta M$ on both sides, we obtain that
\begin{align*}
&    \EE_{m,0}\| \tilde{z}^{(m)}\|^2 \\
\leq &  F \cdot \EE_{m,0}\|  \tilde{z}^{(m-1)}\|^2\\
&+     (1+\gamma)^2 \rho^2 _{\max} \frac{2880}{\lambda_{\widehat{A}} \lambda_{C}} \cdot \Big[\Big(1 + \frac{\gamma \rho_{\max}}{ \min |\lambda(C)|}\Big)^2 \Big(1 + \frac{2  }{\lambda_C }\Big)\cdot  \Big(\rho_{\max} \frac{ 1+\gamma }{ \min |\lambda(C)|}\Big)^2   \frac{\alpha^2}{\beta^2} +  \Big( 1 + \frac{1}{ \min |\lambda(C)|} \Big)^2 \beta   \Big] \\& \times \Big[    2K_2 + \alpha  \cdot 5K_1   \Big] \frac{1}{M}   \\
& + \frac{24}{\lambda_{C}}\Big\{   10 (1+\gamma)^2 \rho^2 _{\max} \cdot \Big[\Big(1 + \frac{\gamma \rho_{\max}}{ \min |\lambda(C)|}\Big)^2 \Big( 1 + \frac{2  }{\lambda_C } \Big)\cdot  \Big(\rho_{\max} \frac{ 1+\gamma }{ \min |\lambda(C)|}\Big)^2  \frac{\alpha^2}{\beta^2}   \\
&+  \Big( 1 + \frac{1}{ \min |\lambda(C)|} \Big)^2 \beta  \Big]   +  120 (1+\gamma)^2 \rho^2 _{\max} \frac{1}{\lambda_{\widehat{A}}} \cdot \Big[\Big(1 + \frac{\gamma \rho_{\max}}{ \min |\lambda(C)|}\Big)^2 \Big(1 + \frac{2  }{\lambda_C }   \Big)\cdot  \Big(\rho_{\max} \frac{ 1+\gamma }{ \min |\lambda(C)|}\Big)^2  \frac{\alpha^2}{\beta^2}   \\
&+  \Big( 1 + \frac{1}{ \min |\lambda(C)|} \Big)^2 \beta   \Big]  \cdot \Big[ \frac{1}{\alpha M}  +  \alpha  \cdot 5(1+\gamma)^2\rho^2_{\max}\Big(1 + \frac{\gamma \rho_{\max}}{ \min |\lambda(C)|}\Big)^2 \Big]  \Big\}\EE_{m,0}\| \tilde{\theta}^{(m-1)} - \theta^\ast\|^2   \\ 
&+  \frac{1}{M}   \frac{48K_3+48K_4}{\lambda_{C}}  + \frac{\beta}{M} \frac{240K_5}{\lambda_{C}}   +  \frac{\alpha^2}{\beta^2} \frac{1}{M}\frac{240}{\lambda_{C}}\Big(\rho_{\max} \frac{ 1+\gamma }{ \min |\lambda(C)|}\Big)^2  K_1 \Big(1 + \frac{2  }{\lambda_C }  \Big),   \numberthis \label{eq: mc-tmppp-02}
\end{align*}
where we also use the fact that $\beta \leq 1$. Recall that we have obtain the final bound of $\EE \| \tilde{\theta}^{(m-1)} - \theta^\ast \|^2$ in Theorem \ref{thm: markov} as follows:
\begin{align*}
&   \EE  \| \tilde{ \theta}^{(m)} - \theta^\ast\|^2 \\
\leq & D^m \cdot \EE \| \tilde{\theta}^{(0)} - \theta^\ast\|^2  + \frac{D^m - E^m}{D - E} \EE \|  \tilde{z}^{(0)}\|^2\\
&+  \frac{32}{\lambda_{\widehat{A}}} \Big\{    \frac{96}{\lambda_{\widehat{A}} \lambda_{C}}\gamma^2 \rho_{\max}^2   \Big[  \frac{1}{\beta M}+ 10\beta + 10\gamma^2\rho_{\max}^2 \Big(1+ \frac{2  }{\lambda_C }  \Big)\cdot  \Big(\rho_{\max} \frac{ 1+\gamma }{ \min |\lambda(C)|}\Big)^2 \frac{\alpha^2}{\beta^2}    \Big] + \alpha  \cdot 5 \gamma^2\rho_{\max}^2  \Big\}\\&\times \Big\{    \beta \cdot \frac{24}{\lambda_C }   H_{\text{VR}}^2 + \frac{\alpha^2}{\beta^2} \cdot \Big(1  + \frac{2 }{\lambda_C }  \Big) \cdot \frac{24}{\lambda_C }\Big(\rho_{\max} \frac{ 1+\gamma }{ \min |\lambda(C)|}\Big)^2 G_{\text{VR}}^2 + \frac{1}{  M} \frac{24}{\lambda_C } (K_3 + K_4) \Big\}\\
&+   \frac{1}{M} \cdot  \frac{16}{\lambda_{\widehat{A}}}  \Big[ \frac{96}{\lambda_{\widehat{A}} \lambda_{C}}\gamma^2 \rho_{\max}^2   (2K_3+2K_4) + 2K_2\Big]  +    \frac{\beta}{M} \cdot   \frac{16}{\lambda_{\widehat{A}}}  \frac{96}{\lambda_{\widehat{A}} \lambda_{C}}\gamma^2 \rho_{\max}^2   10K_5  \\
&+   \frac{\alpha^2}{\beta M} \cdot \frac{16}{\lambda_{\widehat{A}}}  \big(1 + \frac{2  }{\lambda_C }   \big)  \cdot  \frac{96}{\lambda_{\widehat{A}} \lambda_{C}}\gamma^2 \rho_{\max}^2   10\Big(\rho_{\max} \frac{ 1+\gamma }{ \min |\lambda(C)|}\Big)^2  K_1 + \frac{\alpha}{M}  \cdot \frac{80 K_1}{\lambda_{\widehat{A}}}.   \numberthis \label{eq: mc-tmppp-03}
\end{align*}
Taking total expectation on both sides of eq.(\ref{eq: mc-tmppp-02}) and applying eq.(\ref{eq: mc-tmppp-03}), we obtain that
\begin{align*}
&    \EE \| \tilde{z}^{(m)}\|^2 \\
\leq &  F \cdot \EE \|  \tilde{z}^{(m-1)}\|^2\\
&+     (1+\gamma)^2 \rho^2 _{\max} \frac{2880}{\lambda_{\widehat{A}} \lambda_{C}} \cdot \Big[\Big(1 + \frac{\gamma \rho_{\max}}{ \min |\lambda(C)|}\Big)^2 \Big(1 + \frac{2  }{\lambda_C }\Big)\cdot  \Big(\rho_{\max} \frac{ 1+\gamma }{ \min |\lambda(C)|}\Big)^2   \frac{\alpha^2}{\beta^2} +  \Big( 1 + \frac{1}{ \min |\lambda(C)|} \Big)^2 \beta   \Big]  \\&\times \Big\{    2K_2 + \alpha  \cdot 5K_1   \Big\} \frac{1}{M}   \\
& + \frac{24}{\lambda_{C}}\Big\{   10 (1+\gamma)^2 \rho^2 _{\max} \cdot \Big[\Big(1 + \frac{\gamma \rho_{\max}}{ \min |\lambda(C)|}\Big)^2 \Big( 1 + \frac{2  }{\lambda_C } \Big)\cdot  \Big(\rho_{\max} \frac{ 1+\gamma }{ \min |\lambda(C)|}\Big)^2  \frac{\alpha^2}{\beta^2}   \\
&+  \Big( 1 + \frac{1}{ \min |\lambda(C)|} \Big)^2 \beta  \Big]   +  120 (1+\gamma)^2 \rho^2 _{\max} \frac{1}{\lambda_{\widehat{A}}} \cdot \Big[\Big(1 + \frac{\gamma \rho_{\max}}{ \min |\lambda(C)|}\Big)^2 \Big(1 + \frac{2  }{\lambda_C }   \Big)\cdot  \Big(\rho_{\max} \frac{ 1+\gamma }{ \min |\lambda(C)|}\Big)^2  \frac{\alpha^2}{\beta^2}   \\
&+  \Big( 1 + \frac{1}{ \min |\lambda(C)|} \Big)^2 \beta   \Big]  \cdot \Big[ \frac{1}{\alpha M}  +  \alpha  \cdot 5(1+\gamma)^2\rho^2_{\max}\Big(1 + \frac{\gamma \rho_{\max}}{ \min |\lambda(C)|}\Big)^2 \Big]  \Big\}\\ &\times \Big\{ D^{m-1} \cdot \EE \| \tilde{\theta}^{(0)} - \theta^\ast\|^2  + \frac{D^{m-1} - E^{m-1}}{D - E} \EE \|  \tilde{z}^{(0)}\|^2\\
&+  \frac{32}{\lambda_{\widehat{A}}} \Big\{    \frac{96}{\lambda_{\widehat{A}} \lambda_{C}}\gamma^2 \rho_{\max}^2   \Big[  \frac{1}{\beta M}+ 10\beta + 10\gamma^2\rho_{\max}^2 \Big(1+ \frac{2  }{\lambda_C }  \Big)\cdot  \Big(\rho_{\max} \frac{ 1+\gamma }{ \min |\lambda(C)|}\Big)^2 \frac{\alpha^2}{\beta^2}    \Big] + \alpha  \cdot 5 \gamma^2\rho_{\max}^2  \Big\} \\ &\times \Big\{    \beta \cdot \frac{24}{\lambda_C }   H_{\text{VR}}^2 + \frac{\alpha^2}{\beta^2} \cdot \Big(1  + \frac{2 }{\lambda_C }  \Big) \cdot \frac{24}{\lambda_C }\Big(\rho_{\max} \frac{ 1+\gamma }{ \min |\lambda(C)|}\Big)^2 G_{\text{VR}}^2 + \frac{1}{  M} \frac{24}{\lambda_C } (K_3 + K_4) \Big\}\\
&+   \frac{1}{M} \cdot  \frac{16}{\lambda_{\widehat{A}}}  \Big[ \frac{96}{\lambda_{\widehat{A}} \lambda_{C}}\gamma^2 \rho_{\max}^2   (2K_3+2K_4) + 2K_2\Big]  +    \frac{\beta}{M} \cdot   \frac{16}{\lambda_{\widehat{A}}}  \frac{96}{\lambda_{\widehat{A}} \lambda_{C}}\gamma^2 \rho_{\max}^2   10K_5  \\
&+   \frac{\alpha^2}{\beta M} \cdot \frac{16}{\lambda_{\widehat{A}}}  \big(1 + \frac{2  }{\lambda_C }   \big)  \cdot  \frac{96}{\lambda_{\widehat{A}} \lambda_{C}}\gamma^2 \rho_{\max}^2   10\Big(\rho_{\max} \frac{ 1+\gamma }{ \min |\lambda(C)|}\Big)^2  K_1 + \frac{\alpha}{M}  \cdot \frac{80 K_1}{\lambda_{\widehat{A}}} \Big\}\\ 
&+  \frac{1}{M}   \frac{48K_3+48K_4}{\lambda_{C}}  + \frac{\beta}{M} \frac{240K_5}{\lambda_{C}}   +  \frac{\alpha^2}{\beta^2} \frac{1}{M}\frac{240}{\lambda_{C}}\Big(\rho_{\max} \frac{ 1+\gamma }{ \min |\lambda(C)|}\Big)^2  K_1 \Big(1 + \frac{2  }{\lambda_C }  \Big) \\
\leq &  F \cdot \EE \|  \tilde{z}^{(m-1)}\|^2 + D^{m-1} \cdot \EE \| \tilde{\theta}^{(0)} - \theta^\ast\|^2  + \frac{D^{m-1} - E^{m-1}}{D - E} \EE \|  \tilde{z}^{(0)}\|^2\\
&+     (1+\gamma)^2 \rho^2 _{\max} \frac{2880}{\lambda_{\widehat{A}} \lambda_{C}} \cdot \Big[\Big(1 + \frac{\gamma \rho_{\max}}{ \min |\lambda(C)|}\Big)^2 \Big(1 + \frac{2  }{\lambda_C }\Big)\cdot  \Big(\rho_{\max} \frac{ 1+\gamma }{ \min |\lambda(C)|}\Big)^2   \frac{\alpha^2}{\beta^2} +  \Big( 1 + \frac{1}{ \min |\lambda(C)|} \Big)^2 \beta   \Big] \\ &\times \Big\{    2K_2 + \alpha  \cdot 5K_1   \Big\} \frac{1}{M}   \\
& + \frac{24}{\lambda_{C}}\Big\{   10 (1+\gamma)^2 \rho^2 _{\max} \cdot \Big[\Big(1 + \frac{\gamma \rho_{\max}}{ \min |\lambda(C)|}\Big)^2 \Big( 1 + \frac{2  }{\lambda_C } \Big)\cdot  \Big(\rho_{\max} \frac{ 1+\gamma }{ \min |\lambda(C)|}\Big)^2  \frac{\alpha^2}{\beta^2}   \\
&+  \Big( 1 + \frac{1}{ \min |\lambda(C)|} \Big)^2 \beta  \Big]   +  120 (1+\gamma)^2 \rho^2 _{\max} \frac{1}{\lambda_{\widehat{A}}} \cdot \Big[\Big(1 + \frac{\gamma \rho_{\max}}{ \min |\lambda(C)|}\Big)^2 \Big(1 + \frac{2  }{\lambda_C }   \Big)\cdot  \Big(\rho_{\max} \frac{ 1+\gamma }{ \min |\lambda(C)|}\Big)^2  \frac{\alpha^2}{\beta^2}   \\
&+  \Big( 1 + \frac{1}{ \min |\lambda(C)|} \Big)^2 \beta   \Big]  \cdot \Big[ \frac{1}{\alpha M}  +  \alpha  \cdot 5(1+\gamma)^2\rho^2_{\max}\Big(1 + \frac{\gamma \rho_{\max}}{ \min |\lambda(C)|}\Big)^2 \Big]  \Big\}\\
&\times \Big\{  \frac{32}{\lambda_{\widehat{A}}} \Big\{    \frac{96}{\lambda_{\widehat{A}} \lambda_{C}}\gamma^2 \rho_{\max}^2   \Big[  \frac{1}{\beta M}+ 10\beta + 10\gamma^2\rho_{\max}^2 \Big(1+ \frac{2  }{\lambda_C }  \Big)\cdot  \Big(\rho_{\max} \frac{ 1+\gamma }{ \min |\lambda(C)|}\Big)^2 \frac{\alpha^2}{\beta^2}    \Big] + \alpha  \cdot 5 \gamma^2\rho_{\max}^2  \Big\}  \\ &\times \Big\{    \beta \cdot \frac{24}{\lambda_C }   H_{\text{VR}}^2 + \frac{\alpha^2}{\beta^2} \cdot \Big(1  + \frac{2 }{\lambda_C }  \Big) \cdot \frac{24}{\lambda_C }\Big(\rho_{\max} \frac{ 1+\gamma }{ \min |\lambda(C)|}\Big)^2 G_{\text{VR}}^2 + \frac{1}{  M} \frac{24}{\lambda_C } (K_3 + K_4) \Big\}\\
&+   \frac{1}{M} \cdot  \frac{16}{\lambda_{\widehat{A}}}  \Big[ \frac{96}{\lambda_{\widehat{A}} \lambda_{C}}\gamma^2 \rho_{\max}^2   (2K_3+2K_4) + 2K_2\Big]  +    \frac{\beta}{M} \cdot   \frac{16}{\lambda_{\widehat{A}}}  \frac{96}{\lambda_{\widehat{A}} \lambda_{C}}\gamma^2 \rho_{\max}^2   10K_5  \\
&+   \frac{\alpha^2}{\beta M} \cdot \frac{16}{\lambda_{\widehat{A}}}  \big(1 + \frac{2  }{\lambda_C }   \big)  \cdot  \frac{96}{\lambda_{\widehat{A}} \lambda_{C}}\gamma^2 \rho_{\max}^2   10\Big(\rho_{\max} \frac{ 1+\gamma }{ \min |\lambda(C)|}\Big)^2  K_1 + \frac{\alpha}{M}  \cdot \frac{80 K_1}{\lambda_{\widehat{A}}} \Big\}\\ 
&+  \frac{1}{M}   \frac{48K_3+48K_4}{\lambda_{C}}  + \frac{\beta}{M} \frac{240K_5}{\lambda_{C}}   +  \frac{\alpha^2}{\beta^2} \frac{1}{M}\frac{240}{\lambda_{C}}\Big(\rho_{\max} \frac{ 1+\gamma }{ \min |\lambda(C)|}\Big)^2  K_1 \Big(1 + \frac{2  }{\lambda_C }  \Big) 
\end{align*}
where in the second step we assume $\frac{24}{\lambda_{C}}\Big\{   10 (1+\gamma)^2 \rho^2 _{\max} \cdot \Big[\Big(1 + \frac{\gamma \rho_{\max}}{ \min |\lambda(C)|}\Big)^2 \Big( 1 + \frac{2  }{\lambda_C } \Big)\cdot  \Big(\rho_{\max} \frac{ 1+\gamma }{ \min |\lambda(C)|}\Big)^2  \frac{\alpha^2}{\beta^2}   
+  \Big( 1 + \frac{1}{ \min |\lambda(C)|} \Big)^2 \beta  \Big]   +  120 (1+\gamma)^2 \rho^2 _{\max} \frac{1}{\lambda_{\widehat{A}}} \cdot \Big[\Big(1 + \frac{\gamma \rho_{\max}}{ \min |\lambda(C)|}\Big)^2 \Big(1 + \frac{2  }{\lambda_C }   \Big)\cdot  \Big(\rho_{\max} \frac{ 1+\gamma }{ \min |\lambda(C)|}\Big)^2  \frac{\alpha^2}{\beta^2}   +  \Big( 1 + \frac{1}{ \min |\lambda(C)|} \Big)^2 \beta   \Big]  \cdot \Big[ \frac{1}{\alpha M}  +  \alpha  \cdot 5(1+\gamma)^2\rho^2_{\max}\Big(1 + \frac{\gamma \rho_{\max}}{ \min |\lambda(C)|}\Big)^2 \Big]  \Big\} \leq 1$. Lastly, we telescope the above inequality and obtain that
\begin{align*}
&    \EE \| \tilde{z}^{(m)}\|^2 \\ 
\leq &  F^m \cdot \EE \|  \tilde{z}^{(0)}\|^2 + \frac{D^{m-1} - F^{m-1}}{D - F} \cdot \EE \| \tilde{\theta}^{(0)} - \theta^\ast\|^2  + \frac{\frac{D^{m-1} - F^{m-1}}{D - F} - \frac{E^{m-1} - F^{m-1}}{E - F}}{D - E} \EE \|  \tilde{z}^{(0)}\|^2\\
& + \frac{1}{1- F}\frac{24}{\lambda_{C}}\Big\{   10 (1+\gamma)^2 \rho^2 _{\max} \cdot \Big[\Big(1 + \frac{\gamma \rho_{\max}}{ \min |\lambda(C)|}\Big)^2 \Big( 1 + \frac{2  }{\lambda_C } \Big)\cdot  \Big(\rho_{\max} \frac{ 1+\gamma }{ \min |\lambda(C)|}\Big)^2  \frac{\alpha^2}{\beta^2}   \\
&+  \Big( 1 + \frac{1}{ \min |\lambda(C)|} \Big)^2 \beta  \Big]   +  120 (1+\gamma)^2 \rho^2 _{\max} \frac{1}{\lambda_{\widehat{A}}} \cdot \Big[\Big(1 + \frac{\gamma \rho_{\max}}{ \min |\lambda(C)|}\Big)^2 \Big(1 + \frac{2  }{\lambda_C }   \Big)\cdot  \Big(\rho_{\max} \frac{ 1+\gamma }{ \min |\lambda(C)|}\Big)^2  \frac{\alpha^2}{\beta^2}   \\
&+  \Big( 1 + \frac{1}{ \min |\lambda(C)|} \Big)^2 \beta   \Big]  \cdot \Big[ \frac{1}{\alpha M}  +  \alpha  \cdot 5(1+\gamma)^2\rho^2_{\max}\Big(1 + \frac{\gamma \rho_{\max}}{ \min |\lambda(C)|}\Big)^2 \Big]  \Big\}\\
&\times \Big\{  \frac{32}{\lambda_{\widehat{A}}} \Big\{    \frac{96}{\lambda_{\widehat{A}} \lambda_{C}}\gamma^2 \rho_{\max}^2   \Big[  \frac{1}{\beta M}+ 10\beta + 10\gamma^2\rho_{\max}^2 \Big(1+ \frac{2  }{\lambda_C }  \Big)\cdot  \Big(\rho_{\max} \frac{ 1+\gamma }{ \min |\lambda(C)|}\Big)^2 \frac{\alpha^2}{\beta^2}    \Big] + \alpha  \cdot 5 \gamma^2\rho_{\max}^2  \Big\}\\ &\times \Big\{    \beta \cdot \frac{24}{\lambda_C }   H_{\text{VR}}^2 + \frac{\alpha^2}{\beta^2} \cdot \Big(1  + \frac{2 }{\lambda_C }  \Big) \cdot \frac{24}{\lambda_C }\Big(\rho_{\max} \frac{ 1+\gamma }{ \min |\lambda(C)|}\Big)^2 G_{\text{VR}}^2 + \frac{1}{  M} \frac{24}{\lambda_C } (K_3 + K_4) \Big\}\\
&+   \frac{1}{M} \cdot  \frac{16}{\lambda_{\widehat{A}}}  \Big[ \frac{96}{\lambda_{\widehat{A}} \lambda_{C}}\gamma^2 \rho_{\max}^2   (2K_3+2K_4) + 2K_2\Big]  +    \frac{\beta}{M} \cdot   \frac{16}{\lambda_{\widehat{A}}}  \frac{96}{\lambda_{\widehat{A}} \lambda_{C}}\gamma^2 \rho_{\max}^2   10K_5  \\
&+   \frac{\alpha^2}{\beta M} \cdot \frac{16}{\lambda_{\widehat{A}}}  \big(1 + \frac{2  }{\lambda_C }   \big)  \cdot  \frac{96}{\lambda_{\widehat{A}} \lambda_{C}}\gamma^2 \rho_{\max}^2   10\Big(\rho_{\max} \frac{ 1+\gamma }{ \min |\lambda(C)|}\Big)^2  K_1 + \frac{\alpha}{M}  \cdot \frac{80 K_1}{\lambda_{\widehat{A}}} \Big\}\\ 
&+  \frac{1}{M}   \frac{48K_3+48K_4}{\lambda_{C}}  + \frac{\beta}{M} \frac{240K_5}{\lambda_{C}}   +  \frac{\alpha^2}{\beta^2} \frac{1}{M}\frac{240}{\lambda_{C}}\Big(\rho_{\max} \frac{ 1+\gamma }{ \min |\lambda(C)|}\Big)^2  K_1 \Big(1 + \frac{2  }{\lambda_C }  \Big) \\
&+  \frac{1}{1-F} \Big\{  \frac{1}{M}    (1+\gamma)^2 \rho^2 _{\max} \frac{2880}{\lambda_{\widehat{A}} \lambda_{C}} \cdot \Big[\Big(1 + \frac{\gamma \rho_{\max}}{ \min |\lambda(C)|}\Big)^2 \Big(1 + \frac{2  }{\lambda_C }\Big)\cdot  \Big(\rho_{\max} \frac{ 1+\gamma }{ \min |\lambda(C)|}\Big)^2   \frac{\alpha^2}{\beta^2} +  \Big( 1 + \frac{1}{ \min |\lambda(C)|} \Big)^2 \beta   \Big]  \cdot \Big\{    2K_2 + \alpha  \cdot 5K_1   \Big\}  \Big\}.
\end{align*}
To further simplify, note that the first three terms in the right hand side of the above inequality are in the order of $D^m$ ($D>E, D>F$). For the fourth term, it is easy to check that under the relation $\beta = \calO(\alpha^{2/3})$, it is in the order of  $\calO(\beta^3) + \calO(\frac{\beta}{M})$. The other terms are dominated by $\frac{1}{M}$. Therefore, the asymptotic error is in the order of $\calO(\beta^3) + \calO(\frac{1}{M})$. { To further elaborate this, note that the fourth term is a product of three curly brackets. The first bracket is in the order of $\calO(\beta) + \calO(\beta) \times \big(\calO(\frac{1}{\alpha M}) + \calO(\alpha)\big) = \calO(\beta) + \calO(\frac{\beta}{\alpha}\frac{1}{M})$, the second one is in the order of  $\calO(\frac{1}{\beta M}) + \calO(\beta)$ and the last one is in the order of $\calO(\beta)$. Hence their product is in the order of $\big( \calO(\beta) + \calO(\frac{\beta}{\alpha}\frac{1}{M}) \big)\times \big( \calO(\frac{1}{\beta M}) + \calO(\beta) \big) \times \calO(\beta)=\calO(\frac{\beta}{M}) + \calO(\frac{\beta}{\alpha}\frac{1}{M^2}) + \calO(\beta^3) + \calO(\frac{\beta^3}{\alpha} \frac{1}{M}) = \calO(\frac{\beta}{M}) + \calO(\frac{\beta}{\alpha}\frac{1}{M^2}) + \calO(\beta^3)$.} In summary, we have the following asymptotic result:
\begin{align*}
	 \EE \| \tilde{z}^{(m)}\|^2  \leq \calO(D^m) + \calO(\beta^3) + \calO(\frac{1}{M}).
\end{align*}
By following the same proof logic of Theorem \ref{thm: iid} and Theorem \ref{thm: markov}, the sample complexity under the optimal setting is  $\calO(\epsilon^{-1} \log{\epsilon}^{-1})$.  

\end{proof}

\section{Key Lemmas for Proving \Cref{thm: markov}}

\begin{lemma}[Preliminary Bound for $\sum_{t=0}^{M-1}\| \theta_{t}^{(m)} - \theta^\ast\|^2$] \label{lemma: pre-bound theta}Under the same assumptions as those of \Cref{thm: iid}, choose the learning rate $\alpha$ such that
	\begin{align}
	\alpha \leq \min \Big\{  \frac{\lambda_{\widehat{A}}}{30} / \Big[ (1+\gamma)^2\rho^2_{\max}\Big(1 + \frac{\gamma \rho_{\max}}{ \min |\lambda(C)|}\Big)^2 \Big], \frac{3}{5}\frac{1}{\lambda_{\widehat{A}}  } \Big\}.
	\end{align}
	Then, the following preliminary bound holds, 
\begin{align*}
&  \frac{\lambda_{\widehat{A}}}{12}\alpha   \sum_{t=0}^{M-1}\EE_{m,0} \| \theta_{t}^{(m)} - \theta^\ast\|^2 \\
\leq & \Big[ 1 +  \alpha^2M \cdot 5(1+\gamma)^2\rho^2_{\max}\Big(1 + \frac{\gamma \rho_{\max}}{ \min |\lambda(C)|}\Big)^2 \Big] \EE_{m,0}\| \tilde{\theta}^{(m-1)} - \theta^\ast\|^2  + \alpha \cdot 2K_2 + \alpha^2 \cdot5K_1 \\
&+  \alpha \cdot  \frac{6}{\lambda_{\widehat{A}}}\gamma^2 \rho_{\max}^2   \sum_{t=0}^{M-1}\EE_{m,0}\| z_{t}^{(m)}\|^2+ \alpha^2 M \cdot 5 \gamma^2\rho_{\max}^2 \EE_{m,0}\|\tilde{z}^{(m-1)}\|^2,
\end{align*}
where $K_1$ is specified in \cref{eq: def-K1-mc} of Lemma \ref{lemma: K1}, and $K_2$ is specified in   \cref{eq: def-K2-mc} of Lemma \ref{lemma: K2} .
\end{lemma}
\begin{proof}
	Based on the update rule of VRTDC for Markovian samples, we have that
	\[\theta_{t+1}^{(m)} = \Pi_{R_\theta}\Big[ \theta_{t}^{(m)} + \alpha[ G_{t}^{(m)}(\theta_{t}^{(m)}, z_{t}^{(m)})  - G_{t}^{(m)}(\tilde{\theta}^{(m-1)}, \tilde{z}^{(m-1)})  + {G}^{(m)}(\tilde{\theta}^{(m-1)}, \tilde{z}^{(m-1)}) ] \Big].\]
	The above update rule further implies that
	\begin{align*}
		\|\theta_{t+1}^{(m)} - \theta^\ast \|^2 &\overset{(i)}{\leq} \| \theta_{t}^{(m)} -\theta^\ast + \alpha[ G_{t}^{(m)}(\theta_{t}^{(m)}, z_{t}^{(m)})  - G_{t}^{(m)}(\tilde{\theta}^{(m-1)}, \tilde{z}^{(m-1)})  + {G}^{(m)}(\tilde{\theta}^{(m-1)}, \tilde{z}^{(m-1)}) ]\|^2 \\
		&=\| \theta_{t}^{(m)} -\theta^\ast\|^2 + \alpha^2 \| G_{t}^{(m)}(\theta_{t}^{(m)}, z_{t}^{(m)})  - G_{t}^{(m)}(\tilde{\theta}^{(m-1)}, \tilde{z}^{(m-1)})  + {G}^{(m)}(\tilde{\theta}^{(m-1)}, \tilde{z}^{(m-1)}) \|^2 \\
		&\quad + 2\alpha \langle \theta_{t}^{(m)} -\theta^\ast, G_{t}^{(m)}(\theta_{t}^{(m)}, z_{t}^{(m)})  - G_{t}^{(m)}(\tilde{\theta}^{(m-1)}, \tilde{z}^{(m-1)})  + {G}^{(m)}(\tilde{\theta}^{(m-1)}, \tilde{z}^{(m-1)})\rangle, \numberthis \label{eq: mc-tmp-00}
	\end{align*}
	where (i) uses the assumption that $R_\theta \ge \|\theta^*\|$ (i.e., $\theta^*$ is in the ball with radius $R_\theta$) and the fact that $\Pi_{R_\theta}$ is 1-Lipschitz. Then we take $\EE_{m,0}$ on both sides. An upper bound for the second term of \cref{eq: mc-tmp-00} is given in Lemma \ref{lemma: GVR}. Next, we consider the third term of \cref{eq: mc-tmp-00} and obtain that
	\begin{align*}
	 &\EE_{m,0} \langle \theta_{t}^{(m)} -\theta^\ast, G_{t}^{(m)}(\theta_{t}^{(m)}, z_{t}^{(m)})  - G_{t}^{(m)}(\tilde{\theta}^{(m-1)}, \tilde{z}^{(m-1)})  + {G}^{(m)}(\tilde{\theta}^{(m-1)}, \tilde{z}^{(m-1)})\rangle \\
	 =& \EE_{m,0} \langle \theta_{t}^{(m)} -\theta^\ast, G^{(m)}(\theta_{t}^{(m)}, z_{t}^{(m)})  \rangle \\
	 =& \EE_{m,0} \langle \theta_{t}^{(m)} -\theta^\ast, \widehat{A}^{(m)} \theta_{t}^{(m)} + \widehat{b}^{(m)} + \widehat{B}^{(m)} z_t^{(m)}\rangle \\
	 =& \EE_{m,0} \langle \theta_{t}^{(m)} -\theta^\ast, \widehat{A}^{(m)} \theta_{t}^{(m)} - \widehat{A}\theta_{t}^{(m)} + \widehat{A}\theta_{t}^{(m)} + \widehat{b}^{(m)} - \widehat{b} + \widehat{b} \rangle  +  \EE_{m,0} \langle \theta_{t}^{(m)} -\theta^\ast, \widehat{B}^{(m)} z_t^{(m)}\rangle \\
	 =& \EE_{m,0} \langle \theta_{t}^{(m)} -\theta^\ast, (\widehat{A}^{(m)}   - \widehat{A})\theta_{t}^{(m)} +( \widehat{b}^{(m)} - \widehat{b} ) \rangle + \EE_{m,0} \langle \theta_{t}^{(m)} -\theta^\ast,    \widehat{A}\theta_{t}^{(m)} - \widehat{A}\theta^\ast + \widehat{A}\theta^\ast+  \widehat{b}  \rangle \\
	 &\quad+  \EE_{m,0} \langle \theta_{t}^{(m)} -\theta^\ast, \widehat{B}^{(m)} z_t^{(m)}\rangle \\
	 =& \EE_{m,0} \langle \theta_{t}^{(m)} -\theta^\ast, (\widehat{A}^{(m)}   - \widehat{A})\theta_{t}^{(m)} +( \widehat{b}^{(m)} - \widehat{b} ) \rangle +\EE_{m,0} \langle \theta_{t}^{(m)} -\theta^\ast,    \widehat{A}(\theta_{t}^{(m)} -  \theta^\ast  ) \rangle \\
	 &\quad+  \EE_{m,0} \langle \theta_{t}^{(m)} -\theta^\ast, \widehat{B}^{(m)} z_t^{(m)}\rangle  \numberthis \label{eq: mc-tmp-01}
	\end{align*}
	The first term of \cref{eq: mc-tmp-01} is bounded by Lemma \ref{lemma: K2}. The second term of \cref{eq: mc-tmp-01} can be bounded by using the property of negative definite matrix: $\lambda_{\max}(\widehat{A}) \| \theta - \theta^\ast \|^2 \geq (\theta- \theta^\ast )^T \widehat{A} (\theta- \theta^\ast)   \geq \lambda_{\min}(\widehat{A}) \| \theta - \theta^\ast \|^2$.
	Recall that $-\lambda_{\widehat{A}}:=\lambda_{\max}(\widehat{A} + \widehat{A}^T)$, and we obtain that 
	$$\EE_{m,0} \langle \theta_{t}^{(m)} -\theta^\ast,    \widehat{A}(\theta_{t}^{(m)} -  \theta^\ast  ) \rangle \leq -\frac{\lambda_{\widehat{A}}}{2}\EE_{m,0}\| \theta_{t}^{(m)} -\theta^\ast\|^2.$$
 	The third term  of \cref{eq: mc-tmp-01} is bounded using the polarization identity,
	\begin{align*}
		\EE_{m,0} \langle \theta_{t}^{(m)} -\theta^\ast, \widehat{B}^{(m)} z_t^{(m)}\rangle &\leq \frac{1}{2}\cdot \frac{\lambda_{\widehat{A}}}{3}\EE_{m,0}\| \theta_{t}^{(m)} -\theta^\ast\|^2 + \frac{1}{2} \cdot \frac{3}{\lambda_{\widehat{A}}}\gamma^2 \rho_{\max}^2\EE_{m,0}\| z_{t}^{(m)} \|^2 . 
	\end{align*}
	Substituting the above bounds into the third term of \cref{eq: mc-tmp-00}, we obtain that
	\begin{align*}
		&\EE_{m,0} \langle \theta_{t}^{(m)} -\theta^\ast, G_{t}^{(m)}(\theta_{t}^{(m)}, z_{t}^{(m)})  - G_{t}^{(m)}(\tilde{\theta}^{(m-1)}, \tilde{z}^{(m-1)})  + {G}^{(m)}(\tilde{\theta}^{(m-1)}, \tilde{z}^{(m-1)})\rangle \\
		\leq & -\frac{\lambda_{\widehat{A}}}{12}\EE_{m,0}  \|\theta_{t}^{(m)} -\theta^\ast\|^2 +  \frac{K_2}{M}+ \frac{1}{2} \cdot \frac{3}{\lambda_{\widehat{A}}}\gamma^2 \rho_{\max}^2\EE_{m,0}\| z_{t}^{(m)} \|^2 . \numberthis \label{eq: mc-tmp-02}
	\end{align*}
	Then, substituting \cref{eq: mc-tmp-02} into \cref{eq: mc-tmp-00} and re-arranging, we obtain that
	\begin{align*}
		&\EE_{m,0} \|\theta_{t+1}^{(m)} - \theta^\ast \|^2 \\
		\leq&  \EE_{m,0} \| \theta_{t}^{(m)} -\theta^\ast\|^2 +  \alpha\Big[-\frac{\lambda_{\widehat{A}}}{6}\EE_{m,0}  \|\theta_{t}^{(m)} -\theta^\ast\|^2 +  \frac{2 K_2}{M}+   \frac{3}{\lambda_{\widehat{A}}}\gamma^2 \rho_{\max}^2\EE_{m,0}\| z_{t}^{(m)} \|^2 \Big]\\
		& +\alpha^2\Big[  5(1+\gamma)^2\rho^2_{\max}\Big(1 + \frac{\gamma \rho_{\max}}{ \min |\lambda(C)|}\Big)^2\Big( \EE_{m,0} \| \theta_{t}^{(m)} - \theta^\ast\|^2  + \EE_{m,0}\| \tilde{\theta}^{(m-1)} - \theta^\ast\|^2 \Big)\Big] \\
		& + \alpha^2 \Big[ 5 \gamma^2\rho_{\max}^2 \Big(\EE_{m,0}\| z_{t}^{(m)}\|^2  + \EE_{m,0}\|\tilde{z}^{(m-1)}\|^2\Big) + \frac{5K_1}{M}  \Big] \\
		=&\EE_{m,0} \| \theta_{t}^{(m)} -\theta^\ast\|^2 - \Big( \frac{\lambda_{\widehat{A}}}{6}\alpha - \alpha^2 \cdot 5(1+\gamma)^2\rho^2_{\max}\Big(1 + \frac{\gamma \rho_{\max}}{ \min |\lambda(C)|}\Big)^2  \Big) \EE_{m,0} \| \theta_{t}^{(m)} - \theta^\ast\|^2 \\
		&+  \alpha^2 \cdot 5(1+\gamma)^2\rho^2_{\max}\Big(1 + \frac{\gamma \rho_{\max}}{ \min |\lambda(C)|}\Big)^2  \EE_{m,0}\| \tilde{\theta}^{(m-1)} - \theta^\ast\|^2 + \frac{\alpha}{M} \cdot 2K_2 + \frac{\alpha^2 }{M} \cdot 5K_1 \\
		&+ \Big(\alpha \cdot  \frac{3}{\lambda_{\widehat{A}}} \gamma^2 \rho_{\max}^2 + \alpha^2 \cdot 5 \gamma^2\rho_{\max}^2\Big) \EE_{m,0}\| z_{t}^{(m)}\|^2+ \alpha^2 \cdot 5 \gamma^2\rho_{\max}^2 \EE_{m,0}\|\tilde{z}^{(m-1)}\|^2.
	\end{align*}
	Summing the above inequality over $t=0, \dots, M-1$, we obtain the following desired bound 
	\begin{align*}
		& \Big( \frac{\lambda_{\widehat{A}}}{6}\alpha - \alpha^2 \cdot 5(1+\gamma)^2\rho^2_{\max}\Big(1 + \frac{\gamma \rho_{\max}}{ \min |\lambda(C)|}\Big)^2  \Big)  \sum_{t=0}^{M-1}\EE_{m,0} \| \theta_{t}^{(m)} - \theta^\ast\|^2 \\
		\leq & \Big[ 1 +  \alpha^2M \cdot 5(1+\gamma)^2\rho^2_{\max}\Big(1 + \frac{\gamma \rho_{\max}}{ \min |\lambda(C)|}\Big)^2 \Big] \EE_{m,0}\| \tilde{\theta}^{(m-1)} - \theta^\ast\|^2  + \alpha \cdot 2K_2 + \alpha^2 \cdot5K_1 \\
		&+ \Big(\alpha \cdot  \frac{3}{\lambda_{\widehat{A}}}\gamma^2 \rho_{\max}^2 + \alpha^2 \cdot 5 \gamma^2\rho_{\max}^2\Big) \sum_{t=0}^{M-1}\EE_{m,0}\| z_{t}^{(m)}\|^2+ \alpha^2 M \cdot 5 \gamma^2\rho_{\max}^2 \EE_{m,0}\|\tilde{z}^{(m-1)}\|^2. \numberthis \label{eq: mc-tmp-03}
	\end{align*}
	Lastly, we simplify the above bound by choosing sufficiently small $\alpha$ such that
		  $ \frac{\lambda_{\widehat{A}}}{6}\alpha - \alpha^2 \cdot 5(1+\gamma)^2\rho^2_{\max}\Big(1 + \frac{\gamma \rho_{\max}}{ \min |\lambda(C)|}\Big)^2   \geq  \frac{\lambda_{\widehat{A}}}{12}\alpha$, and
		  $ \alpha^2 \cdot 5 \gamma^2\rho_{\max}^2 \leq \alpha \cdot  \frac{3}{\lambda_{\widehat{A}}}\gamma^2 \rho_{\max}^2$.  	Note that these requirements can be implied by 
	\begin{align*}
		\alpha \leq \min \Big\{  \frac{\lambda_{\widehat{A}}}{30} / \Big[ (1+\gamma)^2\rho^2_{\max}\Big(1 + \frac{\gamma \rho_{\max}}{ \min |\lambda(C)|}\Big)^2 \Big], \frac{3}{5}\frac{1}{\lambda_{\widehat{A}}  } \Big\}.
	\end{align*}
	Applying these simplifications, \cref{eq: mc-tmp-03} becomes
	\begin{align*}
	&  \frac{\lambda_{\widehat{A}}}{12}\alpha   \sum_{t=0}^{M-1}\EE_{m,0} \| \theta_{t}^{(m)} - \theta^\ast\|^2 \\
	\leq & \Big[ 1 +  \alpha^2M \cdot 5(1+\gamma)^2\rho^2_{\max}\Big(1 + \frac{\gamma \rho_{\max}}{ \min |\lambda(C)|}\Big)^2 \Big] \EE_{m,0}\| \tilde{\theta}^{(m-1)} - \theta^\ast\|^2  + \alpha \cdot 2K_2 + \alpha^2 \cdot5K_1 \\
	&+  \alpha \cdot  \frac{6}{\lambda_{\widehat{A}}}\gamma^2 \rho_{\max}^2   \sum_{t=0}^{M-1}\EE_{m,0}\| z_{t}^{(m)}\|^2+ \alpha^2 M \cdot 5 \gamma^2\rho_{\max}^2 \EE_{m,0}\|\tilde{z}^{(m-1)}\|^2. 
	\end{align*}
\end{proof}

\begin{lemma}[Convergence of $\EE\| \tilde{z}^{(m)} \|^2$]
	 \label{lemma: conv-z}
	 Under the same assumptions as those of \Cref{thm: markov}, choose the learning rate $\beta$ and the batch size $M$ such that
	 $\beta < 1$ and $M\beta > \frac{12}{\lambda_{C}}$. 
	 Then, the following preliminary bound holds.
	\begin{align*}
\EE\| \tilde{z}^{(m)} \|^2 &\leq   \Big(\frac{1}{M\beta} \cdot \frac{12}{\lambda_C }\Big)^m \EE \|  \tilde{z}^{(0)}\|^2 \\
&\quad  + 2 \cdot \Big[ \beta \cdot \frac{24}{\lambda_C }   H_{\text{VR}}^2 + \frac{\alpha^2}{\beta^2} \cdot \Big(1  + \frac{2 }{\lambda_C }  \Big) \cdot \frac{24}{\lambda_C }\Big(\rho_{\max} \frac{ 1+\gamma }{ \min |\lambda(C)|}\Big)^2 G_{\text{VR}}^2 + \frac{1}{  M} \frac{24}{\lambda_C } (K_3 + K_4)\Big],
\end{align*} 
where $K_3$ is specified in \cref{eq: def-K3-mc} of Lemma \ref{lemma: K3}, and $K_4$ is specified in   \cref{eq: def-K4-mc} of Lemma \ref{lemma: K4}.
\end{lemma}
\begin{proof}
	Similar to Lemma \ref{lemma: iid conv-z}, we can obtain the one-step update rule of $z_{t}^{(m)}$ based on the one-step update rule of $w_{t}^{(m)}$. Combining this update rule with the assumption that $R_w\ge 2\|C^{-1}\|\|A\|R_{\theta}$, we obtain that
	\begin{align*}
		\| z_{t+1}^{(m)}\|^2 &\leq \|  z_t^{(m)} + \beta\Big[   H_{t}^{(m)}(\theta_{t}^{(m)}, z_{t}^{(m)})  - H_{t}^{(m)}(\tilde{\theta}^{(m-1)}, \tilde{z}^{(m-1)})  +  {H}^{(m)}(\tilde{\theta}^{(m-1)}, \tilde{z}^{(m-1)}) \Big] \\
		&\quad+ C^{-1}A(\theta_{t+1}^{(m)}  - \theta_{t}^{(m)}) \|^2 \\
		&= \|  z_t^{(m)}\|^2 + 2\beta^2 \|    H_{t}^{(m)}(\theta_{t}^{(m)}, z_{t}^{(m)})  - H_{t}^{(m)}(\tilde{\theta}^{(m-1)}, \tilde{z}^{(m-1)})  +  {H}^{(m)}(\tilde{\theta}^{(m-1)}, \tilde{z}^{(m-1)}) \|^2   \\
		&\quad + 2\| C^{-1}A(\theta_{t+1}^{(m)}  - \theta_{t}^{(m)})\|^2 \\
		&\quad+ 2\beta \langle z_t^{(m)}, H_{t}^{(m)}(\theta_{t}^{(m)}, z_{t}^{(m)})  - H_{t}^{(m)}(\tilde{\theta}^{(m-1)}, \tilde{z}^{(m-1)})  +  {H}^{(m)}(\tilde{\theta}^{(m-1)}, \tilde{z}^{(m-1)}) \rangle  \\
		&\quad + 2\langle z_t^{(m)}, C^{-1}A(\theta_{t+1}^{(m)}  - \theta_{t}^{(m)}).  \rangle  \numberthis \label{eq: mc-tmp-10}
	\end{align*}
	For the last term of \cref{eq: mc-tmp-10}, we bound it as
	\begin{align}
		2 \langle z_t^{(m)}, C^{-1}A(\theta_{t+1}^{(m)}  - \theta_{t}^{(m)})  \rangle \leq \frac{\lambda_C}{2}\beta \| z_t^{(m)}\|^2 + \frac{2}{\lambda_C } \frac{1}{\beta} \| C^{-1}A (\theta_{t+1} - \theta_{t})\|^2.  \label{eq: mc-tmp-11}
	\end{align}
	Then, we apply Lemma \ref{lemma: GVR-const} to bound the last term of \cref{eq: mc-tmp-11}. Also, we apply Lemma \ref{lemma: HVR-const} to bound the second term of \cref{eq: mc-tmp-10}. Then, we obtain that  
	\begin{align*}
		\| z_{t+1}^{(m)}\|^2 &\leq \|  z_t^{(m)}\|^2 + 2\beta^2 H_{\text{VR}}^2 + \Big(\alpha^2 + \frac{2 \alpha^2}{\lambda_C } \frac{1}{\beta} \Big)\cdot 2\Big(\rho_{\max} \frac{ 1+\gamma }{ \min |\lambda(C)|}\Big)^2 G_{\text{VR}}^2 + \frac{\lambda_C}{2}\beta \| z_t^{(m)}\|^2  \\
		&\quad +  2\beta \langle z_t^{(m)}, H_{t}^{(m)}(\theta_{t}^{(m)}, z_{t}^{(m)})  - H_{t}^{(m)}(\tilde{\theta}^{(m-1)}, \tilde{z}^{(m-1)})  +  {H}^{(m)}(\tilde{\theta}^{(m-1)}, \tilde{z}^{(m-1)}) \rangle \numberthis \label{eq: mc-tmp-12}
	\end{align*}
	Next, we further bound the last term of \cref{eq: mc-tmp-12}.
	\begin{align*}
		& \EE_{m,0} \langle z_t^{(m)}, H_{t}^{(m)}(\theta_{t}^{(m)}, z_{t}^{(m)})  - H_{t}^{(m)}(\tilde{\theta}^{(m-1)}, \tilde{z}^{(m-1)})  +  {H}^{(m)}(\tilde{\theta}^{(m-1)}, \tilde{z}^{(m-1)}) \rangle \\
		= & \EE_{m,0} \langle z_t^{(m)}, H^{(m)}(\theta_{t}^{(m)}, z_{t}^{(m)})   \rangle \\
		= & \EE_{m,0} \langle z_t^{(m)}, \bar{A}^{(m)}\theta_{t}^{(m)} +  \bar{b}^{(m)} +  {C}^{(m)} z_{t}^{(m)} \rangle \\
		= & \EE_{m,0} \langle z_t^{(m)}, \bar{A}^{(m)}\theta_{t}^{(m)} +  \bar{b}^{(m)}\rangle +  \EE_{m,0} \langle z_t^{(m)}, ({C}^{(m)} - C) z_{t}^{(m)} \rangle + \EE_{m,0} \langle z_t^{(m)}, C z_{t}^{(m)} \rangle.   \numberthis \label{eq: mc-tmp-13}
	\end{align*}
	Then, we apply Lemma \ref{lemma: K3} to bound the first term of \cref{eq: mc-tmp-13}, apply Lemma \ref{lemma: K4} to bound the second term of \cref{eq: mc-tmp-13} and apply the negative definiteness of $C$ to bound the last term of \cref{eq: mc-tmp-13}. We obtain that
	\begin{align*}
		& \EE_{m,0} \langle z_t^{(m)}, H_{t}^{(m)}(\theta_{t}^{(m)}, z_{t}^{(m)})  - H_{t}^{(m)}(\tilde{\theta}^{(m-1)}, \tilde{z}^{(m-1)})  +  {H}^{(m)}(\tilde{\theta}^{(m-1)}, \tilde{z}^{(m-1)}) \rangle \\
		\leq &  \Big(\frac{\lambda_C}{8}\EE_{m,0}\|z_t^{(m)}\|^2 + \frac{K_3}{M}\Big)  +  \Big(\frac{\lambda_C}{12}\EE_{m,0}\|z_t^{(m)}\|^2 + \frac{K_4}{M}\Big)   - \frac{\lambda_C}{2} \EE_{m,0}\|z_t^{(m)}\|^2 \\
		=  &  -\frac{7\lambda_C}{24}\EE_{m,0}\|z_t^{(m)}\|^2 + \frac{K_3 + K_4}{M}. \numberthis \label{eq: 100}
	\end{align*} 
	Substituting the above inequality into \cref{eq: mc-tmp-12} yields that
	\begin{align*}
	\EE_{m,0} \| z_{t+1}^{(m)}\|^2 &\leq \EE_{m,0} \|  z_t^{(m)}\|^2 + 2\beta^2 H_{\text{VR}}^2 + \Big(\alpha^2 + \frac{2 \alpha^2}{\lambda_C } \frac{1}{\beta} \Big)\cdot 2\Big(\rho_{\max} \frac{ 1+\gamma }{ \min |\lambda(C)|}\Big)^2 G_{\text{VR}}^2 \\
	&\quad   - \frac{\lambda_C}{12}\beta \EE_{m,0} \| z_t^{(m)}\|^2 + \frac{2K_3 + 2K_4}{M} \beta.
	\end{align*}
	Telescoping the above inequality over one batch, we further obtain that
	\begin{align*}
	\EE_{m,0} \| z_{M}^{(m)}\|^2 &\leq \EE_{m,0} \|  z_0^{(m)}\|^2 + 2\beta^2 M H_{\text{VR}}^2 + \Big(\alpha^2 + \frac{2 \alpha^2}{\lambda_C } \frac{1}{\beta} \Big)M\cdot 2\Big(\rho_{\max} \frac{ 1+\gamma }{ \min |\lambda(C)|}\Big)^2 G_{\text{VR}}^2 \\
	&\quad   - \frac{\lambda_C}{12}\beta \sum_{t=0}^{M-1}\EE_{m,0} \| z_t^{(m)}\|^2 +  (2K_3 + 2K_4)\beta.
	\end{align*}
	Next, we move the term $\sum_{t=0}^{M-1}\EE_{m,0} \| z_t^{(m)}\|^2$ in the above inequality to the left-hand side and apply Jensen's inequality, we obtain that
	\begin{align*}
		\frac{\lambda_C}{12}\beta M \EE_{m,0} \| \tilde{z}^{(m)} \|^2 &\leq  \|  \tilde{z}^{(m-1)}\|^2 + 2\beta^2 M H_{\text{VR}}^2 \\
		&\quad+ \Big(\alpha^2 + \frac{2 \alpha^2}{\lambda_C } \frac{1}{\beta} \Big)M\cdot 2\Big(\rho_{\max} \frac{ 1+\gamma }{ \min |\lambda(C)|}\Big)^2 G_{\text{VR}}^2 +  (2K_3 + 2K_4)\beta.
	\end{align*}
	Lastly, we divide $\frac{\lambda_C}{12}\beta M $ on both sides of the above inequality and obtain that
	\begin{align*}
	\EE\| \tilde{z}^{(m)} \|^2 &\leq   \frac{1}{M\beta} \cdot \frac{12}{\lambda_C }\EE \|  \tilde{z}^{(m-1)}\|^2 +  \beta \cdot \frac{24}{\lambda_C }   H_{\text{VR}}^2 + \Big(\frac{\alpha^2}{\beta} + \frac{2 }{\lambda_C } \frac{\alpha^2}{\beta^2} \Big) \cdot \frac{24}{\lambda_C }\Big(\rho_{\max} \frac{ 1+\gamma }{ \min |\lambda(C)|}\Big)^2 G_{\text{VR}}^2 \\
	&\quad + \frac{1}{  M} \frac{24}{\lambda_C } (K_3 + K_4),
	\end{align*}
	which, after telescoping, leads to 
	\begin{align*}
	\EE\| \tilde{z}^{(m)} \|^2 &\leq   \Big(\frac{1}{M\beta} \cdot \frac{12}{\lambda_C }\Big)^m \EE \|  \tilde{z}^{(0)}\|^2 \\
	&\quad  + \frac{1}{1 - \frac{1}{M\beta} \cdot \frac{12}{\lambda_C } }\Big[ \beta \cdot \frac{24}{\lambda_C }   H_{\text{VR}}^2 + \Big(\frac{\alpha^2}{\beta} + \frac{2 }{\lambda_C } \frac{\alpha^2}{\beta^2} \Big) \cdot \frac{24}{\lambda_C }\Big(\rho_{\max} \frac{ 1+\gamma }{ \min |\lambda(C)|}\Big)^2 G_{\text{VR}}^2 + \frac{1}{  M} \frac{24}{\lambda_C } (K_3 + K_4)\Big]
	\end{align*}
	To further simplify the above inequality, we assume $\beta < 1$ and $M\beta > \frac{24}{\lambda_{C}}$. Then, we have
	\begin{align*}
	\EE\| \tilde{z}^{(m)} \|^2 &\leq   \Big(\frac{1}{M\beta} \cdot \frac{12}{\lambda_C }\Big)^m \EE \|  \tilde{z}^{(0)}\|^2 \\
	&\quad  + 2 \cdot \Big[ \beta \cdot \frac{24}{\lambda_C }   H_{\text{VR}}^2 + \frac{\alpha^2}{\beta^2} \cdot \Big(1  + \frac{2 }{\lambda_C }  \Big) \cdot \frac{24}{\lambda_C }\Big(\rho_{\max} \frac{ 1+\gamma }{ \min |\lambda(C)|}\Big)^2 G_{\text{VR}}^2 + \frac{1}{  M} \frac{24}{\lambda_C } (K_3 + K_4)\Big].
	\end{align*}

\end{proof}

\begin{lemma}[Preliminary Bound for $\sum_{t=0}^{M-1} \|z_{t}^{(m)}\|^2$]
	\label{lemma: pre-bound z}
	Under the same assumptions as those of \Cref{thm: markov}, choose the learning rate $\alpha$ and $\beta$  such that 
	\begin{align*}
		\frac{\lambda_{C}}{12}\beta - 10\beta^2 - 10\gamma^2\rho_{\max}^2 \Big(\alpha^2 + \frac{2 \alpha^2}{\lambda_C } \frac{1}{\beta} \Big)\cdot  \Big(\rho_{\max} \frac{ 1+\gamma }{ \min |\lambda(C)|}\Big)^2 \geq \frac{\lambda_{C}}{16}\beta.
	\end{align*}
	Then, the following preliminary bound holds. 
	\begin{align*}
	&   \frac{\lambda_{C}}{16}\beta  \sum_{t=0}^{M-1}\EE_{m,0}\|z_t^{(m)}\|^2 \\
	\leq & \Big[  1+\Big[10\beta^2 + 10\gamma^2\rho_{\max}^2 \Big(\alpha^2 + \frac{2 \alpha^2}{\lambda_C } \frac{1}{\beta} \Big)\cdot  \Big(\rho_{\max} \frac{ 1+\gamma }{ \min |\lambda(C)|}\Big)^2  \Big]M  \Big]\EE_{m,0}\|  \tilde{z}^{(m-1)}\|^2\\
	&+ 10 (1+\gamma)^2 \rho^2 _{\max} \cdot \Big[\Big(1 + \frac{\gamma \rho_{\max}}{ \min |\lambda(C)|}\Big)^2 \Big(\alpha^2 + \frac{2 \alpha^2}{\lambda_C } \frac{1}{\beta} \Big)\cdot  \Big(\rho_{\max} \frac{ 1+\gamma }{ \min |\lambda(C)|}\Big)^2   \\
	&+  \Big( 1 + \frac{1}{ \min |\lambda(C)|} \Big)^2 \beta^2  \Big] \Big( \sum_{t=0}^{M-1}\EE_{m,0} \| \theta_{t}^{(m)} - \theta^\ast\|^2  + M \EE_{m,0}\| \tilde{\theta}^{(m-1)} - \theta^\ast\|^2 \Big) \\ 
	&+    (2K_3+2K_4)\beta + 10K_5 \beta^2 +  10\Big(\rho_{\max} \frac{ 1+\gamma }{ \min |\lambda(C)|}\Big)^2  K_1 \Big(\alpha^2 + \frac{2 \alpha^2}{\lambda_C } \frac{1}{\beta} \Big) .
	\end{align*}
\end{lemma}
\begin{proof}
	Similar to Lemma \ref{lemma: conv-z}, we firstly consider the one-step update of $z_{t}^{(m)}$: 
	\begin{align*}
	\| z_{t+1}^{(m)}\|^2  
	&\leq \|  z_t^{(m)}\|^2 + 2\beta^2 \|    H_{t}^{(m)}(\theta_{t}^{(m)}, z_{t}^{(m)})  - H_{t}^{(m)}(\tilde{\theta}^{(m-1)}, \tilde{z}^{(m-1)})  +  {H}^{(m)}(\tilde{\theta}^{(m-1)}, \tilde{z}^{(m-1)}) \|^2   \\
	&\quad + 2\| C^{-1}A(\theta_{t+1}^{(m)}  - \theta_{t}^{(m)})\|^2 +2\langle z_t^{(m)}, C^{-1}A(\theta_{t+1}^{(m)}  - \theta_{t}^{(m)})  \rangle \\
	&\quad +  2\beta \langle z_t^{(m)}, H_{t}^{(m)}(\theta_{t}^{(m)}, z_{t}^{(m)})  - H_{t}^{(m)}(\tilde{\theta}^{(m-1)}, \tilde{z}^{(m-1)})  +  {H}^{(m)}(\tilde{\theta}^{(m-1)}, \tilde{z}^{(m-1)}) \rangle \\
	&\leq  \|  z_t^{(m)}\|^2 + 2\beta^2 \|    H_{t}^{(m)}(\theta_{t}^{(m)}, z_{t}^{(m)})  - H_{t}^{(m)}(\tilde{\theta}^{(m-1)}, \tilde{z}^{(m-1)})  +  {H}^{(m)}(\tilde{\theta}^{(m-1)}, \tilde{z}^{(m-1)}) \|^2   \\
	&\quad + 2\| C^{-1}A(\theta_{t+1}^{(m)}  - \theta_{t}^{(m)})\|^2 +\frac{\lambda_C}{2}\beta \| z_t^{(m)}\|^2 + \frac{2}{\lambda_C } \frac{1}{\beta} \| C^{-1}A (\theta_{t+1} - \theta_{t})\|^2. \\
	&\quad +  2\beta \langle z_t^{(m)}, H_{t}^{(m)}(\theta_{t}^{(m)}, z_{t}^{(m)})  - H_{t}^{(m)}(\tilde{\theta}^{(m-1)}, \tilde{z}^{(m-1)})  +  {H}^{(m)}(\tilde{\theta}^{(m-1)}, \tilde{z}^{(m-1)}) \rangle \\
	&\leq \|  z_t^{(m)}\|^2 +\frac{\lambda_C}{2}\beta \| z_t^{(m)}\|^2 + 2\beta^2 \|    H_{t}^{(m)}(\theta_{t}^{(m)}, z_{t}^{(m)})  - H_{t}^{(m)}(\tilde{\theta}^{(m-1)}, \tilde{z}^{(m-1)})  +  {H}^{(m)}(\tilde{\theta}^{(m-1)}, \tilde{z}^{(m-1)}) \|^2   \\
	&\quad + \Big(\alpha^2 + \frac{2 \alpha^2}{\lambda_C } \frac{1}{\beta} \Big)\cdot 2\Big(\rho_{\max} \frac{ 1+\gamma }{ \min |\lambda(C)|}\Big)^2 \|    G_{t}^{(m)}(\theta_{t}^{(m)}, z_{t}^{(m)})  - G_{t}^{(m)}(\tilde{\theta}^{(m-1)}, \tilde{z}^{(m-1)})  +  {G}^{(m)}(\tilde{\theta}^{(m-1)}, \tilde{z}^{(m-1)}) \|^2 \\
	&\quad +  2\beta \langle z_t^{(m)}, H_{t}^{(m)}(\theta_{t}^{(m)}, z_{t}^{(m)})  - H_{t}^{(m)}(\tilde{\theta}^{(m-1)}, \tilde{z}^{(m-1)})  +  {H}^{(m)}(\tilde{\theta}^{(m-1)}, \tilde{z}^{(m-1)}) \rangle \numberthis \label{eq: mc-tmp20}
	\end{align*}
	where the second inequality applies the polarization identity to $\langle z_t^{(m)}, C^{-1}A(\theta_{t+1}^{(m)}  - \theta_{t}^{(m)})  \rangle $ .
	For the last term of \cref{eq: mc-tmp20}, it is bounded by \cref{eq: 100} as 
	\begin{align*}
		&\EE_{m,0} \langle z_t^{(m)}, H_{t}^{(m)}(\theta_{t}^{(m)}, z_{t}^{(m)})  - H_{t}^{(m)}(\tilde{\theta}^{(m-1)}, \tilde{z}^{(m-1)})  +  {H}^{(m)}(\tilde{\theta}^{(m-1)}, \tilde{z}^{(m-1)}) \rangle  \\
		\leq  &-\frac{7\lambda_C}{24}\EE_{m,0}\|z_t^{(m)}\|^2 + \frac{K_3 + K_4}{M}.
	\end{align*}
	To further bound \cref{eq: mc-tmp20}, we apply Lemma \ref{lemma: GVR} to bound its fourth term and apply Lemma \ref{lemma: HVR} to bound its third term. Then, we obtain that
	\begin{align*}
		& \EE_{m,0}\| z_{t+1}^{(m)}\|^2   \\
		\leq & \EE_{m,0}  \|  z_t^{(m)}\|^2 -\frac{\lambda_C}{12}\beta \EE_{m,0}\|z_t^{(m)}\|^2   + \frac{2K_3 + 2K_4}{M}\beta\\
		&+ 2\beta^2 \Big[5\Big(   \EE_{m,0}\| z_{t}^{(m)}\|^2 + \EE_{m,0}\| \tilde{z}^{(m-1)}\|^2\Big)  + \frac{5K_5}{M}\\
		&+    5 (1+\gamma)^2 \rho^2 _{\max}\Big( 1 + \frac{1}{ \min |\lambda(C)|} \Big)^2\Big( \EE_{m,0}\| \theta_{t}^{(m)} - \theta^\ast\|^2 + \EE_{m,0}\| \tilde{\theta}^{(m-1)} - \theta^\ast\|^2  \Big) \Big]\\
		&+ \Big(\alpha^2 + \frac{2 \alpha^2}{\lambda_C } \frac{1}{\beta} \Big)\cdot 2\Big(\rho_{\max} \frac{ 1+\gamma }{ \min |\lambda(C)|}\Big)^2 \Big[5 \gamma^2\rho_{\max}^2 \Big(\EE_{m,0}\| z_{t}^{(m)}\|^2  + \EE_{m,0}\|\tilde{z}^{(m-1)}\|^2\Big) + \frac{5K_1}{M} \\
		&+    5(1+\gamma)^2\rho^2_{\max}\Big(1 + \frac{\gamma \rho_{\max}}{ \min |\lambda(C)|}\Big)^2\Big( \EE_{m,0} \| \theta_{t}^{(m)} - \theta^\ast\|^2  + \EE_{m,0}\| \tilde{\theta}^{(m-1)} - \theta^\ast\|^2 \Big)  \Big].
	\end{align*}
	Next, we re-arrange the above inequality and obtain that
	\begin{align*}
			& \EE_{m,0}\| z_{t+1}^{(m)}\|^2   \\
		\leq & \EE_{m,0}  \|  z_t^{(m)}\|^2 - \Big[ \frac{\lambda_{C}}{12}\beta - 10\beta^2 - 10\gamma^2\rho_{\max}^2 \Big(\alpha^2 + \frac{2 \alpha^2}{\lambda_C } \frac{1}{\beta} \Big)\cdot  \Big(\rho_{\max} \frac{ 1+\gamma }{ \min |\lambda(C)|}\Big)^2 \Big]\EE_{m,0}\|z_t^{(m)}\|^2  \\
		&+ \Big[10\beta^2 + 10\gamma^2\rho_{\max}^2 \Big(\alpha^2 + \frac{2 \alpha^2}{\lambda_C } \frac{1}{\beta} \Big)\cdot  \Big(\rho_{\max} \frac{ 1+\gamma }{ \min |\lambda(C)|}\Big)^2  \Big] \EE_{m,0}\| \tilde{z}^{(m-1)}\|^2 \\
		&+ 10 (1+\gamma)^2 \rho^2 _{\max} \cdot \Big[\Big(1 + \frac{\gamma \rho_{\max}}{ \min |\lambda(C)|}\Big)^2 \Big(\alpha^2 + \frac{2 \alpha^2}{\lambda_C } \frac{1}{\beta} \Big)\cdot  \Big(\rho_{\max} \frac{ 1+\gamma }{ \min |\lambda(C)|}\Big)^2   \\
		&+  \Big( 1 + \frac{1}{ \min |\lambda(C)|} \Big)^2 \beta^2  \Big] \Big( \EE_{m,0} \| \theta_{t}^{(m)} - \theta^\ast\|^2  + \EE_{m,0}\| \tilde{\theta}^{(m-1)} - \theta^\ast\|^2 \Big) \\ 
		&+ \frac{1}{M} \cdot \Big[ (2K_3+2K_4)\beta + 10K_5 \beta^2 +  10\Big(\rho_{\max} \frac{ 1+\gamma }{ \min |\lambda(C)|}\Big)^2  K_1 \Big(\alpha^2 + \frac{2 \alpha^2}{\lambda_C } \frac{1}{\beta} \Big)\Big].
	\end{align*}
	Telescoping the above inequality over one batch, we obtain that
	\begin{align*}
		 & \Big[ \frac{\lambda_{C}}{12}\beta - 10\beta^2 - 10\gamma^2\rho_{\max}^2 \Big(\alpha^2 + \frac{2 \alpha^2}{\lambda_C } \frac{1}{\beta} \Big)\cdot  \Big(\rho_{\max} \frac{ 1+\gamma }{ \min |\lambda(C)|}\Big)^2 \Big]\sum_{t=0}^{M-1}\EE_{m,0}\|z_t^{(m)}\|^2 \\
		 \leq & \Big[  1+\Big[10\beta^2 + 10\gamma^2\rho_{\max}^2 \Big(\alpha^2 + \frac{2 \alpha^2}{\lambda_C } \frac{1}{\beta} \Big)\cdot  \Big(\rho_{\max} \frac{ 1+\gamma }{ \min |\lambda(C)|}\Big)^2  \Big]M  \Big]\EE_{m,0}\|  \tilde{z}^{(m-1)}\|^2\\
		 &+ 10 (1+\gamma)^2 \rho^2 _{\max} \cdot \Big[\Big(1 + \frac{\gamma \rho_{\max}}{ \min |\lambda(C)|}\Big)^2 \Big(\alpha^2 + \frac{2 \alpha^2}{\lambda_C } \frac{1}{\beta} \Big)\cdot  \Big(\rho_{\max} \frac{ 1+\gamma }{ \min |\lambda(C)|}\Big)^2   \\
		 &+  \Big( 1 + \frac{1}{ \min |\lambda(C)|} \Big)^2 \beta^2  \Big] \Big( \sum_{t=0}^{M-1}\EE_{m,0} \| \theta_{t}^{(m)} - \theta^\ast\|^2  + M \EE_{m,0}\| \tilde{\theta}^{(m-1)} - \theta^\ast\|^2 \Big) \\ 
		 &+    (2K_3+2K_4)\beta + 10K_5 \beta^2 +  10\Big(\rho_{\max} \frac{ 1+\gamma }{ \min |\lambda(C)|}\Big)^2  K_1 \Big(\alpha^2 + \frac{2 \alpha^2}{\lambda_C } \frac{1}{\beta} \Big) 
	\end{align*}
	To simplify the above inequality, we assume $ \frac{\lambda_{C}}{12}\beta - 10\beta^2 - 10\gamma^2\rho_{\max}^2 (\alpha^2 + \frac{2 \alpha^2}{\lambda_C } \frac{1}{\beta} ) (\rho_{\max} \frac{ 1+\gamma }{ \min |\lambda(C)|})^2 \geq \frac{\lambda_{C}}{16}\beta$ and obtain the desired result. 
\end{proof}

\section{Other Supporting Lemmas for Proving \Cref{thm: markov}}
\begin{lemma}\label{lemma: K1}Under the same assumption as those of Theorem \ref{thm: markov}, the following inequality holds.
	\begin{align*}
	\EE_{m,0}  \|\widehat{A}^{(m)}\theta^\ast  + \widehat{b}^{(m)} \|^2 \leq \frac{K_1}{M},
	\end{align*}
	where $K_1$ is defined as 
	\begin{align}\label{eq: def-K1-mc}
		K_1 := \Big[ (1+\gamma)R_\theta + r_{\max} \Big]^2\rho^2_{\max}\Big(1 + \frac{\gamma \rho_{\max}}{ \min |\lambda(C)|}\Big)^2  \cdot \Big( 1+\kappa \frac{2 \rho}{1-\rho} \Big).
	\end{align} 
\end{lemma}
\begin{proof}
	Recall that $\widehat{A}^{(m)} = \sum_{t=0}^{M-1}\widehat{A}^{(m)}_t$ and  $\widehat{b}^{(m)} = \sum_{t=0}^{M-1}\widehat{b}_{t}^{(m)}$. We expand the square as follows:
	\begin{align*}
		\|\widehat{A}^{(m)}\theta^\ast  + \widehat{b}^{(m)} \|^2 &\leq  \|\sum_{t=0}^{M-1}\widehat{A}^{(m)}_t\theta^\ast  + \sum_{t=0}^{M-1}\widehat{b}_{t}^{(m)} \|^2 \\
		&= \frac{1}{M^2} \Big[  \sum_{i=j} \| \widehat{A}_i^{(m)}\theta^\ast  + \widehat{b}_i^{(m)}\|^2  + \sum_{i\neq j} \langle \widehat{A}_i^{(m)}\theta^\ast  + \widehat{b}_i^{(m)}, \widehat{A}_j^{(m)}\theta^\ast  + \widehat{b}_j^{(m)}\rangle \Big] \\
		&\leq \frac{1}{M} \cdot \Big[ (1+\gamma)R_\theta + r_{\max} \Big]^2\rho^2_{\max}\Big(1 + \frac{\gamma \rho_{\max}}{ \min |\lambda(C)|}\Big)^2 + \frac{1}{M^2 }\sum_{i\neq j} \langle \widehat{A}_i^{(m)}\theta^\ast  + \widehat{b}_i^{(m)}, \widehat{A}_j^{(m)}\theta^\ast  + \widehat{b}_j^{(m)}\rangle,
	\end{align*}
	where in the last step we apply Lemma \ref{lemma: const-2} to bound the first term. Now we consider the conditional expectation of the second inner product term. Without loss of generality, we assume $i<j$, then 
	\begin{align*}
		&\EE_{m,0} \langle \widehat{A}_i^{(m)}\theta^\ast  + \widehat{b}_i^{(m)}, \widehat{A}_j^{(m)}\theta^\ast  + \widehat{b}_j^{(m)}\rangle \\
		=&\EE_{m,0}\langle \widehat{A}_i^{(m)}\theta^\ast  + \widehat{b}_i^{(m)}, \EE_{m,i}\Big[\widehat{A}_j^{(m)}\theta^\ast  + \widehat{b}_j^{(m)}\Big]\rangle \\
		\leq& \Big[ (1+\gamma)R_\theta + r_{\max} \Big]\rho_{\max}\Big(1 + \frac{\gamma \rho_{\max}}{ \min |\lambda(C)|}\Big) \cdot \EE_{m,0} \|  \EE_{m,i}\Big[\Big(\widehat{A}_j^{(m)}\theta^\ast  - \widehat{A}\theta^\ast \Big)+ \widehat{A}\theta^\ast + \Big(\widehat{b}_j^{(m)} - \widehat{b}\Big) + \widehat{b}\Big]   \| \\
		\leq&  \Big[ (1+\gamma)R_\theta + r_{\max} \Big]\rho_{\max}\Big(1 + \frac{\gamma \rho_{\max}}{ \min |\lambda(C)|}\Big) \cdot\Big[    \EE_{m,0} \Big( \| \EE_{m,i} \widehat{A}_j^{(m)} - \widehat{A} \|  R_\theta \Big) + \EE_{m,0} \Big(\| \EE_{m,i} \widehat{b}_j^{(m)} - \widehat{b} \|\Big)\Big]\numberthis \label{eq: mc-tmp3}
	\end{align*}
	where in the last inequality we apply the equation $\widehat{A}\theta^\ast = (A - BC^{-1}A)(-A^{-1}b) = - b + BC^{-1} b = - \widehat{b}$. Moreover, by Assumption \ref{ass: ergodicity}, we have
	\begin{align*}
		\| \EE_{m,i} \widehat{A}_j^{(m)} - \widehat{A} \| &= \| \int \widehat{A}(s) \ \dd \PP(\cdot | s_{j-i}^{(m)}) - \int \widehat{A}(s)\ \dd \mu_{\pi_b}\|  \\
		&\leq  \| \widehat{A}(s) \| \cdot \text{dist}( \PP(\cdot | s_{j-i}^{(m)}),  \mu_{\pi_b}) \\
		&\leq  (1+\gamma)\rho_{\max}\Big(1 + \frac{\gamma \rho_{\max}}{ \min |\lambda(C)|}\Big) \kappa\rho^{j-i}. \numberthis\label{eq: mc-tmp1}
	\end{align*}  
	Similarly,
	\begin{align}\label{eq: mc-tmp2}
		\| \EE_{m,i} \widehat{b}_j^{(m)} - \widehat{b} \| &\leq \rho_{\max} r_{\max}\Big(1 + \frac{\gamma \rho_{\max}}{ \min |\lambda(C)|}\Big) \kappa\rho^{j-i}.
	\end{align}
	Substituting \cref{eq: mc-tmp1} and  \cref{eq: mc-tmp2} into  \cref{eq: mc-tmp3}  yields that
	\begin{align*}
			&\EE_{m,0} \sum_{i\neq j} \langle \widehat{A}_i^{(m)}\theta^\ast  + \widehat{b}_i^{(m)}, \widehat{A}_j^{(m)}\theta^\ast  + \widehat{b}_j^{(m)}\rangle \\
			\leq& \Big[ (1+\gamma)R_\theta + r_{\max} \Big]^2\rho^2_{\max}\Big(1 + \frac{\gamma \rho_{\max}}{ \min |\lambda(C)|}\Big)^2 \cdot \kappa \frac{2M\rho}{1-\rho},
	\end{align*}
	which leads to the desired bound
	\begin{align*}
		\EE_{m,0}  \|\widehat{A}^{(m)}\theta^\ast  + \widehat{b}^{(m)} \|^2 \leq \frac{1}{M} \cdot   \Big[ (1+\gamma)R_\theta + r_{\max} \Big]^2\rho^2_{\max}\Big(1 + \frac{\gamma \rho_{\max}}{ \min |\lambda(C)|}\Big)^2  \cdot \Big( 1+\kappa \frac{2 \rho}{1-\rho} \Big).
	\end{align*}
\end{proof}

\begin{lemma}\label{lemma: K2}Under the same assumptions as those of Theorem \ref{thm: markov}, the following inequality holds:
	\begin{align*}
		\EE_{m,0} \langle \theta_{t}^{(m)} -\theta^\ast, (\widehat{A}^{(m)}   - \widehat{A})\theta_{t}^{(m)} +( \widehat{b}^{(m)} - \widehat{b} ) \rangle \leq \frac{\lambda_{\widehat{A}}}{4}\EE_{m,0}  \|\theta_{t}^{(m)} -\theta^\ast\|^2 +  \frac{K_2}{M},
	\end{align*}
	where $K_2$ is defined as 
	\begin{align}
	\label{eq: def-K2-mc}
	K_2 :=  \frac{2}{\lambda_{\widehat{A}}}\Big[R_\theta^2 (1+\gamma)^2 + r^2_{\max}\Big] \cdot 4 \rho_{\max}^2  \Big(1 + \frac{\gamma \rho_{\max}}{ \min |\lambda(C)|}\Big)^2 \Big[1 + \kappa \frac{\rho}{1-\rho} \Big].
	\end{align}
\end{lemma}
\begin{proof}
	By the polarization identity and the Jensen's inequality, we obtain that
	\begin{align*}
		&\EE_{m,0} \langle \theta_{t}^{(m)} -\theta^\ast, (\widehat{A}^{(m)}   - \widehat{A})\theta_{t}^{(m)} +( \widehat{b}^{(m)} - \widehat{b} ) \rangle \\
		\leq & \frac{1}{2}\cdot \frac{\lambda_{\widehat{A}}}{2}\EE_{m,0}  \|\theta_{t}^{(m)} -\theta^\ast\|^2 + \frac{1}{2}\cdot \frac{2}{\lambda_{\widehat{A}}}\EE_{m,0} \| (\widehat{A}^{(m)}   - \widehat{A})\theta_{t}^{(m)} +( \widehat{b}^{(m)} - \widehat{b} )\|^2 \\
		\leq & \frac{\lambda_{\widehat{A}}}{4}\EE_{m,0}  \|\theta_{t}^{(m)} -\theta^\ast\|^2 +  \frac{2}{\lambda_{\widehat{A}}} R_\theta^2 \EE_{m,0} \| \widehat{A}^{(m)}   - \widehat{A}\|^2  + \frac{2}{\lambda_{\widehat{A}}}\EE_{m,0}  \| \widehat{b}^{(m)} - \widehat{b} \|^2 .
	\end{align*}
	Next, we further bound $\EE_{m,0}\| \widehat{A}^{(m)}   - \widehat{A}\|^2$ and $\EE_{m,0}  \| \widehat{b}^{(m)} - \widehat{b} \|^2 $, respectively. 
	\begin{align*}
		\EE_{m,0}\| \widehat{A}^{(m)}   - \widehat{A}\|^2 &\leq \EE_{m,0}\| \widehat{A}^{(m)}   - \widehat{A}\|_F^2 \\
		&\leq   \frac{1}{M^2} \EE_{m,0}\Big[  \sum_{i=j}\| \widehat{A}_i^{(m)}   - \widehat{A}\|_F^2   + \sum_{i\neq j} \langle \widehat{A}_i^{(m)}   - \widehat{A},  \widehat{A}_j^{(m)}   - \widehat{A}  \rangle\Big] \\
		&\leq \frac{1}{M^2} \EE_{m,0} \Big[ 4 (1+\gamma)^2\rho_{\max}^2\Big(1 + \frac{\gamma \rho_{\max}}{ \min |\lambda(C)|}\Big)^2 M + \sum_{i\neq j} \langle \widehat{A}_i^{(m)}   - \widehat{A},  \widehat{A}_j^{(m)}   - \widehat{A}  \rangle \Big].   
	\end{align*}
	For the last inner product term, without loss of generality, assume $i<j$ and we obtain that
	\begin{align*}
			&\EE_{m,0} \langle \widehat{A}_i^{(m)}   - \widehat{A},  \widehat{A}_j^{(m)}   - \widehat{A}  \rangle \\
			\leq & \EE_{m,0} \langle \widehat{A}_i^{(m)}   - \widehat{A},  \EE_{m,i} \widehat{A}_j^{(m)}   - \widehat{A}  \rangle \\
			\leq & 2 (1+\gamma)\rho_{\max}\Big(1 + \frac{\gamma \rho_{\max}}{ \min |\lambda(C)|}\Big) \EE_{m,0}\| \EE_{m,i} \widehat{A}_j^{(m)}   - \widehat{A} \|_F \\
			\leq & 2 (1+\gamma)^2\rho^2_{\max}\Big(1 + \frac{\gamma \rho_{\max}}{ \min |\lambda(C)|}\Big)^2 \kappa \rho^{j-i}. \numberthis \label{eq: mc-tmp4}
	\end{align*} 
	Summing the above inequality over all $i\neq j$ yields that
	\begin{align*}
		\EE_{m,0} \sum_{i\neq j} \langle \widehat{A}_i^{(m)}   - \widehat{A},  \widehat{A}_j^{(m)}   - \widehat{A}  \rangle \leq  4 (1+\gamma)^2\rho^2_{\max}\Big(1 + \frac{\gamma \rho_{\max}}{ \min |\lambda(C)|}\Big)^2 \kappa  \frac{ M\rho}{1-\rho},
	\end{align*}
	which implies that
	 \begin{align}
	 \EE_{m,0}\| \widehat{A}^{(m)}   - \widehat{A}\|^2 &\leq \frac{1}{M}\cdot  4 (1+\gamma)^2\rho_{\max}^2\Big(1 + \frac{\gamma \rho_{\max}}{ \min |\lambda(C)|}\Big)^2 \Big[1 + \kappa \frac{\rho}{1-\rho} \Big]. \label{eq: mc-tmp6}
 	 \end{align}
 	 Following the same approach, we obtain that
 	 \begin{align}
 	 	\EE_{m,0}\| \widehat{b}^{(m)}   - \widehat{b}\|^2 &\leq \frac{1}{M}\cdot 4 \rho_{\max}^2 r^2_{\max}\Big(1 + \frac{\gamma \rho_{\max}}{ \min |\lambda(C)|}\Big)^2 \Big[1 + \kappa \frac{\rho}{1-\rho} \Big]. \label{eq: mc-tmp5}
 	 \end{align}
 	 Combining  \cref{eq: mc-tmp4,eq: mc-tmp5,eq: mc-tmp6}, we obtain that
 	 \begin{align*}
 	 	&\EE_{m,0} \langle \theta_{t}^{(m)} -\theta^\ast, (\widehat{A}^{(m)}   - \widehat{A})\theta_{t}^{(m)} +( \widehat{b}^{(m)} - \widehat{b} ) \rangle \\
 	 	\leq & \frac{\lambda_{\widehat{A}}}{4}\EE_{m,0}  \|\theta_{t}^{(m)} -\theta^\ast\|^2 +  \frac{1}{M}\cdot \frac{2}{\lambda_{\widehat{A}}}\Big[R_\theta^2 (1+\gamma)^2 + r^2_{\max}\Big] \cdot 4 \rho_{\max}^2  \Big(1 + \frac{\gamma \rho_{\max}}{ \min |\lambda(C)|}\Big)^2 \Big[1 + \kappa \frac{\rho}{1-\rho} \Big].
 	 \end{align*}
	
\end{proof}

\begin{lemma}\label{lemma: K3}  
	Under the same assumption as those of Theorem \ref{thm: markov}, the following inequality holds:
	\begin{align*}
	\EE_{m,0} \langle z_t^{(m)}, \bar{A}^{(m)}\theta_{t}^{(m)} +  \bar{b}^{(m)}\rangle 
	\leq \frac{\lambda_C}{8}\EE_{m,0}\|z_t^{(m)}\|^2 + \frac{K_3}{M} 
	\end{align*}
	where $K_3$ is defined as 
	\begin{align}\label{eq: def-K3-mc}
		K_3 := \Big( \frac{32}{\lambda_C} \Big[R_\theta^2(1+\gamma)^2+   r^2_{\max}\Big]\cdot \rho_{\max}^2 +  \frac{16}{\lambda_C}  \frac{\rho_{\max}(1+\gamma) R_\theta + \rho_{\max} r_{\max}}{\min |\lambda(C)|}\Big)  \Big[ 1 + \kappa \frac{\rho}{1-\rho} \Big].
	\end{align}
\end{lemma}
\begin{proof}
	Note that the following equations hold.
	\begin{align*}
		\EE_{m,0} \langle z_t^{(m)}, \bar{A}^{(m)}\theta_{t}^{(m)} +  \bar{b}^{(m)}\rangle &= \EE_{m,0} \langle z_t^{(m)},   ({A}^{(m)} - {C}^{(m)}C^{-1} A  ) \theta_{t}^{(m)} +  {b}^{(m)} - C^{(m)}C^{-1}b\rangle \\
		&= \EE_{m,0} \langle z_t^{(m)},   ({A}^{(m)} -  A  ) \theta_{t}^{(m)}+  {b}^{(m)} - b\rangle  \\
		&\quad - \EE_{m,0} \langle z_t^{(m)},   ({C}^{(m)} - {C})C^{-1} A   \theta_{t}^{(m)} + (C^{(m)}-C)C^{-1}b\rangle. \numberthis \label{eq: mc-tmp7} 
	\end{align*} 
	For the first term of \cref{eq: mc-tmp7}, we bound it using the polarization identity and Jensen's inequality as follows.
	\begin{align*}
		&\EE_{m,0} \langle z_t^{(m)},   ({A}^{(m)} -  A  ) \theta_{t}^{(m)}+  {b}^{(m)} - b\rangle \\
		\leq & \frac{1}{2}\cdot \frac{\lambda_C}{8}\EE_{m,0}\|z_t^{(m)}\|^2  + \frac{1}{2} \cdot \frac{8}{\lambda_C} \EE_{m,0}\| ({A}^{(m)} -  A  ) \theta_{t}^{(m)}+  {b}^{(m)} - b \|^2 \\
		\leq &   \frac{\lambda_C}{16}\EE_{m,0}\|z_t^{(m)}\|^2  +   \frac{8}{\lambda_C}\Big( R_\theta^2 \EE_{m,0}\| {A}^{(m)} -  A  \|^2_F  + \EE_{m,0}\| {b}^{(m)} - b \|^2\Big).
	\end{align*}
	Then, following a similar proof logic as that of Lemma \ref{lemma: K2}, we obtain that
	\begin{align*}
		\EE_{m,0}\| {A}^{(m)} -  A  \|^2_F \leq  \frac{1}{M}\cdot  4 (1+\gamma)^2\rho_{\max}^2  \Big[1 + \kappa \frac{\rho}{1-\rho} \Big],
	\end{align*}
	and
	\begin{align*}
	\EE_{m,0}\| {b}^{(m)} -  b  \|^2  \leq \frac{1}{M}\cdot 4 \rho_{\max}^2 r^2_{\max}  \Big[1 + \kappa \frac{\rho}{1-\rho} \Big].
	\end{align*}
	Combining the above bounds, we obtain the following bound for the first term of \cref{eq: mc-tmp7}
	\begin{align}
		&\EE_{m,0} \langle z_t^{(m)},   ({A}^{(m)} -  A  ) \theta_{t}^{(m)}+  {b}^{(m)} - b\rangle \nonumber\\
		&\leq \frac{\lambda_C}{16}\EE_{m,0}\|z_t^{(m)}\|^2  + \frac{1}{M}\cdot \frac{32}{\lambda_C} \Big[R_\theta^2(1+\gamma)^2+   r^2_{\max}\Big]\cdot \rho_{\max}^2\Big[1 + \kappa \frac{\rho}{1-\rho} \Big].\label{eq: mc-tmp9}
	\end{align}
	For the second term  of \cref{eq: mc-tmp7}, we have
	\begin{align*}
		 &- \EE_{m,0} \langle z_t^{(m)},   ({C}^{(m)} - {C})C^{-1} A   \theta_{t}^{(m)} + (C^{(m)}-C)C^{-1}b\rangle  \\
		 \leq & \frac{1}{2}\cdot \frac{\lambda_C}{8}\EE_{m,0}\|z_t^{(m)}\|^2  +  \frac{1}{2}\cdot \frac{8}{\lambda_C}\EE_{m,0}  \|C^{(m)} - C \|^2_F\cdot \frac{\rho_{\max}(1+\gamma) R_\theta + \rho_{\max} r_{\max}}{\min |\lambda(C)|}.
	\end{align*}
	Moreover, note that
	\begin{align*}
		\EE_{m,0}  \|C^{(m)} - C \|^2_F \leq \frac{1}{M} \cdot 4\Big[ 1 + \kappa \frac{\rho}{1-\rho} \Big],
	\end{align*}
	and therefore we have
	\begin{align}
	&- \EE_{m,0} \langle z_t^{(m)},   ({C}^{(m)} - {C})C^{-1} A   \theta_{t}^{(m)} + (C^{(m)}-C)C^{-1}b\rangle  \nonumber\\
	\leq& \frac{\lambda_C}{16}\EE_{m,0}\|z_t^{(m)}\|^2  +   \frac{1}{M} \cdot \frac{16}{\lambda_C}  \frac{\rho_{\max}(1+\gamma) R_\theta + \rho_{\max} r_{\max}}{\min |\lambda(C)|} \Big[ 1 + \kappa \frac{\rho}{1-\rho} \Big] \label{eq: mc-tmp8}
	\end{align} 
	Combining  \cref{eq: mc-tmp7,eq: mc-tmp8,eq: mc-tmp9}, we obtain that
	\begin{align*}
		&\EE_{m,0} \langle z_t^{(m)}, \bar{A}^{(m)}\theta_{t}^{(m)} +  \bar{b}^{(m)}\rangle \\
		\leq& \frac{\lambda_C}{8}\EE_{m,0}\|z_t^{(m)}\|^2 + \frac{1}{M} \cdot \Big( \frac{32}{\lambda_C} \Big[R_\theta^2(1+\gamma)^2+   r^2_{\max}\Big]\cdot \rho_{\max}^2 +  \frac{16}{\lambda_C}  \frac{\rho_{\max}(1+\gamma) R_\theta + \rho_{\max} r_{\max}}{\min |\lambda(C)|}\Big)  \Big[ 1 + \kappa \frac{\rho}{1-\rho} \Big].
	\end{align*} 
\end{proof}

\begin{lemma}\label{lemma: K4}
	Under the same assumptions as those of Theorem \ref{thm: markov}, the following inequality holds:
	$$ \EE_{m,0} \langle z_t^{(m)}, ({C}^{(m)} - C) z_{t}^{(m)} \rangle  \leq \frac{\lambda_C}{12}\EE_{m,0} \|  z_t^{(m)}\|^2  + \frac{K_4}{M},$$
	where $K_4$ is defined as 
	\begin{align}\label{eq: def-K4-mc}
		K_4 := \frac{12}{\lambda_C} R_w^2 \Big[ 1 + \kappa \frac{\rho}{1-\rho} \Big].
	\end{align} 
\end{lemma}
\begin{proof}
	The proof is very similar to that of Lemma \ref{lemma: K2} and Lemma \ref{lemma: K3}, and we only outline the proof below. 
\begin{align*}
	\EE_{m,0} \langle z_t^{(m)}, ({C}^{(m)} - C) z_{t}^{(m)} \rangle &\leq \frac{1}{2} \frac{\lambda_C}{6}\EE_{m,0} \|  z_t^{(m)}\|^2 + \frac{1}{2}\frac{6}{\lambda_C} R_w^2\EE_{m,0} \|  {C}^{(m)} - C\|^2_F \\
	&\leq  \frac{\lambda_C}{12}\EE_{m,0} \|  z_t^{(m)}\|^2 +\frac{1}{M}\cdot \frac{12}{\lambda_C} R_w^2 \Big[ 1 + \kappa \frac{\rho}{1-\rho} \Big].
\end{align*}
\end{proof}

\begin{lemma}\label{lemma: K5}
	Under the same assumptions as those of Theorem \ref{thm: markov}, the following inequality holds:
	$$\EE_{m,0}\|\bar{A}^{(m)}\theta^\ast  + \bar{b}^{(m)} \|^2 \leq \frac{K_5}{M} $$
	where $K_5$ is defined as 
	\begin{align}\label{eq: def-K5-mc}
		K_5:=\Big[(1+\gamma)R_\theta + r_{\max}\Big]^2\rho_{\max}^2\Big( 1 + \frac{1}{ \min |\lambda(C)|} \Big)^2  \cdot \Big( 1   +    \kappa \frac{2\rho}{1-\rho}   \Big). 
	\end{align}
\end{lemma}
\begin{proof}
	The proof is similar to that of Lemma \ref{lemma: K1} and we outline the proof below. We obtain that 
	\begin{align*}
	&\EE_{m,0}\|\bar{A}^{(m)}\theta^\ast  + \bar{b}^{(m)} \|^2\\
	\leq & \frac{1}{M^2}\EE_{m,0} \Big[  \sum_{i=j} \|\bar{A}_i^{(m)}\theta^\ast  + \bar{b}_i^{(m)}  \|^2 + \sum_{i\neq j} \langle \bar{A}_i^{(m)}\theta^\ast  + \bar{b}_i^{(m)}, \bar{A}_j^{(m)}\theta^\ast  + \bar{b}_j^{(m)}\rangle \Big] \\
	\leq & \frac{1}{M^2}\Big[ M \cdot \Big[(1+\gamma)R_\theta + r_{\max}\Big]^2\rho_{\max}^2\Big( 1 + \frac{1}{ \min |\lambda(C)|} \Big)^2 + \sum_{i\neq j} \EE_{m,0} \langle \bar{A}_i^{(m)}\theta^\ast  + \bar{b}_i^{(m)}, \bar{A}_j^{(m)}\theta^\ast  + \bar{b}_j^{(m)}\rangle \Big]. \numberthis \label{eq: tmp00}
	\end{align*}
	Cconsider the last term of \cref{eq: tmp00}. Without loss of generality, assume that $i<j$. Then
	\begin{align*}
		&\EE_{m,0} \langle \bar{A}_i^{(m)}\theta^\ast  + \bar{b}_i^{(m)}, \bar{A}_j^{(m)}\theta^\ast  + \bar{b}_j^{(m)}\rangle \\
		= & \EE_{m,0} \langle \bar{A}_i^{(m)}\theta^\ast  + \bar{b}_i^{(m)}, \EE_{m,i} \bar{A}_j^{(m)}\theta^\ast  + \EE_{m,i}  \bar{b}_j^{(m)}\rangle \\
		\leq & \EE_{m,0}\| \bar{A}_i^{(m)}\theta^\ast  + \bar{b}_i^{(m)} \| \cdot \EE_{m,0}\Big[   \| (\EE_{m,i} \bar{A}_j^{(m)} - \bar{A}) \theta^\ast \| + \|\EE_{m,i} \bar{b}_j^{(m)} - \bar{b}\|     \Big] \\ 
		\leq  &\Big[ (1+\gamma)R_\theta  + r_{\max}\Big]^2\rho_{\max}^2\Big( 1 + \frac{1}{ \min |\lambda(C)|} \Big)^2  \cdot  \kappa \rho^{j-i}  
	\end{align*}
	Summing the above inequality over all $i\neq j$ and substituting it into \cref{eq: tmp00}, we obtain that
	\begin{align*}
		&\EE_{m,0}\|\bar{A}^{(m)}\theta^\ast  + \bar{b}^{(m)} \|^2\\
		\leq & \frac{1}{M} \cdot  \Big[(1+\gamma)R_\theta + r_{\max}\Big]^2\rho_{\max}^2\Big( 1 + \frac{1}{ \min |\lambda(C)|} \Big)^2  \cdot \Big( 1   +    \kappa \frac{2\rho}{1-\rho}   \Big).
	\end{align*}
\end{proof} 
 
\begin{lemma}[One-Step Update of $\theta_{t}^{(m)}$] \label{lemma: GVR}
	Under the same assumption as those of Theorem \ref{thm: markov}, the square norm of one-step update of $\theta_{t}^{(m)}$ using Algorithm \ref{alg: markov} is bounded as 
	\begin{align*}
	&  	\EE_{m,0}\| G_{t}^{(m)}(\theta_{t}^{(m)}, z_{t}^{(m)})  - G_{t}^{(m)}(\tilde{\theta}^{(m-1)}, \tilde{z}^{(m-1)})  +  {G}^{(m)}(\tilde{\theta}^{(m-1)}, \tilde{z}^{(m-1)}) \|^2 \\
	\leq & 5(1+\gamma)^2\rho^2_{\max}\Big(1 + \frac{\gamma \rho_{\max}}{ \min |\lambda(C)|}\Big)^2\Big( \EE_{m,0} \| \theta_{t}^{(m)} - \theta^\ast\|^2  + \EE_{m,0}\| \tilde{\theta}^{(m-1)} - \theta^\ast\|^2 \Big) \\
	& + 5 \gamma^2\rho_{\max}^2 \Big(\EE_{m,0}\| z_{t}^{(m)}\|^2  + \EE_{m,0}\|\tilde{z}^{(m-1)}\|^2\Big) + \frac{5K_1}{M} 
	\end{align*}
	where $K_1$
	is specified in (\ref{eq: def-K1-mc}) of Lemma \ref{lemma: K1}.
\end{lemma}
\begin{proof}
		By the definitions of $G_{t}^{(m)}(\cdot)$ and ${G}^{(m)}(\cdot)$ and the one-step update of  $\theta_{t}^{(m)}$, we obtain that
 \begin{align*}
 	 & \| G_{t}^{(m)}(\theta_{t}^{(m)}, z_{t}^{(m)})  - G_{t}^{(m)}(\tilde{\theta}^{(m-1)}, \tilde{z}^{(m-1)})  +  {G}^{(m)}(\tilde{\theta}^{(m-1)}, \tilde{z}^{(m-1)}) \|^2 \\
 	= & \| \widehat{A}_{t}^{(m)}\theta_{t}^{(m)} + \widehat{b}_{t}^{(m)}+ B_{t}^{(m)} z_{t}^{(m)} - \widehat{A}_{t}^{(m)}\tilde{\theta}^{(m-1)} - \widehat{b}_{t}^{(m)}- B_{t}^{(m)} \tilde{z}^{(m-1)}  + \widehat{A}^{(m)}\tilde{\theta}^{(m-1)} + \widehat{b}^{(m)}+ B^{(m)} \tilde{z}^{(m-1)} \|^2\\
 	= & \| \Big(\widehat{A}_{t}^{(m)}\theta_{t}^{(m)} - \widehat{A}_{t}^{(m)} \theta^\ast\Big) + \Big(\widehat{A}_{t}^{(m)} \theta^\ast - \widehat{A}_{t}^{(m)}\tilde{\theta}^{(m-1)}+ \widehat{A}^{(m)}\tilde{\theta}^{(m-1)} -  \widehat{A}^{(m)}\theta^\ast\Big) +  \Big(\widehat{A}^{(m)}\theta^\ast  + \widehat{b}^{(m)}\Big) \\
 	\quad &  + B_{t}^{(m)} z_{t}^{(m)}  - B_{t}^{(m)} \tilde{z}^{(m-1)}+ B^{(m)} \tilde{z}^{(m-1)} \|^2\\
 	\leq &  5 \|\widehat{A}_{t}^{(m)} \|^2 \| \theta_{t}^{(m)} - \theta^\ast\|^2 + 5\| \Big( \widehat{A}_{t}^{(m)}\tilde{\theta}^{(m-1)} - \widehat{A}_{t}^{(m)} \theta^\ast \Big)- \Big(\widehat{A}^{(m)}\tilde{\theta}^{(m-1)} -  \widehat{A}^{(m)}\theta^\ast\Big) \|^2 + 5 \|\widehat{A}^{(m)}\theta^\ast  + \widehat{b}^{(m)} \|^2 \\
 	\quad & + 5\|B_{t}^{(m)}\|^2 \| z_{t}^{(m)}\|^2 + 5 \| B_{t}^{(m)} \tilde{z}^{(m-1)} -  B^{(m)} \tilde{z}^{(m-1)}\|^2  \numberthis \label{eq: tmp13}
 \end{align*}
  where the last step applies Jensen's inequality. Next, we bound the third term of \cref{eq: tmp13} $\|\widehat{A}^{(m)}\theta^\ast  + \widehat{b}^{(m)} \|^2 $ using Lemma \ref{lemma: K1}, and note that the second term of \cref{eq: tmp13} can be bounded as 
 \begin{align*}
 	&\EE_{m,t-1} \| \Big( \widehat{A}_{t}^{(m)}\tilde{\theta}^{(m-1)} - \widehat{A}_{t}^{(m)} \theta^\ast \Big)- \Big(\widehat{A}^{(m)}\tilde{\theta}^{(m-1)} -  \widehat{A}^{(m)}\theta^\ast\Big) \|^2 \\
 	=& \Var_{m, t-1}\Big( \widehat{A}_{t}^{(m)}\tilde{\theta}^{(m-1)} - \widehat{A}_{t}^{(m)} \theta^\ast \Big) \\
 	\leq& \EE_{m,t-1} \Big( \widehat{A}_{t}^{(m)}\tilde{\theta}^{(m-1)} - \widehat{A}_{t}^{(m)} \theta^\ast \Big)^2 \\
 	\leq& \EE_{m,t-1} \| \widehat{A}_{t}^{(m)}\|^2 \| \tilde{\theta}^{(m-1)} - \theta^\ast\|^2, \numberthis \label{eq: tmp12}
 \end{align*} 
 and similarly, the last term of \cref{eq: tmp13} can be bounded as 
 \begin{align}\label{eq: tmp11}
 	\EE_{m,t-1} \| B_{t}^{(m)} \tilde{z}^{(m-1)} -  B^{(m)} \tilde{z}^{(m-1)}\|^2  \leq  \EE_{m,t-1} \| B_{t}^{(m)} \|^2 \|\tilde{z}^{(m-1)}\|^2
 \end{align}
 Combining Lemma \ref{lemma: K1}, (\ref{eq: tmp11}), (\ref{eq: tmp12}), and (\ref{eq: tmp13}), we get the desired bound 
 \begin{align*}
  &  	\EE_{m,0}\| G_{t}^{(m)}(\theta_{t}^{(m)}, z_{t}^{(m)})  - G_{t}^{(m)}(\tilde{\theta}^{(m-1)}, \tilde{z}^{(m-1)})  +  {G}^{(m)}(\tilde{\theta}^{(m-1)}, \tilde{z}^{(m-1)}) \|^2 \\
\leq & 5(1+\gamma)^2\rho^2_{\max}\Big(1 + \frac{\gamma \rho_{\max}}{ \min |\lambda(C)|}\Big)^2\Big( \EE_{m,0} \| \theta_{t}^{(m)} - \theta^\ast\|^2  + \EE_{m,0}\| \tilde{\theta}^{(m-1)} - \theta^\ast\|^2 \Big) \\
& + 5 \gamma^2\rho_{\max}^2 \Big(\EE_{m,0}\| z_{t}^{(m)}\|^2  + \EE_{m,0}\|\tilde{z}^{(m-1)}\|^2\Big)   +    \frac{1}{M} \cdot   5 K_1.
 \end{align*}
\end{proof}

\begin{lemma}[One-Step Update of $z_{t}^{(m)}$] \label{lemma: HVR}
	Under the same assumption as those of Theorem \ref{thm: markov}, the square norm of one-step update of $z_{t}^{(m)}$ using Algorithm \ref{alg: markov} is bounded as
	\begin{align*}
	& \EE_{m,0}\| H_{t}^{(m)}(\theta_{t}^{(m)}, z_{t}^{(m)})  - H_{t}^{(m)}(\tilde{\theta}^{(m-1)}, \tilde{z}^{(m-1)})  +  {H}^{(m)}(\tilde{\theta}^{(m-1)}, \tilde{z}^{(m-1)}) \|^2 \\
	\leq &  5 (1+\gamma)^2 \rho^2 _{\max}\Big( 1 + \frac{1}{ \min |\lambda(C)|} \Big)^2\Big( \EE_{m,0}\| \theta_{t}^{(m)} - \theta^\ast\|^2 + \EE_{m,0}\| \tilde{\theta}^{(m-1)} - \theta^\ast\|^2  \Big) \\
	& + 5\Big(   \EE_{m,0}\| z_{t}^{(m)}\|^2 + \EE_{m,0}\| \tilde{z}^{(m-1)}\|^2\Big)  + \frac{5K_5}{M}.
	\end{align*}
	where $K_5$ is specified in \cref{eq: def-K5-mc} of Lemma \ref{lemma: K5}.
\end{lemma}
\begin{proof}
	The proof is very similar to that of Lemma \ref{lemma: GVR} and we outline the proof below. We obtain that
	\begin{align*}
		& \| H_{t}^{(m)}(\theta_{t}^{(m)}, z_{t}^{(m)})  - H_{t}^{(m)}(\tilde{\theta}^{(m-1)}, \tilde{z}^{(m-1)})  +  {H}^{(m)}(\tilde{\theta}^{(m-1)}, \tilde{z}^{(m-1)}) \|^2 \\
		\leq &  5 \|\bar{A}_{t}^{(m)} \|^2 \| \theta_{t}^{(m)} - \theta^\ast\|^2 + 5\| \Big( \bar{A}_{t}^{(m)}\tilde{\theta}^{(m-1)} - \bar{A}_{t}^{(m)} \theta^\ast \Big)- \Big(\bar{A}^{(m)}\tilde{\theta}^{(m-1)} -  \bar{A}^{(m)}\theta^\ast\Big) \|^2 + 5 \|\bar{A}^{(m)}\theta^\ast  + \bar{b}^{(m)} \|^2 \\
		\quad & + 5\|C_{t}^{(m)}\|^2 \| z_{t}^{(m)}\|^2 + 5 \| C_{t}^{(m)} \tilde{z}^{(m-1)} -  C^{(m)} \tilde{z}^{(m-1)}\|^2 .
	\end{align*}
	Taking $\EE_{m,0}$ on both sides of the above inequality yields that
	\begin{align*}
	& \EE_{m,0}\| H_{t}^{(m)}(\theta_{t}^{(m)}, z_{t}^{(m)})  - H_{t}^{(m)}(\tilde{\theta}^{(m-1)}, \tilde{z}^{(m-1)})  +  {H}^{(m)}(\tilde{\theta}^{(m-1)}, \tilde{z}^{(m-1)}) \|^2 \\
	\leq &  5 (1+\gamma)^2 \rho^2 _{\max}\Big( 1 + \frac{1}{ \min |\lambda(C)|} \Big)^2\Big( \EE_{m,0}\| \theta_{t}^{(m)} - \theta^\ast\|^2 + \EE_{m,0}\| \tilde{\theta}^{(m-1)} - \theta^\ast\|^2  \Big) \\
	& + 5\Big(   \EE_{m,0}\| z_{t}^{(m)}\|^2 + \EE_{m,0}\| \tilde{z}^{(m-1)}\|^2\Big)  + \frac{5K_5}{M}.
	\end{align*}
\end{proof}  
\section{Other Lemmas on Constant-Level Bounds}

\begin{lemma}\label{lemma: const-1} The following constant-level bounds hold.
	\begin{itemize}
		\item $\| \widehat{A}_t^{(m)} \| \leq (1+\gamma)\rho_{\max}\big(1 + \frac{\gamma \rho_{\max}}{ \min |\lambda(C)|}\big)$
		\item $\| \widehat{b}^{(m)}\| \leq  \rho_{\max} r_{\max}\big(1 + \frac{\gamma \rho_{\max}}{ \min |\lambda(C)|}\big)$
		\item $\|\bar{A}_t^{(m)} \| \leq (1+\gamma)\rho_{\max}\big( 1 + \frac{1}{ \min |\lambda(C)|} \big)$
		\item $\|\bar{b}_t^{(m)} \| \leq  \rho_{\max} r_{\max}\big( 1 + \frac{1}{ \min |\lambda(C)|} \big)$
	\end{itemize} 
\end{lemma}
\begin{proof}
	We prove the first inequality as an example. For the rest of them, we omit the proof details. Recall that $\widehat{A}_t^{(m)} :=A_{t}^{(m)} - B_{t}^{(m)} C^{-1}A$, from which we obtain that
	\begin{align*}
	\| \widehat{A}_t^{(m)} \| &= \|A_{t}^{(m)} - B_{t}^{(m)} C^{-1}A\| \\
	&\leq \| A_{t}^{(m)}\| + \| B_{t}^{(m)}\| \| C^{-1}\| \|A \| \\
	&\leq (1+\gamma)\rho_{\max} + \gamma \rho_{\max} \frac{1}{ \min |\lambda(C)|} (1+\gamma)\rho_{\max}\\
	&= (1+\gamma)\rho_{\max}\big(1 + \frac{\gamma \rho_{\max}}{ \min |\lambda(C)|}\big),
	\end{align*}
	where the third step applies Assumption \ref{ass: boundedness} to the definitions in \cref{eq:def_abc}, more precisely, $\| A_{t}^{(m)}\| = \| \rho(s_t^{(m)}, a_t^{(m)}) \phi(s_t^{(m)})(\gamma \phi(s_{t+1}^{(m)} - \phi(s_t^{(m)} )^\top \| \leq (1+\gamma) \rho_{\max}$ and $\| B_{t}^{(m)}\| = \| - \gamma \rho(s_t^{(m)}, a_t^{(m)}) \phi(s_{t+1}^{(m)}) \phi(s_t^{(m)})^\top \|$. Similarly, we can obtain the following results: 
	\begin{align*}
	\| \widehat{b}_t^{(m)}\| &= \|b_{t}^{(m)} - B_{t}^{(m)} C^{-1} b \| \\
	&\leq \|b_{t}^{(m)}\| + \| B_{t}^{(m)} \| \| C^{-1}\| \| b \| \\
	&\leq \rho_{\max} r_{\max} + \gamma \rho_{\max}   \frac{1}{ \min |\lambda(C)|} \rho_{\max} r_{\max} \\
	&=  \rho_{\max} r_{\max}\big(1 + \frac{\gamma \rho_{\max}}{ \min |\lambda(C)|}\big);
	\end{align*} 
	\begin{align*}
		\|\bar{A}_t^{(m)} \| &= \|A_{t}^{(m)} - C_{t}^{(m)} C^{-1}A\| \\
		&\leq (1+\gamma)\rho_{\max} +  \frac{1}{ \min |\lambda(C)|} (1+\gamma)\rho_{\max}\\
		&= (1+\gamma)\rho_{\max}\big( 1 + \frac{1}{ \min |\lambda(C)|} \big);
	\end{align*}  
	\begin{align*}
	\|\bar{b}_t^{(m)} \| &= b_{t}^{(m)} - C_{t}^{(m)} C^{-1} b \\
	&\leq \rho_{\max} r_{\max} +  \frac{1}{ \min |\lambda(C)|} \rho_{\max} r_{\max} \\
	&= \rho_{\max} r_{\max}\big( 1 + \frac{1}{ \min |\lambda(C)|} \big).
	\end{align*} 
\end{proof} 
\begin{lemma}\label{lemma: const-2}The following constant-level bound holds.
\begin{align*}
    \|\widehat{A}_t^{(m)}\theta^\ast  + \widehat{b}_t^{(m)} \| \leq \big[ (1+\gamma)R_\theta + r_{\max} \big]\rho_{\max}\big(1 + \frac{\gamma \rho_{\max}}{ \min |\lambda(C)|}\big).
\end{align*} 
\end{lemma}
\begin{proof}
	Combining Lemma \ref{lemma: const-1} and the assumption that $\|\theta^\ast\| \leq R_\theta$, we have
	\begin{align*}
		\|\widehat{A}_t^{(m)}\theta^\ast  + \widehat{b}_t^{(m)} \| &\leq \|\widehat{A}_t^{(m)} \| R_\theta + \| \widehat{b}_t^{(m)}\| \\
		&\leq (1+\gamma)\rho_{\max}\big(1 + \frac{\gamma \rho_{\max}}{ \min |\lambda(C)|}\big) R_\theta +  \rho_{\max} r_{\max}\big(1 + \frac{\gamma \rho_{\max}}{ \min |\lambda(C)|}\big) \\
		&= \big[ (1+\gamma)R_\theta + r_{\max} \big]\rho_{\max}\big(1 + \frac{\gamma \rho_{\max}}{ \min |\lambda(C)|}\big).
	\end{align*}
\end{proof}

\begin{lemma}\label{lemma: const-3}The following constant-level bound holds.
\begin{align*}
    \|\bar{A}^{(m)}\theta^\ast  + \bar{b}^{(m)} \| \leq \big[(1+\gamma)R_\theta + r_{\max}\big]\rho_{\max}\big( 1 + \frac{1}{ \min |\lambda(C)|} \big)
\end{align*} 
\end{lemma}
\begin{proof}
	Combining Lemma \ref{lemma: const-1} and $\|\theta^\ast\| \leq R_\theta$, we have
	\begin{align*}
		\|\bar{A}^{(m)}\theta^\ast  + \bar{b}^{(m)} \| &\leq \|\bar{A}_t^{(m)} \| R_\theta + \| \bar{b}_t^{(m)}\| \\
		&\leq (1+\gamma)\rho_{\max}\big( 1 + \frac{1}{ \min |\lambda(C)|} \big) R_\theta + \rho_{\max} r_{\max}\big( 1 + \frac{1}{ \min |\lambda(C)|} \big)\\
		&= \big[(1+\gamma)R_\theta + r_{\max}\big]\rho_{\max}\big( 1 + \frac{1}{ \min |\lambda(C)|} \big).
	\end{align*}
\end{proof}

\begin{lemma}[One-Step Update of $\theta_{t}^{(m)}$] \label{lemma: GVR-const} The upper bound of one-step update of $\theta_{t}^{(m)}$ given in both Algorithm \ref{alg: iid} and Algorithm \ref{alg: markov} is given by
\begin{align*}
    \| G_{t}^{(m)}(\theta_{t}^{(m)}, z_{t}^{(m)})  - G_{t}^{(m)}(\tilde{\theta}^{(m-1)}, \tilde{z}^{(m-1)})  +  {G}^{(m)}(\tilde{\theta}^{(m-1)}, \tilde{z}^{(m-1)}) \| \leq G_{\text{VR}},
\end{align*} 
	where $G_{\text{VR}}$ is defined as 
	\begin{align} \label{eq: def-G-VR}
G_{\text{VR}} := 3 \big[ (1+\gamma)R_\theta + r_{\max} \big]\rho_{\max}\big(1 + \frac{\gamma \rho_{\max}}{ \min |\lambda(C)|}\big).	
\end{align}
\end{lemma}
\begin{proof} Recall that $G_{t}^{(m)}(\theta , z ):= \widehat{A}_t^{(m)} \theta +  \widehat{b}_t^{(m)} + C_t^{(m)} w$. Using Lemma \ref{lemma: const-2} and Lemma \ref{lemma: const-3}, we obtain that
	$$\|G_{t}^{(m)}(\theta , z )\| \leq \big[ (1+\gamma)R_\theta + r_{\max} \big]\rho_{\max}\big(1 + \frac{\gamma \rho_{\max}}{ \min |\lambda(C)|}\big).$$
	By the definition of ${G}^{(m)}(\tilde{\theta}^{(m-1)}, \tilde{z}^{(m-1)})$ and Jensen's inequality, we obtain that
	$$\|{G}^{(m)}(\tilde{\theta}^{(m-1)}, \tilde{z}^{(m-1)}) \| \leq \big[ (1+\gamma)R_\theta + r_{\max} \big]\rho_{\max}\big(1 + \frac{\gamma \rho_{\max}}{ \min |\lambda(C)|}\big).$$
	Combining the above upper bounds for $\|G_{t}^{(m)}(\theta , z )\|$ and $\|{G}^{(m)}(\tilde{\theta}^{(m-1)}, \tilde{z}^{(m-1)}) \|$, we further obtain that
	\begin{align*}
		& \| G_{t}^{(m)}(\theta_{t}^{(m)}, z_{t}^{(m)})  - G_{t}^{(m)}(\tilde{\theta}^{(m-1)}, \tilde{z}^{(m-1)})  +  {G}^{(m)}(\tilde{\theta}^{(m-1)}, \tilde{z}^{(m-1)}) \| \\
		\leq& \| G_{t}^{(m)}(\theta_{t}^{(m)}, z_{t}^{(m)}) \| + \|G_{t}^{(m)}(\tilde{\theta}^{(m-1)}, \tilde{z}^{(m-1)})  \| +\| {G}^{(m)}(\tilde{\theta}^{(m-1)}, \tilde{z}^{(m-1)}) \|\\ 
		\leq& 3 \big[ (1+\gamma)R_\theta + r_{\max} \big]\rho_{\max}\big(1 + \frac{\gamma \rho_{\max}}{ \min |\lambda(C)|}\big).
	\end{align*}
\end{proof}

\begin{lemma}[One-Step Update of $w_{t}^{(m)}$] \label{lemma: HVR-const} The upper bound of one-step update of $w_{t}^{(m)}$ given in both Algorithm \ref{alg: iid} and Algorithm \ref{alg: markov} is given by
\begin{align*}
    \| H_{t}^{(m)}(\theta_{t}^{(m)}, z_{t}^{(m)})  - H_{t}^{(m)}(\tilde{\theta}^{(m-1)}, \tilde{z}^{(m-1)})  +  {H}^{(m)}(\tilde{\theta}^{(m-1)}, \tilde{z}^{(m-1)}) \| \leq H_{\text{VR}}
\end{align*}
	where $H_{\text{VR}}$ is defined as 
	\begin{align}\label{eq: def-H-VR}
	    H_{\text{VR}} := 3 \big[(1+\gamma)R_\theta + r_{\max}\big]\rho_{\max}\big( 1 + \frac{1}{ \min |\lambda(C)|} \big).
	\end{align}
\end{lemma}
\begin{proof}
	The proof is very similar to that of Lemma \ref{lemma: GVR-const} and we omit the proof.
	\begin{align*}
	& \| H_{t}^{(m)}(\theta_{t}^{(m)}, z_{t}^{(m)})  - H_{t}^{(m)}(\tilde{\theta}^{(m-1)}, \tilde{z}^{(m-1)})  +  {H}^{(m)}(\tilde{\theta}^{(m-1)}, \tilde{z}^{(m-1)}) \| \\
	\leq& \| H_{t}^{(m)}(\theta_{t}^{(m)}, z_{t}^{(m)}) \| + \|H_{t}^{(m)}(\tilde{\theta}^{(m-1)}, \tilde{z}^{(m-1)})  \| +\| {H}^{(m)}(\tilde{\theta}^{(m-1)}, \tilde{z}^{(m-1)}) \|\\ 
	\leq& 3 \big[(1+\gamma)R_\theta + r_{\max}\big]\rho_{\max}\big( 1 + \frac{1}{ \min |\lambda(C)|} \big).
	\end{align*}
\end{proof} 

\end{document}